\documentclass[10pt]{article}

\usepackage[preprint]{tmlr}

\usepackage{amsmath,amssymb,mathtools}
\usepackage{amsthm}
\usepackage{bm}
\usepackage{algorithm}
\usepackage{algpseudocode}
\usepackage{booktabs}
\usepackage{array}
\usepackage{longtable}
\usepackage{graphicx}
\usepackage{adjustbox}
\usepackage{float}
\usepackage{needspace}
\usepackage{subcaption}
\usepackage[section]{placeins}
\usepackage{url}
\usepackage{hyperref}
\hypersetup{hidelinks}

\newcommand{\States}{\mathcal U}
\newcommand{\Plans}{\Pi}

\newcommand{\Controllers}{\mathcal D}
\newcommand{\RepSpace}{\mathcal Z}
\newcommand{\Loss}{\ell}

\newcommand{\Dinfo}{\mathfrak D^{\mathrm{info}}}

\newcommand{\Rfull}{R_{\mathrm{full}}}
\newcommand{\Rinterface}{R_{Z}^{\star}}
\newcommand{\Rrestricted}{R_{\mathcal D}}
\newcommand{\Rlearn}{R_{\widehat{\delta}}}

\newcommand{\Dinterface}{\Delta_{\mathrm{info}}}
\newcommand{\Drule}{\Delta_{\mathrm{access}}}
\newcommand{\Dtrain}{\Delta_{\mathrm{train}}}

\newcommand{\AcceptSet}{\mathcal G}
\newcommand{\ConflictHypergraph}{\mathcal H_{\varepsilon}}
\newcommand{\ConflictEdges}{\mathcal E_{\varepsilon}}

\newcommand{\HyperChromatic}{\chi_{\mathrm{w}}}
\newcommand{\InterfaceFamily}{\mathfrak K}
\newcommand{\FiniteAlphabet}{\mathcal Z_{\mathrm f}}
\newcommand{\AuditDomain}{\mathcal U_{\mathrm{aud}}}

\newcommand{\ConditionalConflict}{\kappa}
\newcommand{\CapacityCurve}{\mathsf M^{\star}}
\newcommand{\RegretCurve}{\varepsilon^{\star}}
\newcommand{\Modulus}{\omega}
\newcommand{\LossMatrix}{\mathbf L}
\newcommand{\RegretMatrix}{\mathbf E}
\newcommand{\Indicator}{\mathbf 1}

\newtheorem{assumption}{Assumption}[section]
\newtheorem{definition}[assumption]{Definition}
\newtheorem{proposition}[assumption]{Proposition}

\newtheorem{theorem}[assumption]{Theorem}
\newtheorem{corollary}[assumption]{Corollary}
\theoremstyle{remark}

\newtheorem{example}[assumption]{Example}

\DeclareMathOperator*{\argmin}{arg\,min}
\DeclareMathOperator*{\argmax}{arg\,max}

\DeclareMathOperator{\range}{range}

\algrenewcommand\algorithmicrequire{\textbf{Input:}}
\algrenewcommand\algorithmicensure{\textbf{Output:}}
\floatname{algorithm}{Algorithm}

\newcommand{\auditgraphic}[2]{%
  \IfFileExists{#1}{\includegraphics[width=\linewidth]{#1}}{%
    \fbox{\parbox[c][1.35in][c]{0.94\linewidth}{\centering
      Figure file not found at compile time.\\[2pt]
      \texttt{\detokenize{#1}}\\[4pt]
      \small #2}}}}


\newcommand{\PaperTitle}{Decision Sufficiency of Learned Representation Interfaces for Safe Generative Control}
\newcommand{\FirstAuthor}{Jipeng Li}
\newcommand{\SecondAuthor}{Yanning Shen}
\newcommand{\AuthorAffiliation}{University of California, Irvine}

\hypersetup{
  pdftitle={\PaperTitle},
  pdfauthor={\FirstAuthor; \SecondAuthor},
  pdfdisplaydoctitle=true,
  pdflang={en-US}
}

\title{\PaperTitle}
\author{\name \FirstAuthor \qquad \SecondAuthor\\
\addr \AuthorAffiliation}
\date{}

\begin{document}
\maketitle

\begin{abstract}
A controller can guide a language model toward safe and useful responses by
choosing when to assist, ask for clarification, or refuse.
When the controller uses selected internal features, poor performance alone
cannot distinguish missing information from limitations of its decision
rules or training.
We study \emph{decision sufficiency}: whether the best controller using a
representation can match the minimum average loss attainable with full state
information, under the same strategies and evaluation criteria.
For interfaces with finitely many outputs, we compute the optimal loss
exactly and separate the three sources of excess loss.
Pairwise checks can miss information loss: states with the same controller
input may share an optimal strategy in every pair but not across the whole
group.
We bound the number of states needed to reveal a conflict and characterize
the fewest outputs an unrestricted interface needs to avoid all conflicts.
We audit selected features from sparse autoencoders, which encode model
activations using relatively few active features, in Qwen2.5, SmolLM2, and Phi.
On SmolLM2 responses scored by two human raters, restricting the controller
to two selected on/off features increases its minimum average loss by
$0.0345$--$0.1166$ relative to full state information, even with the governing
policy supplied.
Both independently trained autoencoders show the increase on both
evaluation sets.
The audit identifies decision failures that require changing the
controller's input, not merely its architecture or training.
\end{abstract}

\section{Introduction}
\label{sec:introduction}

A language model may need to assist with a request, ask for missing
information, or refuse. The appropriate choice depends on the facts
available about the request and the policy governing allowed behavior.
Internal activations, the numerical vectors computed inside a model,
can inform monitoring and steering
\citep{li2023inference,zou2023representation,karnik2025reachability}.
We study a controller that uses selected activation features to choose
a response strategy, which the model then executes.
A \emph{representation interface} specifies the controller's input.
We ask whether the input supports appropriate choices, independently of
how well a particular controller is trained.

Suppose a policy calls for assistance when authorization is verified and
clarification when evidence is missing. If both cases give the controller
identical inputs, it must use the same choice or distribution over choices.
Refusing both may avoid unauthorized assistance but also withholds permitted
help. The challenge is to distinguish missing information from a
controller's failure to use information that remains available.

A sparse autoencoder (SAE) learns to reconstruct an activation using
relatively few active features
\citep{cunningham2024sparse,gao2024scaling,lieberum2024gemmascope}.
A controller may receive only selected feature values or on/off indicators.
Accurate reconstruction from the full code does not establish what the
selected input preserves. A \emph{probe}, a predictor fitted to the
selected features, can fail because information is missing, its allowed
rules are too restrictive, or its fitting is imperfect.

We compare the best controller using the interface with the best controller
using the \emph{full audited state}: the situation, governing policy, and
recorded evidence. Both select from a fixed library of
\emph{continuations}, meaning response strategies that may react to later
observations. The execution protocol, loss function, and state
distribution are identical in each comparison.
The loss evaluates safety and response quality; \emph{risk} is its
average under the fixed state distribution.
An interface is \emph{decision-sufficient} when its minimum risk equals
the minimum with full state information.
The comparison builds on statistical decision theory and state abstraction
\citep{blackwell1953equivalent,lecam1964sufficiency,abel2020abstractions}.

Let $U$ denote the full audited state and $V$ the complete controller input:
\[
\underbrace{U}_{\text{full audited state}}
\longrightarrow
\underbrace{V}_{\text{complete controller input}}
\longrightarrow
\text{choice of continuation}.
\]
Let $Z$ denote the selected representation features and $S$ the governing
policy text. The complete input is $V=Z$ for features alone and $V=(Z,S)$
when the policy is also supplied. For a deterministic interface, states
with identical complete inputs $V$ must share a choice or distribution
over choices.
Supplying policy text $S$ can distinguish states that selected features $Z$
alone merge. Auditing features alone would then evaluate a different
information restriction.

We compare four risk levels: the minimum with full state information,
the minimum with the interface, the minimum within an allowed family of
rules using the interface, and the risk of the fitted rule.
The successive differences measure information loss, rule restrictions,
and fitting error, including estimation and optimization error.
All three terms can be positive simultaneously.
For a finite deterministic interface, states with identical input form a
\emph{cell}. Choosing the continuation with the smallest conditional average
loss in each cell computes the minimum over all interface rules.
The \emph{information gap} is the best interface risk minus the minimum
using the full state.

To explain the average information gap, we examine which states cannot
share a decision.
Consider three states with the same controller input whose optimal
strategies are, respectively, clarification or assistance, clarification
or refusal, and assistance or refusal.
Every pair shares an optimum, but the group does not.
Randomization cannot attain all three optima unless every strategy it
selects with positive probability is optimal at every state.
More generally, \emph{regret} is loss above a state's minimum.
A group is \emph{compatible} when one distribution over continuations
keeps regret within a specified nonnegative tolerance at every state.
A \emph{minimal conflict} is an incompatible group that becomes compatible
whenever one state is removed.
To make every cell compatible, an interface must split each conflicting
group across at least two cells.

We audit selected SAE features in Qwen2.5, SmolLM2, and Phi, with generation
and scoring protocols specified for each model.
With the policy supplied, the best controller using two selected on/off
features still has higher human rating loss than the full state reference.
On 96 SmolLM2 states from policy areas excluded from fitting, the gaps are
$0.1166$ and $0.1082$ for two independently trained SAEs.
The original 80-state sample gives gaps of $0.0345$ and $0.0353$ but
includes fitting areas.
Qwen2.5 comparisons vary the features, rule family, and fitting data
separately; Phi retains positive gaps with policy supplied under two
automatic evaluators.
A controlled example realizes a conflict missed by all pairwise checks,
whereas the searched natural language example does not replicate.
The empirical optima concern the recorded responses and scores, not
unknown losses on future responses.

\Needspace{7\baselineskip}
\paragraph{Contributions.}
\begin{enumerate}
 \item \textbf{Exact diagnosis of decision loss.}
 We compute the minimum risk for finite interfaces and separate missing
 information from limitations of the allowed rules and their fitting.

 \item \textbf{Group conflicts and interface requirements.}
 A hypergraph records states as vertices and minimal incompatible groups
 as edges. For a library of $m$ continuations with randomization allowed,
 every minimal conflict has at most $m$ states, and the bound is tight.
 We connect group incompatibility to risk and characterize the fewest
 outputs needed by an unrestricted deterministic interface, with global
 validation of every proposed cell.

 \item \textbf{Evidence across inputs and outcome evaluations.}
 SmolLM2 human ratings and Phi automatic scores show that providing the
 policy need not remove the selected interface's information gap.
 Qwen2.5 comparisons separately examine rule restrictions and fitting.
\end{enumerate}

Sections~\ref{sec:problem-formulation} and~\ref{sec:conflict-structure}
develop the risk and conflict analyses.
Sections~\ref{sec:empirical-audit} and~\ref{sec:results} present the
experiments, followed by related work, discussion, and conclusions.
Appendices give conditions for estimated losses, conflict selection after
search, and continuous observations.

Better rules or fitting can reduce loss above the interface's optimum.
Closing a remaining information gap requires changing what the controller
can observe.
\section{Decision Sufficiency and Exact Risk Analysis}
\label{sec:problem-formulation}

We compare optimal risks under the same losses and continuation library,
changing only the information available for selection. For a finite
interface, the comparison reduces to separate optimizations at each output.

\subsection{States, continuations, and permitted information}

Fix a context $x$, a policy specification $\sigma$, and a probability
distribution $\mu$ on the state space $\States:=\States_x$.
Let $U\sim\mu$ be the full audited state. The continuation library is
$\Plans:=\Plans_x=\{\pi_1,\ldots,\pi_m\}$, with $m\ge2$ and
$[m]:=\{1,\ldots,m\}$. Each continuation $\pi_j$ is a response policy
that may react to later observations. The governing specification $\sigma$
defines allowed behavior; it is not itself a continuation.

The expected loss of continuation $\pi_j$ at state $u$ is
$\ell_j(u):=\Loss_{\sigma,x}(u,\pi_j)\in[0,1]$. The expectation includes
randomness in its execution. Define the minimum loss at that state and
the regret of continuation $\pi_j$ by
\[
 \ell^\star(u):=\min_{j\in[m]}\ell_j(u),\qquad
 e_j(u):=\ell_j(u)-\ell^\star(u).
\]
Regret is the nonnegative excess above the best available continuation.
Zero regret need not mean zero loss.

Let $\Delta_m:=\{q\in[0,1]^m:\sum_{j=1}^m q_j=1\}$ be the probability
simplex over continuations. A randomized choice $q\in\Delta_m$ selects
$\pi_j$ with probability $q_j$, then executes that policy. Its loss and
regret are
\[
 \ell(u,q):=\sum_jq_j\ell_j(u),\qquad
 e(u,q):=\ell(u,q)-\ell^\star(u)=\sum_jq_je_j(u).
\]
Randomization mixes complete policies, not their internal actions. Once
the losses $\ell_j(u)$ are specified, the analysis applies to any finite
library; no additional temporal assumption is used in the risk formulas.

Let $\RepSpace$ be the observation space. An interface is a Markov kernel
$K(dz\mid u)$, the conditional distribution of observation $Z$ given state
$u$. A controller observes $Z\sim K(\cdot\mid U)$ and uses a measurable
rule $\delta:\RepSpace\to\Delta_m$. The component $\delta_j(z)$ is the
probability of selecting continuation $\pi_j$ at input $z$. Write
$\mathfrak D_Z$ for all such rules, $\Controllers\subseteq\mathfrak D_Z$
for a nonempty allowed family, and $\widehat\delta\in\Controllers$ for
the fitted rule.
\begin{assumption}[Audit setting]
\label{ass:basic-audit-setting}
The state space $\States$ and observation space $\RepSpace$ are standard
Borel spaces. The $m$ continuation losses are measurable and bounded in
$[0,1]$, and $K$ is a Markov kernel. The nonempty family
$\Controllers\subseteq\mathfrak D_Z$ contains the fitted rule
$\widehat\delta$ on the entire declared observation domain.
\end{assumption}

Conditional on state $u$ and selection of $\pi_j$, the expected execution
loss is $\ell_j(u)$; it has no further dependence on the selector or its
observation randomness. All risks use the joint law $\mu(du)K(dz\mid u)$ and the same execution
protocol for each continuation. Evaluation of $\widehat\delta$ holds its
training data and randomness fixed. The policy specification $\sigma$ is
initially fixed; Section~\ref{sec:observation-contract} handles varying
policy text that is supplied to the controller.

\subsection{Four risks and three excess losses}

The four risks compare access to the full state, unrestricted use of the
interface, a restricted family of rules, and the fitted rule:
\begin{align}
 \Rfull&:=\mathbb E[\ell^\star(U)],\label{eq:full-state-risk}\\
 \Rinterface&:=\inf_{\delta\in\mathfrak D_Z}\mathbb E[\ell(U,\delta(Z))],
     \label{eq:best-interface-risk}\\
 \Rrestricted&:=\inf_{\delta\in\Controllers}\mathbb E[\ell(U,\delta(Z))],
     \label{eq:restricted-risk}\\
 \Rlearn&:=\mathbb E[\ell(U,\widehat\delta(Z))].\label{eq:fitted-risk}
\end{align}
Define the information gap $\Dinterface:=\Rinterface-\Rfull$, the
accessibility gap $\Drule:=\Rrestricted-\Rinterface$, and the fitting gap
$\Dtrain:=\Rlearn-\Rrestricted$. When the interface and distribution need
to be explicit, write
\[
 \Dinfo_\sigma(K;\mu):=\Dinterface
   =\inf_{\delta\in\mathfrak D_Z}\mathbb E[e(U,\delta(Z))].
\]
The interface has \emph{decision sufficiency under $\mu$} when its
information gap $\Dinfo_\sigma(K;\mu)$ is zero.
\begin{theorem}[Risk ordering and decomposition]
\label{thm:nested-risks}
Under Assumption~\ref{ass:basic-audit-setting},
\[
 \Rfull\le\Rinterface\le\Rrestricted\le\Rlearn,
\]
and
\begin{equation}
 \Rlearn-\Rfull
 =\underbrace{\Rinterface-\Rfull}_{\Dinterface}
 +\underbrace{\Rrestricted-\Rinterface}_{\Drule}
 +\underbrace{\Rlearn-\Rrestricted}_{\Dtrain}.
 \label{eq:risk-decomposition}
\end{equation}
\end{theorem}

The information gap remains even for the best rule on the interface.
The accessibility gap measures the extra cost of restricting that rule;
the fitting gap measures the remaining error within the allowed family.
The last gap can include both estimation and optimization error.
The identity separates these comparisons, but does not by itself predict
how more training data will change them. All comparisons fix the recorded
state, continuation library, and loss.

\subsection{Exact optimization for a finite interface}
\label{sec:finite-interface-theory}

Let $\AuditDomain=\{u_1,\ldots,u_n\}$ be a finite audit domain, with
$n\ge1$ and state weights $\mu_i\ge0$ summing to one. Let
$\FiniteAlphabet=\{z_1,\ldots,z_L\}$, with $L\ge1$, be the complete
observation alphabet. Define the observation probabilities, loss entries,
and joint costs by
\[
 P_{ir}:=K(\{z_r\}\mid u_i),\qquad L_{ij}:=\ell_j(u_i),\qquad
 C_{rj}:=\sum_{i=1}^n\mu_iP_{ir}L_{ij}.
\]
Each row of $P$ sums to one. The loss matrix $\LossMatrix=(L_{ij})$
contains one entry per state and continuation. If
$p_r:=\sum_i\mu_iP_{ir}>0$, then
$C_{rj}=p_r\mathbb E[\ell_j(U)\mid Z=z_r]$.
Thus $C_{rj}$ already includes the probability of output $z_r$.
A rule assigning $q_r=\delta(z_r)$ has risk $\sum_{r,j}C_{rj}q_{rj}$.
\begin{theorem}[Optimum for a finite interface]
\label{thm:finite-interface-optimum}
For the finite audit distribution and observation channel,
\[
 \Rfull=\sum_{i=1}^n\mu_i\min_{j\in[m]}L_{ij},
 \qquad
 \Rinterface=\sum_{r=1}^L\min_{j\in[m]}C_{rj}.
\]
At each output $z_r$, every distribution supported on $\argmin_jC_{rj}$
is optimal. In particular, selecting one minimizing continuation suffices.
Outputs of probability zero can be assigned arbitrarily.
\end{theorem}

The choices at different outputs have separate simplex constraints, and
the objective is linear in each choice. Forming the dense cost table takes
$O(nLm)$ operations; selecting the row minima takes $O(Lm)$ comparisons
(Algorithm~\ref{alg:finite-interface-diagnosis}). For a deterministic
map $f:\AuditDomain\to\FiniteAlphabet$, the nonempty sets
$C_r:=\{u_i:f(u_i)=z_r\}$ are its \emph{cells}. Every state in a cell
receives the same continuation distribution. With strictly positive state
weights, the gap is zero exactly when each cell has a common optimal
continuation (Corollary~\ref{cor:finite-cell-zero-gap}).
Appendix~\ref{app:proofs-risk-interface} gives the proofs and the
corresponding posterior characterization for a general channel.

\subsection{Policy information and an observed collision}
\label{sec:observation-contract}

When policies vary across states, let $S=s(U)$ denote the governing policy
text. Each loss $\ell_j(u)$ is evaluated under the policy in state $u$;
the continuation library and execution protocols remain fixed. Let $V$
be the complete controller input: $V=Z$ for features alone and $V=(Z,S)$
when policy is also supplied. Define
\[
 R_V^\star:=\inf_\delta\mathbb E[\ell(U,\delta(V))],\qquad
 \Delta_V:=R_V^\star-\Rfull,
\]
where the infimum ranges over all measurable rules using $V$.
For the finite deterministic audit, the cells with policy supplied are
$C_{z,s}:=\{u:f(u)=z,\ s(u)=s\}$. Let $\Indicator\{A\}$ denote the
indicator of event $A$. Applying the finite optimum gives
\begin{equation}
 \begin{aligned}
 R_{Z,S}^\star&=\sum_{z,s}\min_j
    \mathbb E[\Indicator\{Z=z,S=s\}\ell_j(U)],\\
 0&\le\Delta_{Z,S}\le\Delta_Z.
 \end{aligned}
 \label{eq:policy-aware-risk}
\end{equation}
The sum runs over observed code and policy pairs. A rule that receives
policy text $S$ can ignore it, which proves the inequality. A positive
gap using features $Z$ alone need not remain when policy is supplied.

\Needspace{12\baselineskip}
\paragraph{An example from the human ratings.}
Table~\ref{tab:human-collision-main} gives two SmolLM2 states with the same
policy and the same two SAE activation indicators, \texttt{00}. The
three continuations ask for clarification, refuse, or assist. Entries are
rating losses in $[0,1]$, with lower values preferred; the scoring rule is
defined in Section~\ref{sec:human-calibration}.
\begin{table}[H]
\centering
\caption{An observed policy-aware collision under the SmolLM2 seed-314159, two-bit interface. Both requests have the same policy and code \texttt{00}. Entries are final human composite losses; smaller is better. The pair is a post hoc illustration, not an independent or higher-order test. Exact text and state IDs are given in Appendix~\ref{sec:real-collision-record}.}
\label{tab:human-collision-main}
\small
\begin{tabular}{@{}p{0.57\linewidth}rrr@{}}
\toprule
Request summary & \textsc{ask} & \textsc{decline} & \textsc{assist} \\
\midrule
Acknowledgment, accessible meeting, and referral & 0.350 & 0.525 & 0.000 \\
Fabricated supporting documentation & 0.500 & 0.100 & 0.800 \\
\bottomrule
\end{tabular}
\end{table}

With equal state weights, access to the full state gives mean loss
$(0+0.100)/2=0.050$. A common choice has column means
$(0.425,0.3125,0.400)$, so refusal is optimal with loss $0.3125$.
The two-state gap is $0.2625$. The pair was selected after inspecting the
data to illustrate an existing collision, not as an independent test or
a conflict missed by pairwise checks. Its full record is in
Appendix~\ref{sec:real-collision-record}.

\subsection{Restricted rules and estimated losses}

The optimum within a restricted rule family requires a separate calculation.
Our affine family maps visible features to continuation probabilities, with
simplex constraints on the entire declared observation alphabet. Its optimum
is a linear program; fitting changes the costs, not the feasible domain
(Appendix~\ref{app:restricted-affine}).

An exact optimum for a supplied loss table need not equal the optimum for
unknown expected losses. Taking a minimum of noisy estimates can favor
negative estimation errors. Proposition~\ref{prop:loss-table-perturbation}
controls the resulting risk changes under a uniform bound on loss errors.
For continuous observations, a finite set of distinct vectors can be
memorized by an unrestricted rule, so probe failure alone does not identify
missing information. Quantization also changes the observed input.
Appendix~\ref{app:continuous-certificates} gives the conditions needed to
transfer a calculation on quantized data to the original observation and
to obtain bounds from estimated losses.

\section{Group Conflicts and Interface Requirements}
\label{sec:conflict-structure}
\label{sec:higher-order-conflicts}
\label{sec:capacity-certification}
\label{sec:end-to-end-pipeline}

The average risk gap does not identify which states cannot share a decision.
We describe those groups and determine how many outputs are needed to keep
every group compatible. The grouping results concern deterministic interfaces;
the risk analysis in Section~\ref{sec:problem-formulation} also permits
stochastic channels.

\subsection{Shared choices and average risk}

Fix a nonempty finite domain $\AuditDomain=\{u_1,\ldots,u_n\}$, the
$m\ge2$ continuations from Section~\ref{sec:problem-formulation}, and
their regrets $e_j(u)=\ell_j(u)-\ell^\star(u)$. A complete observation
map $f:\AuditDomain\to[M]$, where $M\ge1$ and $[M]:=\{1,\ldots,M\}$, has nonempty
cells $C_c=f^{-1}(c)$ for $c\in\range(f)$. Every state in a cell
receives the same continuation distribution $q\in\Delta_m$. If policy
text is supplied, it is part of $f$; states with different policy text
cannot belong to the same cell.

For an allowed regret $\varepsilon\ge0$, the acceptable choices at state
$u$ form the set
\[
 \AcceptSet^\varepsilon(u):=\{q\in\Delta_m:e(u,q)\le\varepsilon\}.
\]
Each set is nonempty, closed, and convex: a pure optimum is feasible, and
regret is linear in the choice probabilities. For a nonempty group
$W\subseteq\AuditDomain$, define its \emph{margin} as the smallest
possible maximum regret across its states:
\begin{equation}
 \Gamma(W):=\min_{q\in\Delta_m}\max_{u\in W}e(u,q).
 \label{eq:shared-continuation-margin}
\end{equation}
The minimum is attained by compactness of $\Delta_m$ and continuity of
the objective. A group is $\varepsilon$-compatible exactly when
\begin{equation}
 \Gamma(W)\le\varepsilon
 \quad\Longleftrightarrow\quad
 \bigcap_{u\in W}\AcceptSet^\varepsilon(u)\ne\varnothing.
 \label{eq:group-compatibility}
\end{equation}
Otherwise the group is incompatible. Removing states cannot increase the
margin: $\Gamma(V)\le\Gamma(W)$ for nonempty $V\subseteq W$.

Let $\Delta(W):=\{w=(w_u)_{u\in W}:w_u\ge0,\ \sum_{u\in W}w_u=1\}$ be the distributions
over states in $W$. For fixed state weights $w\in\Delta(W)$, define
\[
 \ConditionalConflict(W,w):=
 \min_{q\in\Delta_m}\sum_{u\in W}w_ue(u,q)
 =\min_{j\in[m]}\sum_{u\in W}w_ue_j(u).
\]
The quantity $\ConditionalConflict(W,w)$ minimizes average regret, whereas
$\Gamma(W)$ minimizes the largest regret. The weights $w$ describe states;
the probabilities $q$ describe choices of continuation.
\begin{theorem}[Relation between average and maximum regret]
\label{thm:average-robust-bridge}
For every nonempty finite group $W$ and $w\in\Delta(W)$, let
$\beta(W,w):=\min_{u\in W}w_u$. Then
\[
 \beta(W,w)\Gamma(W)\le\ConditionalConflict(W,w)\le\Gamma(W).
\]
\end{theorem}

Let $\mu$ be a probability distribution on the finite audit domain. The
induced channel is $K_f(\{c\}\mid u)=\Indicator\{f(u)=c\}$. For each cell of positive
mass $p_c:=\sum_{u\in C_c}\mu(u)$, set $w_u^{(c)}:=\mu(u)/p_c$.
The information gap satisfies
\begin{equation}
 \Dinfo_\sigma(K_f;\mu)
 =\sum_{c:p_c>0}p_c\ConditionalConflict(C_c,w^{(c)})
 \le\sum_{c:p_c>0}p_c\Gamma(C_c).
 \label{eq:cell-average-gap}
\end{equation}
Hence a gap exceeding $\varepsilon$ implies that some cell has margin
exceeding $\varepsilon$. Positivity alone does not imply failure at an
arbitrary positive tolerance.
\label{rem:average-vs-uniform}
With positive state weights, zero average regret and zero maximum regret
are equivalent. At positive tolerance, a rare state can prevent a uniform
guarantee while contributing little to average loss.

\subsection{Conflicts missed by pairwise checks}

\begin{example}[Every pair shares an optimum, but the group does not]
\label{ex:shared-optimum-triple}
Consider three states with losses
\[
 \begin{array}{c|ccc}
 &\textsc{ask}&\textsc{decline}&\textsc{assist}\\\hline
 u_1&0&1&0\\
 u_2&0&0&1\\
 u_3&1&0&0
 \end{array}
\]
Each state has minimum loss zero, so losses equal regrets. The pairs
$\{u_1,u_2\}$, $\{u_1,u_3\}$, and $\{u_2,u_3\}$ can choose
\textsc{ask}, \textsc{assist}, and \textsc{decline}, respectively, with
zero regret. For a common choice
$q=(q_{\rm ask},q_{\rm decline},q_{\rm assist})$, the three regrets are
$q_{\rm decline}$, $q_{\rm assist}$, and $q_{\rm ask}$. Their maximum
is at least $1/3$; the uniform distribution attains $1/3$. The triple
is therefore incompatible for $0\le\varepsilon<1/3$, while every pair
is compatible.
\end{example}

A \emph{minimal conflict} is an incompatible group for which every proper
subset is compatible. Minimality refers to inclusion, not the smallest
cardinality among all conflicts.
\begin{definition}[Conflict hypergraph]
\label{def:conflict-hypergraph}
Let $\ConflictHypergraph=(\AuditDomain,\ConflictEdges)$, where a nonempty
group $W$ belongs to $\ConflictEdges$ if
\[
 \bigcap_{u\in W}\AcceptSet^\varepsilon(u)=\varnothing,
 \qquad
 \bigcap_{u\in V}\AcceptSet^\varepsilon(u)\ne\varnothing
 \quad\text{for every proper subset }V\subsetneq W.
\]
The intersection over no states is defined as $\Delta_m$. Vertices are audit states, and edges
are minimal incompatible groups.
\end{definition}

The global edge set depends on the losses, library, states, and tolerance,
not on the current interface. Its subset
\[
 \mathcal E_\varepsilon(f):=
 \{W\in\ConflictEdges:W\subseteq C_c\text{ for some }c\}
\]
contains conflicts hidden within the cells of $f$. A new interface may
merge states from different old cells, so redesign must consider global
conflicts rather than only $\mathcal E_\varepsilon(f)$.
\begin{theorem}[Maximum size of a minimal conflict]
\label{thm:tight-witness-rank}
Every minimal conflict $W\in\ConflictEdges$ satisfies $2\le|W|\le m$.
For each $m\ge2$, some bounded loss table and tolerance have a minimal
conflict of exactly $m$ states.
\end{theorem}

The upper bound follows from Helly's theorem in the $(m-1)$-dimensional
simplex of continuation probabilities \citep{helly1923mengen}. With three
continuations, every incompatible group thus contains a conflicting pair
or triple. Appendix~\ref{app:proofs-higher-order} proves tightness and
explains when unique optimal choices reduce all minimal conflicts to pairs.

\subsection{The risk interpretation of a conflict margin}

Allow state weights $\lambda\in\Delta(W)$ to vary while $q\in\Delta_m$
remains the controller's choice.
\begin{theorem}[Minimax interpretation of the group margin]
\label{thm:witness-duality}
For every nonempty finite group $W$,
\[
 \Gamma(W)=\max_{\lambda\in\Delta(W)}
                  \min_{j\in[m]}\sum_{u\in W}\lambda_ue_j(u).
\]
If the complete observation map $f$ is constant on $W$, then
\[
 \Gamma(W)=\max_{\lambda\in\Delta(W)}\Dinfo_\sigma(K_f;\lambda),
\]
where the state distribution $\lambda$ is supported on $W$.
\end{theorem}

The margin equals the \emph{largest} information gap over distributions
on the group, not the gap under any fixed audit distribution $\mu$.
Finite minimax duality justifies the exchange of optimizations
\citep{sion1958general}. Introducing a scalar $t$ gives the linear program
\[
 \Gamma(W)=\min_{q,t}\left\{t:q\in\Delta_m,\quad
       \sum_jq_je_j(u)\le t\text{ for every }u\in W\right\}.
\]
The primal solution gives a best common choice; the dual gives state
weights attaining the maximum. For fixed $m$, enumerating all candidates
requires $O(n^m)$ margin problems
(Algorithm~\ref{alg:witness-extraction}).

\subsection{Output requirements and global validation}

A redesigned interface is a map $g:\AuditDomain\to[n]$ assigning new
labels to states. It is uniformly $\varepsilon$-sufficient when every
nonempty cell has margin at most $\varepsilon$. Define its minimum
number of distinct outputs by
\[
 \CapacityCurve(\varepsilon):=
 \min_{g:\AuditDomain\to[n]}
 \left\{|\range(g)|:\Gamma(g^{-1}(c))\le\varepsilon
                  \text{ for every }c\in\range(g)\right\}.
\]
An edge is monochromatic if all its states have the same label. The weak
chromatic number $\HyperChromatic(\ConflictHypergraph)$ is the smallest
number of labels with no monochromatic edge.
\begin{theorem}[Capacity as a coloring problem]
\label{thm:hypergraph-capacity}
For unrestricted deterministic interfaces,
\[
 \CapacityCurve(\varepsilon)=\HyperChromatic(\ConflictHypergraph).
\]
\end{theorem}

A cell is incompatible exactly when it contains a minimal conflict;
avoiding monochromatic edges is therefore equivalent to compatibility
\citep{berge1989hypergraphs}. Assigning every state an optimal continuation
index uses at most $m$ labels. Three such labels fit in two bits, although
two selected SAE bits need not implement that assignment.
Appendix~\ref{app:capacity-extensions} treats restricted encoders, zero
regret, and policy text supplied without charging it to the code alphabet.

To avoid enumerating every edge, start with any
$\mathcal F_0\subseteq\ConflictEdges$. At iteration $t$, a coloring
mixed-integer linear program (MILP) using $\mathcal F_t$ returns the
minimum label count $M_t$ and a labeling $g_t$.
\begin{proposition}[Capacity by constraint generation]
\label{prop:capacity-constraint-generation}
With exact solves, $M_t\le\CapacityCurve(\varepsilon)$. If every nonempty
cell of $g_t$ has margin at most $\varepsilon$, then
$M_t=\CapacityCurve(\varepsilon)$ and $g_t$ is optimal. Otherwise extract
a minimal conflict $W_t$ from a violating cell, set
$\mathcal F_{t+1}=\mathcal F_t\cup\{W_t\}$, and repeat. Exact margin
comparisons and MILP solves yield finite termination at the global capacity.
\end{proposition}

The retained constraints give a lower bound. A validated labeling gives
a matching upper bound. A failed validation adds an edge not yet retained;
there are finitely many edges. Algorithms~\ref{alg:capacity-milp}
and~\ref{alg:capacity-global-validation} give the procedures, including
the handling of incomplete solves. Numerical output is reported with
solver status and tolerances, not as an exact arithmetic proof.

\subsection{Guarantees after searching for a conflict}

Let $\widehat\ell_j(u)\in[0,1]$ estimate loss $\ell_j(u)$. Suppose,
with probability at least $1-\alpha$, all errors on the finite audit
satisfy $|\widehat\ell_j(u)-\ell_j(u)|\le r_\ell$, where
$0<\alpha<1$ and $r_\ell\ge0$. Define
$\widehat e_j(u):=\widehat\ell_j(u)-\min_k\widehat\ell_k(u)$ and
let $\widehat\Gamma(W)$ be the group margin using these estimated regrets.
Every regret and group margin changes by at most $2r_\ell$ on that same
event. For any group $\widehat W$ selected using the estimates,
\[
 \Gamma(\widehat W)\ge\widehat\Gamma(\widehat W)-2r_\ell
 \quad\text{with probability at least }1-\alpha.
\]
A lower bound exceeding $\varepsilon$ certifies incompatibility after
selection. Minimality also requires upper bounds for every group formed
by deleting one state. Appendix~\ref{app:extended-witness-certification}
gives the full results, interval constructions, and conditions for a
larger continuation library. Resampling the existing policy areas does
not supply the simultaneous loss bounds assumed here.

\section{Experimental Design}
\label{sec:empirical-audit}
\label{sec:audit-design}

The experiments compare exact risks under selected SAE inputs, then vary
the input, allowed rules, fitting data, or evaluator. Human ratings are
available for SmolLM2. Appendix~\ref{app:full-protocol} gives complete
model and rating protocols.

\subsection{States, response choices, and controller inputs}

Each state records a synthetic situation, its governing policy, and
relevant evidence. The original collection has 240 states from 20 policy
areas, with three required actions and four paraphrases per area.
Twelve areas (144 states) are used for setting selection and fitting;
eight (96 states) form the original confirmation collection. SmolLM2 adds
96 states from eight other areas, fixed before generation and scoring.
Phi generates new responses on both collections; the states themselves
are reused.

The available strategies are
$\mathcal A:=\{\textsc{ask},\textsc{decline},\textsc{assist}\}$:
request clarification, refuse, or provide assistance. Every state is
evaluated under all three choices. The restricted input belongs to the
controller selecting a strategy, not to every operation of the generator.
The experiments concern one response turn, not arbitrary conversation
policies.

Let $F_B$ be a set of $B$ SAE coordinates and $a_k(u)\ge0$ the activation
of selected coordinate $k$ at state $u$. The controller observes the
activation indicators
\[
 z_k(u):=\Indicator\{a_k(u)>0\},\qquad Z=(z_k)_{k\in F_B}.
\]
Its complete input is $V=Z$ or $V=(Z,S)$, where $S$ is policy text.
When policy is supplied, a cell requires equality of both code and policy
string. The required action label is never a controller input.
SmolLM2 and Phi use $B\in\{2,4,8,16,32,64\}$, with $B=2$ primary.
Features are ranked individually by the reduction in fitting risk relative
to a constant input, followed by activation balance and index. The first
$B$ coordinates are retained. The ranking is not a joint optimization of
all subsets, and evaluation conditions do not reselect features.

\subsection{Model protocols and automatic scoring}
\label{sec:cross-family-independent-sae}
\label{sec:phi-design}

Qwen2.5-1.5B uses public SAEs after layers 17, 18, 19, and 21, with a
second Qwen scale as a supporting check. SmolLM2-1.7B and Phi-4-mini-instruct
use layers 17 and 23, respectively. Each has two TopK SAEs trained with
different initializations and training orders. Within a pair, the model,
layer, architecture, and training activation corpus are shared.
Layer indices are zero based; full settings appear in
Appendix~\ref{app:full-protocol}.

The Qwen generator receives the public request, fact name, and selected
action, but not the private condition or complete policy. Its responses
are reused across states sharing the public request. The Qwen evaluator
receives the full context and required action label. SmolLM2 and Phi
supply policy and situation to their generators but do not supply the
required action to their evaluators. The comparison across families thus
changes generation and scoring protocols as well as the model; it is not
a controlled comparison of model identity alone.

Qwen scores safety, utility, fact handling, action adherence, and overall
quality from zero to four with weights $(0.40,0.25,0.15,0.10,0.10)$.
Its score is an expectation under label probabilities corrected for the
evaluator's label prior. Loss is one minus the weighted score divided by
four. SmolLM2 and Phi instead request integer JSON scores from zero to
three on the five dimensions below. Each generates five responses per
state and action, totaling 5,040 responses per model. Automatic risk
uses the mean loss of the five responses. Mistral is their primary
evaluator; Phi is also evaluated with Qwen-7B and the mean of the two
losses. Normalization to $[0,1]$ does not make the scoring rubrics equivalent.

\begin{table}[tbp]
\centering
\caption{Models, controller inputs, and evaluations. Human ratings cover
SmolLM2 only. Generator inputs and scoring rules differ across families.
$^\dagger$Qwen2.5-1.5B results with policy supplied are recovered from
rounded summaries for individual policy areas.}
\label{tab:family-evidence-contract}
\small
\begin{tabular}{@{}p{0.19\linewidth}p{0.25\linewidth}p{0.48\linewidth}@{}}
\toprule
Study & Controller input & Outcome evaluation \\
\midrule
Qwen2.5 & $Z$ at two scales; $(Z,S)$ at 1.5B$^\dagger$ &
Automatic scores; separate changes to features, rules, and fitting data. \\
SmolLM2 & Paired $Z$ and $(Z,S)$ &
Automatic and human scores; two SAE training seeds. \\
Phi & Paired $Z$ and $(Z,S)$ &
Two automatic evaluators and their mean; two SAE training seeds. \\
\bottomrule
\end{tabular}
\end{table}

\subsection{Human ratings and matched observations}
\label{sec:human-calibration}
\label{sec:human-evaluation-design}

Human evaluation covers 80 original SmolLM2 states, including fitting
areas, and all 96 locked states. A SHA-256 ordering independent of scores
fixes one response per state and action, giving 528 items. After a separate
36-item pilot, two raters independently score every item. They see the
policy, situation, selected action, and response, but not automatic scores,
model or SAE identity, required action, split, or the other rater's labels.
The pilot is excluded from reported statistics. Predefined disagreement
rules select 100 items for joint adjudication.

Let $s_{\rm safety},s_{\rm policy},s_{\rm use},s_{\rm fact},s_{\rm overall}
\in\{0,1,2,3\}$ denote ratings for safety, policy adherence, usefulness,
factuality, and overall quality, respectively. The human loss is
\[
 L^{\rm human}=1-\frac{0.30s_{\rm safety}+0.30s_{\rm policy}
              +0.15s_{\rm use}+0.10s_{\rm fact}
              +0.15s_{\rm overall}}{3}.
\]
The 100 adjudicated items use resolved scores; the other 428 use the mean
of the independent losses. The composite rating is not an unsafe event
probability. Automatic and human calibration compares scores for the same
fixed response, not a human rating with a mean over five responses.
Appendix~\ref{app:human-protocol-details} gives the rubric and procedure.

Comparing $Z$ with $(Z,S)$ holds features, states, responses, and losses
fixed. The SmolLM2 comparison was added after the original study; both
Phi observations were specified before generation. Qwen2.5-1.5B results
with policy supplied are recovered from rounded summaries with a stated
error bound, not full precision arrays. Qwen2.5-0.5B uses code alone.

\subsection{Risk calculation and resampling}
\label{sec:empirical-uncertainty}

For action $a\in\mathcal A$, let $L(a,u)\in[0,1]$ be the measured loss
at state $u$. Let $\widehat p(v)$ be the evaluation frequency of input
$v$, and $\widehat{\mathbb E}[L(a,U)\mid V=v]$ the mean loss in its
cell. The empirical optima are
\begin{equation}
 \widehat R_V^\star=\sum_v\widehat p(v)\min_{a\in\mathcal A}
    \widehat{\mathbb E}[L(a,U)\mid V=v],\qquad
 \widehat R_{\rm full}=\widehat{\mathbb E}
    \bigl[\min_{a\in\mathcal A}L(a,U)\bigr].
 \label{eq:finite-audit-bayes-rule}
\end{equation}
Their difference is the information gap on the scored table. Both are
reference optimizations over the observed scores, not trained policies
tested on future responses. Comparisons of rule families restrict the
allowed mapping; the fitting experiment varies samples from a fixed
finite population.

We resample whole policy areas and recompute both minima. Tables report
individual two-sided 95\% bootstrap intervals and one-sided 95\% lower
confidence bounds (LCBs). Phi uses 2,000 draws and eight areas per
collection. With policy supplied, resampling reweights fixed gaps within
policy areas. A positive LCB describes stability to that reweighting;
it does not control error in the underlying response losses. The intervals
are not simultaneous across settings, and a zero fraction of nonpositive
draws is not a zero $p$-value. Formal guarantees require the conditions in
Appendices~\ref{app:continuous-certificates}
and~\ref{app:extended-witness-certification}.

\section{Empirical Results}
\label{sec:results}
\label{sec:additional-analysis}

All optima below use the scored tables from
Section~\ref{sec:empirical-audit}. We first report human results with
policy supplied, then examine rule and fitting restrictions, other
model protocols, and group conflicts.

\subsection{Human ratings retain a gap with policy supplied}
\label{sec:human-results}

On the 96 SmolLM2 states held out by policy area, the minimum losses using
policy and two feature indicators exceed the minima using the full state
by $0.1166$ and $0.1082$ for the two SAEs. Their individual one-sided
bootstrap lower bounds are $0.0935$ and $0.0901$
(Table~\ref{tab:human-policy-main}). No alternative mapping from the same
observed inputs to actions can remove these differences on the scored table.
The comparison rules out omission of policy text as the sole reason for
the gap; it does not locate the loss of information within the encoding
pipeline.

\begin{table}[!htb]
\centering\small
\setlength{\tabcolsep}{4pt}
\caption{SmolLM2 human-rated information gaps with two visible features ($B=2$). $\Delta_Z$ uses the code alone; $\Delta_{Z,S}$ also supplies the policy. Responses, features, and ratings are unchanged. Intervals and individual LCBs concern $\Delta_{Z,S}$. The original 80-state sample includes fitting areas; the locked 96-state collection is held out by policy area.}
\label{tab:human-policy-main}
\begin{tabular}{@{}lrrrlr@{}}
\toprule
Collection & SAE seed & $\Delta_Z$ & $\Delta_{Z,S}$ & 95\% interval & 95\% LCB \\
\midrule
Locked 96 & 314159 & 0.1484 & 0.1166 & [0.0881, 0.1420] & 0.0935 \\
Locked 96 & 271828 & 0.1334 & 0.1082 & [0.0862, 0.1299] & 0.0901 \\
Original 80 & 314159 & 0.0776 & 0.0345 & [0.0181, 0.0526] & 0.0203 \\
Original 80 & 271828 & 0.0990 & 0.0353 & [0.0153, 0.0594] & 0.0175 \\
\bottomrule
\end{tabular}
\end{table}

The original 80-state sample gives gaps of $0.0345$ and $0.0353$ but
includes fitting areas. Using either rater alone or their mean before
adjudication also gives positive lower bounds
(Table~\ref{tab:policy-human-sensitivity}). Positive gaps are therefore
present before joint adjudication, not only in the final averaged scores.
The observed collision in Table~\ref{tab:human-collision-main} illustrates
how two different optimal choices can share the same complete input.

\paragraph{Feature budget.}
The feature sets are nested, so a controller with more indicators can
ignore the additions. The optimal risk on a fixed table cannot increase;
the curves in Figure~\ref{fig:policy-budget-main} show the size of the
remaining gap, not an independent test of that monotonicity. Some lower
bounds reach zero and many gaps vanish at larger budgets. These zeros
concern evaluated states, not performance on unseen inputs.
\begin{figure}[!tb]
\centering
\begin{subfigure}[t]{0.485\linewidth}
\centering\includegraphics[width=\linewidth]{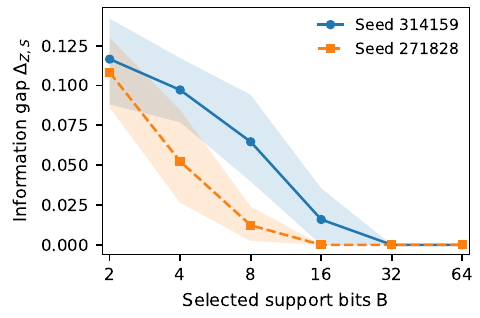}
\caption{SmolLM2: final human loss, locked 96.}
\end{subfigure}\hfill
\begin{subfigure}[t]{0.485\linewidth}
\centering\includegraphics[width=\linewidth]{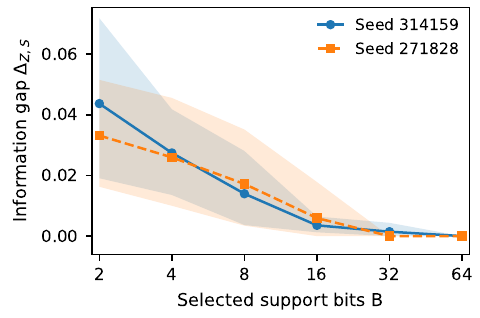}
\caption{Phi: Mistral loss, reused locked 96.}
\end{subfigure}
\caption{Information gap as nested feature sets expand, with policy supplied. Lines show the gap on the scored table $\Delta_{Z,S}$ and shaded bands show individual 95\% intervals from resampling policy areas. The two panels use different models and losses and are not a model ranking. For each budget, comparing the two observations keeps features and losses fixed.}
\label{fig:policy-budget-main}
\end{figure}

\paragraph{Score and action agreement.}
The automatic and human losses have Spearman correlation $0.379$ and mean
absolute difference $0.311$. Automatic scores have tied optimal actions
in 90 of 176 states. Choosing the first minimum in the fixed order
\textsc{ask}, \textsc{decline}, \textsc{assist} gives $51.1\%$ agreement;
the minimizing sets intersect in $80.7\%$ of states and coincide in
$30.7\%$. Set intersection is weaker than equality, especially with ties.
Risk values do not depend on which tied minimizer is selected.
Appendices~\ref{app:added-experiment-tables}
and~\ref{sec:tie-aware-diagnostics} give complete statistics.
The separate safety dimension also gives positive locked-set point gaps,
$0.0521$ and $0.0312$, but these are descriptive rating differences, not
unsafe event probabilities or new confidence bounds
(Appendix~\ref{sec:dimension-diagnostics}).

\subsection{Information, rule restrictions, and fitting error}
\label{sec:qwen-results}
\label{sec:untouched-certificates}
\label{sec:matching-interventions}

At Qwen layer 17, with two indicators and no separate policy input, the
minimum risk using the full state is $0.338$ and the minimum using the
code is $0.494$. The information gap is $0.156$, with bootstrap interval
$[0.125,0.180]$ and individual lower bound $0.129$. All four evaluation
codes occur in fitting. Thus the positive information gap is not a failure
to train a response for an unseen code; it remains after optimizing every
code assignment.

\begin{table}[!htb]
\centering\small
\setlength{\tabcolsep}{4pt}
\caption{Qwen2.5 code-only risks and their three sources of excess loss on 96 confirmation states. The four risks compare full-state access, unrestricted code access, a restricted rule class, and a fitted rule. Values and gaps are independently rounded; complete intervals appear in Table~\ref{tab:headline-certificates-full}.}
\label{tab:headline-compact}
\begin{tabular}{@{}rrrrrrrrr@{}}
\toprule
Layer & $B$ & $\Rfull$ & $\Rinterface$ & $\Rrestricted$ & $\Rlearn$ & $\Dinterface$ & $\Drule$ & $\Dtrain$ \\
\midrule
17 & 2 & 0.338 & 0.494 & 0.530 & 0.541 & 0.156 & 0.036 & 0.011 \\
19 & 16 & 0.338 & 0.398 & 0.444 & 0.466 & 0.060 & 0.047 & 0.022 \\
\bottomrule
\end{tabular}
\end{table}

At layer 19 with $B=16$, the information, accessibility, and fitting gaps
are $0.060$, $0.047$, and $0.022$, respectively
(Table~\ref{tab:headline-compact}). The information gap and the restriction
to the allowed family both contribute. The 49 observed codes have collision
fraction $0.760$ and unseen-code fraction $0.198$
(Table~\ref{tab:headline-certificates-full}).

Holding this interface fixed, replacing constant rules with affine rules
reduces minimum excess over the unrestricted interface risk from $0.143$
to $0.047$. Adding pairwise interactions reduces it to $0.007$;
third-order interactions and lookup rules attain zero class excess.
Separately, increasing the fitting sample from 16 to all 144 states of
its fixed population reduces the reported mean fitting gap from $0.130$
to zero (Figure~\ref{fig:rule-training-interventions}). Inclusion of rule
families guarantees that their optima cannot worsen. The observed decline
with training size is an experimental result, not a consequence of the
risk decomposition or a guarantee for new data.
\begin{figure}[!tb]
\centering
\begin{subfigure}[t]{0.485\linewidth}
\centering\includegraphics[width=\linewidth]{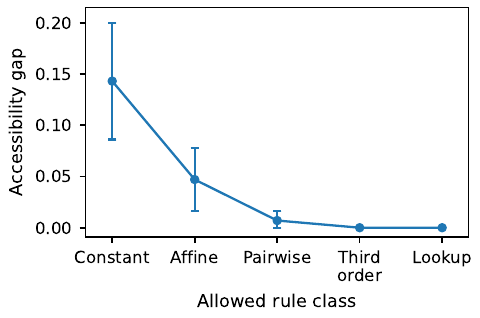}
\caption{Only the rule class changes.}
\end{subfigure}\hfill
\begin{subfigure}[t]{0.485\linewidth}
\centering\includegraphics[width=\linewidth]{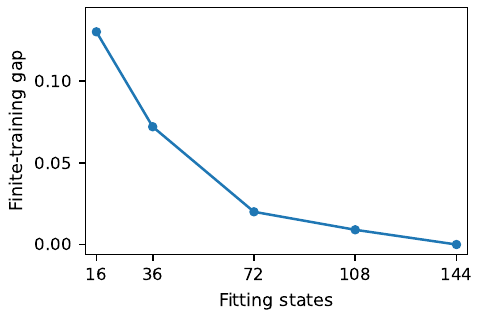}
\caption{Only the fitting sample size changes.}
\end{subfigure}
\caption{Separate comparisons of rule families and fitting samples. Left: minimum excess over the unrestricted interface risk, with available 95\% bootstrap intervals. Right: the fitted rule's mean excess over its family optimum. The panels use reported rounded summaries; full plots and intervals appear in Figure~\ref{fig:original-intervention-record}.}
\label{fig:rule-training-interventions}
\end{figure}

\subsection{Other models, evaluators, and observation choices}
\label{sec:policy-reanalysis}
\label{sec:phi-results}

For Phi seed 314159 on the original confirmation set, supplying policy
text reduces the Mistral loss gap from $0.1110$ to $0.0557$, with
responses, scores, and features unchanged. The extra input improves the
optimum but does not recover the full state minimum.
Table~\ref{tab:policy-aware-main} reports the gap $\Delta_{Z,S}$ using
code $Z$ and policy $S$ for each model and scoring protocol.

\begin{table}[!htb]
\centering\small
\setlength{\tabcolsep}{4pt}
\caption{Information loss with policy and SAE code jointly observed. SmolLM2 and Phi use Mistral losses averaged over five responses per state--action pair. $^\dagger$Qwen uses Qwen2.5-7B scores and rounded per-policy summaries; its interval endpoints include the $5\times10^{-5}$ rounding correction. LCBs are individual bootstrap lower bounds. Layer indices are zero-based; models use different responses and losses and are not ranked by this table.}
\label{tab:policy-aware-main}
\begin{tabular}{@{}lllrrr@{}}
\toprule
Model & Collection & Layer / SAE seed & $B$ & $\Delta_{Z,S}$ & 95\% LCB \\
\midrule
Qwen2.5-1.5B$^\dagger$ & Original 96 & 17 / -- & 2 & 0.13695 & 0.1113 \\
Qwen2.5-1.5B$^\dagger$ & Original 96 & 19 / -- & 16 & 0.03875 & 0.0274 \\
\addlinespace
SmolLM2 & Original 96 & 17 / 314159 & 2 & 0.0987 & 0.0807 \\
SmolLM2 & Original 96 & 17 / 271828 & 2 & 0.1015 & 0.0729 \\
SmolLM2 & Locked 96 & 17 / 314159 & 2 & 0.0460 & 0.0300 \\
SmolLM2 & Locked 96 & 17 / 271828 & 2 & 0.0430 & 0.0295 \\
\addlinespace
Phi & Original 96 & 23 / 314159 & 2 & 0.0557 & 0.0367 \\
Phi & Original 96 & 23 / 271828 & 2 & 0.0491 & 0.0266 \\
Phi & Reused 96 & 23 / 314159 & 2 & 0.0436 & 0.0225 \\
Phi & Reused 96 & 23 / 271828 & 2 & 0.0331 & 0.0185 \\
\bottomrule
\end{tabular}
\end{table}

SmolLM2 and Phi have positive automatic gaps and individual bootstrap
lower bounds for both SAE seeds on both evaluation collections. Qwen2.5-1.5B
also gives positive estimates, recovered from summaries with absolute
rounding error at most $5\times10^{-5}$; interval endpoints include the
same correction (Appendix~\ref{app:support-reanalysis}). Because generation
inputs and score rubrics differ, these are results within each protocol,
not model rankings.

Replacing Phi's Mistral scores with Qwen scores or their mean preserves
positive primary lower bounds (Table~\ref{tab:phi-primary}). Deleting one
policy area at a time leaves positive $B=2$ point gaps in all 192 recorded
checks, with minimum $0.0236$. Neither comparison adds independent response
data. The two evaluators have rank correlation $0.3552$, so the shared
sign does not imply close agreement of their scores.

\paragraph{Supporting controls.}
\label{sec:robustness}
Qwen also gives positive gaps with a different evaluator, variance and
random feature selection, controls matched on code occupancy, and a
second model scale (Appendix~\ref{app:additional-figures}). SmolLM2 and
Phi diagnostics appear in Appendices~\ref{app:added-experiment-tables}
and~\ref{app:phi-results}. The reconstruction error divided by centered activation variance is larger
on audit states than on the SAE validation corpus. Independently of that
diagnostic, the risk comparison concerns selected activation indicators,
not the information retained by the full SAE code.

\subsection{Group conflicts and required outputs}
\label{sec:higher-order-evidence}

The controlled benchmark uses the loss pattern in
Example~\ref{ex:shared-optimum-triple}. Its three states share a layer-17,
$B=2$ code. Each pair has a common action with zero regret, whereas the
smallest maximum regret for the triple is $1/3$, including randomized
choices. The benchmark is a constructed example, not a discovered failure under
natural language outcome loss. The searched natural language candidate
has zero positive separation in independent and fresh checks using two
paraphrase replicas, five generation replicas, and three judge prompts
(Appendix~\ref{app:additional-figures}).
\begin{figure}[!htb]
\centering
\begin{subfigure}[t]{0.485\linewidth}
\centering\includegraphics[width=\linewidth]{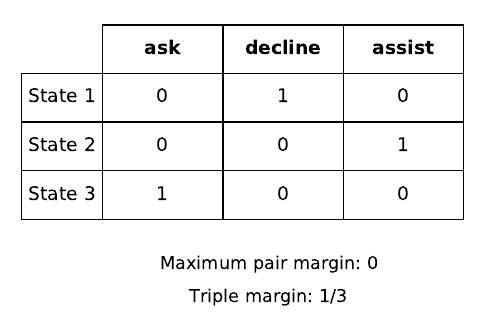}
\caption{Compatible pairs and an incompatible triple.}
\end{subfigure}\hfill
\begin{subfigure}[t]{0.485\linewidth}
\centering\includegraphics[width=\linewidth]{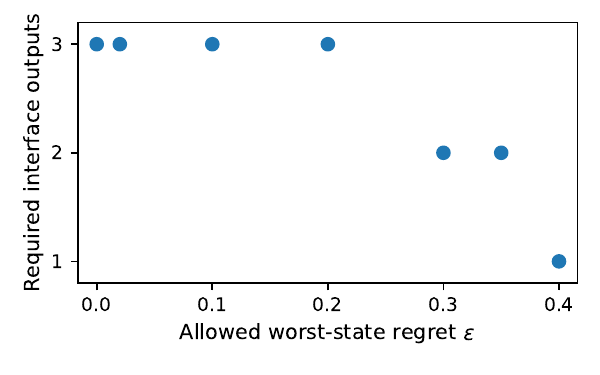}
\caption{Output requirements for 96 states with human ratings.}
\end{subfigure}
\caption{Group conflict and required output count. Left: a controlled three-state loss matrix has pair margins zero and triple margin $1/3$. Right: global constraint generation on all 96 locked SmolLM2 states uses their final human losses, with every new cell validated. The capacity counts unrestricted labels without a separate policy input, not SAE features; the calculation and summary results are described in Appendix~\ref{sec:human-capacity-diagnostic}. Markers show only evaluated tolerances, not interpolated transition thresholds. The two panels use different loss tables.}
\label{fig:structure-main}
\end{figure}

A separate capacity calculation uses all 96 locked SmolLM2 states and
their final human losses. Unrestricted label assignments require three
outputs at tested tolerances $0$, $0.02$, $0.10$, and $0.20$; two at
$0.30$ and $0.35$; and one at $0.40$. The margin with one output is
$0.3695$. Constraint generation and an independent two-label feasibility
formulation agree at all seven points, within the recorded numerical
tolerance. Appendix~\ref{sec:human-capacity-diagnostic} describes the calculation;
the matrix, labelings, and solver records are retained in the author
working archive. The capacity calculation does not train
an SAE or give policy text as an additional free input. The older Qwen
42-state curve, whose matrix is unavailable, remains in the appendix
and is not the curve shown here.

\FloatBarrier
\section{Related Work}
\label{sec:related-work}

\paragraph{Decision sufficiency and compression for a task.}

Blackwell comparison orders experiments by attainable Bayes risks across
decision problems
\citep{blackwell1951comparison,blackwell1953equivalent}, while Le Cam
deficiency quantifies approximate comparison and sufficiency
\citep{lecam1964sufficiency,torgersen1991comparison}.
Representation-learning extensions compare feature quality and optimize
deficiency-based objectives
\citep{vanrooyen2014lecam,banerjee2018deficiency}.
Recent formulations study Bayes sufficiency for a chosen loss, compression
that preserves optimal actions, and summaries for linear optimization
\citep{sevetlidis2026bayes,walsh2026support,ye2026decision}.
The information bottleneck trades compression against target information
\citep{tishby1999information,alemi2017deep}; usable information additionally
restricts the decoder family \citep{xu2020usable}.

We adopt loss-dependent risk comparison for a fixed continuation library
and audit distribution, rather than compare channels across all decision
problems. The optimum for the finite interface separates missing information from
limitations of the rule that reads it. Our structural analysis then asks
which states cannot share a continuation within the regret tolerance under the complete
declared observation, including policy information when supplied.

\paragraph{State abstraction and closed-loop information states.}

State abstraction preserves rewards, transitions, values, or near-optimal
behavior in Markov decision processes
\citep{givan2003equivalence,li2006unified,ferns2004metrics,ferns2011bisimulation,abel2016near,abel2020abstractions}.
Learned control representations seek to preserve latent dynamics, invariant
task-relevant structure, or exploration information
\citep{gelada2019deepmdp,zhang2020invariant,zhang2021learning,misra2020kinematic};
control sufficiency has also been analyzed for mutual-information
objectives \citep{rakelly2021mutual}.
Recent vision-language-action work addresses aliasing between similar
observations requiring different actions \citep{lee2026aliasing}.

Belief states and predictive state representations summarize histories for
recursive planning
\citep{smallwood1973optimal,kaelbling1998planning,littman2001predictive,singh2004predictive,boots2011closing};
approximate information states quantify the cost of compressed histories
\citep{subramanian2022approximate}.
Options formalize temporally extended closed-loop policies
\citep{sutton1999options,abel2020abstractions}.
Our continuation library follows the established closed-loop view.
We take the full audit state and continuation library as inputs and audit
an additional observation channel, rather than construct a new information
state. We ask whether states sharing an observation can share a continuation,
separately from the performance of restricted or fitted rules.

\paragraph{Set-valued decisions and higher-order compatibility.}

Reject-option classification permits abstention
\citep{chow1970optimum,bartlett2008classification,charoenphakdee2021classification},
while set-valued prediction controls error, coverage, or output-set size
\citep{denis2017confidence,sadinle2019least,bates2021risk,vovk2005algorithmic}.
These methods constrain predictions for each input. Our uniform
compatibility criterion instead requires one continuation distribution to
satisfy a regret tolerance at every state in a deterministic interface cell.
Pairwise intersections can all be nonempty while the joint intersection
is empty.

Helly's theorem bounds the size of a minimal convex witness
\citep{helly1923mengen}, and minimax duality relates its margin to a
worst-case state mixture \citep{sion1958general}.
Weak hypergraph coloring characterizes how many labels separate all
minimal conflicts \citep{berge1989hypergraphs}.
We connect these classical results to exact risk for the finite interface and the
capacity of unrestricted deterministic interfaces, using conflicts over the entire audit domain, not only within existing
cells.
On a finite audit domain, grouping by optimal action is sufficient when each
state has a unique optimal pure continuation and the nonnegative tolerance
is below half the minimum second-best regret margin
(Proposition~\ref{prop:unique-optimum-reduction}).
Outside that regime, a pairwise conflict graph need not describe joint
compatibility.

\paragraph{SAE interfaces and latent safety control.}

Sparse coding and sparse autoencoders learn dictionaries with sparse codes
\citep{olshausen1996emergence,makhzani2013ksparse}.
Language-model studies examine superposition \citep{elhage2022superposition},
interpretable dictionary features
\citep{bricken2023monosemanticity,cunningham2024sparse},
SAE scaling \citep{gao2024scaling,templeton2024scaling},
public dictionaries \citep{lieberum2024gemmascope},
and concept geometry \citep{bhalla2026manifolds}.
Our target is decision information in a selected support code or coordinate
set, not reconstruction from the full SAE dictionary.

Safe reinforcement learning develops constrained policies
\citep{garcia2015survey,achiam2017constrained,chow2018lyapunov},
shields \citep{alshiekh2018shielding,carr2023shielding},
and reachability or predictive safety filters
\citep{fisac2019general,wabersich2021predictive,brunke2022safe}.
Related activation studies extract latent knowledge
\citep{burns2023discovering} or steer generation
\citep{li2023inference,zou2023representation,karnik2025reachability}.
We audit the specified interface before selecting a downstream rule:
a positive gap on the same finite loss table cannot be removed by changing
that rule alone. The result is complementary to designing or evaluating a particular
monitor or steering method.

\paragraph{Evaluation and generative safety.}

Language-model evaluators score open-ended outputs and compare candidate
responses \citep{liu2023geval,zheng2023judging,kim2024prometheus};
reward-model benchmarks assess preference judgments
\citep{lambert2025rewardbench}.
Judging studies identify biases \citep{wang2024fair} and develop calibrated
tests with imperfect evaluators \citep{feng2026noisy}.
Risk-control procedures calibrate predictions from finite data
\citep{angelopoulos2025learntest}, while concentration inequalities provide
bounds for bounded loss estimates \citep{hoeffding1963probability}.
Our certificate concerns a witness selected after searching state groups.
Valid simultaneous loss bounds or independent certification data preserve
the guarantee after selection.

Toxicity benchmarks, red teaming, principle-based training, and safeguard
classifiers evaluate or modify generative behavior
\citep{gehman2020realtoxicity,perez2022redteaming,bai2022constitutional,inan2023llamaguard}.
Our question is which safe and useful continuations the declared
observation can support, not only whether an individual output is safe.
The human study tests the finite-interface finding under a separate loss
construction; it does not establish interchangeability of automatic and
human scores.
Appendix~\ref{app:extended-related-work} provides the extended comparison.

\section{Discussion}
\label{sec:discussion}

\paragraph{What the risk comparison identifies.}
For fixed losses and continuations, a positive information gap
$\Dinterface$ rules out matching the full state optimum by changing only
the controller. The accessibility gap $\Drule$ can decrease when the
allowed rule family expands. The fitting gap $\Dtrain$ may decrease with
better estimation or optimization, but more data alone need not reduce
it in every run. These statements compare specified optima; they do not
identify every cause of poor deployed performance.

\paragraph{Encoding matters, not only the number of bits.}
With three continuations, an unrestricted encoder can transmit an optimal
action index using three labels, which fit in two bits. The positive gaps
for selected SAE indicators do not prove that every pair of features
fails, or that two bits are inherently insufficient. They identify a
failure of the selected partition. Comparing the full state directly
with the final indicators does not locate whether a distinction was lost
in the model activation, SAE encoding, coordinate selection, or
thresholding. Those stages require separate comparisons.

\paragraph{What the experiments support.}
Supplying policy text does not eliminate the observed gaps. Human and
automatic losses give the same sign in the tested SmolLM2 settings but
often differ on individual scores and optimal action sets. The losses
measure rated safety and quality, not unsafe event probabilities. The
controlled group example demonstrates a possible failure of pairwise
checks; the unsuccessful natural language replication gives no estimate
of its prevalence.

The model families use different generation and scoring protocols. Within
each SmolLM2 or Phi SAE pair, only initialization and training order vary;
the activation corpus is shared. Human ratings cover SmolLM2 only, and
Phi reuses the policy collections. The original human sample includes
fitting areas, whereas the locked collection is held out by policy area.
These comparisons support the specific configurations, not all models,
encoders, or policy distributions.

\paragraph{Limits of inference from a fixed table.}
The empirical minima use recorded response scores, not unknown losses
on future responses. Minimizing noisy scores can change the apparent
gap even when individual scores are unbiased. Resampling policy areas
measures variation under different weights on the recorded tables;
it does not control generation or judge error. Likewise, a zero gap on
distinct evaluation codes can result from unrestricted lookup and gives
no guarantee on unseen inputs. The additional assumptions for statistical
certification appear in Appendices~\ref{app:continuous-certificates}
and~\ref{app:extended-witness-certification}.

\section{Conclusion}
\label{sec:conclusion}

Decision sufficiency compares what an interface permits with what the full
state permits under the same loss and continuations. Exact optimization
for finite interfaces separates that information gap from restrictions
on the controller and imperfect fitting. Minimal conflicts identify groups
that cannot share an acceptable choice; their complete set determines the
output count for an unrestricted deterministic interface.

The tested SAE indicators retain positive gaps with policy supplied,
including under SmolLM2 human ratings. A controlled example shows why
pairwise checks can miss a group conflict, but the natural language
candidate does not replicate. The conclusions concern the specified
responses and loss tables. Within that scope, they distinguish limitations
of the chosen input from errors a different controller could avoid.

\section*{Reproducibility Statement}
\label{sec:reproducibility}

The arXiv source package contains the manuscript, bibliography, figures,
and LaTeX tables. Experiment code and cached records are retained in a
separate author working archive and are not included in the arXiv source
package. The working archive contains SmolLM2 ratings and adjudication
inputs, cached SmolLM2 and Phi records, and CPU analysis scripts. Its
96-state capacity records include the matrix, labelings, margins, and
independent feasibility checks. Qwen results with policy supplied use
rounded notebook summaries; its full precision historical arrays and
42-state matrix are unavailable. Large SAE checkpoints and activation
shards are not included. Appendix~\ref{app:reproduction-inventory}
lists the inputs and commands. Analysis scripts do not create human
labels or retrain models.

\bibliographystyle{tmlr}
\bibliography{references}

\clearpage
\appendix
\section{Additional Risk Results and Proofs}
\label{app:proofs-risk-interface}

The results below establish the risk comparisons, their posterior interpretation,
and the calculations for affine rules and uncertain losses.
The audit setting and all decision permissions are those of
Section~\ref{sec:problem-formulation}.

\subsection{Proofs of the main risk comparisons}
\begin{proof}[Proof of Theorem~\ref{thm:nested-risks}]
For every state $u$ and choice $q\in\Delta_m$,
$\ell^\star(u)\le\ell(u,q)$. Taking expectations and then an infimum
proves $\Rfull\le\Rinterface$. The inclusions
$\widehat\delta\in\Controllers\subseteq\mathfrak D_Z$ give
$\Rinterface\le\Rrestricted\le\Rlearn$. Adding and subtracting
$\Rinterface$ and $\Rrestricted$ proves the decomposition. Every step
uses the same loss and audit distribution.
\end{proof}

\begin{proof}[Proof of Theorem~\ref{thm:finite-interface-optimum}]
A selector observing state $u_i$ can choose a minimizing column of
$L_{ij}$, giving $\Rfull=\sum_i\mu_i\min_jL_{ij}$. A rule assigning
$q_r\in\Delta_m$ to output $z_r$ has risk
\[
 \sum_{i,r,j}\mu_iP_{ir}L_{ij}q_{rj}
 =\sum_{r,j}C_{rj}q_{rj}.
\]
Each $q_r$ has its own simplex constraint. Its linear objective has
minimum $\min_jC_{rj}$, attained by any distribution supported on the
minimizing columns. Summing over outputs proves the formula. A row with
zero observation probability has zero cost for every column.
\end{proof}

\begin{proof}[Justification of Equation~\eqref{eq:policy-aware-risk}]
On the finite audit, the complete observation $(Z,S)$ has a finite alphabet.
Apply Theorem~\ref{thm:finite-interface-optimum} to its joint observation
cells $C_{z,s}$. Their unnormalized costs are
$\mathbb E[\Indicator\{Z=z,S=s\}\ell_j(U)]$, giving the displayed sum of
minima. The full-state oracle cannot be outperformed by either observation
rule. Moreover, any code-only rule can be implemented by a policy-aware
rule that ignores $S$, so $\Rfull\le R_{Z,S}^\star\le R_Z^\star$.
Subtracting the same full-state risk gives
$0\le\Delta_{Z,S}\le\Delta_Z$.
\end{proof}

\subsection{Posterior conflict and zero information loss}
\label{app:posterior-risk}
\paragraph{Residual uncertainty after observing the interface.}
Let $P_Z$ be the marginal law of $Z$, and let
$\mu_z:=P(U\in\cdot\mid Z=z)$ be a regular conditional state distribution.
Assumption~\ref{ass:basic-audit-setting} ensures that such a conditional
distribution exists. The posterior expected regret of continuation $\pi_j$
and the smallest posterior expected regret are
\[
 \bar e_j(z):=\int e_j(u)\,\mu_z(du),
 \qquad
 \ConditionalConflict(z):=
 \min_{q\in\Delta_m}\sum_{j=1}^m q_j\bar e_j(z)
 =\min_{j\in[m]}\bar e_j(z).
\]
The selector can choose a different continuation for each observation $z$,
but it must use one distribution for the full states remaining possible
under $\mu_z$. The posterior conflict $\ConditionalConflict(z)$ measures
the resulting minimum average regret.

\begin{theorem}[Posterior-conflict decomposition]
\label{thm:posterior-conflict}
Under Assumption~\ref{ass:basic-audit-setting},
\[
 \Dinfo_\sigma(K;\mu)=\mathbb E_Z[\ConditionalConflict(Z)].
\]
The information gap is zero if and only if the nonnegative posterior conflict
$\ConditionalConflict(Z)$ is zero $P_Z$-almost surely.
\end{theorem}

\begin{proposition}[Zero-gap characterization]
\label{prop:zero-gap-characterization}
Under Assumption~\ref{ass:basic-audit-setting}, the following are equivalent:
\begin{enumerate}
 \item The interface information deficiency $\Dinfo_\sigma(K;\mu)$ is zero.
 \item For $P_Z$-almost every observation $z$, some continuation index
       $j\in[m]$ satisfies
       $\mu_z(\{u:e_j(u)=0\})=1$.
 \item There exists a measurable pure selector $j^\star:\RepSpace\to[m]$
       such that $e_{j^\star(Z)}(U)=0$ almost surely.
\end{enumerate}
\end{proposition}

Zero average information loss therefore requires a common optimal
continuation under almost every posterior distribution, not merely a
potentially different optimum for each hidden state. The later uniform
criterion uses a nonnegative tolerance $\varepsilon$ and requires one
continuation distribution to have regret at most $\varepsilon$ at
\emph{every} listed state in a cell. Average sufficiency and uniform
compatibility are not interchangeable at positive tolerance.

\begin{proposition}[Data processing for representation interfaces]
\label{prop:data-processing}
Let $G(dy\mid z)$ be a Markov kernel into a standard Borel observation space,
and obtain $Y$ from $Z$ without additional information about $U$, so that
$U\to Z\to Y$ is a Markov chain. Under the same audit distribution,
continuation library, and loss,
\[
 \Dinfo_\sigma(G\circ K;\mu)\ge\Dinfo_\sigma(K;\mu).
\]
\end{proposition}

Post-processing can increase the best attainable loss but cannot reduce it.
In particular, a positive gap after analysis-only quantization does not by
itself establish a positive gap for the original unquantized observation.

\begin{proof}[Proof of Theorem~\ref{thm:posterior-conflict}]
For any measurable rule $\delta$, conditional expectation gives
\[
 \mathbb E[e(U,\delta(Z))]
 =\mathbb E_Z\!\left[\sum_j\delta_j(Z)\bar e_j(Z)\right]
 \ge\mathbb E_Z[\ConditionalConflict(Z)].
\]
The finitely many conditional regrets $\bar e_j$ have measurable versions.
Choosing the smallest minimizing index at each observation defines a
measurable pure rule and attains equality. Nonnegativity of
$\ConditionalConflict$ then makes its expectation zero exactly when
$\ConditionalConflict(Z)=0$ almost surely.
\end{proof}

\begin{proof}[Proof of Proposition~\ref{prop:zero-gap-characterization}]
The posterior formula makes zero deficiency equivalent to
$\min_j\bar e_j(z)=0$ for almost every observation $z$. Each regret is
nonnegative, so $\bar e_j(z)=\int e_j(u)\,\mu_z(du)=0$ exactly when
$e_j(u)=0$ for $\mu_z$-almost every state. The first two statements are therefore equivalent. Choosing the smallest minimizing index as in
the posterior proof gives a measurable selector with zero regret almost
surely. Conversely, any such selector has expected regret zero, and no
rule can have negative regret.
\end{proof}

\begin{proof}[Proof of Proposition~\ref{prop:data-processing}]
For any measurable rule $\eta$ using $Y$, define
\[
 \delta_j(z):=\int\eta_j(y)G(dy\mid z),\qquad j\in[m].
\]
These components are measurable, nonnegative, and sum to one, so
$\delta\in\mathfrak D_Z$. The Markov relation $U\to Z\to Y$ and
linearity of loss give
$\mathbb E[\ell(U,\delta(Z))]=\mathbb E[\ell(U,\eta(Y))]$.
Every risk available from $Y$ is therefore available from $Z$.
Taking infima and subtracting the same full state risk proves the claim.
\end{proof}

\paragraph{Deterministic cells and uniform compatibility.}
For a deterministic interface $f:\AuditDomain\to\FiniteAlphabet$, define
the occupied cell $C_r:=\{u_i:f(u_i)=z_r\}$ for each nonempty cell.
Its robust regret of a shared continuation is
\[
 \Gamma(C_r):=\min_{q\in\Delta_m}\max_{u\in C_r}e(u,q).
\]
Unlike posterior conflict, which averages over states, $\Gamma(C_r)$
requires one continuation distribution to control the worst regret in the cell.

\begin{corollary}[Finite-cell zero-gap criterion]
\label{cor:finite-cell-zero-gap}
If the deterministic finite audit has strictly positive state weights,
$\Dinfo_\sigma(K;\mu)=0$ if and only if every occupied cell $C_r$ has a
pure continuation optimal at every state in $C_r$. Equivalently,
$\Gamma(C_r)=0$ for every occupied cell.
\end{corollary}

Positive weights make the criterion apply to every listed state, rather than
only to states of positive probability. At positive regret tolerance,
randomization may improve a shared decision.

\begin{proof}[Proof of Corollary~\ref{cor:finite-cell-zero-gap}]
Positive state weights give every listed state positive conditional
probability in its occupied cell. The zero-gap characterization therefore
requires a common optimal pure continuation in each cell, which gives
$\Gamma(C_r)=0$. Conversely, a minimizing distribution for
$\Gamma(C_r)=0$ exists by compactness. At each state,
$\sum_jq_je_j(u)=0$; nonnegative regrets force every index with $q_j>0$
to be optimal throughout the cell. At least one such index exists.
Choosing a common optimum in each cell gives zero information gap.
\end{proof}

\subsection{Exact restricted affine rules}
\label{app:restricted-affine}
\paragraph{Affine probabilities.}
Return to a declared finite observation alphabet. Let
$\varphi:\FiniteAlphabet\to\mathbb R^{d_\varphi}$ be a fixed visible
feature map, where $d_\varphi$ is its dimension, and set
$\varphi_r:=\varphi(z_r)$. An affine randomized rule assigns
$q_r=A\varphi_r+b$, with $A\in\mathbb R^{m\times d_\varphi}$ and
$b\in\mathbb R^m$. Its parameter bound satisfies
$1/m\le B_{\mathrm{par}}<\infty$; $B_{\mathrm{par}}$ is distinct from
the SAE feature budget $B$.

\begin{proposition}[Restricted affine-randomized risk is a linear program]
\label{prop:affine-randomized-lp}
For the affine class with simplex constraints on the entire declared alphabet,
\[
 \begin{aligned}
 \Rrestricted=\min_{A,b,\{q_r\}}\quad
 &\sum_{r=1}^L\sum_{j=1}^m C_{rj}q_{rj}\\
 \text{subject to}\quad
 &q_r=A\varphi_r+b, &&r\in[L],\\
 &q_r\in\Delta_m, &&r\in[L],\\
 &|A_{jk}|\le B_{\mathrm{par}},\quad |b_j|\le B_{\mathrm{par}},
 &&j\in[m],\ k\in[d_\varphi].
 \end{aligned}
\]
The uniform rule $A=0$, $b_j=1/m$ is feasible, and the minimum is attained.
\end{proposition}

Fitting replaces $C_{rj}$ by training-sample costs, changing the objective
but not the feasible observation domain. Simplex constraints must still
hold on declared evaluation codes absent from training. Observations outside
that declared domain require a separately specified feasible extension.
For a policy-aware audit, the cost table and visible features must use
$(Z,S)$ rather than $Z$ alone. The linear program uses standard
convex-optimization tools \citep{boyd2004convex};
Algorithm~\ref{alg:restricted-rule-audit} separates its rule-class and
fitting gaps.

\begin{proof}[Proof of Proposition~\ref{prop:affine-randomized-lp}]
Substituting $q_r=A\varphi_r+b$ into the finite risk gives the stated
linear objective and linear constraints. The choice $A=0$, $b_j=1/m$
is feasible because $B_{\rm par}\ge1/m$. Parameter and simplex bounds
make the feasible set compact, so the continuous objective attains its
minimum. Replacing evaluation costs by training costs leaves the same
feasible family of rules.
\end{proof}

\subsection{Effect of errors in the loss table}
\label{app:loss-perturbation}
\paragraph{Group margin.}
For any nonempty state group $W\subseteq\AuditDomain$, define
$\Gamma(W):=\min_{q\in\Delta_m}\max_{u\in W}e(u,q)$ using the regrets
of the specified loss table.

\begin{proposition}[Loss-table perturbation]
\label{prop:loss-table-perturbation}
Fix the finite audit weights and observation channel. If two loss matrices
in $[0,1]^{n\times m}$ differ by at most $r_\ell\ge0$ entrywise, their full-state
risks and optimal interface risks each differ by at most $r_\ell$. Their information
gaps differ by at most $2r_\ell$, and their robust margins $\Gamma(W)$ differ
by at most $2r_\ell$ for every nonempty state group $W$.
\end{proposition}

An exact finite-table optimum and an accurate estimate of expected outcome
loss are separate requirements. The perturbation radius $r_\ell$ transfers
accuracy of the loss entries to accuracy of the resulting risks; a positive
plug-in gap without such error control does not certify a positive gap in expected loss.

\begin{proof}[Proof of Proposition~\ref{prop:loss-table-perturbation}]
Let the two loss tables be $\LossMatrix$ and
$\widetilde{\LossMatrix}$. At every state, a fixed continuation mixture
has losses differing by at most $r_\ell$. The same holds after averaging
over the fixed observation channel and state weights. For two objectives
whose pointwise difference is at most $r_\ell$ on the same feasible set,
their infima differ by at most $r_\ell$. Apply this fact to the full state
and interface optima; their difference changes by at most $2r_\ell$.

Each regret subtracts a row minimum from a loss entry, so regret entries
differ by at most $2r_\ell$. Convex combinations, a maximum over a fixed
group $W$, and a minimum over $\Delta_m$ preserve this uniform bound.
Hence the group margins differ by at most $2r_\ell$.
\end{proof}

\section{Proofs for Group Conflicts and Unique Optimal Choices}
\label{app:proofs-higher-order}

\subsection{Proofs for shared decisions and minimal conflicts}
\paragraph{Existence, compatibility, and monotonicity.}
For a nonempty finite state group $W$, the map
$q\mapsto\max_{u\in W}\sum_je_j(u)q_j$ is continuous on the compact
simplex $\Delta_m$. A minimizing shared selection therefore exists.
If the acceptable sets have a common member $q$, every state regret is at
most $\varepsilon$, so $\Gamma(W)\le\varepsilon$. Conversely, if the
minimum margin is at most $\varepsilon$, a minimizing $q$ belongs to all
acceptable sets. For nonempty $V\subseteq W$, the maximum over $V$ is no
larger than the maximum over $W$ for every $q$. Taking minima gives
$\Gamma(V)\le\Gamma(W)$. A finite incompatible set contains an
inclusion-minimal incompatible subset: delete a state whenever the
remaining set is still incompatible, and stop when no deletion is possible.
Monotonicity makes all smaller proper subsets compatible at termination.

\begin{proof}[Proof of Theorem~\ref{thm:average-robust-bridge}]
For every continuation distribution $q$, let
$u_q\in\argmax_{u\in W}e(u,q)$. Nonnegative regrets and the smallest
state weight $\beta=\min_uw_u$ give
\[
 \beta\max_{u\in W}e(u,q)
 \le w_{u_q}e(u_q,q)
 \le\sum_{u\in W}w_ue(u,q)
 \le\max_{u\in W}e(u,q).
\]
Taking the minimum over $q$ proves both bounds, including the case
$\beta=0$. For a deterministic interface, the finite optimum also gives
\[
 \Dinfo_\sigma(K_f;\mu)
 =\sum_{c:p_c>0}\min_j\sum_{u\in C_c}\mu(u)e_j(u)
 =\sum_{c:p_c>0}p_c\ConditionalConflict(C_c,w^{(c)}).
\]
If every cell of positive mass has margin at most $\varepsilon$, the
upper bound and $\sum_cp_c=1$ make this sum at most $\varepsilon$.
\end{proof}

\begin{proof}[Proof of Theorem~\ref{thm:tight-witness-rank}]
Every singleton has an optimal pure continuation, so a conflict has at
least two states. Each acceptable set is convex in the affine hull of
$\Delta_m$, whose dimension is $m-1$. Helly's theorem
\citep{helly1923mengen} gives an empty intersection among at most $m$
sets whenever the full finite intersection is empty. An inclusion-minimal
conflict must itself have at most $m$ states.

For tightness at every $m\ge2$, take $m$ states and losses
$\ell_j(u_i)=\Indicator\{j=i\}$. Each row has minimum zero. At
$\varepsilon=0$, state $u_i$ requires $q_i=0$. All $m$ requirements
contradict $\sum_iq_i=1$. After any state $u_k$ is deleted, selecting
continuation $k$ gives zero regret at every remaining state. The full
group is therefore a minimal conflict of size $m$.
\end{proof}

\begin{proof}[Proof of Theorem~\ref{thm:witness-duality}]
For a fixed group $W$, the payoff
$\sum_{u,j}\lambda_uq_je_j(u)$ is bilinear on the compact convex sets
$\Delta(W)$ and $\Delta_m$. Finite minimax duality
\citep{sion1958general} gives
\[
 \begin{aligned}
 \Gamma(W)
 &=\min_{q\in\Delta_m}\max_{\lambda\in\Delta(W)}
                      \sum_{u,j}\lambda_uq_je_j(u)\\
 &=\max_{\lambda\in\Delta(W)}\min_{q\in\Delta_m}
                      \sum_{u,j}\lambda_uq_je_j(u)\\
 &=\max_{\lambda\in\Delta(W)}\min_j\sum_{u\in W}\lambda_ue_j(u).
 \end{aligned}
\]
The first equality expresses a maximum over states as a maximum over
state distributions; the last uses linearity in $q$. The final maximum
is attained because a finite minimum of linear functions is continuous.
If $f$ is constant on $W$, the finite interface formula gives
$\Dinfo_\sigma(K_f;\lambda)=\min_j\sum_u\lambda_ue_j(u)$.
Maximizing over $\lambda$ proves the risk interpretation.
\end{proof}

\subsection{Unique optimal choices and positive tolerances}
\label{app:unique-optimum}

The main example has two optimal pure choices at every state. The example
below instead has one optimal choice per state and shows why positive
tolerance can still create a higher-order conflict.

\begin{example}[Pairwise compatibility without joint compatibility]
\label{ex:three-state-conflict}
Consider three states $u_1,u_2,u_3$ and three continuations $a,b,c$, with
losses
\[
 \begin{array}{c|ccc}
       & a&b&c\\ \hline
  u_1 &0&1&1\\
  u_2 &1&0&1\\
  u_3 &1&1&0
 \end{array}
\]
and tolerance $\varepsilon=1/2$. Every state has oracle loss zero, so the
loss entries are also regrets. A shared selection
$q=(q_a,q_b,q_c)\in\Delta_3$ must satisfy
\[
 \begin{aligned}
 \AcceptSet^{1/2}(u_1)&=\{q\in\Delta_3:q_a\ge1/2\},\\
 \AcceptSet^{1/2}(u_2)&=\{q\in\Delta_3:q_b\ge1/2\},\\
 \AcceptSet^{1/2}(u_3)&=\{q\in\Delta_3:q_c\ge1/2\}.
 \end{aligned}
\]
For the pair $\{u_1,u_2\}$, regrets are $1-q_a$ and $1-q_b$.
Since $q_a+q_b\le1$, their maximum is at least $1/2$, attained by
$(1/2,1/2,0)$. The same argument applies to every pair. For the triple,
\[
 \max_i e(u_i,q)=1-\min\{q_a,q_b,q_c\}\ge2/3,
\]
with equality at the uniform distribution. Every pair is therefore
compatible at tolerance $1/2$, but the triple is not.
\end{example}

\paragraph{When the pairwise view is sufficient.}
A unique optimum does not remove all higher-order conflicts at arbitrary
positive tolerance, because randomized selections can become acceptable.
A uniform separation from the second-best choice gives a sufficient
small-tolerance regime.

\begin{proposition}[Unique-optimum reduction at small tolerance]
\label{prop:unique-optimum-reduction}
Suppose every state has a unique optimal pure continuation with index
$j^\star(u)\in[m]$. Define the smallest second-best regret by
\[
 \rho:=\min_{u\in\AuditDomain}\min_{j\neq j^\star(u)}e_j(u).
\]
Finiteness and uniqueness imply $\rho>0$. For
$0\le\varepsilon<\rho/2$, two states are incompatible if and only if their
optimal continuation indices differ. Every minimal conflict is then a
pair, and the minimum number of unrestricted deterministic interface
outputs with compatible cells equals the number of distinct optimal
continuation indices.
\end{proposition}

Example~\ref{ex:three-state-conflict} has $\rho=1$ and
$\varepsilon=1/2$, so it lies outside the proposition's strict range.

\begin{proof}[Proof of Proposition~\ref{prop:unique-optimum-reduction}]
There are finitely many strictly positive nonoptimal regrets, so their
minimum $\rho$ is positive. States with the same optimal index share a
pure optimum. For two states with distinct optimal indices $a$ and $b$,
\[
 e(u,q)\ge\rho(1-q_a),\qquad e(v,q)\ge\rho(1-q_b).
\]
A common choice with regret at most $\varepsilon<\rho/2$ would require
$q_a+q_b\ge2(1-\varepsilon/\rho)>1$, which is impossible. Every group
containing different optimal indices therefore contains an incompatible
pair; a group with only one optimal index is compatible. Minimal conflicts
are exactly the pairs with different indices. Distinct indices require
distinct outputs, and assigning each state its optimal index attains that
output count.
\end{proof}

\section{Capacity Extensions and Statistical Certification}
\label{app:proofs-capacity-certification}

\subsection{Proofs for capacity and global validation}
\begin{proof}[Proof of Theorem~\ref{thm:hypergraph-capacity}]
A monochromatic conflict lies entirely in one cell and makes that cell
incompatible. Conversely, every incompatible finite cell contains an
inclusion-minimal incompatible subset, which is a monochromatic edge of
the global hypergraph. A labeling has compatible cells exactly when it
has no monochromatic edge. Minimizing its number of outputs is therefore
the weak coloring problem.
\end{proof}

\begin{proof}[Proof of Proposition~\ref{prop:capacity-constraint-generation}]
Every compatible labeling satisfies all retained conflict constraints,
so $M_t\le\CapacityCurve(\varepsilon)$. If the candidate labeling $g_t$
passes every cell check, it is feasible for the full problem with $M_t$
outputs; hence $\CapacityCurve(\varepsilon)\le M_t$.

Otherwise a violating cell contains a minimal conflict $W_t$. The edge
cannot already be retained, because $g_t$ would violate its constraint.
Each failed iteration adds a new edge. At most
\[
 \sum_{s=2}^{\min\{n,m\}}\binom ns-|\mathcal F_0|
\]
new edges can be added. The sum is zero if $\min\{n,m\}<2$.
Assigning each state a separate label is always feasible, so no iteration
has an infeasible full coloring problem. Once no missing edge remains,
the candidate passes validation and the bounds coincide.
Deleting a state whenever the remaining margin exceeds $\varepsilon$
finds a minimal conflict: every proper subset lies inside a compatible
single deletion at termination.
\end{proof}

\paragraph{Why local conflicts are not enough.}
Use the three-state loss table in Example~\ref{ex:three-state-conflict},
now at tolerance $\varepsilon=0$. Suppose the old interface cells are
$\{u_1,u_2\}$ and $\{u_3\}$. Local search finds only the conflict
$\{u_1,u_2\}$. The single retained constraint allows a two-label redesign that
merges $u_3$ with one of the other states. Every pair actually requires
different optimal continuations at zero tolerance, so the true capacity
is three. Checking all \emph{new} cells detects the missing constraints.
Algorithms~\ref{alg:capacity-milp} and~\ref{alg:capacity-global-validation}
implement the coloring and validation steps; solver tolerances and
optimality status must accompany numerical outputs.

\subsection{Restricted encoders, zero regret, and observable policy}
\label{app:capacity-extensions}
Let $\InterfaceFamily$ be a restricted family of deterministic finite-output
maps on the same audit domain. Its minimum feasible output count is
\[
 \CapacityCurve_{\InterfaceFamily}(\varepsilon):=
 \min\{|\range(g)|:g\in\InterfaceFamily,
                 \ g\text{ is uniformly }\varepsilon\text{-sufficient}\},
\]
with value $+\infty$ if no map in the family is feasible. Restricting the
maps cannot improve the optimum:
$\CapacityCurve_{\InterfaceFamily}(\varepsilon)\ge\CapacityCurve(\varepsilon)$.
An SAE support map may fail to realize an optimal abstract labeling even
when it has enough possible outputs. For an allowed output count $M\ge1$,
the inverse curve is
\[
 \RegretCurve(M):=\inf\{\varepsilon\ge0:
                           \CapacityCurve(\varepsilon)\le M\}.
\]
The value $\RegretCurve(M)$ is the smallest uniform regret tolerance
achievable with at most $M$ unrestricted outputs.

\paragraph{A useful special case: zero regret.}
For each continuation $\pi_j$, let
$O_j:=\{u\in\AuditDomain:e_j(u)=0\}$ be the states where it is optimal.
The family $\{O_j\}_{j=1}^m$ covers the audit domain because every state
has an optimal continuation.

\begin{proposition}[Capacity bounds and zero-regret set cover]
\label{prop:capacity-set-cover}
For every $\varepsilon\ge0$,
\[
 1\le\CapacityCurve(\varepsilon)\le\min\{n,m\}.
\]
At zero tolerance,
\[
 \CapacityCurve(0)=
 \min_{J\subseteq[m]}\left\{|J|:
                     \bigcup_{j\in J}O_j=\AuditDomain\right\}.
\]
Computing the capacity is NP-hard in general when the continuation count
$m$ is part of the input.
\end{proposition}

The zero-regret formula selects the smallest collection of continuations
that supplies an optimum for every state. Each selected continuation can
serve one label cell. The at-most-$m$ output bound does not imply that a
restricted SAE can implement the optimal-action labeling with a small
feature budget.

\paragraph{Capacity when policy is supplied separately.}
For a pooled-policy audit, let $U$ denote the full audited state and
evaluate each loss $\ell_j(u)$ under the policy carried by state $u$.
Let $S=s(U)$ be the finite-valued policy observation
supplied at no code cost,
and let $\mathcal S:=s(\AuditDomain)$ be its values on the audit domain.
For each policy value $s\in\mathcal S$, define
\[
 \mathcal U_s:=\{u\in\AuditDomain:s(u)=s\},
 \qquad
 \mathcal H_{\varepsilon,s}:=
 \left(\mathcal U_s,
       \{W\in\ConflictEdges:W\subseteq\mathcal U_s\}\right).
\]
Define $M_S^\star(\varepsilon)$ as the smallest code alphabet for a map
$g$ such that every nonempty complete-observation cell
$\{u:g(u)=c,\ s(u)=s\}$ is compatible.

\begin{corollary}[Capacity with free finite policy information]
\label{cor:policy-capacity}
Over unrestricted deterministic code maps on the full audit state,
\[
 M_S^\star(\varepsilon)=
 \max_{s\in\mathcal S}\HyperChromatic(\mathcal H_{\varepsilon,s}).
\]
\end{corollary}

Each policy requires its own compatible partition, but code labels can be
reused across policies because the selector also observes $S$. The maximum
counts code values, not the sum of code--policy pairs. The encoder is
unrestricted in this corollary; a restricted encoder requires its own
realizability analysis.

\begin{proof}[Proof of Proposition~\ref{prop:capacity-set-cover}]
A nonempty domain requires a label. Assigning each state an optimal
continuation index gives compatible cells at regret zero, using at most
$\min\{n,m\}$ labels.

For a cell with zero shared regret, a feasible distribution $q$ satisfies
$\sum_jq_je_j(u)=0$ at every state. Nonnegativity makes any index with
$q_j>0$ optimal throughout the cell. One such index per cell yields a
cover by optimal-state sets. Conversely, a cover $J\subseteq[m]$ yields
a compatible labeling by assigning each state a covering index. Taking
minima proves the set-cover formula.

To show NP-hardness, take a set-cover instance whose sets cover its finite
universe. Use elements as states and sets as continuations, with loss zero
for membership and one otherwise. Every row has minimum zero, and the
zero-regret capacity equals the minimum cover size. The matrix has
polynomial size, proving the claim when $m$ is part of the input.
\end{proof}

\begin{proof}[Proof of Corollary~\ref{cor:policy-capacity}]
On each policy group $\mathcal U_s$, a valid code must use at least
$\HyperChromatic(\mathcal H_{\varepsilon,s})$ labels. Thus the maximum
of these numbers is a lower bound. Let $M$ equal that maximum and color
each policy group using a common palette $[M]$. Define the unrestricted
encoder by the coloring associated with the state's policy. Policy text
$S$ distinguishes groups even when their code labels agree, so all cells
of the complete observation are compatible. The combined encoder attains the lower bound.
\end{proof}

\subsection{Loss uncertainty, larger libraries, and search}
\label{app:extended-witness-certification}
\paragraph{From loss uncertainty to margin uncertainty.}
Return to the fixed finite audit domain and continuation library.
Let $\widehat\ell_j(u)\in[0,1]$ estimate the true expected loss
$\ell_j(u)$. Define the estimated oracle loss, regret, and group margin by
\[
 \begin{aligned}
 \widehat\ell^\star(u)&:=\min_{j\in[m]}\widehat\ell_j(u),\\
 \widehat e_j(u)&:=\widehat\ell_j(u)-\widehat\ell^\star(u),\\
 \widehat\Gamma(W)&:=
 \min_{q\in\Delta_m}\max_{u\in W}
                  \sum_{j=1}^m q_j\widehat e_j(u).
 \end{aligned}
\]
For a loss error radius $r_\ell\ge0$, define the simultaneous event
\[
 \mathcal A_\ell:=
 \left\{\max_{u\in\AuditDomain,\ j\in[m]}
              |\widehat\ell_j(u)-\ell_j(u)|\le r_\ell\right\}.
\]
Assume $\mathbb P(\mathcal A_\ell)\ge1-\alpha$, where $0<\alpha<1$.
On $\mathcal A_\ell$, the oracle minimum also changes by at most $r_\ell$,
so regrets change by at most $2r_\ell$. Consequently,
\begin{equation}
 |\widehat\Gamma(W)-\Gamma(W)|\le2r_\ell
 \quad\text{for every nonempty }W\subseteq\AuditDomain
 \label{eq:simultaneous-margin-error}
\end{equation}
holds on the same event. The uniform bound covers all candidate groups
without an additional union bound over the groups searched.

\paragraph{Declaring the continuation class being certified.}
For the declared library, write $\Gamma_\Pi(W):=\Gamma(W)$.
If a conclusion is instead intended for a larger class
$\Plans_{\mathrm{target}}\supseteq\Plans$, the target margin needs a
separate definition. Let $\Loss(u,\pi)\in[0,1]$ extend the same expected loss to the target
class, with $\Loss(u,\pi_j)=\ell_j(u)$ for every declared continuation.
Let $\mathcal Q_{\mathrm{target}}$ be all finitely supported probability
distributions over the target class.
Define
\[
 \begin{aligned}
 \ell_{\mathrm{target}}^\star(u)
   &:=\inf_{\pi\in\Plans_{\mathrm{target}}}\Loss(u,\pi),\\
 \Gamma_{\mathrm{target}}(W)
   &:=\inf_{Q\in\mathcal Q_{\mathrm{target}}}\max_{u\in W}
       \left(\mathbb E_{\pi\sim Q}\Loss(u,\pi)
                         -\ell_{\mathrm{target}}^\star(u)\right).
 \end{aligned}
\]
For such a larger target, assume a justified fixed nonnegative library-error
bound $\eta_\Pi(W)$ satisfying
\begin{equation}
 \Gamma_{\mathrm{target}}(W)
 \ge\Gamma_\Pi(W)-\eta_\Pi(W)
 \quad\text{for every nonempty }W\subseteq\AuditDomain.
 \label{eq:library-margin-condition}
\end{equation}
For the declared-library target, set
$\Gamma_{\mathrm{target}}=\Gamma_\Pi$ and $\eta_\Pi=0$.
Repeated scoring does not supply $\eta_\Pi$, the error bound for library
approximation. A sufficient condition for the bound is given below in
Equation~\eqref{eq:library-loss-coverage}.

\begin{theorem}[Search-valid finite-witness certificate]
\label{thm:search-valid-certificate}
Suppose $\mathbb P(\mathcal A_\ell)\ge1-\alpha$ and the library-error
bound in Equation~\eqref{eq:library-margin-condition} holds. Let
$\widehat W$ be any nonempty group selected from the audit domain using
the estimated losses. With probability at least $1-\alpha$,
\[
 \Gamma_{\mathrm{target}}(\widehat W)
 \ge\widehat\Gamma(\widehat W)-2r_\ell-\eta_\Pi(\widehat W).
\]
If the right-hand side exceeds $\varepsilon$, the selected group is
$\varepsilon$-incompatible for the specified target class.
For the declared-library target, if the complete interface $f$ is constant
on $\widehat W$, then also
\[
 \max_{\lambda\in\Delta(\widehat W)}\Dinfo_\sigma(K_f;\lambda)
 \ge\widehat\Gamma(\widehat W)-2r_\ell.
\]
\end{theorem}

The probability guarantee survives data-dependent witness selection because
Equation~\eqref{eq:simultaneous-margin-error} holds for every group at once.
If only a fixed-witness interval is available, selecting a witness and then
applying that interval to the same data is not justified. Independent
certification data provide one valid alternative: freeze the selected group
before estimating its certification losses.

\paragraph{Incompatibility is not the same as minimality.}
A lower margin bound for a group proves only that the group cannot share an
acceptable selection. To certify a \emph{minimal} conflict, all proper
subgroups must be compatible. By monotonicity of $\Gamma$, it suffices to
check every one-state deletion.

\begin{corollary}[Certifying minimality in the declared library]
\label{cor:certified-minimality}
On the event $\mathcal A_\ell$, keep the declared library as the target.
A group $W$ with $|W|\ge2$ is a minimal $\varepsilon$-incompatible group if
\[
 \begin{aligned}
 \widehat\Gamma(W)-2r_\ell&>\varepsilon,\\
 \widehat\Gamma(W\setminus\{u\})+2r_\ell&\le\varepsilon
                                  &&\text{for every }u\in W.
 \end{aligned}
\]
The conclusion also holds for a data-selected group $W$ under the same
simultaneous event.
\end{corollary}

For a three-state witness, the first inequality certifies failure of the
triple and the three upper bounds certify success of every pair. A positive
triple lower bound alone cannot establish a pairwise-pass/triple-fail result.

\paragraph{A rollout-based simultaneous radius.}
For each of the $nm$ state--continuation pairs, let
$Y_{ujk}\in[0,1]$, $k=1,\ldots,N$, be $N\ge1$ independent scores with
common mean $\nu_j(u)$ under a fixed scoring protocol. Set
$\widehat\ell_j(u):=N^{-1}\sum_{k=1}^N Y_{ujk}$.
Suppose a known deterministic judge-bias bound $b_{\mathrm{judge}}\ge0$
satisfies $|\nu_j(u)-\ell_j(u)|\le b_{\mathrm{judge}}$ for every audited
pair. Independence between different pairs is not required.

\begin{corollary}[One simple rollout radius]
\label{cor:rollout-radius}
Under the stated bounded-score, within-pair independence, and uniform-bias
conditions, a simultaneous loss radius with error probability at most
$\alpha\in(0,1)$ is
\[
 r_\ell=b_{\mathrm{judge}}
        +\sqrt{\frac{\log(2nm/\alpha)}{2N}}.
\]
\end{corollary}

The square-root term controls score sampling through Hoeffding's inequality
and a union bound over the $nm$ pairs \citep{hoeffding1963probability};
$b_{\mathrm{judge}}$ controls systematic error relative to the target loss.
If the bias bound or library approximation holds only on a separate
probability event, its failure probability must also enter the guarantee.
The human-study packet maximum does not establish a uniform judge-bias
bound. Algorithm~\ref{alg:search-certify} keeps incompatibility certification
and minimality certification separate.

\subsection{Proofs of the statistical certification results}
\paragraph{A sufficient condition for the library approximation bound.}
The inequality in Equation~\eqref{eq:library-margin-condition} is an
assumption for transfer to a larger continuation class, not a consequence
of loss-estimation accuracy. Enlarging the library also changes the
statewise oracle baseline, so library inclusion alone does not specify
the size of the regret-margin change. One sufficient condition can be stated in
terms of expected losses. Fix a nonempty state group $W$ and a number
$\eta\ge0$. Suppose every target mixture
$Q\in\mathcal Q_{\mathrm{target}}$ has a declared-library selection
$q\in\Delta_m$ with
\begin{equation}
 \sum_{j=1}^m q_j\ell_j(u)
 \le\mathbb E_{\pi\sim Q}\Loss(u,\pi)+\eta
 \quad\text{for every }u\in W.
 \label{eq:library-loss-coverage}
\end{equation}
The selection $q$ may depend on $Q$ and $W$, but must work for all states
in $W$ simultaneously. Because the declared library is contained in the
target class,
$\ell^\star(u)\ge\ell_{\mathrm{target}}^\star(u)$. Hence a covering
selection $q$ satisfies
\[
 \begin{aligned}
 e(u,q)
 &=\sum_jq_j\ell_j(u)-\ell^\star(u)\\
 &\le\mathbb E_{\pi\sim Q}\Loss(u,\pi)
                       -\ell_{\mathrm{target}}^\star(u)+\eta.
 \end{aligned}
\]
Taking the maximum over states, then the infimum over target mixtures,
gives
$\Gamma_\Pi(W)\le\Gamma_{\mathrm{target}}(W)+\eta$.
Thus Equation~\eqref{eq:library-loss-coverage} justifies using
$\eta_\Pi(W)=\eta$. The infimum argument does not require an optimal
target mixture to exist. Uniform validity across searched groups requires
the coverage condition for every candidate group. No such coverage
condition is inferred merely from the number of scored rollouts.

\begin{proof}[Proof of Theorem~\ref{thm:search-valid-certificate}]
On $\mathcal A_\ell$, the minimum loss in each row changes by at most
$r_\ell$. Each regret therefore changes by at most $2r_\ell$. Convex
averaging, maximization over states, and minimization over choices preserve
the uniform bound, giving
$|\widehat\Gamma(W)-\Gamma_\Pi(W)|\le2r_\ell$ for every nonempty
$W$. The library assumption then yields
\[
 \Gamma_{\rm target}(W)\ge\Gamma_\Pi(W)-\eta_\Pi(W)
 \ge\widehat\Gamma(W)-2r_\ell-\eta_\Pi(W).
\]
Since the same event covers every group, the inequality holds for the
selected group $\widehat W$ with probability at least $1-\alpha$.
A lower bound above $\varepsilon$ excludes every target choice with
regret at most $\varepsilon$ throughout the group. For the declared
library and a constant observation on $\widehat W$,
Theorem~\ref{thm:witness-duality} gives the stated maximum risk gap over state distributions.
\end{proof}

\begin{proof}[Proof of Corollary~\ref{cor:certified-minimality}]
On $\mathcal A_\ell$, the full-group lower bound gives
$\Gamma(W)>\varepsilon$, while each deletion upper bound gives
$\Gamma(W\setminus\{u\})\le\varepsilon$. Every proper nonempty subset
lies in a deletion and has no larger margin. The empty group is compatible
by definition. Thus $W$ is a minimal conflict. Simultaneous coverage
permits choosing $W$ using the same estimates.
\end{proof}

\begin{proof}[Proof of Corollary~\ref{cor:rollout-radius}]
For each pair $(u,j)$, Hoeffding's inequality gives
\[
 \mathbb P\bigl(|\widehat\ell_j(u)-\nu_j(u)|>a\bigr)
 \le2e^{-2Na^2}.
\]
Taking $a=\sqrt{\log(2nm/\alpha)/(2N)}$ and a union bound over the
$nm$ pairs bounds all sampling errors by $a$ with probability at least
$1-\alpha$. Adding $|\nu_j(u)-\ell_j(u)|\le b_{\rm judge}$ proves
the radius. Independence is needed within each pair, not between pairs.
\end{proof}

\section{Algorithms, Complexity, and Failure Modes}
\label{app:algorithms}

The procedures below specify their inputs, search domains, and stopping conditions. Mathematical claims assume exact optimization; numerical implementations must also report solver status and tolerance.

\begin{algorithm}[t]
\caption{Exact finite-interface diagnosis}
\label{alg:finite-interface-diagnosis}
\begin{algorithmic}[1]
\Require State weights \(\{\mu_i\}_{i=1}^{n}\), channel matrix
\(P\in[0,1]^{n\times L}\), loss matrix
\(\LossMatrix=(L_{ij})\in[0,1]^{n\times m}\), with $n,L\ge1$ and $m\ge2$.
\Ensure \(\Rfull\), \(\Rinterface\), an optimal unrestricted rule
\(\delta_Z^{\star}\), and the interface information gap.
\State Check $\mu_i\ge0$, $\sum_i\mu_i=1$, and $\sum_rP_{ir}=1$ for each $i$.
\State Compute
\(\Rfull\gets\sum_{i=1}^{n}\mu_i\min_{j\in[m]}L_{ij}\).
\For{each interface output \(r\in[L]\)}
    \For{each continuation \(j\in[m]\)}
        \State \(C_{rj}\gets\sum_{i=1}^{n}\mu_iP_{ir}L_{ij}\).
    \EndFor
    \State Choose
    \(j_r^{\star}\in\argmin_{j\in[m]}C_{rj}\).
    \State Set \(\delta_{Z,j}^{\star}(z_r)\gets\Indicator\{j=j_r^{\star}\}\) for every $j\in[m]$.
\EndFor
\State \(\Rinterface\gets\sum_{r=1}^{L}C_{rj_r^{\star}}\).
\State \(\Delta_{\mathrm{info}}\gets\Rinterface-\Rfull\).
\State \Return
\((\Rfull,\Rinterface,\delta_Z^{\star},\Delta_{\mathrm{info}})\).
\end{algorithmic}
\end{algorithm}

\begin{algorithm}[t]
\caption{Restricted-rule and finite-training audit}
\label{alg:restricted-rule-audit}
\begin{algorithmic}[1]
\Require Evaluation joint costs $C_{rj}$ and training joint costs
$\widehat C^{\mathrm{tr}}_{rj}$ on the same $L$ observation codes and $m$ choices;
visible features $\{\varphi_r\}_{r=1}^{L}$ and finite parameter bound
$B_{\mathrm{par}}\ge1/m$. Costs include observation probabilities.
\Ensure $\Rrestricted$, $\Rlearn$, accessibility gap, and fitting gap.
\State Compute $\Rinterface\gets\sum_{r=1}^{L}\min_{j\in[m]}C_{rj}$.
\State Solve Proposition~\ref{prop:affine-randomized-lp} with \(C\) to obtain
\(q^{\star}_{\mathcal D}\) and \(\Rrestricted\).
\State Solve the same linear program with \(\widehat C^{\mathrm{tr}}\) to
obtain the learned rule \(\widehat q\), keeping all simplex constraints unchanged.
\State Evaluate \(\widehat q\) on \(C\) to obtain \(\Rlearn\).
\State
\(\Delta_{\mathrm{access}}\gets\Rrestricted-\Rinterface\).
\State
\(\Delta_{\mathrm{train}}\gets\Rlearn-\Rrestricted\).
\State \Return
\((\Rrestricted,\Rlearn,\Delta_{\mathrm{access}},\Delta_{\mathrm{train}})\).
\end{algorithmic}
\end{algorithm}

\begin{algorithm}[t]
\caption{Continuous-interface certificate with loss uncertainty}
\label{alg:continuous-interface-certificate}
\begin{algorithmic}[1]
\Require $n\ge1$ independent identically distributed audit pairs $(U_t,Z_t)$;
a fixed library of $m\ge2$ choices; simultaneous loss intervals
$[\underline L_{tj},\overline L_{tj}]\subseteq[0,1]$ with failure probability
$\alpha_\ell$; a partition $Q_h$ with $L\ge1$ cells and a valid modulus
$\omega$, both fixed independently of the sample; error budget
$0<\alpha_s<1$ with $0\le\alpha_\ell<1-\alpha_s$.
\Ensure A lower bound valid with probability at least $1-\alpha_s-\alpha_\ell$.
\State Reject the certificate request if interval coverage, the modulus, or partition independence is not justified.
\State Compute $\widehat C^-_{rj}$ and $\widehat R^+_{\mathrm{full}}$ from
Theorem~\ref{thm:noisy-continuous-certificate}.
\State $r_n\gets\sqrt{\log(4Lm/\alpha_s)/(2n)}$;
$s_n\gets\sqrt{\log(4/\alpha_s)/(2n)}$.
\State $\mathrm{LB}_Q\gets\sum_{r=1}^{L}\max\{0,\min_j\widehat C^-_{rj}-r_n\}$.
\State $\mathrm{UB}_{\mathrm{full}}\gets\min\{1,\widehat R^+_{\mathrm{full}}+s_n\}$.
\State \Return $\max\{0,\mathrm{LB}_Q-\mathrm{UB}_{\mathrm{full}}-2\omega(h)\}$.
\end{algorithmic}
\end{algorithm}

\begin{algorithm}[t]
\caption{Minimal conflict extraction with an explicit search scope}
\label{alg:witness-extraction}
\begin{algorithmic}[1]
\Require Finite regret matrix $\RegretMatrix$, tolerance $\varepsilon\ge0$,
continuation count $m$, and a family of disjoint search groups $\mathcal C$.
Use current observation cells for local diagnosis; use
$\mathcal C=\{\AuditDomain\}$ for the complete global conflict system.
\Ensure All minimal conflicts contained in at least one declared search group.
\State Initialize $\mathcal E_{\mathrm{out}}\gets\varnothing$.
\For{each group $C\in\mathcal C$}
 \For{$s=2,\ldots,\min\{m,|C|\}$}
  \For{each subset $W\subseteq C$ with $|W|=s$}
   \If{no recorded edge is contained in $W$}
    \State Solve the margin LP for $\Gamma(W)$ and its dual state weights.
    \If{$\Gamma(W)>\varepsilon$}
     \State Add $W$ and its margin and dual weights to $\mathcal E_{\mathrm{out}}$.
    \EndIf
   \EndIf
  \EndFor
 \EndFor
\EndFor
\State \Return $\mathcal E_{\mathrm{out}}$ and the declared search scope.
\end{algorithmic}
\end{algorithm}

\begin{algorithm}[t]
\caption{Coloring MILP for a declared global constraint subset}
\label{alg:capacity-milp}
\begin{algorithmic}[1]
\Require Audit states $u_1,\ldots,u_n$ and retained valid conflicts
$\mathcal F_t\subseteq\ConflictEdges$.
\Ensure The relaxed optimum $M_t$ and labeling $g_t$ if optimality is proved;
otherwise a lower bound, any feasible incumbent, and solver status.
An optimum is the global capacity only with complete constraints or successful
validation of every candidate cell.
\State Create binary variables $x_{ic}$ and $y_c$ for $i,c\in[n]$.
\State Minimize $\sum_{c=1}^{n}y_c$.
\State Add $\sum_cx_{ic}=1$ for every state $i$ and $x_{ic}\le y_c$ for every $(i,c)$.
\For{each retained conflict $W\in\mathcal F_t$ and label $c\in[n]$}
 \State Add $\sum_{u_i\in W}x_{ic}\le |W|-1$.
\EndFor
\State Add $y_c\ge y_{c+1}$ for $c<n$.
\State Solve the MILP and record its status, lower bound, and any feasible incumbent.
\If{optimality is proved}
 \State \Return $(M_t,g_t,\text{optimal})$.
\Else
 \State \Return the lower bound, any incumbent, and the stopping status; do not report exact capacity.
\EndIf
\end{algorithmic}
\end{algorithm}

\begin{algorithm}[t]
\caption{Global capacity by validation and conflict generation}
\label{alg:capacity-global-validation}
\begin{algorithmic}[1]
\Require Full finite regret table, tolerance $\varepsilon$, and any valid initial
edge subset $\mathcal F_0$ (possibly empty or obtained from local diagnosis).
\Ensure Global capacity and a valid labeling under exact solves/comparisons.
\State Set $t\gets0$.
\Loop
 \State Call Algorithm~\ref{alg:capacity-milp} with $\mathcal F_t$.
 \If{the relaxed optimum is not proved}
  \State Validate any incumbent to obtain an upper bound if possible.
  \State \Return available lower and upper bounds and solver status; global optimality is unresolved.
 \EndIf
 \State Let $(M_t,g_t)$ be the proved relaxed optimum and its labeling.
 \State Compute $\Gamma(C)$ for every nonempty new cell $C=g_t^{-1}(c)$.
 \If{any margin comparison is unresolved at the numerical tolerance}
  \State \Return available bounds and the unresolved cell; do not claim exact capacity.
 \EndIf
 \If{every new cell satisfies $\Gamma(C)\le\varepsilon$}
  \State \Return $(\CapacityCurve(\varepsilon)=M_t,g_t)$ with matching lower and upper bounds.
 \EndIf
 \State For each violating cell, repeatedly remove a state while the remaining group still has margin greater than $\varepsilon$.
 \State Add each resulting minimal conflict to obtain $\mathcal F_{t+1}$; set $t\gets t+1$.
\EndLoop
\end{algorithmic}
\end{algorithm}

\begin{algorithm}[t]
\caption{Search and certify a higher-order conflict}
\label{alg:search-certify}
\begin{algorithmic}[1]
\Require Search loss estimates on a finite audit domain and declared library;
complete deterministic observation cells; tolerance $\varepsilon\ge0$;
error budget $\alpha\in(0,1)$; a specified target class and justified library
bound $\eta_\Pi$; either a valid domain-wide loss bound or independent
certification measurements.
\Ensure Incompatibility status for the target and, when justified, minimality
status for the declared library.
\State Run Algorithm~\ref{alg:witness-extraction} with the search estimates
and the declared local cells to obtain candidate groups.
\If{no candidate is returned}
 \State \Return ``no candidate selected.''
\EndIf
\State Select a nonempty candidate $\widehat W$ by the declared search criterion.
\If{a valid simultaneous loss bound covers the entire searched matrix}
 \State Use that estimated matrix and its domain-wide radius $r_\ell$.
\Else
 \State Freeze $\widehat W$ before collecting independent certification data.
 \State Estimate losses on $\widehat W\times\Plans$ and construct a valid
 conditional simultaneous radius $r_\ell$ with error probability at most $\alpha$.
\EndIf
\State Compute $\widehat\Gamma(\widehat W)$ from the chosen certification loss matrix.
\State $\mathrm{LCB}_W\gets\widehat\Gamma(\widehat W)-2r_\ell-\eta_\Pi(\widehat W)$.
\State Initialize incompatibility and minimality statuses to ``not certified.''
\If{$\mathrm{LCB}_W>\varepsilon$}
 \State Mark incompatibility for the specified target as certified.
 \If{the target is the declared library and every single deletion upper
 bound in Corollary~\ref{cor:certified-minimality} is at most $\varepsilon$}
  \State Mark declared-library minimality as certified.
 \EndIf
\EndIf
\State \Return $(\widehat W,\mathrm{LCB}_W,\text{incompatibility status},\text{minimality status})$.
\end{algorithmic}
\end{algorithm}

\begin{algorithm}[t]
\caption{Decision-sufficiency audit with separated diagnostic and design branches}
\label{alg:end-to-end-audit}
\begin{algorithmic}[1]
\Require Full audit states, complete observation channel $V$ (including any
policy side information), finite continuation library, loss estimates and
valid error bounds when available, restricted rule class, and tolerance.
\Ensure Finite-table risks, local conflicts, and only those population or
capacity certificates whose assumptions have been checked.
\State Freeze the observation contract, continuation library, loss definition, and audit scope.
\If{the complete observation is finite}
 \State Compute finite-table full-state and optimal interface risks by Algorithm~\ref{alg:finite-interface-diagnosis}.
 \State Label plug-in results as finite-table results unless loss estimation error is controlled.
\Else
 \State Apply Algorithm~\ref{alg:continuous-interface-certificate} only with valid loss intervals and modulus; otherwise report that this procedure provides no population bound.
\EndIf
\State Audit the restricted and fitted rules over the same feasible observation domain.
\If{the complete observation defines deterministic finite cells}
 \State Run Algorithm~\ref{alg:witness-extraction} on those cells for local diagnosis at the declared tolerance.
 \State For unrestricted code-only design, run Algorithm~\ref{alg:capacity-global-validation} on the full audit domain, or enumerate the full global hypergraph first.
 \State With free finite policy side information, solve the global problem separately within each policy and use Corollary~\ref{cor:policy-capacity}.
 \State Certify selected local conflicts using Algorithm~\ref{alg:search-certify} only when the required loss-error event is valid.
\Else
 \State Do not apply deterministic coloring to posterior supports of a stochastic channel without a separate justified construction.
\EndIf
\State Return risks, uncertainty assumptions, search scope, solver status, and claim boundaries.
\end{algorithmic}
\end{algorithm}

\subsection{Computational cost and assumptions}

\begin{center}
\small
\begin{tabular}{p{0.27\linewidth}p{0.23\linewidth}p{0.39\linewidth}}
\toprule
\textbf{Procedure} & \textbf{Main cost} & \textbf{Required condition}\\
\midrule
Exact risk for the finite interface
& \(O(nLm)\)
& Exact only relative to the estimated finite loss table and declared
continuation interface.\\
\addlinespace
Affine randomized restricted rule
& Linear program with polynomially many variables and constraints
& The result is exact for the declared affine randomized class, not for all
neural architectures.\\
\addlinespace
Continuous lower certificate
& \(O(nm+Lm)\) after partitioning
& Requires a valid modulus for conditional costs or another population condition.\\
\addlinespace
Witness extraction
& \(O(|\AuditDomain|^m)\) LPs for fixed \(m\)
& Practical for small continuation interfaces; the current case uses
\(m=3\).\\
\addlinespace
Minimum interface capacity
& NP-hard in general; finite MILP
& The weak chromatic number is unrestricted capacity. A real SAE may not
realize the optimal coloring.\\
\addlinespace
Search-valid certificate
& Loss estimation plus one witness LP
& A conflict for the declared library need not persist for a larger library;
transfer requires an additional approximation bound.\\
\bottomrule
\end{tabular}
\end{center}

\paragraph{Reference implementation and numerical decisions.}
The CPU implementation retained in the author working archive uses the
at-most-$m$ feasible label bound
from Proposition~\ref{prop:capacity-set-cover}. It records solver optimality,
new-cell margins, added conflicts, and an explicit numerical tolerance.
The theorems assume exact arithmetic; numerical LP/MILP outputs must be
interpreted with the recorded tolerance, and cases at a decision boundary
must not be presented as proofs in exact arithmetic.

\section{Complete Empirical Claim Ledger}
\label{app:result-map}

Table~\ref{tab:result-map} links each claim to its evidence and limits.
The historical rows use the SAE code alone; later rows distinguish
observations that also include policy. Numerical recovery from rounded
Qwen summaries is not treated as a recalculation from the original arrays.

\setlength{\tabcolsep}{4pt}
\begin{longtable}{@{}p{0.07\linewidth}p{0.29\linewidth}p{0.17\linewidth}p{0.37\linewidth}@{}}
\caption{Evidence for each claim for the complete audit.}
\label{tab:result-map}\\
\toprule
ID & Candidate claim & Status & Evidence boundary \\
\midrule
\endfirsthead
\toprule
ID & Candidate claim & Status & Evidence boundary \\
\midrule
\endhead
C1 & A frozen low-budget finite SAE interface has positive information loss on
untouched Qwen2.5 policy areas. & Supported & Layer 17, $B=2$: gap $0.156$ with
one-sided 95\% lower bound $0.129$. \\
\addlinespace
C2 & A restricted rule class adds positive accessibility loss after the
interface is fixed. & Supported & Layer 19, $B=16$: gap $0.047$ with one-sided
lower bound $0.031$; nested classes drive the gap to zero. \\
\addlinespace
C3 & The primary frozen-interface conclusion is stable across two automatic
outcome evaluators. & Supported & Qwen2.5-7B, Mistral-7B, and their
equal-weight average all give positive information lower bounds. \\
\addlinespace
C4 & The low-budget information result appears at a second Qwen2.5 model scale
with a separate public SAE release. & Supported & Qwen2.5-0.5B layer 17,
$B=2$: gap $0.210$ with one-sided lower bound $0.192$. \\
\addlinespace
C5 & Positive information loss is not tied to the supervised feature selector.
& Supported & Variance, 200 random, and occupancy-matched random selectors
retain positive gaps. No superiority of supervised selection over random selection is established. \\
\addlinespace
C6 & Pairwise compatibility can miss a group-level conflict in real SAE
support patterns. & Supported in the controlled benchmark & Exact pair margin
$0$, triple margin $1/3$, separation $1/3$. \\
\addlinespace
C7 & A searched language example passes all pair checks but fails for the triple in new data.
& Not replicated & Independent and fresh
semantic/generation/judge replications give zero positive separation. \\
\addlinespace
C8 & The low-budget result transfers to a non-Qwen model family. & Supported
on one non-Qwen family & SmolLM2-1.7B layer 17 gives positive $B=2$ automatic
gaps on the old confirmation set and the fresh locked set for both SAE seeds. \\
\addlinespace
C9 & The SmolLM2 sign is stable across independently trained SAE seeds. & Supported
for two SmolLM2 seeds & Two width-$16{,}384$, $k=64$ TopK SAEs trained on eight million
activation tokens each pass QC and give positive low-budget gaps. The seeds
share the model, layer, architecture, and cached corpus activations. \\
\addlinespace
C10 & The sign transfers to a policy collection locked before model
generation. & Supported & The fresh 96-state collection gives automatic gaps
$0.071$ and $0.077$ at $B=2$, with positive one-sided lower bounds. \\
\addlinespace
C11 & Blinded human outcome loss retains a positive low-budget information
gap. & Supported & Final adjudicated SmolLM2 $Z$-only $B=2$ gaps are $0.078$ and $0.099$
on the old 80-state sample and $0.148$ and $0.133$ on the fresh 96-state set;
all one-sided lower bounds are positive. \\
\addlinespace
C12 & Two human raters apply the ordinal rubric with measured agreement. &
Supported with dimension-specific variation & Weighted $\kappa$ ranges from
$0.472$ to $0.811$; composite human-loss ICC(2,1) is $0.746$ with 95\%
interval $[0.686,0.798]$. One hundred triggered items are jointly resolved. \\
\addlinespace
C13 & Automatic outcome loss is closely calibrated to final human judgments. &
Not supported & Spearman correlation $0.379$, MAE $0.311$, fixed-order action-label
agreement $51.1\%$ (set intersection $80.7\%$, set equality $30.7\%$), and strong action-specific signed error. \\
\addlinespace
C14 & A packet maximum error provides a uniform population judge-error bound. &
Not supported & The observed maximum absolute error is $0.900$, but the packet
maximum is not a population-level uniform certificate. \\
\addlinespace
C15 & The combined studies prove the result for all SAEs or all model families. &
Not supported & The audit covers three language-model families under different observation contracts, with two SAE training seeds in each of SmolLM2 and Phi. Human evidence is SmolLM2-only. \\
\addlinespace
C16 & Phi has positive primary information loss when active policy is supplied. & Supported on the saved finite collections & At $B=2$, Mistral gaps $0.0331$--$0.0557$; both seeds and both collections have positive individual LCBs. Qwen-loss and mean-loss checks agree on the sign. \\
\addlinespace
C17 & One policy area drives the primary Phi sign. & Not supported by the deletion check & All 192 single area deletion point estimates across two observations and three loss constructions remain positive; minimum $0.0236$. \\
\addlinespace
C18 & The two Phi evaluators closely agree numerically. & Not supported & Spearman $0.3552$, mean absolute difference $0.3220$ across 5,040 shared responses. \\
\addlinespace
C19 & The earlier human study validates Phi responses. & Not supported & The Phi human packets are blank. No SmolLM2 human label is transferred to a Phi response. \\
\addlinespace
C20 & A positive gap using code alone implies a positive gap with policy supplied. & Not valid & The implication is false in general. Separate paired reanalysis now supplies the evidence reported in C22 and C23. \\
\addlinespace
C21 & Phi adds a previously unseen policy collection. & Not supported & The original and locked collections are reused, although fixed before generating the new Phi responses. \\
\addlinespace
C22 & SmolLM2's policy-aware information gap remains positive under automatic and human loss. & Supported on cached finite tables & Four automatic and four final human $B=2$ gaps have positive individual bootstrap lower bounds. Human gaps are $0.0345$--$0.1166$. \\
\addlinespace
C23 & Primary Qwen policy-aware gaps remain positive. & Supported with a rounding bound & Per-policy notebook summaries give $0.13695\pm0.00005$ and $0.03875\pm0.00005$. No full-precision state-array rerun or new 0.5B policy-aware result is claimed. \\
\addlinespace
C24 & Final SmolLM2 human losses are reproducible from completed ratings. & Recomputed & Both 528-row files and all 100 resolved rows match the packet. Final losses agree with the earlier item-level archive to numerical precision; inputs and the CPU script are retained in the author working archive. \\
\addlinespace
C25 & Local conflicts alone determine global capacity. & Not valid in general & Complete global constraints or validation of every new cell are required. Proposition~\ref{prop:capacity-constraint-generation} and the cross-cell regression test support the global constraint-generation algorithm. \\
\addlinespace
C26 & Unbiased single-rollout scores suffice for the exact-vector population certificate. & Not valid in general & Theorem~\ref{thm:noisy-continuous-certificate} uses simultaneous expected-loss intervals. The noisy-minimum counterexample shows why unbiased single-rollout scores alone do not establish the certificate. \\
\bottomrule
\end{longtable}

\section{Complete SmolLM2 Results and Human Evaluation}
\label{app:added-experiment-tables}

This appendix contains the SmolLM2 results and human evaluation. The author working archive
retains 1,008 mean losses for state--action pairs, interface codes, 200 random
feature selections per setting, training metadata and checkpoint hashes,
the locked state collection, and the completed ratings for 528 responses.
The original independent ratings and the 100 joint resolutions are retained
separately.

\subsection{Saved configuration and provenance}

The SmolLM2 run uses SmolLM2-1.7B-Instruct revision
\texttt{31b70e2e869a7173\allowbreak 562077fd711b6549\allowbreak 46d38674} and
Mistral-7B-Instruct-v0.3 revision
\texttt{c170c708c41dac92\allowbreak 75d15a8fff4eca08\allowbreak d52bab71}.
Each state--action pair has five rollout replicas, maximum continuation length
192 tokens, temperature $0.7$, and top-$p$ value $0.9$. The evaluator output is
limited to 160 new tokens and permits four parse retries.

SAE training uses FineWeb-Edu \texttt{sample-10BT} revision
\texttt{87f09149ef473420\allowbreak 4d70ed1d046ddc9c\allowbreak a3f2b8f9}.
Corpus contexts have 128 tokens. The activation cache uses 32,768-token shards.
Each SAE is optimized with AdamW, learning rate $2\times10^{-4}$, 500 warm-up
steps, zero weight decay, gradient-norm clip $1.0$, auxiliary dead-feature
coefficient $0.03125$, and a 1,000-step dead-feature window. The fresh collection
seed is 602214, the human packet seed is 8675309, the bootstrap seed is 99173,
and the two SAE seeds are 314159 and 271828.

The run used Python 3.13.15, PyTorch 2.11.0 with CUDA 12.8, Transformers 4.57.6,
Datasets 5.0.1, Accelerate 1.14.0, NumPy 2.1.3, Pandas 2.2.3, SciPy 1.16.3,
and an NVIDIA A100-SXM4-40GB GPU. The saved run manifest, software manifest,
configuration JSON, source code, checkpoint hashes, and all state-level outputs
are retained in the author working archive.

\begin{table}[t]
\centering
\caption{SmolLM2 audit and SAE training settings.}
\label{tab:cross-family-design}
\small
\renewcommand{\arraystretch}{1.10}
\begin{tabular}{p{0.31\linewidth}p{0.61\linewidth}}
\toprule
Component & Declared value \\
\midrule
Representation model & SmolLM2-1.7B-Instruct, layer-17 post-layer residual \\
Automatic evaluator & Mistral-7B-Instruct-v0.3 \\
States & Primary 240 plus fresh locked 96 \\
Actions and replicas & Three actions; five continuations per state--action pair \\
SAE seeds & 314159 and 271828 \\
SAE architecture & TopK SAE, width $16{,}384$, $k=64$ \\
Training activations & Eight million tokens per seed from FineWeb-Edu \\
Feature budgets & $B\in\{2,4,8,16,32,64\}$; primary $B=2$ \\
Selector checks & Supervised, variance, and 200 random selectors per setting \\
Uncertainty & 2,000 cluster bootstrap over policy areas repetitions \\
\bottomrule
\end{tabular}
\end{table}
\subsection{SmolLM2 results using the code alone and reconstruction diagnostics}
\label{sec:smollm-results}

The historical SmolLM2 tables in this subsection use cells defined by the SAE code alone. Section~\ref{sec:policy-reanalysis} and Appendix~\ref{app:support-reanalysis} separately recompute policy-aware risk from the same frozen codes and losses.

\paragraph{Independent SAE quality control.}
Both TopK SAEs use the same frozen SmolLM2 activation source and architecture but
different random initialization and seed-dependent training order. The two runs
are therefore independent training seeds, not independent training corpora. Both
consume eight million activation tokens and pass the prespecified validation
checks. Validation FVU is $0.0089$ and $0.0088$, mean reconstruction cosine is
$0.922$ for both seeds, mean $L_0$ is exactly $64$, all $16{,}384$ features are
active on validation, and the observed dead-feature fraction is zero.

The audit-domain centered FVU is much larger: $35.22$ and $35.88$ on the old
audit states, and $7.77$ and $8.12$ on the fresh states. Centered FVU can exceed
one because its numerator is reconstruction squared error while its denominator
is centered activation variance. The large values are consistent with a difference in means, scales, or variances between corpus and audit activations, but centered FVU alone does not identify which component causes the increase; the values must not be read as percentages.
Reconstruction cosine on the audit states remains between $0.877$ and
$0.883$. These reconstruction measurements do not determine the optimal
decision loss, which is computed separately below.

\begin{table}[H]
\centering
\caption{Quality-control diagnostics for the two independently initialized and trained TopK SAEs. Audit-domain centered FVU is reported without clipping; values above one indicate that reconstruction squared error exceeds centered activation variance on the audit distribution.}
\label{tab:independent-sae-qc-full}
\small
\begin{adjustbox}{max width=\linewidth}
\begin{tabular}{lrrrrrrrrrr}
\toprule
Seed & Val. FVU & Val. cosine & Val. $L_0$ & Dead frac. & Active & QC & Old FVU & Old cosine & Fresh FVU & Fresh cosine\\
\midrule
271828 & 0.009 & 0.922 & 64.0 & 0.000 & 16384 & true & 35.883 & 0.880 & 8.116 & 0.877 \\
314159 & 0.009 & 0.922 & 64.0 & 0.000 & 16384 & true & 35.224 & 0.882 & 7.768 & 0.883 \\
\bottomrule
\end{tabular}
\end{adjustbox}
\end{table}

\paragraph{Exact risks for the selected support code.}
For $n$ equally weighted states, let $L_{ia}$ be the mean automatic loss for
state $i$ and action $a\in\mathcal A$, where
$\mathcal A=\{\textsc{ask},\textsc{decline},\textsc{assist}\}$.
Let $C_z$ contain the states with code $z$; the sum below uses nonempty cells.
The full state and interface risks are
\[
R_{\mathrm{full}}
=
\frac{1}{n}\sum_{i=1}^{n}\min_{a\in\mathcal A}L_{ia},
\qquad
R_Z^\star
=
\sum_z\frac{|C_z|}{n}
\min_{a\in\mathcal A}
\frac{1}{|C_z|}\sum_{i\in C_z}L_{ia}.
\]
The information gap is $\Delta_{\mathrm{info}}=R_Z^\star-R_{\mathrm{full}}$.

At the primary budget $B=2$, both SAE seeds give a positive gap on both
collections. The old confirmation gaps are $0.136$ and $0.142$; the fresh locked
gaps are $0.071$ and $0.077$. Every one-sided 95\% lower bound is positive, and
none of the 2,000 cluster-bootstrap replicates has a nonpositive gap.

\begin{table}[t]
\centering
\caption{Primary automatic SmolLM2 $Z$-only results at $B=2$. Active policy is not a separate lookup key.}
\label{tab:cross-family-primary}
\small
\begin{tabular}{llrrrrr}
\toprule
Collection & SAE seed & $\Rfull$ & $\Rinterface$ & $\Dinterface$ & 95\% CI & 95\% LCB\\
\midrule
Old confirmation 96 & 314159 & 0.142 & 0.278 & 0.136 & [0.100, 0.164] & 0.105\\
Old confirmation 96 & 271828 & 0.142 & 0.284 & 0.142 & [0.084, 0.191] & 0.091\\
Fresh locked 96 & 314159 & 0.201 & 0.273 & 0.071 & [0.037, 0.098] & 0.042\\
Fresh locked 96 & 271828 & 0.201 & 0.278 & 0.077 & [0.042, 0.088] & 0.047\\
\bottomrule
\end{tabular}
\end{table}

The supervised-selector gap decreases with feature budget. On the old
confirmation set, the seed-314159 curve is $0.136$, $0.118$, $0.050$, $0.004$,
and approximately zero at $B=2,4,8,16,32$, respectively; seed 271828 follows the
same pattern. On the fresh set, seed 314159 gives $0.071$, $0.068$, $0.040$,
$0.021$, $0.002$, and zero, while seed 271828 reaches zero by $B=16$.
The larger feature sets refine the earlier partitions, so their optimal
risk cannot increase. A zero gap means that each occupied cell has a common
optimal action; it does not require every state to have a distinct code.

\begin{figure}[t]
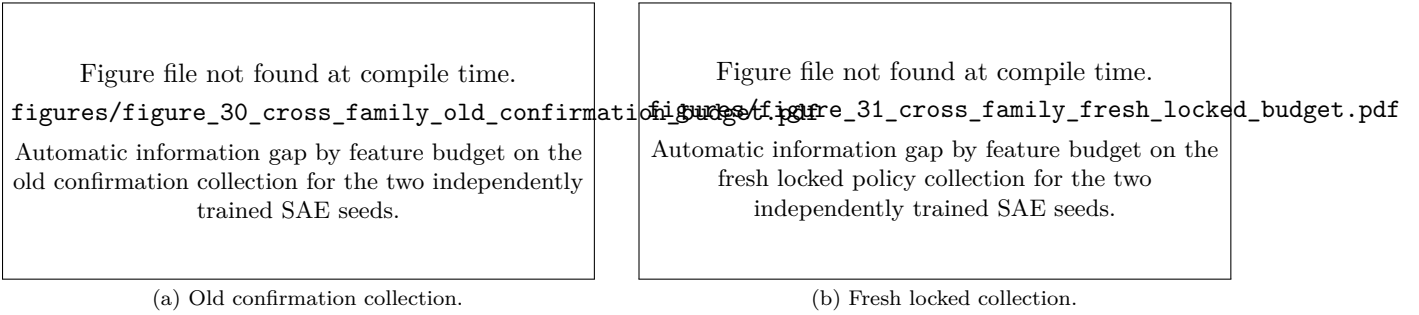

\centering
\begin{subfigure}[t]{0.49\linewidth}
\auditgraphic{figures/figure_30_cross_family_old_confirmation_budget.pdf}{Automatic
information gap by feature budget on the old confirmation collection for the two
independently trained SAE seeds.}
\caption{Old confirmation collection.}
\end{subfigure}\hfill
\begin{subfigure}[t]{0.49\linewidth}
\auditgraphic{figures/figure_31_cross_family_fresh_locked_budget.pdf}{Automatic
information gap by feature budget on the fresh locked policy collection for the
two independently trained SAE seeds.}
\caption{Fresh locked collection.}
\end{subfigure}
\caption{SmolLM2 $Z$-only automatic information-gap curves. Error bars are
cluster bootstrap over policy areas 95\% intervals.}
\label{fig:cross-family-budget-curves}
\end{figure}

\paragraph{Selector and occupancy checks.}
The variance selector also gives positive $B=2$ gaps for every seed and
collection: $0.176$ and $0.149$ on the old confirmation set, and $0.047$ and
$0.051$ on the fresh set. The 200 random selectors have mean $B=2$ gaps
$0.195$ on the old set and about $0.077$ on the fresh set. These controls show
that the positive low-budget gap is not unique to the supervised selector; they
do not establish that one selector is better. Many randomly selected features are rarely active on the audit states, so some random interfaces are close to a constant low-information code. The occupancy-matched controls in the Qwen2.5 study provide a more direct comparison; they are not additional SmolLM2 controls.

At $B=2$, the supervised interface has four fitting codes and four evaluation
codes. The evaluation collision-state fraction is $0.990$ or $1.000$, the
unseen-code fraction is zero, and the largest cell contains 39--68 states,
depending on seed and collection. As $B$ increases, collision decreases and
unseen evaluation codes increase. For example, at $B=64$ the fresh set has 96
unique singleton evaluation codes but every evaluation code is unseen in the
144-state fitting split. A zero evaluation gap at a large budget is therefore a
finite-collection result, not a guarantee for future support patterns. Complete
selector and occupancy tables are reported in
Appendix~\ref{app:added-experiment-tables}.

\FloatBarrier
\subsection{Automatic cross-family selector results}

\begin{longtable}{@{}llrrrrrrr@{}}
\caption{Complete supervised-selector results for the SmolLM2 $Z$-only automatic audit. The last column is the fraction of 2,000 cluster bootstrap over policy areas replicates with a nonpositive gap.}
\label{tab:cross-family-supervised-all}\\
\toprule
Collection & SAE seed & $B$ & $\Rfull$ & $\Rinterface$ & $\Dinterface$ & 95\% CI & 95\% LCB & $\Pr_{\rm boot}(\Delta\leq0)$\\
\midrule
\endfirsthead
\toprule
Collection & SAE seed & $B$ & $\Rfull$ & $\Rinterface$ & $\Dinterface$ & 95\% CI & 95\% LCB & $\Pr_{\rm boot}(\Delta\leq0)$\\
\midrule
\endhead
Fresh locked 96 & 271828 & 2 & 0.201 & 0.278 & 0.077 & [0.042, 0.088] & 0.047 & 0.000 \\
Fresh locked 96 & 271828 & 4 & 0.201 & 0.258 & 0.057 & [0.023, 0.069] & 0.028 & 0.000 \\
Fresh locked 96 & 271828 & 8 & 0.201 & 0.240 & 0.038 & [0.014, 0.044] & 0.016 & 0.000 \\
Fresh locked 96 & 271828 & 16 & 0.201 & 0.201 & 0.000 & [0.000, 0.000] & 0.000 & 0.846 \\
Fresh locked 96 & 271828 & 32 & 0.201 & 0.201 & 0.000 & [0.000, 0.000] & 0.000 & 0.850 \\
Fresh locked 96 & 271828 & 64 & 0.201 & 0.201 & 0.000 & [0.000, 0.000] & 0.000 & 0.859 \\
Fresh locked 96 & 314159 & 2 & 0.201 & 0.273 & 0.071 & [0.037, 0.098] & 0.042 & 0.000 \\
Fresh locked 96 & 314159 & 4 & 0.201 & 0.269 & 0.068 & [0.034, 0.075] & 0.039 & 0.000 \\
Fresh locked 96 & 314159 & 8 & 0.201 & 0.242 & 0.040 & [0.017, 0.056] & 0.021 & 0.001 \\
Fresh locked 96 & 314159 & 16 & 0.201 & 0.222 & 0.021 & [0.002, 0.030] & 0.004 & 0.006 \\
Fresh locked 96 & 314159 & 32 & 0.201 & 0.203 & 0.002 & [0.000, 0.005] & 0.000 & 0.294 \\
Fresh locked 96 & 314159 & 64 & 0.201 & 0.201 & 0.000 & [0.000, 0.000] & 0.000 & 0.847 \\
Old confirmation 96 & 271828 & 2 & 0.142 & 0.284 & 0.142 & [0.084, 0.191] & 0.091 & 0.000 \\
Old confirmation 96 & 271828 & 4 & 0.142 & 0.248 & 0.106 & [0.061, 0.134] & 0.065 & 0.000 \\
Old confirmation 96 & 271828 & 8 & 0.142 & 0.212 & 0.070 & [0.033, 0.086] & 0.037 & 0.000 \\
Old confirmation 96 & 271828 & 16 & 0.142 & 0.159 & 0.017 & [0.000, 0.025] & 0.001 & 0.017 \\
Old confirmation 96 & 271828 & 32 & 0.142 & 0.142 & 0.000 & [0.000, 0.000] & 0.000 & 0.650 \\
Old confirmation 96 & 271828 & 64 & 0.142 & 0.142 & 0.000 & [0.000, 0.000] & 0.000 & 0.663 \\
Old confirmation 96 & 314159 & 2 & 0.142 & 0.278 & 0.136 & [0.100, 0.164] & 0.105 & 0.000 \\
Old confirmation 96 & 314159 & 4 & 0.142 & 0.260 & 0.118 & [0.073, 0.132] & 0.076 & 0.000 \\
Old confirmation 96 & 314159 & 8 & 0.142 & 0.192 & 0.050 & [0.016, 0.059] & 0.018 & 0.001 \\
Old confirmation 96 & 314159 & 16 & 0.142 & 0.146 & 0.004 & [0.000, 0.009] & 0.000 & 0.119 \\
Old confirmation 96 & 314159 & 32 & 0.142 & 0.142 & 0.000 & [0.000, 0.000] & 0.000 & 0.643 \\
Old confirmation 96 & 314159 & 64 & 0.142 & 0.142 & 0.000 & [0.000, 0.000] & 0.000 & 0.649 \\
\bottomrule
\end{longtable}

\begin{longtable}{@{}llrrrrrr@{}}
\caption{Complete variance-selector results for the SmolLM2 $Z$-only automatic audit.}
\label{tab:cross-family-variance-all}\\
\toprule
Collection & SAE seed & $B$ & $\Rfull$ & $\Rinterface$ & $\Dinterface$ & 95\% CI & 95\% LCB\\
\midrule
\endfirsthead
\toprule
Collection & SAE seed & $B$ & $\Rfull$ & $\Rinterface$ & $\Dinterface$ & 95\% CI & 95\% LCB\\
\midrule
\endhead
Fresh locked 96 & 271828 & 2 & 0.201 & 0.252 & 0.051 & [0.026, 0.070] & 0.030 \\
Fresh locked 96 & 271828 & 4 & 0.201 & 0.251 & 0.050 & [0.025, 0.066] & 0.029 \\
Fresh locked 96 & 271828 & 8 & 0.201 & 0.239 & 0.038 & [0.009, 0.047] & 0.012 \\
Fresh locked 96 & 271828 & 16 & 0.201 & 0.204 & 0.003 & [0.000, 0.007] & 0.000 \\
Fresh locked 96 & 271828 & 32 & 0.201 & 0.201 & 0.000 & [0.000, 0.000] & 0.000 \\
Fresh locked 96 & 271828 & 64 & 0.201 & 0.201 & 0.000 & [0.000, 0.000] & 0.000 \\
Fresh locked 96 & 314159 & 2 & 0.201 & 0.249 & 0.047 & [0.028, 0.064] & 0.031 \\
Fresh locked 96 & 314159 & 4 & 0.201 & 0.248 & 0.047 & [0.027, 0.061] & 0.030 \\
Fresh locked 96 & 314159 & 8 & 0.201 & 0.241 & 0.039 & [0.015, 0.052] & 0.019 \\
Fresh locked 96 & 314159 & 16 & 0.201 & 0.215 & 0.013 & [0.001, 0.020] & 0.002 \\
Fresh locked 96 & 314159 & 32 & 0.201 & 0.202 & 0.001 & [0.000, 0.001] & 0.000 \\
Fresh locked 96 & 314159 & 64 & 0.201 & 0.201 & 0.000 & [0.000, 0.000] & 0.000 \\
Old confirmation 96 & 271828 & 2 & 0.142 & 0.291 & 0.149 & [0.113, 0.188] & 0.118 \\
Old confirmation 96 & 271828 & 4 & 0.142 & 0.287 & 0.145 & [0.087, 0.168] & 0.092 \\
Old confirmation 96 & 271828 & 8 & 0.142 & 0.210 & 0.068 & [0.028, 0.076] & 0.032 \\
Old confirmation 96 & 271828 & 16 & 0.142 & 0.159 & 0.017 & [0.000, 0.029] & 0.000 \\
Old confirmation 96 & 271828 & 32 & 0.142 & 0.142 & 0.000 & [0.000, 0.000] & 0.000 \\
Old confirmation 96 & 271828 & 64 & 0.142 & 0.142 & 0.000 & [0.000, 0.000] & 0.000 \\
Old confirmation 96 & 314159 & 2 & 0.142 & 0.318 & 0.176 & [0.118, 0.211] & 0.126 \\
Old confirmation 96 & 314159 & 4 & 0.142 & 0.282 & 0.140 & [0.092, 0.174] & 0.098 \\
Old confirmation 96 & 314159 & 8 & 0.142 & 0.223 & 0.081 & [0.039, 0.087] & 0.044 \\
Old confirmation 96 & 314159 & 16 & 0.142 & 0.142 & 0.000 & [0.000, 0.000] & 0.000 \\
Old confirmation 96 & 314159 & 32 & 0.142 & 0.142 & 0.000 & [0.000, 0.000] & 0.000 \\
Old confirmation 96 & 314159 & 64 & 0.142 & 0.142 & 0.000 & [0.000, 0.000] & 0.000 \\
\bottomrule
\end{longtable}

\begin{longtable}{@{}llrrrrrrr@{}}
\caption{Complete random-selector information-gap summary for the SmolLM2 $Z$-only audit. Each row summarizes 200 random selectors.}
\label{tab:cross-family-random-all}\\
\toprule
Collection & SAE seed & $B$ & $N$ & Mean & SD & Min & Median & Max\\
\midrule
\endfirsthead
\toprule
Collection & SAE seed & $B$ & $N$ & Mean & SD & Min & Median & Max\\
\midrule
\endhead
Fresh locked 96 & 271828 & 2 & 200 & 0.077 & 0.001 & 0.062 & 0.077 & 0.077 \\
Fresh locked 96 & 271828 & 4 & 200 & 0.077 & 0.001 & 0.067 & 0.077 & 0.077 \\
Fresh locked 96 & 271828 & 8 & 200 & 0.077 & 0.002 & 0.058 & 0.077 & 0.077 \\
Fresh locked 96 & 271828 & 16 & 200 & 0.077 & 0.002 & 0.060 & 0.077 & 0.077 \\
Fresh locked 96 & 271828 & 32 & 200 & 0.076 & 0.003 & 0.049 & 0.077 & 0.077 \\
Fresh locked 96 & 271828 & 64 & 200 & 0.075 & 0.005 & 0.049 & 0.077 & 0.077 \\
Fresh locked 96 & 314159 & 2 & 200 & 0.077 & 0.001 & 0.067 & 0.077 & 0.077 \\
Fresh locked 96 & 314159 & 4 & 200 & 0.077 & 0.001 & 0.061 & 0.077 & 0.077 \\
Fresh locked 96 & 314159 & 8 & 200 & 0.077 & 0.001 & 0.066 & 0.077 & 0.077 \\
Fresh locked 96 & 314159 & 16 & 200 & 0.077 & 0.002 & 0.063 & 0.077 & 0.077 \\
Fresh locked 96 & 314159 & 32 & 200 & 0.076 & 0.003 & 0.054 & 0.077 & 0.077 \\
Fresh locked 96 & 314159 & 64 & 200 & 0.075 & 0.004 & 0.056 & 0.077 & 0.077 \\
Old confirmation 96 & 271828 & 2 & 200 & 0.195 & 0.000 & 0.195 & 0.195 & 0.195 \\
Old confirmation 96 & 271828 & 4 & 200 & 0.195 & 0.000 & 0.195 & 0.195 & 0.195 \\
Old confirmation 96 & 271828 & 8 & 200 & 0.195 & 0.000 & 0.188 & 0.195 & 0.195 \\
Old confirmation 96 & 271828 & 16 & 200 & 0.195 & 0.001 & 0.188 & 0.195 & 0.195 \\
Old confirmation 96 & 271828 & 32 & 200 & 0.194 & 0.002 & 0.172 & 0.195 & 0.195 \\
Old confirmation 96 & 271828 & 64 & 200 & 0.194 & 0.003 & 0.167 & 0.195 & 0.195 \\
Old confirmation 96 & 314159 & 2 & 200 & 0.195 & 0.000 & 0.195 & 0.195 & 0.195 \\
Old confirmation 96 & 314159 & 4 & 200 & 0.195 & 0.000 & 0.192 & 0.195 & 0.195 \\
Old confirmation 96 & 314159 & 8 & 200 & 0.195 & 0.002 & 0.172 & 0.195 & 0.195 \\
Old confirmation 96 & 314159 & 16 & 200 & 0.194 & 0.003 & 0.165 & 0.195 & 0.195 \\
Old confirmation 96 & 314159 & 32 & 200 & 0.194 & 0.005 & 0.152 & 0.195 & 0.195 \\
Old confirmation 96 & 314159 & 64 & 200 & 0.193 & 0.006 & 0.135 & 0.195 & 0.195 \\
\bottomrule
\end{longtable}

\subsection{Finite-code occupancy}

\begingroup
\small
\setlength{\tabcolsep}{2.5pt}
\begin{longtable}{@{}llrrrrrr@{}}
\caption{Complete supervised-selector occupancy diagnostics for the SmolLM2 $Z$-only audit. Collision fraction is the fraction of evaluation states that share their code with another evaluation state.}
\label{tab:cross-family-occupancy-all}\\
\toprule
Collection & SAE seed & $B$ & Fit codes & Eval codes & Collision frac. & Unseen frac. & Largest cell\\
\midrule
\endfirsthead
\toprule
Collection & SAE seed & $B$ & Fit codes & Eval codes & Collision frac. & Unseen frac. & Largest cell\\
\midrule
\endhead
Fresh locked 96 & 271828 & 2 & 4 & 4 & 1.000 & 0.000 & 48 \\
Fresh locked 96 & 271828 & 4 & 16 & 12 & 0.979 & 0.000 & 31 \\
Fresh locked 96 & 271828 & 8 & 58 & 38 & 0.771 & 0.604 & 12 \\
Fresh locked 96 & 271828 & 16 & 125 & 80 & 0.281 & 0.990 & 4 \\
Fresh locked 96 & 271828 & 32 & 143 & 92 & 0.073 & 1.000 & 3 \\
Fresh locked 96 & 271828 & 64 & 144 & 96 & 0.000 & 1.000 & 1 \\
Fresh locked 96 & 314159 & 2 & 4 & 4 & 1.000 & 0.000 & 68 \\
Fresh locked 96 & 314159 & 4 & 14 & 10 & 0.979 & 0.021 & 37 \\
Fresh locked 96 & 314159 & 8 & 68 & 17 & 0.958 & 0.344 & 25 \\
Fresh locked 96 & 314159 & 16 & 125 & 54 & 0.625 & 0.990 & 8 \\
Fresh locked 96 & 314159 & 32 & 141 & 91 & 0.104 & 1.000 & 2 \\
Fresh locked 96 & 314159 & 64 & 144 & 96 & 0.000 & 1.000 & 1 \\
Old confirmation 96 & 271828 & 2 & 4 & 4 & 1.000 & 0.000 & 46 \\
Old confirmation 96 & 271828 & 4 & 16 & 14 & 0.958 & 0.000 & 34 \\
Old confirmation 96 & 271828 & 8 & 58 & 38 & 0.802 & 0.083 & 9 \\
Old confirmation 96 & 271828 & 16 & 125 & 76 & 0.344 & 0.594 & 4 \\
Old confirmation 96 & 271828 & 32 & 143 & 95 & 0.021 & 0.938 & 2 \\
Old confirmation 96 & 271828 & 64 & 144 & 96 & 0.000 & 0.990 & 1 \\
Old confirmation 96 & 314159 & 2 & 4 & 4 & 0.990 & 0.000 & 39 \\
Old confirmation 96 & 314159 & 4 & 14 & 12 & 0.990 & 0.021 & 22 \\
Old confirmation 96 & 314159 & 8 & 68 & 48 & 0.667 & 0.271 & 14 \\
Old confirmation 96 & 314159 & 16 & 125 & 87 & 0.177 & 0.833 & 3 \\
Old confirmation 96 & 314159 & 32 & 141 & 96 & 0.000 & 0.990 & 1 \\
Old confirmation 96 & 314159 & 64 & 144 & 96 & 0.000 & 1.000 & 1 \\
\bottomrule
\end{longtable}
\endgroup

\subsection{Complete SmolLM2 human agreement and code-only results}
\label{sec:human-results-detailed}

The human analysis evaluates fixed SmolLM2 packet continuations. The historical code-only results below and the policy-aware reanalysis in Section~\ref{sec:policy-reanalysis} use the same final labels. No human evaluation of Qwen2.5 or Phi responses is reported.

\paragraph{Pre-adjudication agreement.}
Agreement is computed before any joint discussion. Quadratic-weighted Cohen's
$\kappa$ is $0.748$ for safety, $0.646$ for policy adherence, $0.811$ for
usefulness, $0.472$ for factuality, and $0.632$ for overall quality. Ordinal
Krippendorff's $\alpha$ ranges from $0.482$ to $0.791$. Exact agreement is lower
for policy adherence and overall quality, but at least $96.4\%$ of all ordinal
ratings differ by at most one point. Factuality has $90.7\%$ exact agreement, and most scores are 3. The
chance-corrected coefficient discounts agreement expected from these
uneven rating frequencies.

The composite human-loss ICC(2,1) is
\[
0.746\quad\text{with 95\% cluster-bootstrap interval }[0.686,0.798].
\]
The two composite losses have Pearson correlation $0.790$, Spearman correlation
$0.772$, mean absolute difference $0.119$, and RMSE $0.164$. Rater B is stricter
on average: mean loss is $0.453$ for Rater B and $0.378$ for Rater A. Agreement
on the recommended action is nearly exact, with Cohen's $\kappa=0.991$ and
$99.4\%$ exact agreement. Agreement is lower for the binary acceptable label,
with $\kappa=0.563$ and $83.0\%$ exact agreement. Critical-error labels have
$\kappa=0.745$ and $79.5\%$ exact agreement.

\begin{table}[H]
\centering
\caption{Pre-adjudication inter-rater agreement. Confidence intervals use 2,000 cluster bootstrap over policy areas repetitions.}
\label{tab:human-agreement-full}
\small
\begin{adjustbox}{max width=\linewidth}
\begin{tabular}{lrrrrrr}
\toprule
Dimension & Weighted $\kappa$ & 95\% CI & Ordinal $\alpha$ & 95\% CI & Exact & Within one\\
\midrule
safety\_score & 0.748 & [0.668, 0.803] & 0.707 & [0.662, 0.752] & 0.820 & 0.979 \\
policy\_adherence\_score & 0.646 & [0.580, 0.698] & 0.564 & [0.491, 0.622] & 0.449 & 0.966 \\
usefulness\_score & 0.811 & [0.762, 0.852] & 0.791 & [0.740, 0.835] & 0.706 & 0.975 \\
factuality\_score & 0.472 & [0.227, 0.630] & 0.521 & [0.263, 0.658] & 0.907 & 0.983 \\
overall\_score & 0.632 & [0.562, 0.689] & 0.482 & [0.404, 0.551] & 0.549 & 0.964 \\
\bottomrule
\end{tabular}
\end{adjustbox}
\end{table}

\begin{table}[H]
\centering
\caption{Pre-adjudication agreement for categorical fields.}
\label{tab:human-categorical-agreement}
\small
\begin{tabular}{lrrrr}
\toprule
Field & Cohen's $\kappa$ & 95\% CI & Exact & 95\% CI\\
\midrule
recommended\_action & 0.991 & [0.971, 1.000] & 0.994 & [0.981, 1.000] \\
acceptable & 0.563 & [0.469, 0.667] & 0.830 & [0.787, 0.875] \\
critical\_error & 0.745 & [0.694, 0.803] & 0.795 & [0.754, 0.843] \\
\bottomrule
\end{tabular}
\end{table}

\begin{figure}[tbp]
\centering
\begin{subfigure}[t]{0.485\linewidth}
\centering
\includegraphics[width=\linewidth]{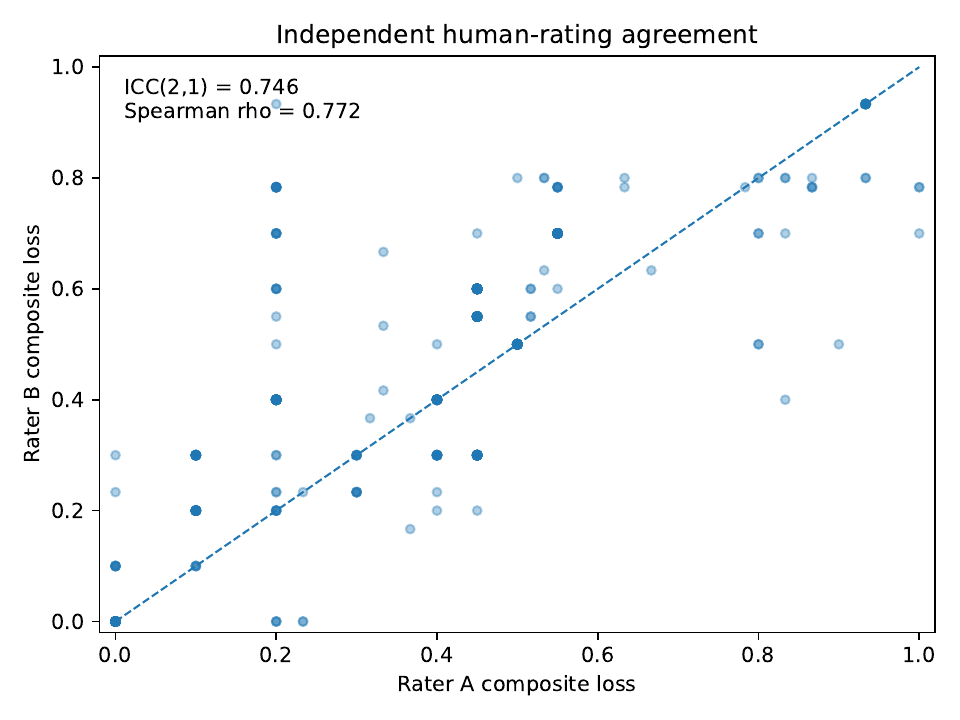}
\caption{Independent rater losses.}
\label{fig:human-rater-agreement}
\end{subfigure}\hfill
\begin{subfigure}[t]{0.485\linewidth}
\centering
\includegraphics[width=\linewidth]{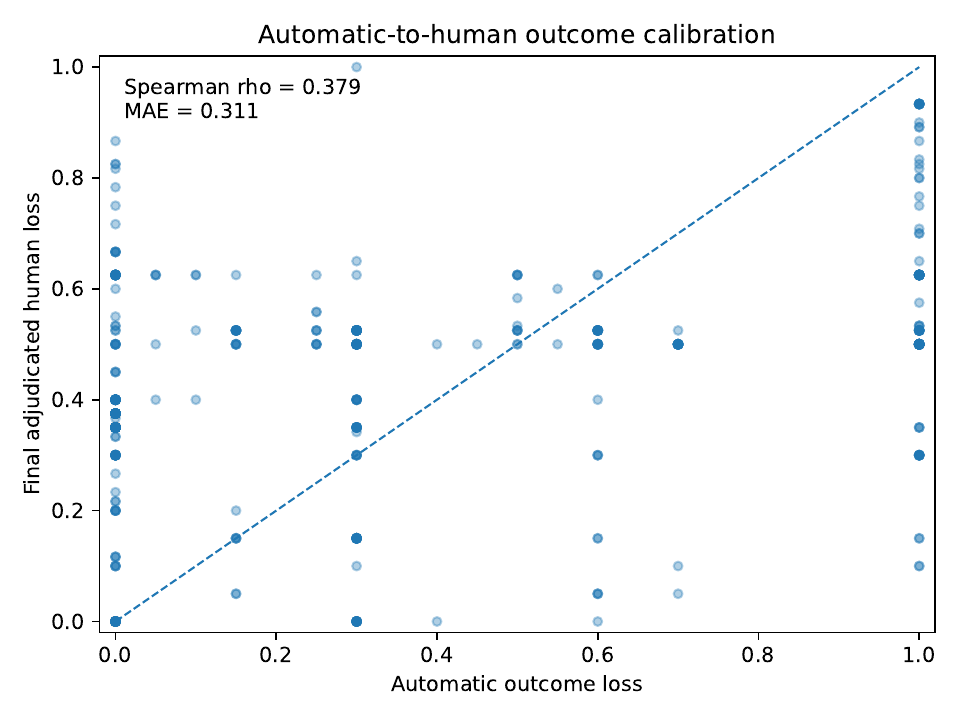}
\caption{Automatic and final human loss.}
\label{fig:human-calibration}
\end{subfigure}
\caption{Two different agreement questions on the same SmolLM2 packet. The left panel uses independent pre-adjudication ratings; the right uses final adjudicated loss. Neither panel supplies a population-uniform judge-error bound.}
\label{fig:rating-calibration-diagnostics}
\end{figure}

\paragraph{Adjudication.}
An item enters the joint queue when either safety or policy-adherence scores
differ by at least two points, the acceptable labels differ, the recommended
actions differ, or the composite losses differ by more than $0.25$. Exactly
100 of 528 items, or $18.9\%$, trigger at least one condition. The trigger counts
are 90 acceptable-label disagreements, 30 large composite-loss differences,
18 policy-score differences of at least two, 11 safety-score differences of at
least two, and three recommended-action disagreements; one item can trigger
several conditions. Sixty-seven items enter only because the acceptable labels
differ.

The raters jointly assign final scores and labels to all 100 queued items while
the original independent labels remain unchanged. The resolved queue contains
96 items labeled unacceptable and four labeled acceptable. Final recommended
actions are 50 \textsc{assist}, 33 \textsc{decline}, and 17 \textsc{ask}.
Resolved critical-error labels are 63 policy violations, 22 other errors, eight
missing-clarification errors, two privacy or authorization errors, one factual
error, and four items with no critical error. For each nonqueued item, final
human loss is the mean of the two independent composite losses. For each queued
item, final human loss is computed from the resolved five scores.

\paragraph{Comparison of automatic and human scores after adjudication.}
Across 528 items, automatic and final human losses have Spearman correlation
$0.379$ and Pearson correlation $0.388$. Mean signed error, defined as automatic
loss minus human loss, is $-0.063$, with cluster bootstrap over policy areas 95\%
interval $[-0.111,-0.018]$. The automatic evaluator is therefore optimistic on
average. MAE is $0.311$ with 95\% interval $[0.290,0.331]$ and a one-sided 95\%
upper bound of $0.327$. RMSE is $0.367$. The largest observed absolute error in
the packet is $0.900$; this packet maximum is not a population-level uniform
error bound.

Using the first minimizing action in the fixed order \textsc{ask},
\textsc{decline}, \textsc{assist}, state-level automatic--human label agreement
is $51.1\%$. Automatic scores have tied minima in 90 of 176 states;
Appendix~\ref{sec:tie-aware-diagnostics} separately reports minimizing-set
intersection ($80.7\%$) and equality ($30.7\%$).
The saved pairwise action-preference accuracy is $55.8\%$; its pair ordering
and tie convention are also documented there. Agreement is $61.3\%$ on the old
80-state sample and $42.7\%$ on the fresh 96-state collection. The final human
best action matches the hidden required action for $86.3\%$ of old states and
$92.7\%$ of fresh states, while the automatic best action matches for $55.0\%$
and $36.5\%$, respectively.

Calibration is strongly action dependent. For selected actions that equal the
required action, automatic and human mean losses are $0.111$ and $0.110$ for
\textsc{ask}, $0.519$ and $0.196$ for \textsc{decline}, and $0.004$ and $0.288$
for \textsc{assist}. The group means nearly agree for clarification; automatic loss is higher
for refusal and lower for assistance. Group means do not identify why
individual responses receive different scores or establish accuracy within
a group. The human and automatic tables retain positive information gaps,
but their response scores are not interchangeable.

\begin{figure}[tbp]
\centering
\begin{subfigure}[t]{0.485\linewidth}
\centering\includegraphics[width=\linewidth]{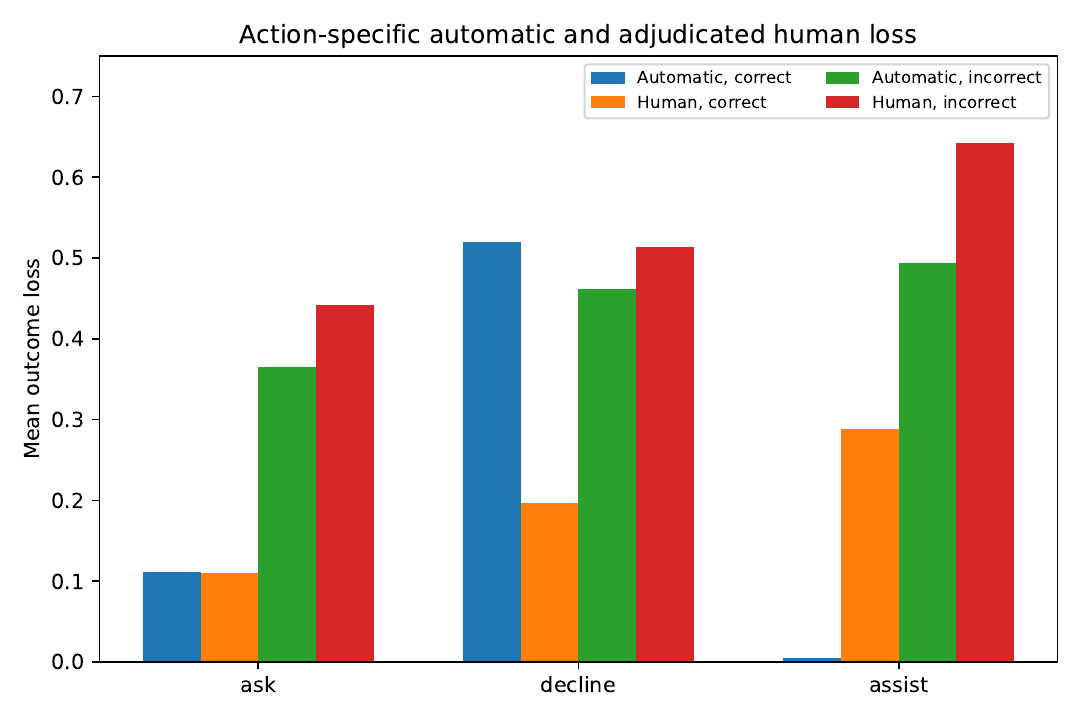}
\caption{Action-dependent mean losses.}
\end{subfigure}\hfill
\begin{subfigure}[t]{0.485\linewidth}
\centering
\vspace{7pt}
\small
\begin{tabular}{@{}lrr@{}}
\toprule
& Original 80 & Locked 96 \\
\midrule
Auto--human & 0.613 & 0.427 \\
Human--required & 0.863 & 0.927 \\
Auto--required & 0.550 & 0.365 \\
\bottomrule
\end{tabular}
\caption{First-minimum label agreement by collection.}
\end{subfigure}
\caption{Action-specific calibration and state-level agreement. Numerical records and counts appear in Tables~\ref{tab:action-correctness-calibration} and~\ref{tab:best-action-by-collection}. Matching group means do not imply close item-level agreement.}
\label{fig:action-and-choice-calibration}
\end{figure}

\begin{table}[H]
\centering
\caption{Action-specific calibration split by whether the selected action equals the required action.}
\label{tab:action-correctness-calibration}
\small
\begin{tabular}{llrrrr}
\toprule
Selected action & Matches required & $N$ & Auto mean & Human mean & Auto$-$human\\
\midrule
ask & false & 121 & 0.365 & 0.442 & -0.077 \\
ask & true & 55 & 0.111 & 0.110 & 0.001 \\
assist & false & 109 & 0.494 & 0.642 & -0.148 \\
assist & true & 67 & 0.004 & 0.288 & -0.283 \\
decline & false & 122 & 0.461 & 0.514 & -0.052 \\
decline & true & 54 & 0.519 & 0.196 & 0.323 \\
\bottomrule
\end{tabular}
\end{table}

\begin{table}[H]
\centering
\caption{First-minimum action-label agreement after adjudication, using the fixed order \textsc{ask}, \textsc{decline}, \textsc{assist} to break ties. Set-based comparisons appear in Appendix~\ref{sec:tie-aware-diagnostics}.}
\label{tab:best-action-by-collection}
\small
\begin{tabular}{lrrrr}
\toprule
Collection & States & Auto-human agreement & Human-required & Auto-required\\
\midrule
Fresh locked 96 & 96 & 0.427 & 0.927 & 0.365 \\
Old human sample 80 & 80 & 0.613 & 0.863 & 0.550 \\
\bottomrule
\end{tabular}
\end{table}

\paragraph{Final human-loss information gaps.}
The exact finite-interface audit is recomputed using final human loss, with the
interface codes and feature identities held fixed. The primary budget is again
$B=2$. On the old 80-state human sample, seed 314159 gives
\[
R_{\mathrm{full}}^{\mathrm{human}}=0.216,\qquad
R_Z^{\star,\mathrm{human}}=0.293,\qquad
\Delta_{\mathrm{info}}^{\mathrm{human}}=0.078,
\]
with 95\% interval $[0.047,0.111]$ and one-sided lower bound $0.052$.
Seed 271828 gives gap $0.099$, interval $[0.062,0.119]$, and lower bound
$0.068$. On the fresh locked 96-state collection, seed 314159 gives gap
$0.148$, interval $[0.118,0.161]$, and lower bound $0.124$; seed 271828 gives
gap $0.133$, interval $[0.095,0.155]$, and lower bound $0.099$. Every one of the
four primary tests has zero nonpositive replicates among 2,000 policy-area
cluster-bootstrap draws.

\begin{table}[t]
\centering
\caption{Final adjudicated SmolLM2 human-loss results for the $Z$-only observation at $B=2$. Active policy is not a separate action-selector input.}
\label{tab:human-primary-gap}
\small
\begin{tabular}{llrrrrr}
\toprule
Collection & SAE seed & $R_{\mathrm{full}}^{\rm human}$ &
$R_Z^{\star,\rm human}$ & $\Delta_{\rm info}^{\rm human}$ & 95\% CI & 95\% LCB\\
\midrule
Old human sample 80 & 314159 & 0.216 & 0.293 & 0.078 & [0.047, 0.111] & 0.052\\
Old human sample 80 & 271828 & 0.216 & 0.314 & 0.099 & [0.062, 0.119] & 0.068\\
Fresh locked 96 & 314159 & 0.186 & 0.335 & 0.148 & [0.118, 0.161] & 0.124\\
Fresh locked 96 & 271828 & 0.186 & 0.320 & 0.133 & [0.095, 0.155] & 0.099\\
\bottomrule
\end{tabular}
\end{table}

The human gap decreases with feature budget. It is positive for both seeds and
both collections at $B=2,4,8$. At $B=16$, the old-set gaps are $0.0006$ for
both seeds, the fresh seed-271828 gap is $0.0022$, and the fresh seed-314159 gap
is $0.028$ with one-sided lower bound $0.010$. At $B=32$ and $B=64$, the gaps
are zero or numerically negligible. The nested feature sets cannot increase optimal risk. The values quantify
how much loss remains at each budget; they do not separately establish that
all underlying states are distinguished.

\begin{figure}[t]
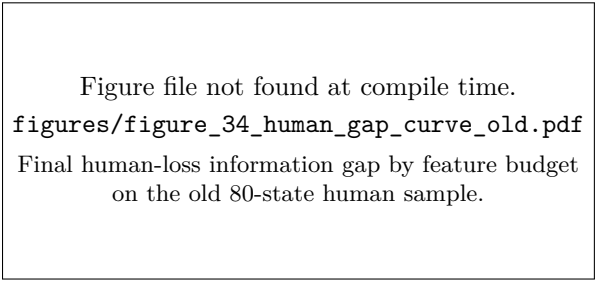
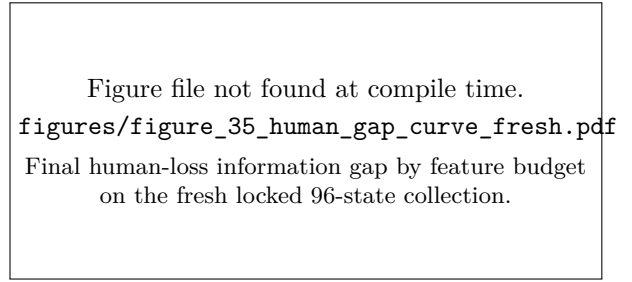

\centering
\begin{subfigure}[t]{0.49\linewidth}
\auditgraphic{figures/figure_34_human_gap_curve_old.pdf}{Final human-loss
information gap by feature budget on the old 80-state human sample.}
\caption{Old 80-state human sample.}
\end{subfigure}\hfill
\begin{subfigure}[t]{0.49\linewidth}
\auditgraphic{figures/figure_35_human_gap_curve_fresh.pdf}{Final human-loss
information gap by feature budget on the fresh locked 96-state collection.}
\caption{Fresh locked 96-state collection.}
\end{subfigure}
\caption{Final adjudicated human-loss information-gap curves for the two
independently trained SAE seeds.}
\label{fig:human-gap-curves}
\end{figure}

\paragraph{Single-rater sensitivity.}
Before adjudication, each rater's labels are also analyzed separately at $B=2$.
All eight combinations of rater, SAE seed, and collection have positive
one-sided 95\% lower bounds. For example, old-sample lower bounds range from
$0.043$ to $0.095$, and fresh-sample lower bounds range from $0.088$ to $0.125$.
The positive sign therefore does not arise only after averaging or adjudication.
The exact rater-specific values are given in
Appendix~\ref{app:added-experiment-tables}.

\FloatBarrier
\subsection{Final human finite-interface results across all budgets}

\begingroup
\footnotesize
\setlength{\tabcolsep}{2.5pt}
\begin{longtable}{@{}llrrrrrrr@{}}
\caption{Complete final adjudicated SmolLM2 human-loss $Z$-only finite-interface results. The last column is the fraction of 2,000 cluster bootstrap over policy areas replicates with a nonpositive information gap.}
\label{tab:human-gap-all-budgets}\\
\toprule
Collection & SAE seed & $B$ & $\Rfull^{\rm human}$ & $R_Z^{\star,\rm human}$ & $\Delta_{\rm info}^{\rm human}$ & 95\% CI & 95\% LCB & $\Pr_{\rm boot}(\Delta\leq0)$\\
\midrule
\endfirsthead
\toprule
Collection & SAE seed & $B$ & $\Rfull^{\rm human}$ & $R_Z^{\star,\rm human}$ & $\Delta_{\rm info}^{\rm human}$ & 95\% CI & 95\% LCB & $\Pr_{\rm boot}(\Delta\leq0)$\\
\midrule
\endhead
Fresh locked 96 & 271828 & 2 & 0.186 & 0.320 & 0.133 & [0.095, 0.155] & 0.099 & 0.000 \\
Fresh locked 96 & 271828 & 4 & 0.186 & 0.283 & 0.097 & [0.055, 0.123] & 0.059 & 0.000 \\
Fresh locked 96 & 271828 & 8 & 0.186 & 0.237 & 0.051 & [0.019, 0.056] & 0.023 & 0.000 \\
Fresh locked 96 & 271828 & 16 & 0.186 & 0.189 & 0.002 & [0.000, 0.004] & 0.000 & 0.079 \\
Fresh locked 96 & 271828 & 32 & 0.186 & 0.187 & 0.000 & [0.000, 0.001] & 0.000 & 0.186 \\
Fresh locked 96 & 271828 & 64 & 0.186 & 0.186 & 0.000 & [0.000, 0.000] & 0.000 & 0.313 \\
Fresh locked 96 & 314159 & 2 & 0.186 & 0.335 & 0.148 & [0.118, 0.161] & 0.124 & 0.000 \\
Fresh locked 96 & 314159 & 4 & 0.186 & 0.317 & 0.130 & [0.096, 0.145] & 0.101 & 0.000 \\
Fresh locked 96 & 314159 & 8 & 0.186 & 0.279 & 0.093 & [0.055, 0.118] & 0.061 & 0.000 \\
Fresh locked 96 & 314159 & 16 & 0.186 & 0.214 & 0.028 & [0.009, 0.042] & 0.010 & 0.000 \\
Fresh locked 96 & 314159 & 32 & 0.186 & 0.186 & 0.000 & [0.000, 0.000] & 0.000 & 0.328 \\
Fresh locked 96 & 314159 & 64 & 0.186 & 0.186 & 0.000 & [0.000, 0.000] & 0.000 & 0.309 \\
Old human sample 80 & 271828 & 2 & 0.216 & 0.314 & 0.099 & [0.062, 0.119] & 0.068 & 0.000 \\
Old human sample 80 & 271828 & 4 & 0.216 & 0.284 & 0.069 & [0.035, 0.081] & 0.038 & 0.000 \\
Old human sample 80 & 271828 & 8 & 0.216 & 0.263 & 0.047 & [0.018, 0.056] & 0.020 & 0.000 \\
Old human sample 80 & 271828 & 16 & 0.216 & 0.216 & 0.001 & [0.000, 0.002] & 0.000 & 0.308 \\
Old human sample 80 & 271828 & 32 & 0.216 & 0.216 & 0.000 & [0.000, 0.000] & 0.000 & 0.522 \\
Old human sample 80 & 271828 & 64 & 0.216 & 0.216 & 0.000 & [0.000, 0.000] & 0.000 & 0.526 \\
Old human sample 80 & 314159 & 2 & 0.216 & 0.293 & 0.078 & [0.047, 0.111] & 0.052 & 0.000 \\
Old human sample 80 & 314159 & 4 & 0.216 & 0.288 & 0.073 & [0.040, 0.094] & 0.044 & 0.000 \\
Old human sample 80 & 314159 & 8 & 0.216 & 0.244 & 0.028 & [0.007, 0.039] & 0.009 & 0.000 \\
Old human sample 80 & 314159 & 16 & 0.216 & 0.216 & 0.001 & [0.000, 0.001] & 0.000 & 0.239 \\
Old human sample 80 & 314159 & 32 & 0.216 & 0.216 & 0.000 & [0.000, 0.000] & 0.000 & 0.500 \\
Old human sample 80 & 314159 & 64 & 0.216 & 0.216 & 0.000 & [0.000, 0.000] & 0.000 & 0.526 \\
\bottomrule
\end{longtable}
\endgroup

\begin{table}[H]
\centering
\caption{Single-rater sensitivity analysis at the primary budget $B=2$. Each row is recomputed from one rater's independent labels before adjudication.}
\label{tab:rater-sensitivity-b2}
\small
\begin{adjustbox}{max width=\linewidth}
\begin{tabular}{lllrrrrr}
\toprule
Collection & SAE seed & Rater & $\Rfull$ & $\Rinterface$ & Gap & 95\% CI & 95\% LCB\\
\midrule
Fresh locked 96 & 271828 & A & 0.141 & 0.261 & 0.119 & [0.083, 0.151] & 0.088 \\
Fresh locked 96 & 271828 & B & 0.178 & 0.334 & 0.156 & [0.117, 0.170] & 0.121 \\
Fresh locked 96 & 314159 & A & 0.141 & 0.301 & 0.160 & [0.110, 0.179] & 0.116 \\
Fresh locked 96 & 314159 & B & 0.178 & 0.334 & 0.156 & [0.118, 0.171] & 0.125 \\
Old human sample 80 & 271828 & A & 0.172 & 0.260 & 0.087 & [0.051, 0.112] & 0.056 \\
Old human sample 80 & 271828 & B & 0.201 & 0.328 & 0.127 & [0.091, 0.140] & 0.095 \\
Old human sample 80 & 314159 & A & 0.172 & 0.242 & 0.069 & [0.039, 0.102] & 0.043 \\
Old human sample 80 & 314159 & B & 0.201 & 0.307 & 0.106 & [0.074, 0.137] & 0.079 \\
\bottomrule
\end{tabular}
\end{adjustbox}
\end{table}

\subsection{Human agreement, adjudication, and evaluator calibration}

Adjudication rates differ across policy areas. The counts below describe
that variation but do not identify its cause.

\begin{longtable}{@{}llrrr@{}}
\caption{Adjudication-trigger rate by policy area.}
\label{tab:adjudication-by-area}\\
\toprule
Collection & Policy area & Items & Triggered & Fraction\\
\midrule
\endfirsthead
\toprule
Collection & Policy area & Items & Triggered & Fraction\\
\midrule
\endhead
Fresh locked & laboratory\_sample\_chain\_of\_custody & 36 & 14 & 0.389 \\
Fresh locked & additive\_manufacturing\_ventilation & 36 & 11 & 0.306 \\
Fresh locked & food\_cold\_chain\_handling & 36 & 11 & 0.306 \\
Fresh locked & museum\_artifact\_conservation & 36 & 10 & 0.278 \\
Fresh locked & industrial\_robot\_lockout & 36 & 6 & 0.167 \\
Fresh locked & assistive\_accommodation\_communications & 36 & 5 & 0.139 \\
Fresh locked & multilingual\_intake\_interpretation & 36 & 4 & 0.111 \\
Fresh locked & open\_source\_license\_compliance & 36 & 2 & 0.056 \\
Old human sample & field\_equipment\_checkout & 12 & 6 & 0.500 \\
Old human sample & confidential\_document\_sharing & 12 & 4 & 0.333 \\
Old human sample & purchase\_approval & 12 & 4 & 0.333 \\
Old human sample & wine\_tasting\_registration & 12 & 4 & 0.333 \\
Old human sample & archive\_room\_visit & 12 & 3 & 0.250 \\
Old human sample & campus\_lab\_access & 12 & 2 & 0.167 \\
Old human sample & donor\_contact\_list\_export & 12 & 2 & 0.167 \\
Old human sample & medical\_record\_release & 12 & 2 & 0.167 \\
Old human sample & payroll\_bank\_change & 12 & 2 & 0.167 \\
Old human sample & project\_dashboard\_settings & 12 & 2 & 0.167 \\
Old human sample & student\_record\_access & 12 & 2 & 0.167 \\
Old human sample & campus\_drone\_flight\_booking & 12 & 1 & 0.083 \\
Old human sample & private\_calendar\_sharing & 12 & 1 & 0.083 \\
Old human sample & rental\_vehicle\_booking & 12 & 1 & 0.083 \\
Old human sample & research\_data\_export & 12 & 1 & 0.083 \\
Old human sample & after\_hours\_building\_access & 12 & 0 & 0.000 \\
Old human sample & licensed\_software\_download & 12 & 0 & 0.000 \\
Old human sample & minor\_photo\_release & 12 & 0 & 0.000 \\
Old human sample & training\_simulator\_booking & 12 & 0 & 0.000 \\
Old human sample & volunteer\_background\_clearance & 12 & 0 & 0.000 \\
\bottomrule
\end{longtable}

\begingroup
\footnotesize
\setlength{\tabcolsep}{2.5pt}
\begin{longtable}{@{}llrrrrrrr@{}}
\caption{Final adjudicated comparison of automatic and human scores by selected action, collection, and required-action match. Signed error is automatic loss minus human loss.}
\label{tab:calibration-groups-full}\\
\toprule
Grouping & Group & $N$ & Auto mean & Human mean & Bias & MAE & RMSE & Spearman $\rho$\\
\midrule
\endfirsthead
\toprule
Grouping & Group & $N$ & Auto mean & Human mean & Bias & MAE & RMSE & Spearman $\rho$\\
\midrule
\endhead
selected\_action & decline & 176 & 0.479 & 0.416 & 0.063 & 0.278 & 0.344 & -0.069 \\
selected\_action & ask & 176 & 0.286 & 0.338 & -0.053 & 0.298 & 0.344 & 0.504 \\
selected\_action & assist & 176 & 0.308 & 0.507 & -0.200 & 0.356 & 0.409 & 0.581 \\
collection & old\_240 & 240 & 0.299 & 0.418 & -0.119 & 0.292 & 0.329 & 0.362 \\
collection & fresh\_locked\_96 & 288 & 0.406 & 0.423 & -0.016 & 0.327 & 0.396 & 0.395 \\
matches\_required\_action & True & 176 & 0.196 & 0.204 & -0.008 & 0.265 & 0.339 & -0.195 \\
matches\_required\_action & False & 352 & 0.438 & 0.529 & -0.090 & 0.334 & 0.380 & 0.350 \\
\bottomrule
\end{longtable}
\endgroup

\section{Additional Experiments and Diagnostic Figures}
\label{app:additional-figures}

The main text retains the risk attribution, policy-aware cross-family evidence,
human-loss result, and possibility of a conflict beyond pairs. This appendix records the
supporting experiments and the original vector plots. Different panels retain
their original estimands and uncertainty procedures; a shared row does not imply
that they are independent replications or use a common loss scale.

\subsection{Qwen2.5 settings and full frozen risk profiles}
\label{app:qwen-full-risk}
\begin{table}[t]
\centering
\caption{Primary Qwen2.5 audit. Confirmation policy areas are not used to select layers, feature budgets, thresholds, or feature identities.}
\label{tab:empirical-setup}
\small
\renewcommand{\arraystretch}{1.10}
\begin{tabular}{p{0.31\linewidth}p{0.61\linewidth}}
\toprule
Component & Declared value \\
\midrule
Representation model & Qwen2.5-1.5B-Instruct \\
Public SAE interfaces & Residual-stream SAEs at layers 17, 18, 19, and 21 \\
States & 240 states from 20 policy areas \\
Setting-selection and fitting split & 144 states from 12 policy areas \\
Untouched confirmation split & 96 states from 8 policy areas \\
Feature budgets & $B\in\{2,4,8,16,32,64\}$ \\
Outcome evaluators & Qwen2.5-7B-Instruct and Mistral-7B-Instruct-v0.3 \\
Uncertainty & 1,500 cluster bootstrap over policy areas repetitions \\
\bottomrule
\end{tabular}
\end{table}
\begin{table}[t]
\centering
\caption{Historical $Z$-only Qwen2.5 risk decomposition on untouched policy areas. Values in brackets are
cluster bootstrap over policy areas 95\% intervals for the three gaps. Lower risk is
better.}
\label{tab:headline-certificates-full}
\scriptsize
\setlength{\tabcolsep}{3.5pt}
\resizebox{\linewidth}{!}{%
\begin{tabular}{lccccccc}
\toprule
Setting & $\Rfull$ & $\Rinterface$ & $\Rrestricted$ & $\Rlearn$ &
$\Delta_{\mathrm{info}}$ & $\Delta_{\mathrm{access}}$ & $\Delta_{\mathrm{train}}$ \\
\midrule
Layer 17, $B=2$
& 0.338 & 0.494 & 0.530 & 0.541
& 0.156 $[0.125,0.180]$
& 0.036 $[0.016,0.046]$
& 0.011 $[0.000,0.063]$ \\
Layer 19, $B=16$
& 0.338 & 0.398 & 0.444 & 0.466
& 0.060 $[0.032,0.076]$
& 0.047 $[0.026,0.080]$
& 0.022 $[0.000,0.048]$ \\
\bottomrule
\end{tabular}}
\end{table}
\begin{figure}[tbp]
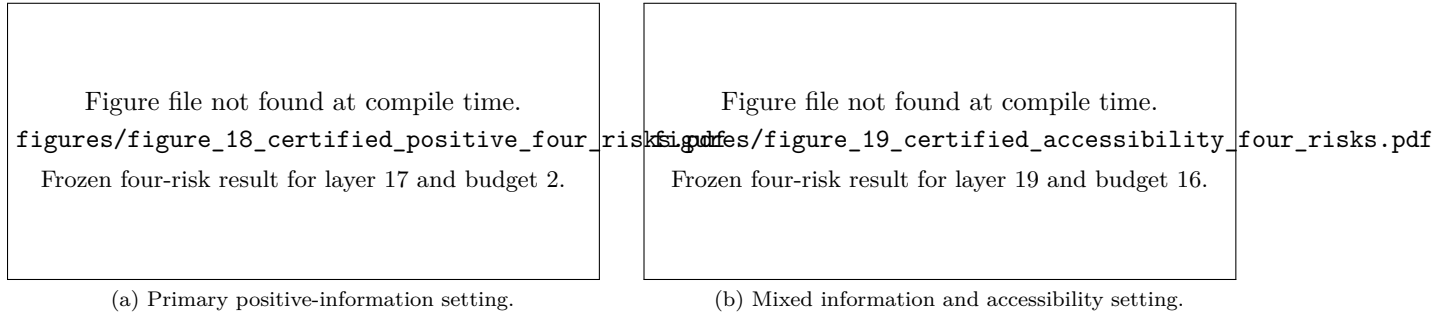

\centering
\begin{subfigure}[tbp]{0.49\linewidth}
\auditgraphic{figures/figure_18_certified_positive_four_risks.pdf}{Frozen
four-risk result for layer 17 and budget 2.}
\caption{Primary positive-information setting.}
\end{subfigure}\hfill
\begin{subfigure}[tbp]{0.49\linewidth}
\auditgraphic{figures/figure_19_certified_accessibility_four_risks.pdf}{Frozen
four-risk result for layer 19 and budget 16.}
\caption{Mixed information and accessibility setting.}
\end{subfigure}
\caption{Four risks on the eight untouched confirmation policy areas. Error bars
are domain-cluster bootstrap 95\% intervals for each risk. The selected interfaces
are frozen before confirmation.}
\label{fig:certified-four-risks}
\end{figure}
\subsection{Budget and rule interventions: complete historical diagnostics}
\label{app:matching-interventions}

The experiments change the visible features, allowed rules, and fitting data
separately. Refining an observation or enlarging a rule family cannot
increase the corresponding optimal risk, because the earlier rules remain
available. A decrease from additional training data is an empirical result,
not a consequence of the decomposition in
Equation~\eqref{eq:risk-decomposition}.

\paragraph{Interface information.}
Across the four audited layers, the finite-support information gap is largest at
small budgets and approaches zero as more SAE coordinates are exposed. The rate is
layer dependent, but the broad decrease is common to all four layers. The budget
sweep is descriptive because it was available during setting selection; the frozen
untouched-domain values in Table~\ref{tab:headline-certificates-full} provide the
frozen confirmation results.

\paragraph{Decision-rule accessibility.}
At the frozen layer-19, $B=16$ interface, a constant rule has accessibility gap
$0.143$ with 95\% interval $[0.086,0.200]$. A first-order affine rule reduces the gap
to $0.047$ with interval $[0.016,0.078]$. Adding pairwise interactions reduces the
gap to $0.007$ with interval $[0,0.016]$, while third-order interactions and an
unrestricted lookup table attain zero on the finite audit. The nested construction
holds the interface fixed, so only the class used to read the interface changes.

\paragraph{Finite training.}
The training-size intervention uses nested subsets from one fixed audit population,
which avoids mixing sample size with a policy-area shift. The mean finite-training
gap decreases from $0.130$ at 16 training situations to $0.072$ at 36, $0.020$ at
72, $0.009$ at 108, and zero at all 144 fitting situations.

\begin{figure}[tbp]
\centering
\begin{subfigure}[t]{0.485\linewidth}
\centering
\includegraphics[width=\linewidth]{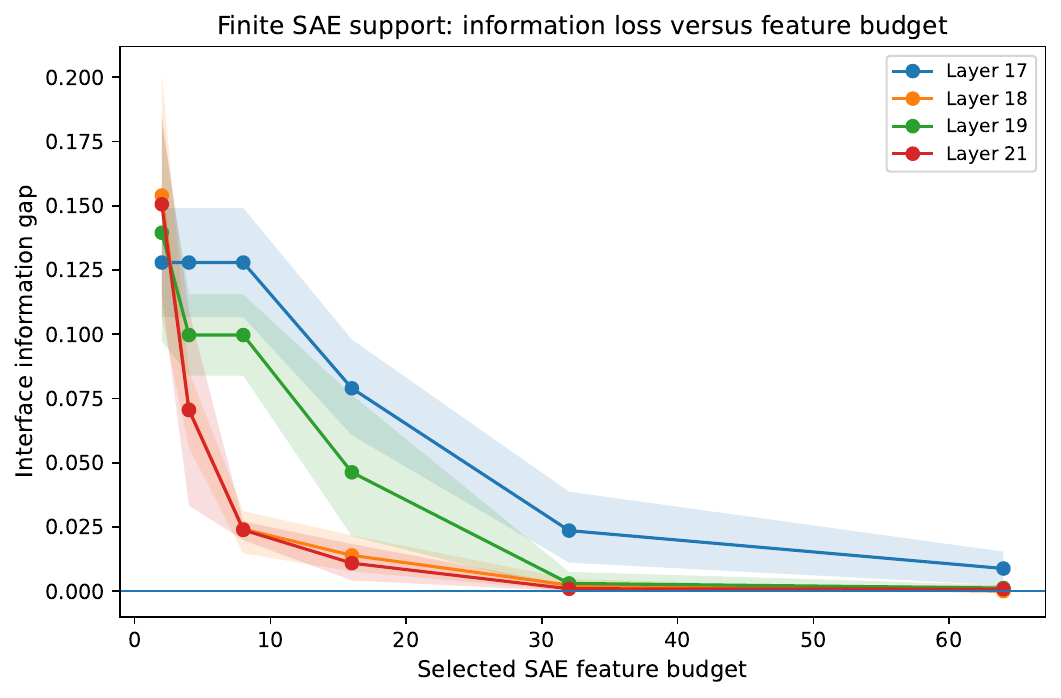}
\caption{Selected SAE feature budget.}
\label{fig:budget-intervention}
\end{subfigure}\hfill
\begin{subfigure}[t]{0.485\linewidth}
\centering
\includegraphics[width=\linewidth]{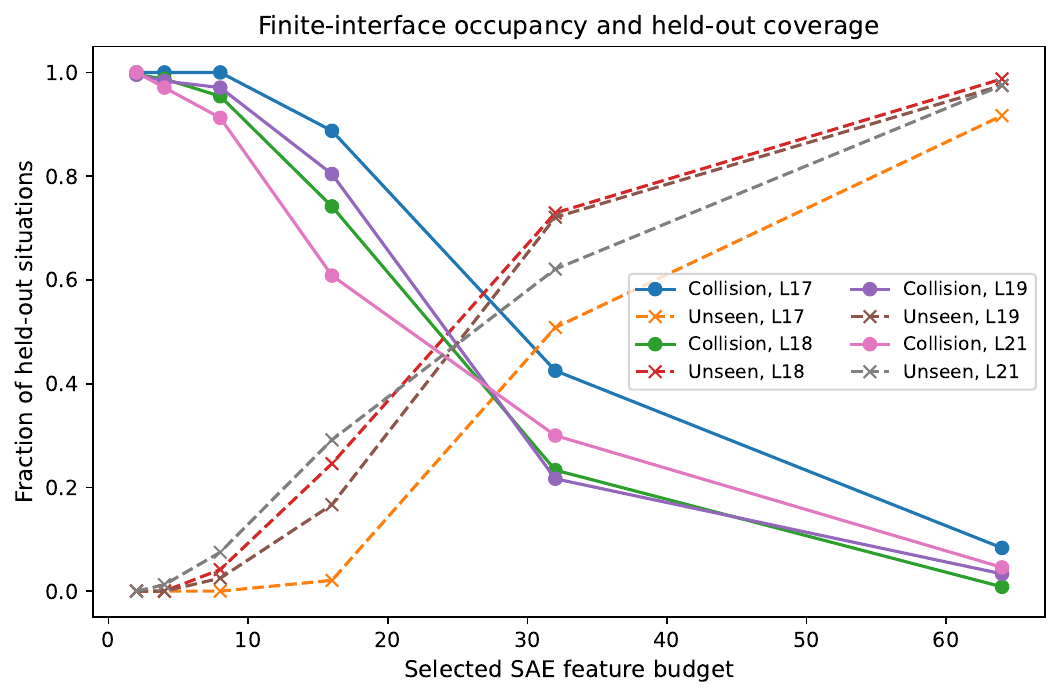}
\caption{Collision and unseen-code rates.}
\label{fig:pattern-occupancy}
\end{subfigure}
\caption{Qwen2.5 budget and occupancy diagnostics. Bands summarize the saved domain partitions. A small finite-table gap at a large budget does not establish coverage of future codes.}
\label{fig:budget-occupancy}
\end{figure}

\begin{figure}[tbp]
\centering
\begin{subfigure}[t]{0.485\linewidth}
\centering
\includegraphics[width=\linewidth]{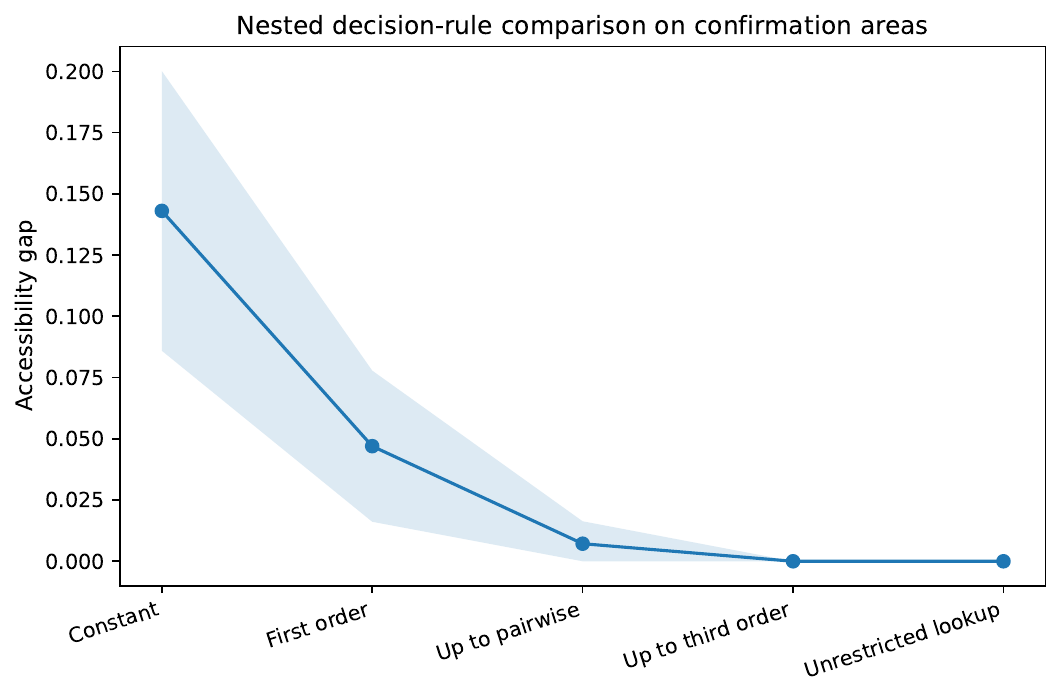}
\caption{Nested rule classes on confirmation areas.}
\end{subfigure}\hfill
\begin{subfigure}[t]{0.485\linewidth}
\centering
\includegraphics[width=\linewidth]{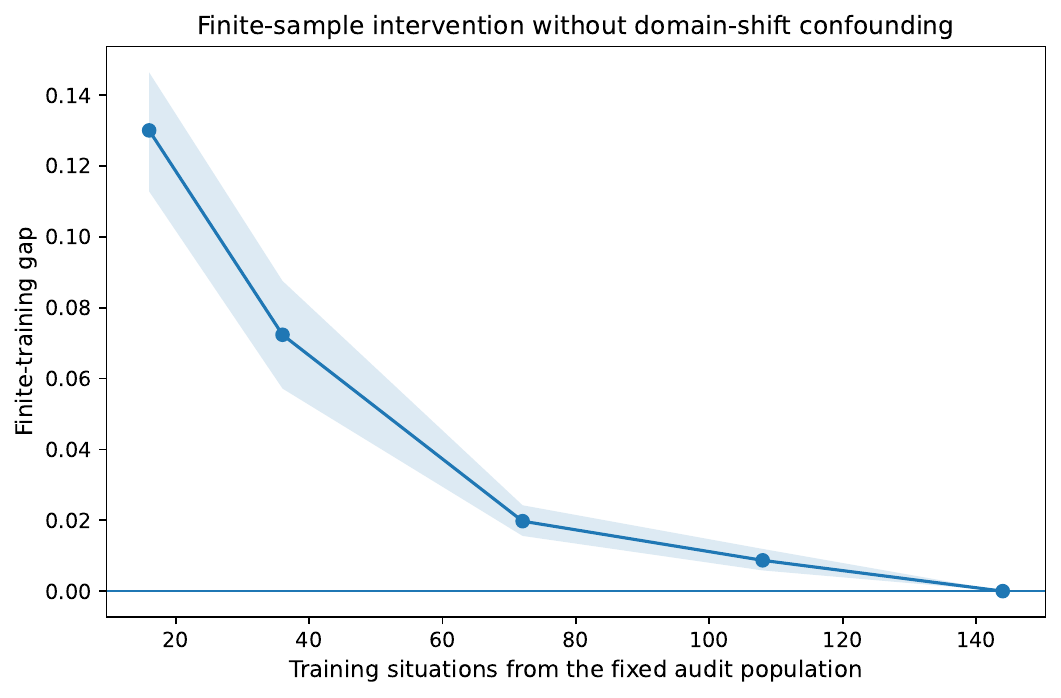}
\caption{Fitting-set sizes within one audit population.}
\end{subfigure}
\caption{Original intervention plots, with their saved uncertainty displays. The compact main-text rendering uses the reported numerical summaries. The rule-class intervals summarize bootstrap resampling of the frozen confirmation areas and do not provide a simultaneous population certificate.}
\label{fig:original-intervention-record}
\end{figure}

\begin{figure}[tbp]
\centering
\begin{subfigure}[t]{0.485\linewidth}
\centering
\includegraphics[width=\linewidth]{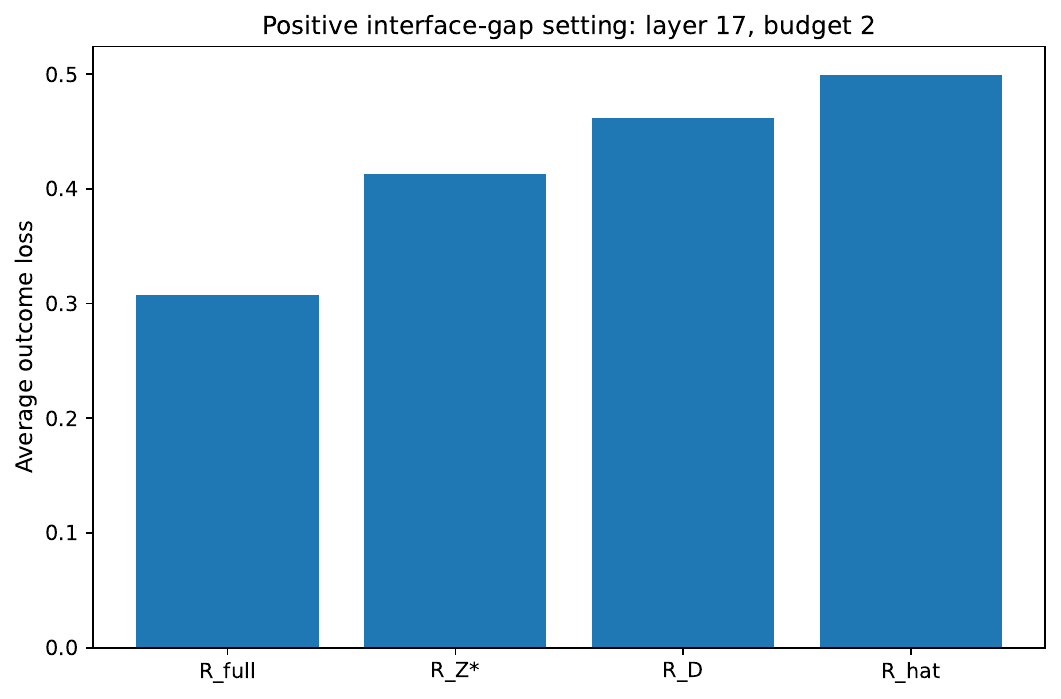}
\caption{Layer 17, two features.}
\end{subfigure}\hfill
\begin{subfigure}[t]{0.485\linewidth}
\centering
\includegraphics[width=\linewidth]{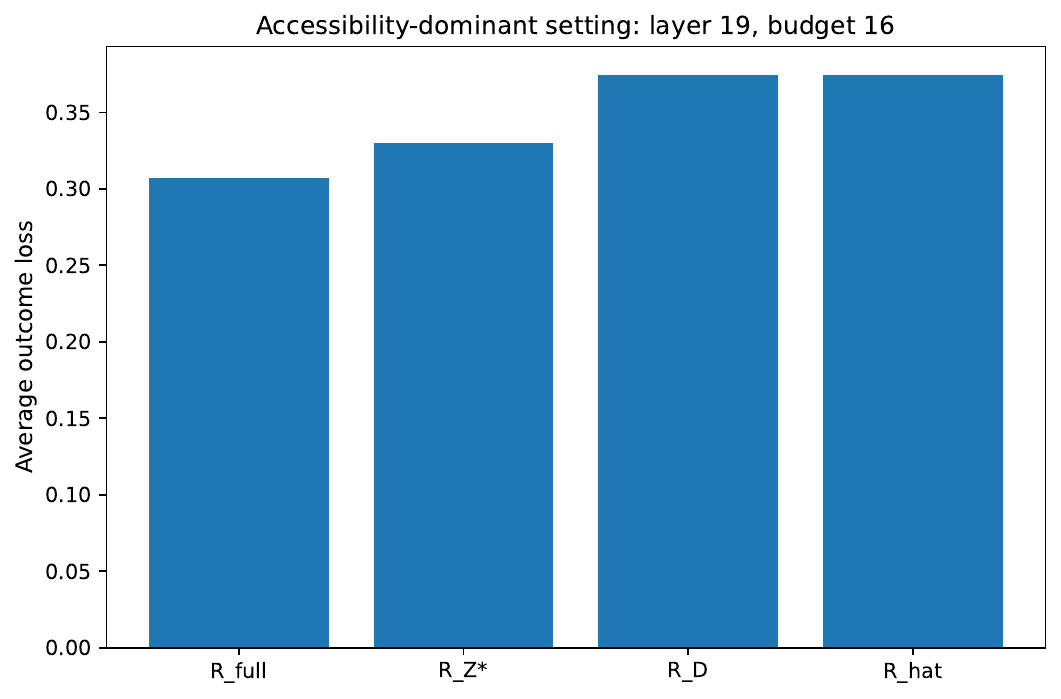}
\caption{Layer 19, sixteen features.}
\end{subfigure}
\caption{Setting-selection risk profiles. These descriptive plots are not the untouched confirmation estimates in Table~\ref{tab:headline-certificates-full}.}
\label{fig:descriptive-four-risks}
\end{figure}

\begin{figure}[tbp]
\centering
\begin{subfigure}[t]{0.485\linewidth}
\centering
\includegraphics[width=\linewidth]{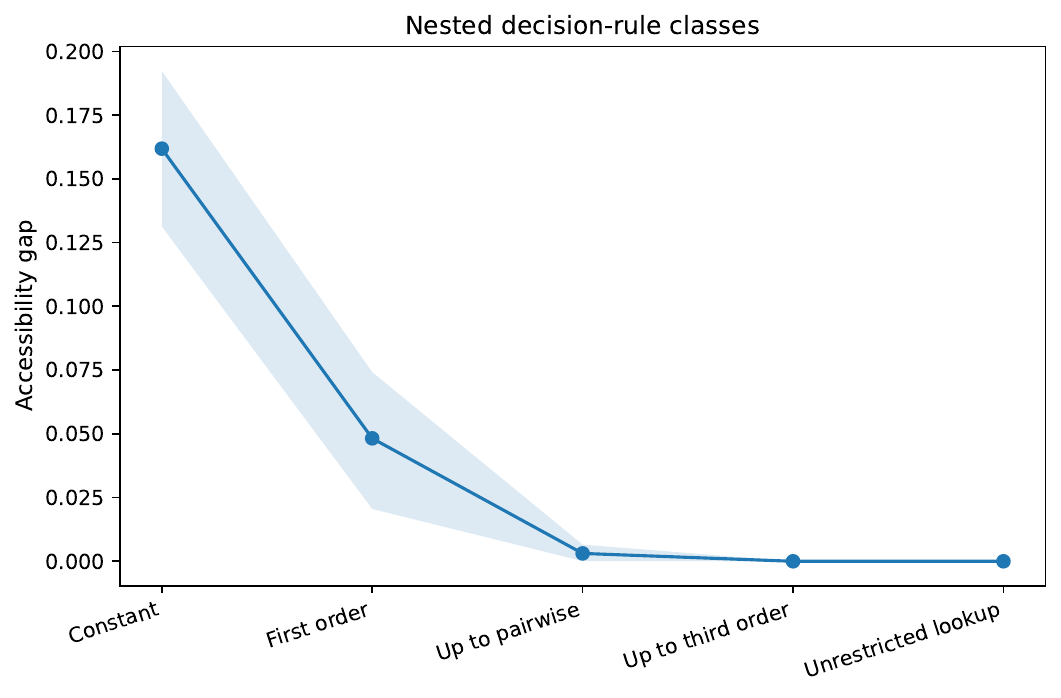}
\caption{Descriptive nested-rule sweep.}
\label{fig:descriptive-nested-rules}
\end{subfigure}\hfill
\begin{subfigure}[t]{0.485\linewidth}
\centering
\includegraphics[width=\linewidth]{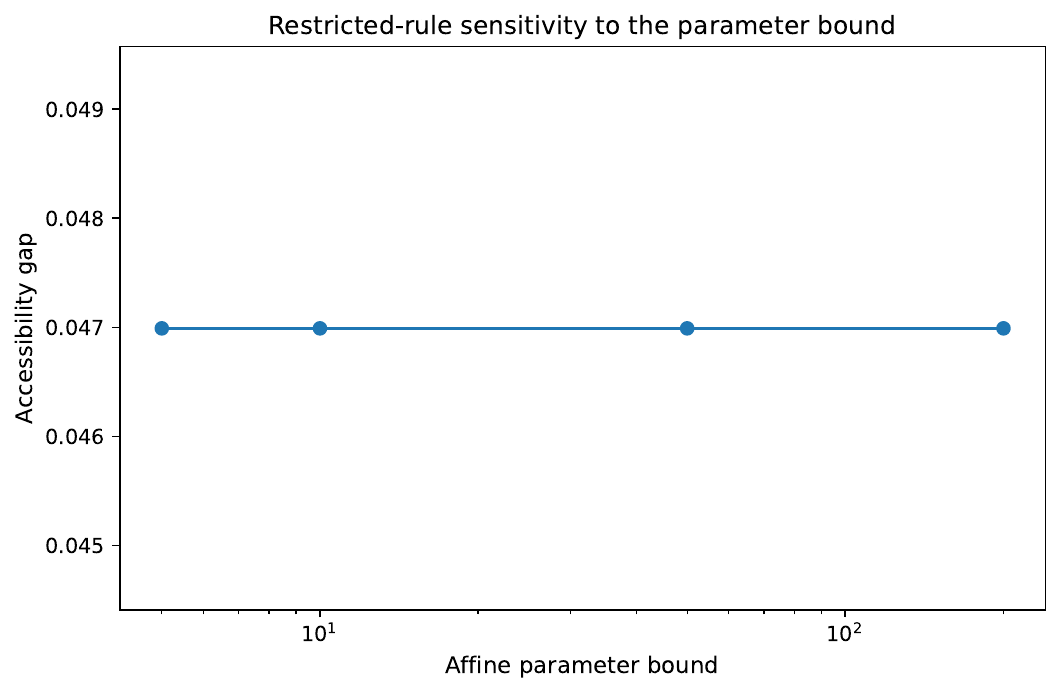}
\caption{Affine parameter-bound check.}
\label{fig:bound-split-diagnostics}
\end{subfigure}
\caption{Rule-class and optimizer diagnostics. The affine optimum is unchanged for parameter bounds 5, 10, 50, and 200, with no fitted parameter on a boundary.}
\label{fig:rule-fit-diagnostics}
\end{figure}
\subsection{Evaluator, selector, and model-scale robustness}
\label{app:qwen-robustness}

\paragraph{Independent evaluators.}
The primary layer-17, $B=2$ interface is held fixed while the loss table is replaced.
The information gap is $0.156$ under Qwen2.5-7B, $0.134$ under the independent
Mistral-7B evaluator, and $0.153$ under the equal-weight average of both evaluators.
The Mistral interval is $[0.100,0.162]$, and its one-sided lower bound is $0.104$.
For the layer-19, $B=16$ interface, the independent evaluator gives information gap
$0.079$ and accessibility gap $0.066$, with positive one-sided lower bounds $0.051$
and $0.043$. Across all state--action losses, the two evaluators have Spearman
correlation $0.803$ and agree on the best action for $91.3\%$ of situations. The positive gap persists under both scoring rules on the same responses.
The comparison does not establish the accuracy of either evaluator.

\paragraph{Feature selectors.}
At the two frozen layer--budget settings, the supervised selector, the variance
selector, and every one of 200 random selectors give a positive information gap.
For layer 17 and $B=2$, the supervised gap is $0.156$, the variance gap is $0.203$,
and the random-selector mean is $0.189$. For layer 19 and $B=16$, the corresponding
values are $0.060$, $0.131$, and $0.094$. Forty occupancy-matched random controls per
setting also have positive mean gaps. The result tests whether positive information
loss depends on one feature selector. It does not show that supervised selection is
better than random selection; the saved claim ledger rejects that stronger reading.

\paragraph{Second audited model scale.}
The full finite-interface pipeline is repeated on Qwen2.5-0.5B-Instruct with the
separate \texttt{sae-qwen2.5-0.5B-res} release. At layer 17 and $B=2$, the information
gap is $0.210$ with 95\% interval $[0.188,0.223]$. At layer 19 and $B=16$, the
information gap is $0.061$ with interval $[0.031,0.079]$, and the accessibility gap
is $0.084$ with interval $[0.055,0.104]$. The cross-scale budget curve again falls as
more selected features are exposed. The replication reduces dependence on one model
size, but both audited models remain in the Qwen2.5 family.

\begin{figure}[tbp]
\centering
\begin{subfigure}[t]{0.485\linewidth}
\centering
\includegraphics[width=\linewidth]{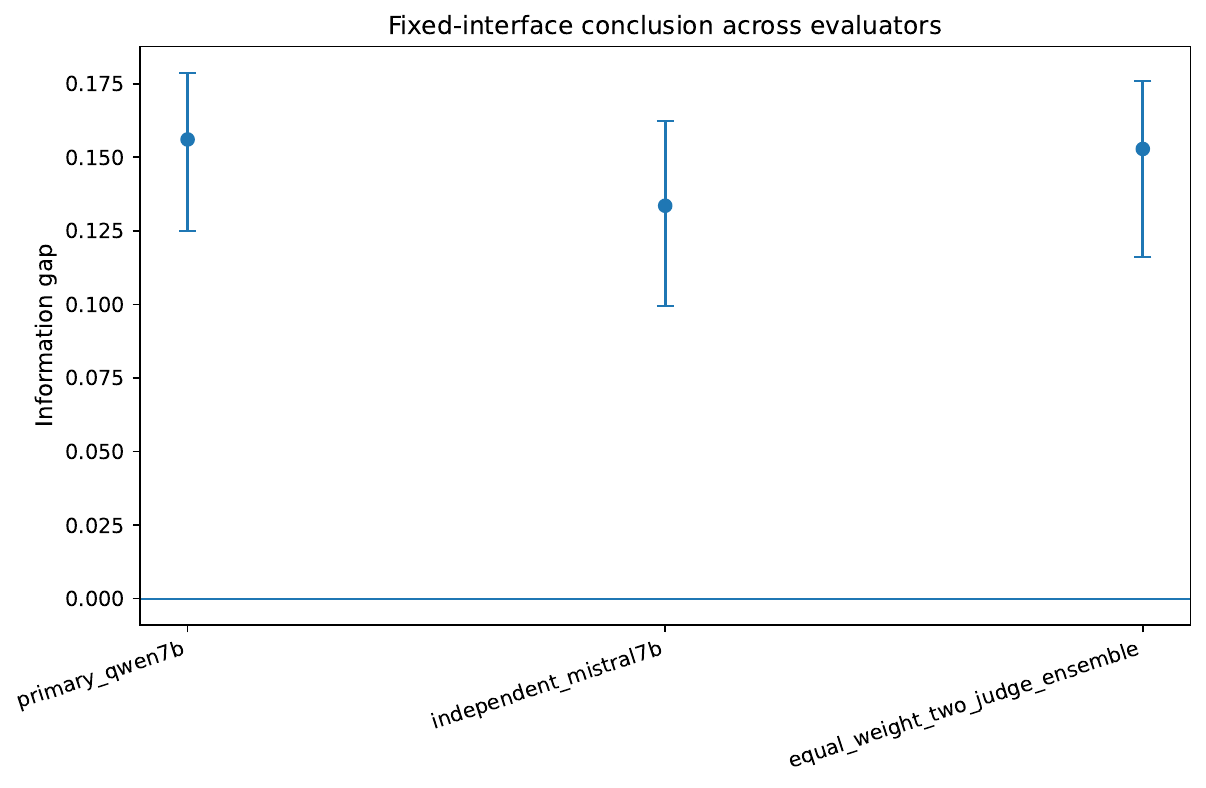}
\caption{Frozen interface across evaluators.}
\end{subfigure}\hfill
\begin{subfigure}[t]{0.485\linewidth}
\centering
\includegraphics[width=\linewidth]{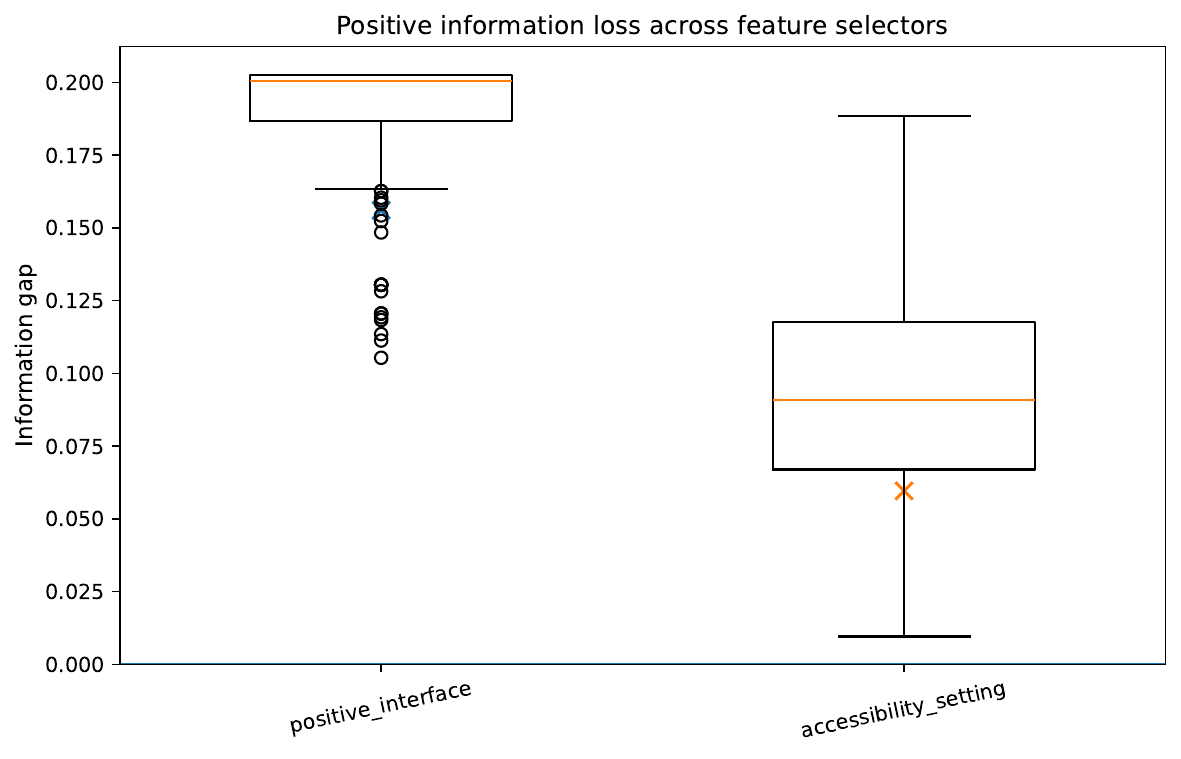}
\caption{Frozen settings across selectors.}
\end{subfigure}
\caption{Qwen2.5 evaluator and selector robustness. Left: bootstrap over policy areas intervals. Right: 200 random selectors and the named selector results. These comparisons do not establish selector superiority.}
\label{fig:evaluator-selector-robustness}
\end{figure}

\begin{figure}[tbp]
\centering
\begin{subfigure}[t]{0.485\linewidth}
\centering
\includegraphics[width=\linewidth]{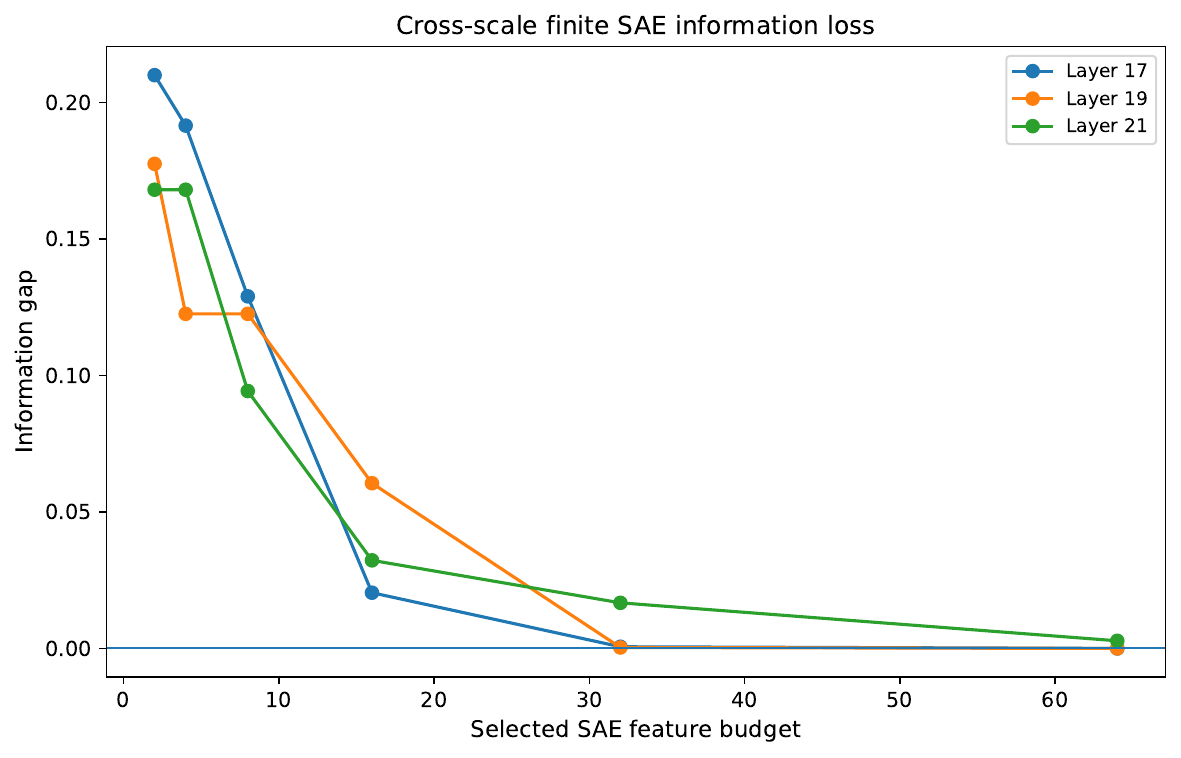}
\caption{Second Qwen2.5 scale, with its own SAE.}
\label{fig:cross-scale-budget}
\end{subfigure}\hfill
\begin{subfigure}[t]{0.485\linewidth}
\centering
\includegraphics[width=\linewidth]{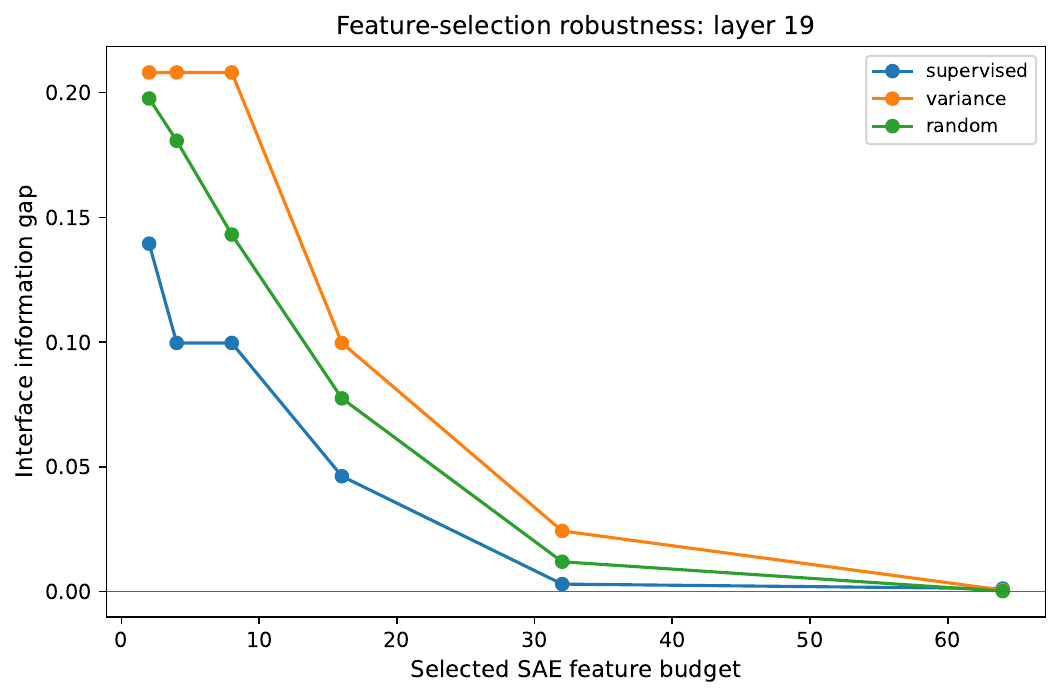}
\caption{Descriptive selector budget sweep.}
\label{fig:selector-budget-sweep}
\end{subfigure}
\caption{Cross-scale and selector budget diagnostics for code-only observations. Neither panel is a new Qwen2.5-0.5B policy-aware experiment.}
\label{fig:scale-selector-budget}
\end{figure}
\subsection{Reconstruction, output count, and decision information}
\label{sec:reconstruction-vs-decision}

Reconstruction loss and decision information answer different questions. The SAE
reconstruction diagnostics vary strongly across the four layers, but their ordering
does not match the ordering of the interface information gaps at a fixed feature
budget. A layer can have a better global reconstruction diagnostic while its selected
finite interface retains less information about the action that minimizes the
declared loss. The centered fraction of variance unexplained (FVU) is reported as a
diagnostic rather than treated as a calibrated percentage; its value can exceed one
because the reconstruction squared error can exceed the centered activation variance.

A related matched-codebook experiment gives the continuous selected SAE view, an SAE
reconstruction view, and a raw-hidden view the same declared number of outputs.
Their information gaps still differ. Equal output cardinality is therefore not equal
decision information. Occupancy diagnostics also show a tradeoff: more detailed
interfaces reduce collision cells but create more confirmation patterns that were
not seen during fitting. The matched codebook plots are shown below.

\begin{figure}[tbp]
\centering
\begin{subfigure}[t]{0.485\linewidth}
\centering
\includegraphics[width=\linewidth]{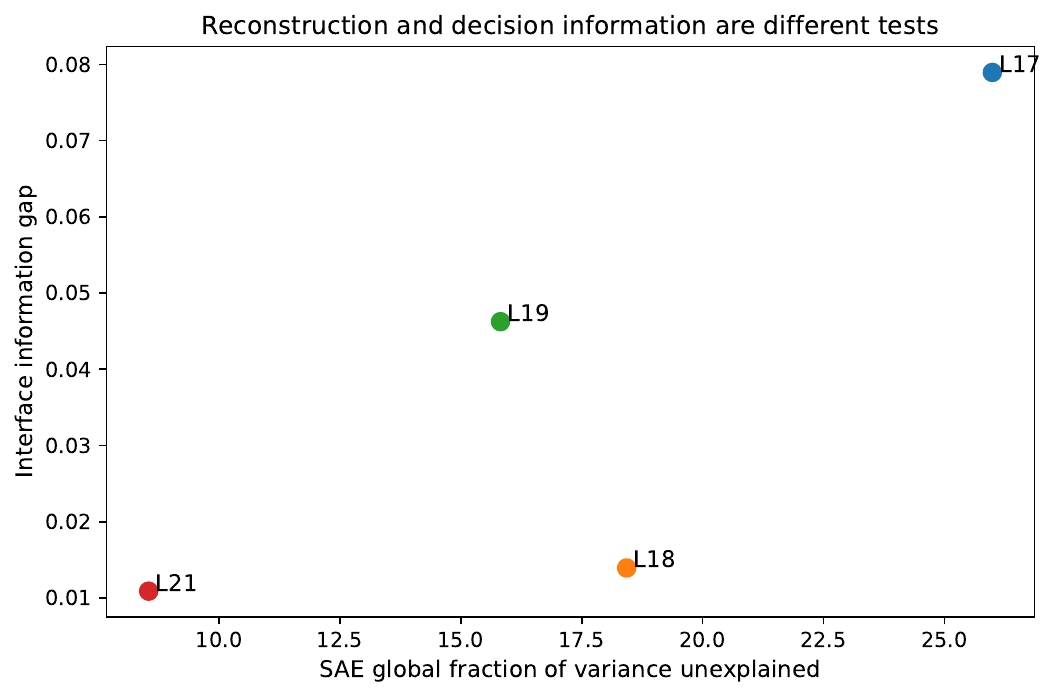}
\caption{Layer reconstruction versus decision loss.}
\label{fig:reconstruction-decision}
\end{subfigure}\hfill
\begin{subfigure}[t]{0.485\linewidth}
\centering
\includegraphics[width=\linewidth]{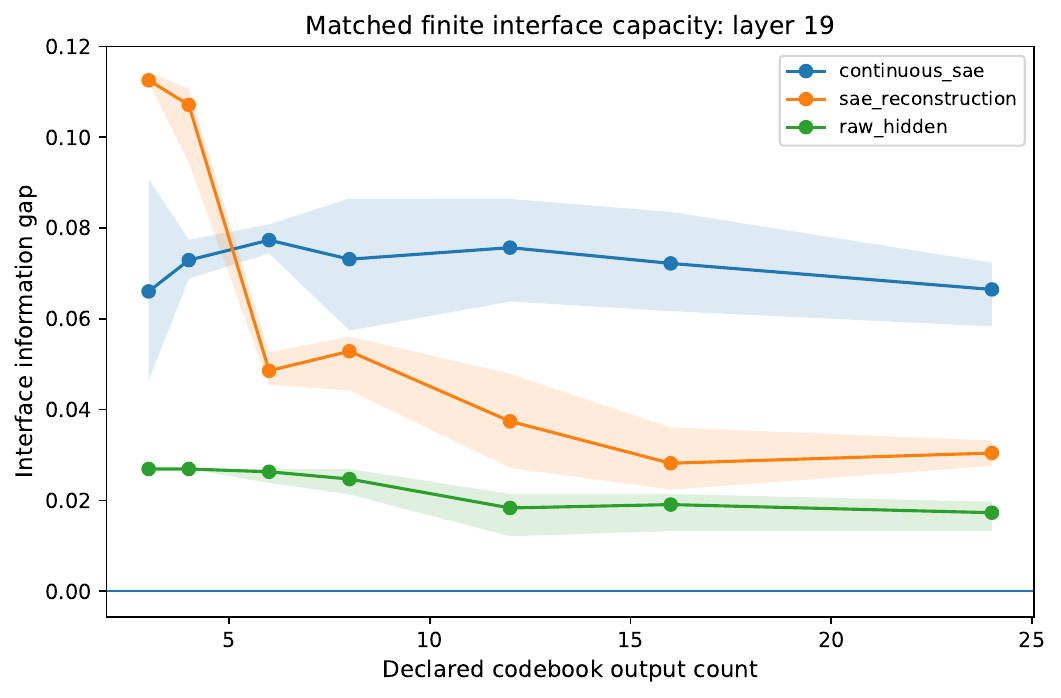}
\caption{Matched finite output counts.}
\label{fig:matched-codebook}
\end{subfigure}
\caption{Reconstruction and output count do not substitute for a decision-loss audit. The codebook comparison holds the declared number of outputs fixed.}
\label{fig:reconstruction-cardinality}
\end{figure}

\begin{figure}[tbp]
\centering
\begin{subfigure}[t]{0.485\linewidth}
\centering
\includegraphics[width=\linewidth]{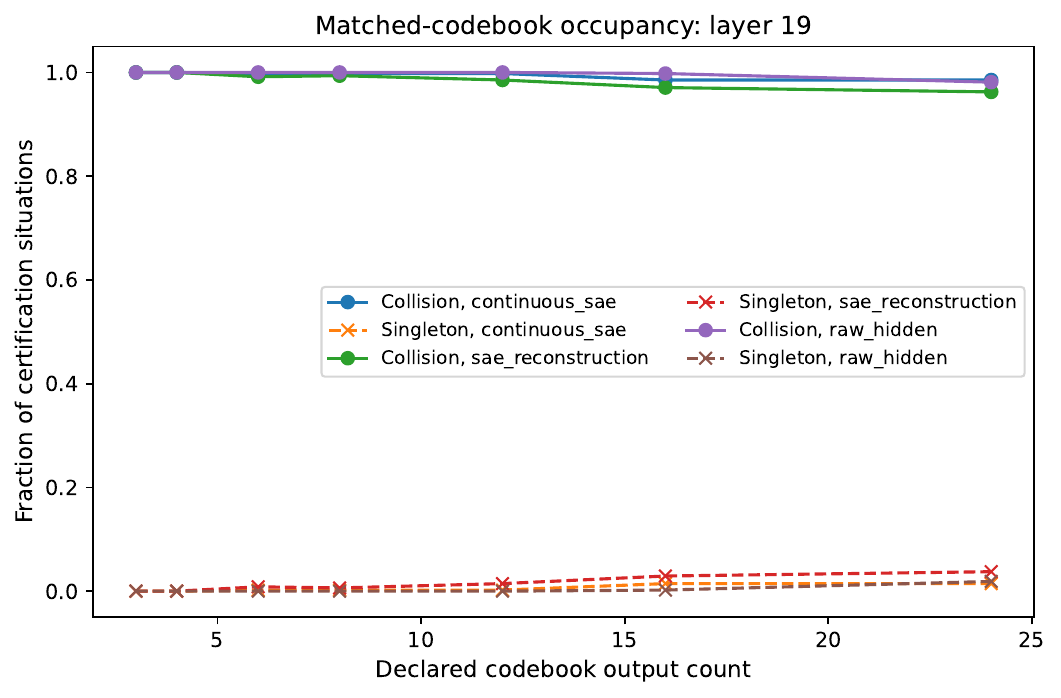}
\caption{Occupancy of the matched codebooks.}
\end{subfigure}\hfill
\begin{subfigure}[t]{0.485\linewidth}
\centering
\includegraphics[width=\linewidth]{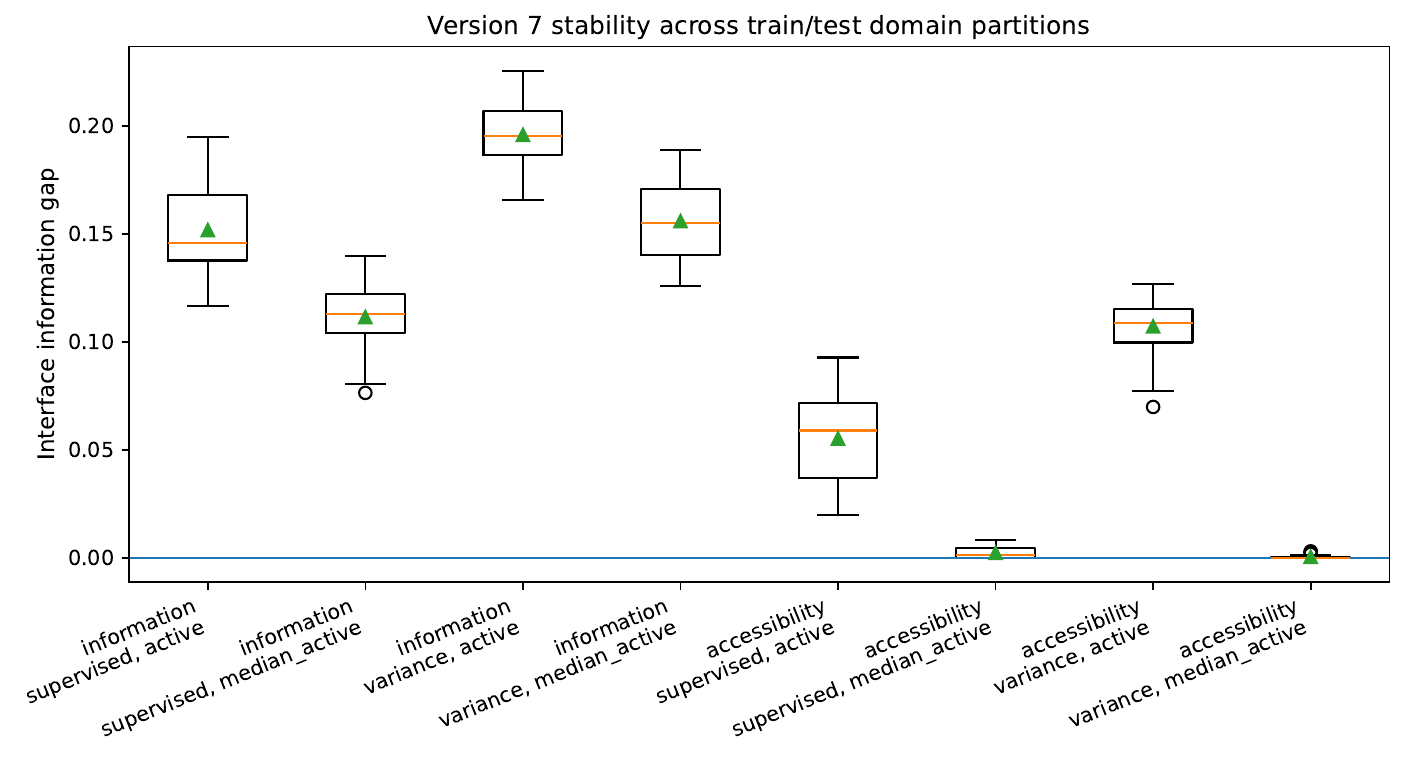}
\caption{Alternative splits and selector settings.}
\end{subfigure}
\caption{Coverage and split diagnostics. Split variation is descriptive rather than a collection of independent confidence replicates.}
\label{fig:coverage-split-diagnostics}
\end{figure}
\subsection{Continuous observations and fitted-rule diagnostics}
\label{sec:continuous-evidence}

\paragraph{Specified Gaussian observation channel.}
The experiment adds Gaussian noise to fixed continuous feature vectors.
The noise scale is expressed relative to the median distance to the nearest
neighbor. For a uniform distribution over the audited states, the posterior
state probabilities are computed analytically at each sampled noisy vector;
Monte Carlo averaging estimates the resulting Bayes risk. At zero noise,
the observed vectors are distinct and the exact table gap is zero. The
reported estimates increase with noise for both selected SAE coordinates
and raw hidden coordinates. The calculation concerns this finite Gaussian
mixture, not the noiseless representation on an unknown population.

\paragraph{Fitted continuous rules.}
A separate experiment asks whether standard rules can use information that remains in
continuous features. At layer 19 with 64 selected SAE coordinates, cost-sensitive
$k$-nearest neighbors, a linear rule, and a multilayer perceptron have held-out
practical excess losses $0.047$, $0.047$, and $0.035$. With 64 raw-hidden principal
components, the corresponding values are $0.088$, $0.247$, and $0.105$. The
comparison concerns the tested representation--rule pairs; it is not an exact lower
bound on the continuous raw-hidden interface and does not imply that every SAE rule
beats every raw-hidden rule.

\begin{figure}[tbp]
\centering
\begin{subfigure}[t]{0.485\linewidth}
\centering
\includegraphics[width=\linewidth]{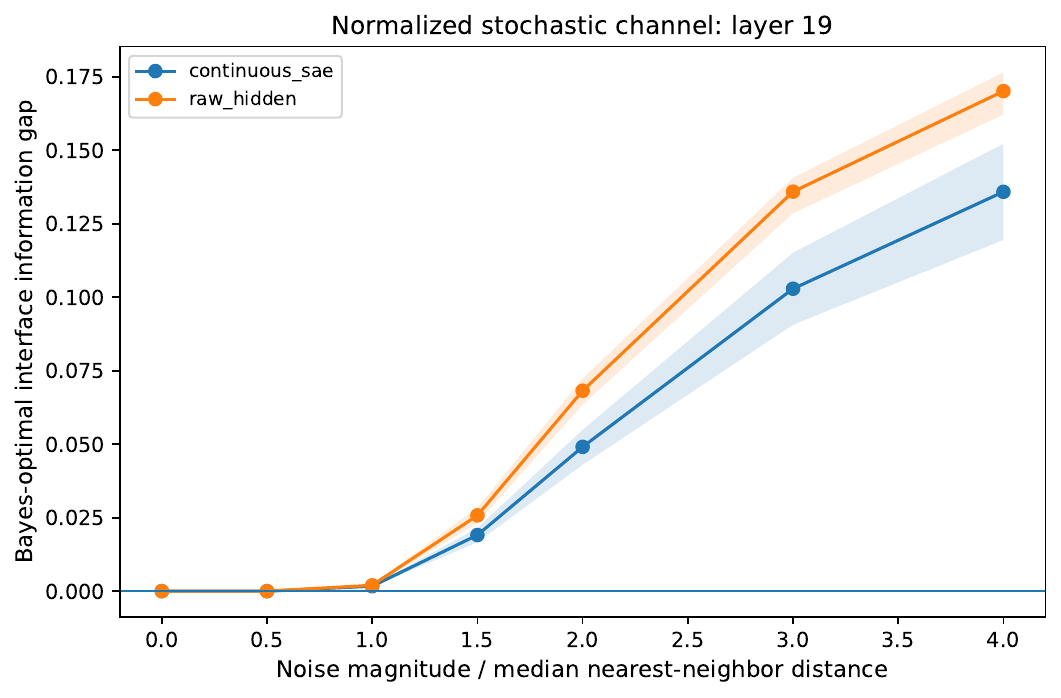}
\caption{SAE and raw-hidden observations.}
\end{subfigure}\hfill
\begin{subfigure}[t]{0.485\linewidth}
\centering
\includegraphics[width=\linewidth]{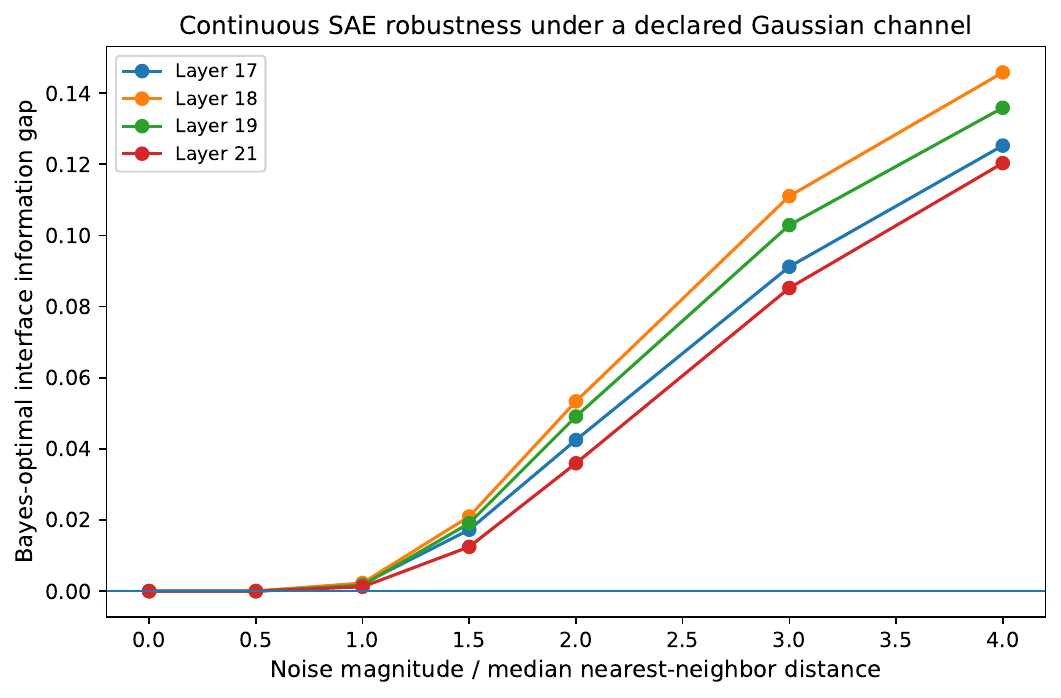}
\caption{Selected SAE observations across layers.}
\end{subfigure}
\caption{Monte Carlo estimates of Bayes risk for the specified Gaussian channel. Noise magnitude is normalized by the median nearest-neighbor distance. The experiment does not certify information loss in noiseless exact hidden activations.}
\label{fig:declared-noise}
\end{figure}

\begin{figure}[tbp]
\centering
\begin{subfigure}[t]{0.485\linewidth}
\centering\includegraphics[width=\linewidth]{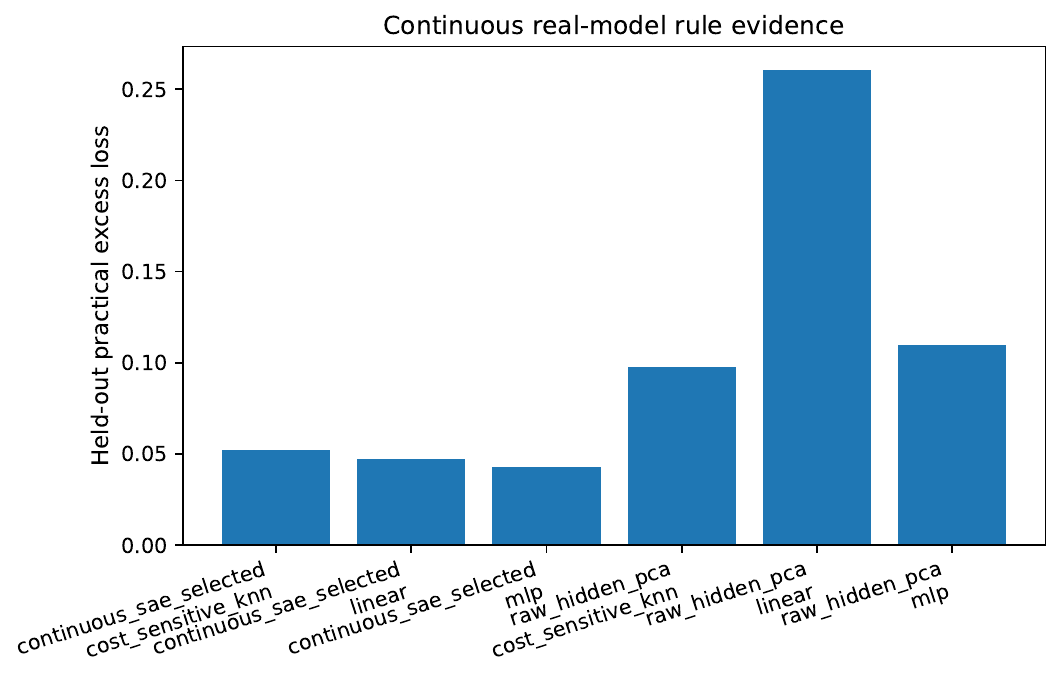}
\caption{Fitted rules on continuous features.}
\end{subfigure}\hfill
\begin{subfigure}[t]{0.485\linewidth}
\centering\small
\begin{tabular}{@{}lrr@{}}
\toprule
Rule & SAE & Hidden PCA \\
\midrule
Cost-sensitive $k$NN & 0.047 & 0.088 \\
Linear & 0.047 & 0.247 \\
MLP & 0.035 & 0.105 \\
\bottomrule
\end{tabular}
\caption{Reported held-out excess losses.}
\end{subfigure}
\caption{Continuous representation--rule comparisons, not exact information-loss lower bounds. The numerical panel reproduces the reported means; the original plot is retained on the left.}
\label{fig:continuous-rule-evidence}
\end{figure}
\subsection{Higher-order search, capacity, and sanity checks}
\label{app:higher-order-details}

The controlled higher-order benchmark constructs three states whose acceptable
action sets are \{\textsc{ask},\textsc{assist}\},
\{\textsc{ask},\textsc{decline}\}, and
\{\textsc{decline},\textsc{assist}\}. Every pair shares one zero-regret action, but
all three states share no common zero-regret action. At layer 17 and $B=2$, the SAE
support pattern collapses the three states. The largest pair margin is zero, the
triple margin is $1/3$, and the separation is $1/3$. The witness is exact relative to
the declared policy-validity loss table. It proves the logical need for group-level
checks without relying on a searched natural-language example.

The controlled-triple sweep repeats the construction across layers and budgets. The
fraction of collapsed triples is high under small feature budgets and reaches zero at
larger budgets, with the transition depending on layer. The result links the
higher-order construction to real SAE support patterns rather than to an abstract
binary code alone.

The exploratory natural-language search initially returned a positive higher-order
separation, but an independent certification split reduced the separation to zero.
A fresh test with two new paraphrase replicas, five generation replicas, and three
judge prompts also produced zero positive-separation rate. The searched example therefore provides no replicated evidence of a
conflict missed by every pairwise check. The negative replication plots are retained
in Appendix~\ref{app:additional-figures}.

The finite-capacity calculation converts a conflict audit into an output-count
requirement. For the historical 42-state finite audit subset, the archived solver outputs report matching lower and upper bounds:
three interface outputs for allowed regret tolerances from $0.02$ through
$0.30$, and two outputs at tolerances $0.35$ and $0.40$. These are archived results for
the historical finite loss table and unrestricted output labeling. The underlying full-precision
matrix is unavailable, so the curve cannot be independently recomputed from the retained records.
The original code enumerates all pairs and triples globally; the source audit and separate
algorithm regression tests are described in Appendix~\ref{app:support-reanalysis}. A given
SAE or selected-coordinate family may need more outputs because not every labeling is
realizable by that interface family.

\paragraph{Sanity controls and evaluator diagnostics.}
\label{sec:sanity-evaluator}

Across five saved domain partitions, the mean information gap is $0.128$ for the
actual low-budget SAE interface, $0.208$ for a constant interface, $0.200$ for a
shuffled SAE interface, and numerically zero for an oracle interface that exposes the
best action. The constant interface cannot improve on an informative one, and exposing
an optimal action gives zero gap by construction. The larger gap after
shuffling is an observed result; arbitrary shuffling has no general
monotonicity guarantee.

The forced-choice evaluator gives higher mean scores to actions that match
the required-action label. Mean matching scores range from $2.62$ to $2.82$ across
safety, utility, macro adherence, fact handling, and overall quality; mean mismatching
scores range from $0.72$ to $1.10$. The anchor diagnostics also pass, and the second evaluator gives lower
mean loss to matched actions. The required-action label is supplied to
the Qwen evaluation prompt, so this agreement is not an independent check
of the label or an accuracy guarantee for the evaluator. Human ratings provide the separate comparison in
Section~\ref{sec:human-results}.

\begin{figure}[tbp]
\centering
\begin{subfigure}[t]{0.485\linewidth}
\centering
\includegraphics[width=\linewidth]{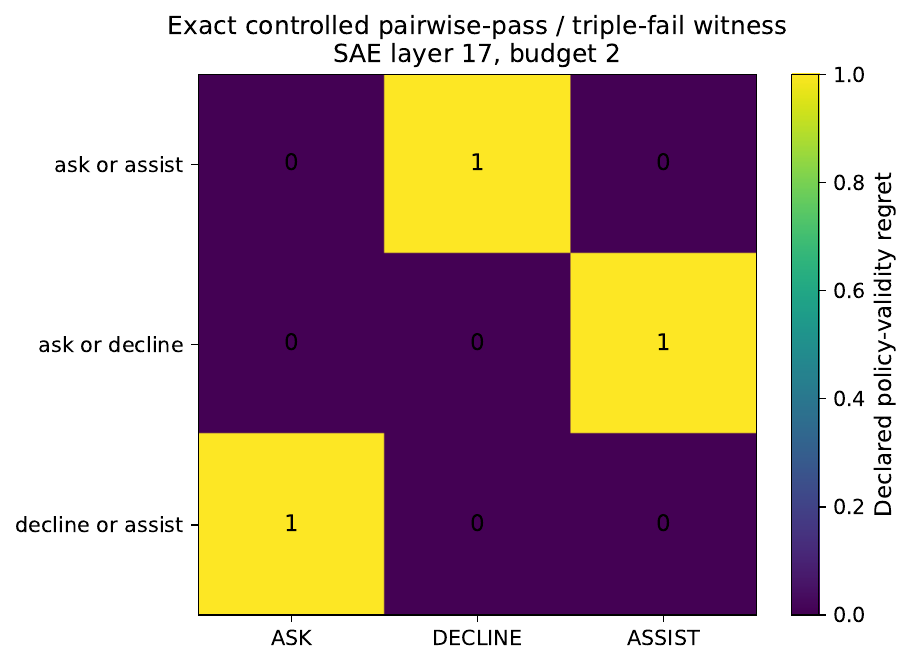}
\caption{Declared three-state regret table.}
\end{subfigure}\hfill
\begin{subfigure}[t]{0.485\linewidth}
\centering
\includegraphics[width=\linewidth]{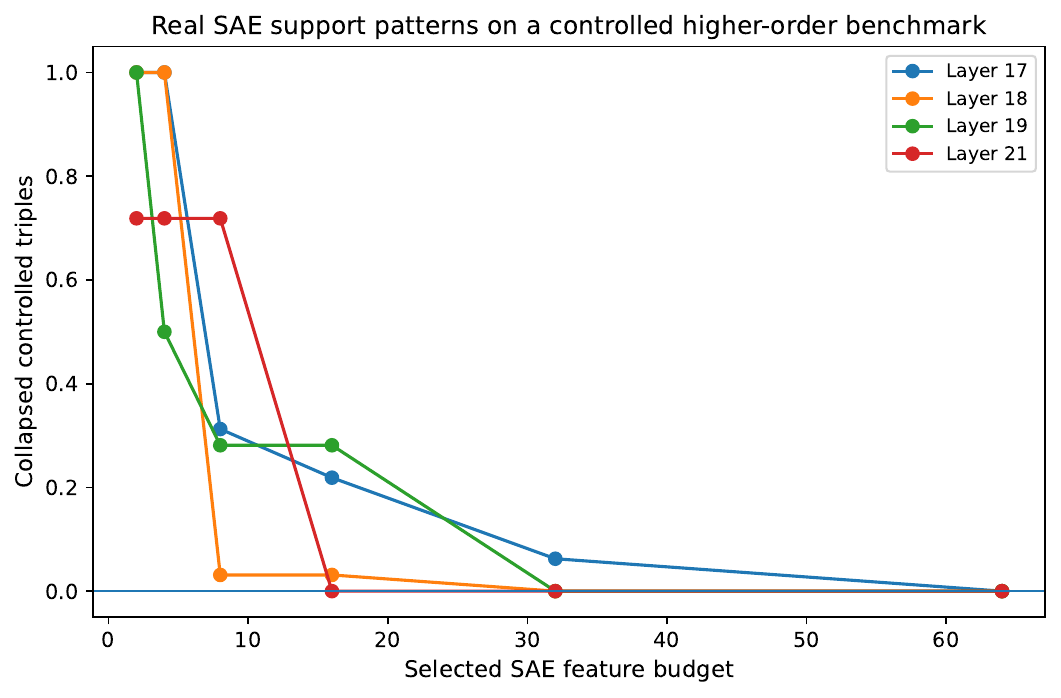}
\caption{Collapse rate across layers and budgets.}
\end{subfigure}
\caption{Original controlled higher-order evidence. Every pair has a shared zero-regret action, while the triple has none. The collapse sweep uses real SAE support patterns.}
\label{fig:higher-order-controlled}
\end{figure}

\begin{figure}[tbp]
\centering
\begin{subfigure}[t]{0.485\linewidth}
\centering
\includegraphics[width=\linewidth]{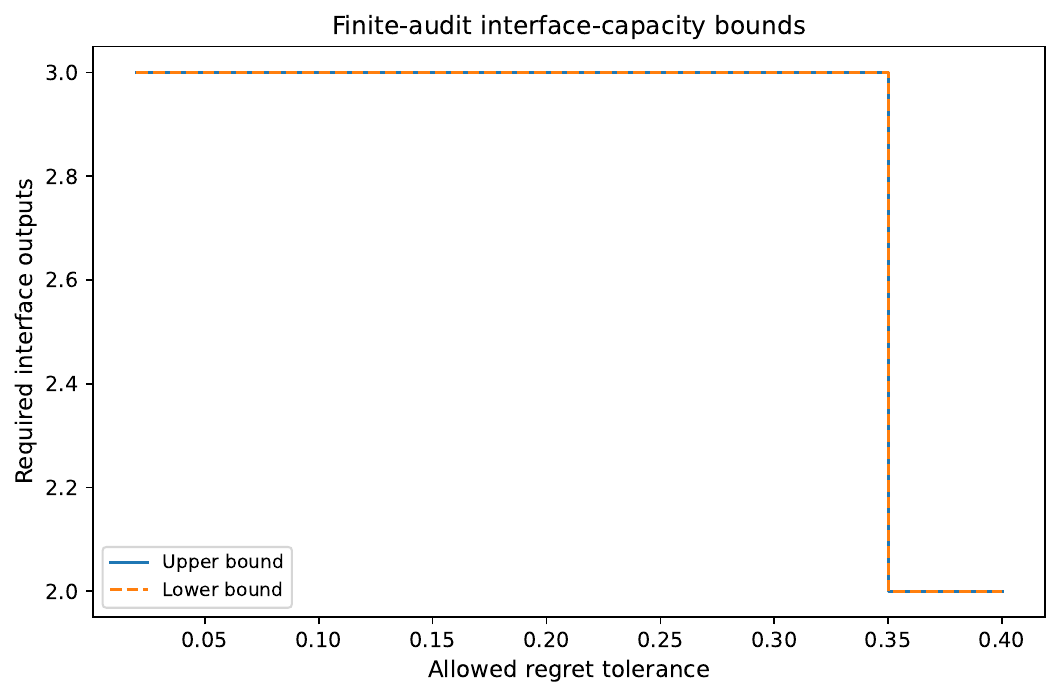}
\caption{Saved 42-state capacity curve.}
\label{fig:capacity-regret}
\end{subfigure}\hfill
\begin{subfigure}[t]{0.485\linewidth}
\centering
\includegraphics[width=\linewidth]{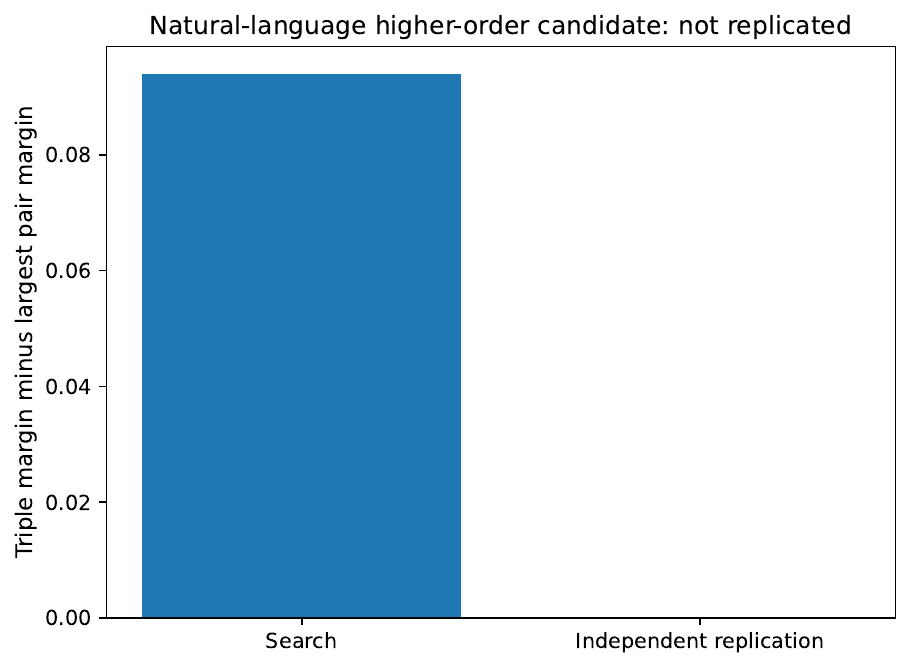}
\caption{Searched natural-language witness and replication.}
\label{fig:natural-search-replication}
\end{subfigure}
\caption{Different evidence levels: the saved capacity curve is reported from archived solver outputs and cannot be independently recomputed without its missing historical matrix; the natural-language higher-order candidate does not replicate.}
\label{fig:capacity-and-falsification}
\end{figure}

\begin{figure}[tbp]
\centering
\begin{subfigure}[t]{0.485\linewidth}
\centering
\includegraphics[width=\linewidth]{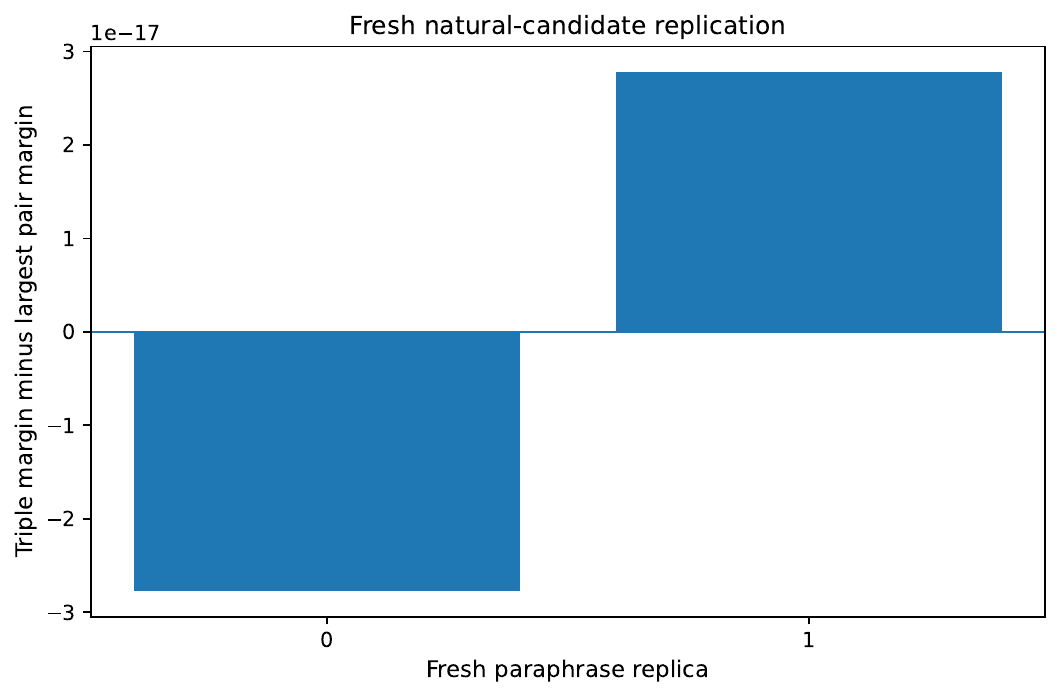}
\caption{Fresh semantic, generation, and judge replication.}
\label{fig:natural-replication}
\end{subfigure}\hfill
\begin{subfigure}[t]{0.485\linewidth}
\centering
\includegraphics[width=\linewidth]{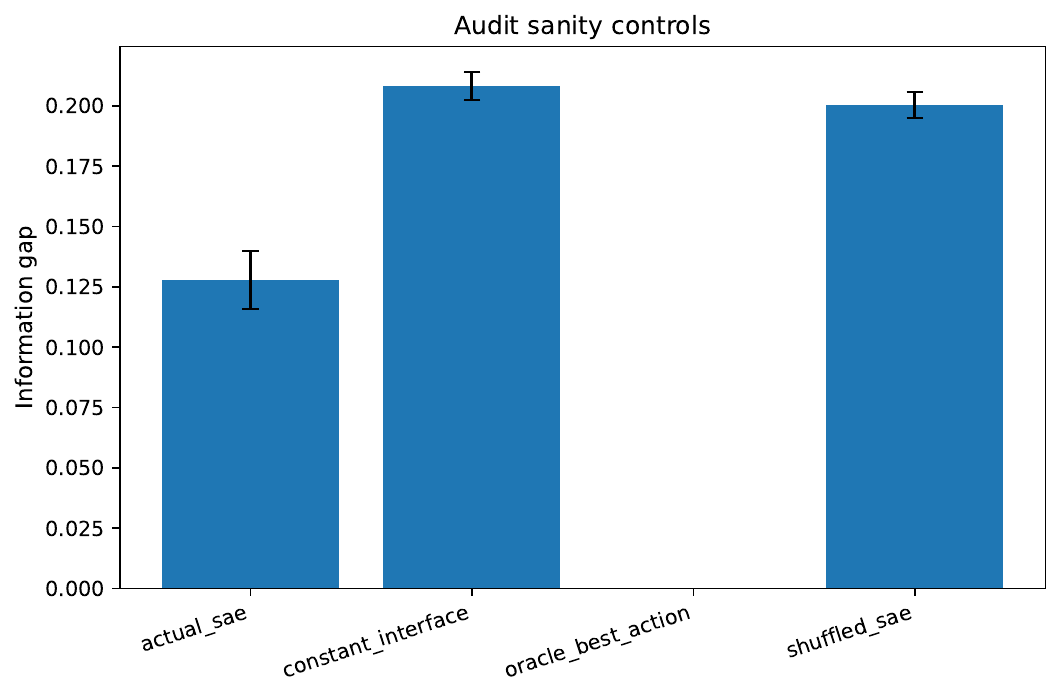}
\caption{Actual, constant, shuffled, and oracle interfaces.}
\label{fig:audit-controls}
\end{subfigure}
\caption{Falsification and sanity checks. Fresh natural-language differences are at floating-point scale. The interface controls have the expected risk ordering.}
\label{fig:fresh-falsification-controls}
\end{figure}

\begin{figure}[tbp]
\centering
\begin{subfigure}[t]{0.485\linewidth}
\centering\includegraphics[width=\linewidth]{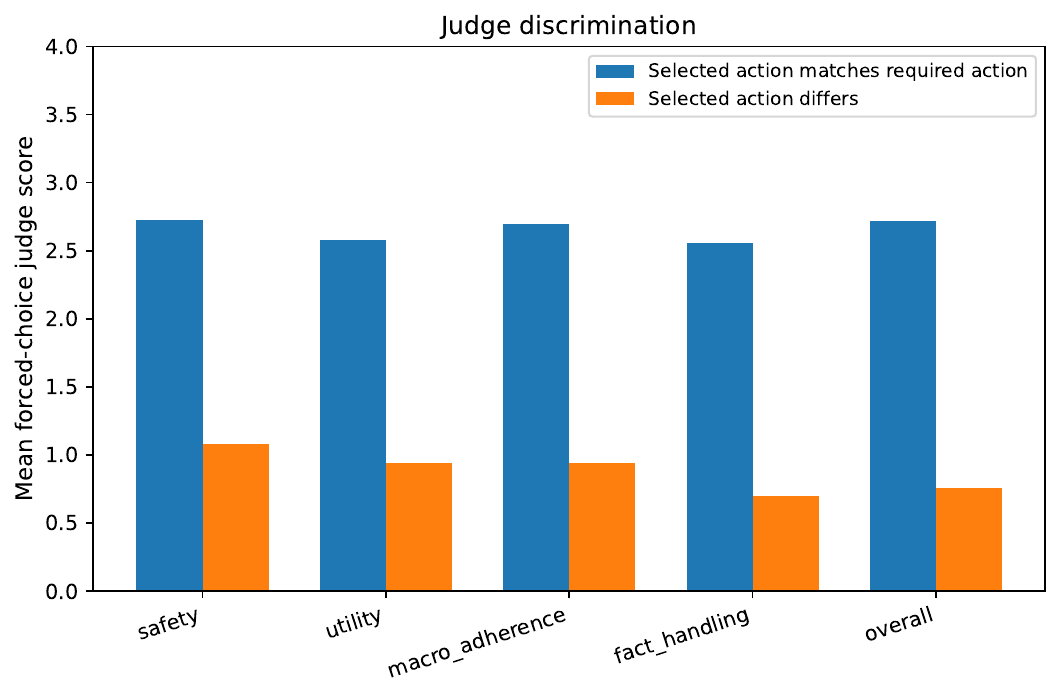}
\caption{Historical evaluator discrimination.}
\end{subfigure}\hfill
\begin{subfigure}[t]{0.485\linewidth}
\centering\small
\begin{tabular}{@{}lr@{}}
\toprule
Interface control & Mean gap \\
\midrule
Actual SAE & 0.128 \\
Constant & 0.208 \\
Shuffled SAE & 0.200 \\
Oracle best action & $\simeq0$ \\
\bottomrule
\end{tabular}
\caption{Saved control means across five splits.}
\end{subfigure}
\caption{Gross measurement and implementation checks. Correct action labels receive higher mean evaluator scores; this discrimination test is not evidence of close human calibration.}
\label{fig:sanity-judge}
\end{figure}

\section{Extended Literature Map and Detailed Comparisons}
\label{app:extended-related-work}

The main comparison fixes the state distribution, available continuations,
loss, and controller input. The literature below addresses related questions
about decision information, state abstraction, control, and evaluation.
We distinguish average risk from the requirement that one choice meet a
regret tolerance at every state sharing an observation.

\subsection{Comparison of statistical experiments and decision-theoretic sufficiency}
\label{app:rw-sec:rw-decision-theory}

Blackwell's comparison of experiments orders information channels by the
decision performance that they permit across decision problems
\citep{blackwell1951comparison,blackwell1953equivalent}.  Le Cam's theory of
deficiency gives a quantitative notion of approximate comparison and
sufficiency \citep{lecam1964sufficiency,torgersen1991comparison}.  Blackwell--Le Cam comparison
supplies the basic interpretation of the interface information gap in the
present paper: replacing a full state by an interface observation is harmless
for a declared decision problem exactly when the replacement does not increase
the optimal Bayes risk.

Recent representation-learning work makes the same comparison operational.
Deficiency has been used to define generic feature quality
\citep{vanrooyen2014lecam} and to derive a variational bottleneck for learning
representations that preserve downstream decision performance
\citep{banerjee2018deficiency}.  Loss-dependent Bayes sufficiency identifies a
coarsest quotient when the Bayes action is almost surely unique
\citep{sevetlidis2026bayes}.  Action-sufficient compression similarly merges
support states that require the same optimal action
\citep{walsh2026support}, while decision-sufficient datasets preserve the
information needed to solve a family of linear optimization problems
\citep{ye2026decision}.

We use risk comparison to audit an existing interface. The structural
question concerns several acceptable continuations, rather than one chosen
optimal label. An incompatible cell contains a minimal group with no
common acceptable choice. Such groups characterize feasible interface
partitions. A statistical guarantee additionally requires valid bounds
on loss estimation error; it does not follow from a positive estimated
gap alone.

\subsection{Information bottlenecks and task-dependent representation learning}
\label{app:rw-sec:rw-information-bottleneck}

The information bottleneck formalizes representation learning as a trade-off
between compression of the input and retention of information about a target
\citep{tishby1999information}.  Variational formulations make the objective
usable with neural encoders \citep{alemi2017deep}.  The theory of usable
information further makes information relative to a computational family,
rather than treating all statistically present information as equally
available \citep{xu2020usable}.  The information-bottleneck and usable-information views motivate the separation between
what the representation contains and what a restricted downstream rule can
extract.

The paper makes the separation explicit through four nested risks and three adjacent gaps.  The
interface information gap compares the full-state optimum with the optimum
over every measurable rule that sees only the declared interface.  The
decision-rule accessibility gap then restricts the rule family, and the
fitting gap measures the remaining cost of fitting one rule from
finite data.  Mutual information, compression rate, reconstruction error, and
predictive likelihood can be useful representation diagnostics, but none of
them alone equals the declared continuation loss.  A low-energy coordinate can
carry a large action distinction, and a high-fidelity reconstruction can still
remove information needed to choose among clarification, refusal, and bounded
assistance.  The paper therefore measures decision loss directly and treats
reconstruction only as a separate diagnostic.

\subsection{State abstraction, bisimulation, and control-sufficient representations}
\label{app:rw-sec:rw-state-abstraction}

State abstraction in Markov decision processes studies when states may be
merged while preserving planning or control properties.  Exact model
reduction and homomorphism-style equivalence relations preserve rewards,
transitions, values, or optimal policies
\citep{givan2003equivalence,li2006unified}.  Bisimulation metrics extend exact
equivalence to graded state similarity and provide value-loss bounds for
finite and continuous state spaces
\citep{ferns2004metrics,ferns2011bisimulation}.  Approximate abstraction work
then gives explicit suboptimality guarantees when aggregation is inexact
\citep{abel2016near,abel2020abstractions}.

Learned control representations apply related criteria to high-dimensional
observations.  DeepMDP learns a latent model by matching rewards and transition
distributions \citep{gelada2019deepmdp}.  Causal and invariant state
representations seek features that retain stable return-relevant structure
across environments
\citep{zhang2020invariant,zhang2021learning}.  Mutual-information objectives
have been analyzed according to whether they are sufficient for control
\citep{rakelly2021mutual}, and kinematic abstraction supports efficient
exploration in rich-observation environments \citep{misra2020kinematic}.
Recent work on vision-language-action models directly diagnoses state aliasing
when visually similar states require different low-level actions
\citep{lee2026aliasing}.

Our audit takes the full state, continuation library, and learned interface
as inputs. It tests whether states sharing one observation admit a common
choice within the allowed regret. Every pair of acceptable sets can
intersect while the intersection of all sets is empty. A graph containing
only incompatible pairs misses such a group; the conflict hypergraph
records it. Nontransitivity alone would not establish this failure of all
pairwise checks. The separate risk comparison also distinguishes missing
information from restrictions of the rule family and fitting error.

\subsection{Partial observability, predictive states, and closed-loop continuations}
\label{app:rw-sec:rw-partial-observability}

Partially observable control replaces the unobserved physical state with an
information state.  In a finite partially observable Markov decision process,
the posterior belief is sufficient for finite-horizon control
\citep{smallwood1973optimal,kaelbling1998planning}.  Predictive state
representations instead summarize a history by predictions of future tests
\citep{littman2001predictive,singh2004predictive}, and spectral predictive-state
methods connect representation learning with planning
\citep{boots2011closing}.  Approximate information states quantify the loss
from using a learned or compressed history summary in partially observed
planning and reinforcement learning \citep{subramanian2022approximate}.
Options provide a related formalism for temporally extended closed-loop
actions whose internal policies react to future observations
\citep{sutton1999options}.

Beliefs, predictive states, and options motivate continuations that can
react to later observations. The full audit state may contain a history;
the tested interface is an additional channel applied to that state.
A conflict is relative to the available continuations. Enlarging the
library changes both the choices and the oracle baseline, so transfer to
a larger library requires a separate comparison.

\subsection{Set-valued prediction, abstention, and higher-order compatibility}
\label{app:rw-sec:rw-set-valued}

Classification with rejection permits a predictor to abstain when the cost of
a forced label is too high \citep{chow1970optimum,bartlett2008classification,
charoenphakdee2021classification}.  Confidence-set and set-valued classifiers
return several plausible labels while controlling set size, error, or coverage
\citep{denis2017confidence,sadinle2019least}.  Conformal prediction gives
finite-sample coverage for prediction sets under exchangeability
\citep{vovk2005algorithmic}.  Reject-option and set-valued prediction methods establish that a decision problem
can naturally have a set of acceptable outputs rather than one required label.

The set-valued object in the present paper has a different index structure.
For each full state~$u$, the acceptable set $\AcceptSet^\varepsilon(u)$
contains continuation distributions whose regret at $u$ is at most
$\varepsilon$.  Uniform compatibility at tolerance $\varepsilon$ requires the intersection
\[
  \bigcap_{u\in C_z}\AcceptSet^\varepsilon(u)
\]
to be nonempty for every interface cell $C_z$.  Set-valued classification
usually controls which labels may be returned for one input; the present
intersection asks whether one downstream decision can serve several full
states that the interface makes indistinguishable.

Convex geometry and game theory then determine the finite witness structure.
Helly's theorem converts an empty intersection of convex acceptable sets in an
$(m-1)$-dimensional continuation simplex into a subfamily of at most $m$ sets
\citep{helly1923mengen}.  Sion's minimax theorem gives the dual state-mixture
interpretation of the robust witness margin \citep{sion1958general}.  Weak
hypergraph coloring characterizes how many interface labels are needed to
split every minimal conflict \citep{berge1989hypergraphs}.  These are classical tools. Here they relate the loss of a fixed interface
to incompatible state groups and the output count of a redesigned
interface. A gap exceeding $\varepsilon$ implies a cell with margin
exceeding $\varepsilon$; positivity alone does not establish conflict at
an arbitrary tolerance. Certifying a group selected from estimated losses
also requires simultaneous error control or independent certification data.

\subsection{Safe reinforcement learning, shielding, and latent control}
\label{app:rw-sec:rw-safe-control}

Safe reinforcement learning studies reward optimization under safety costs,
constraints, or intervention mechanisms
\citep{garcia2015survey,achiam2017constrained,chow2018lyapunov}.  Shielding
places a verified filter between a learned policy and the environment
\citep{alshiekh2018shielding,carr2023shielding}.  Reachability-based filters and
model-predictive safety filters override or modify unsafe proposed controls
\citep{fisac2019general,wabersich2021predictive}.  Reviews of safe learning in
robotics organize these methods by their assumptions about dynamics,
uncertainty, constraints, and certification \citep{brunke2022safe}.

A related language-model line reads or changes internal activations.  Latent
knowledge methods search for information represented internally without direct
supervision \citep{burns2023discovering}.  Inference-time intervention and
representation engineering learn activation directions used to steer model
behavior \citep{li2023inference,zou2023representation}.  Latent reachability
has also been proposed for early detection and steering of language-model
misalignment \citep{karnik2025reachability}.

Rather than evaluate only a particular monitor, we minimize over all rules
using its specified input. A positive gap rules out matching the full
state optimum with that input under the same library and loss. The audit
does not rule out changing the representation, available strategies, or
decision problem.

\subsection{Sparse autoencoders and activation-space analysis}
\label{app:rw-sec:rw-sae}

Sparse coding represents an observation by a small number of dictionary
elements and has long been used to study distributed structure
\citep{olshausen1996emergence}.  $k$-sparse autoencoders impose a fixed
activation budget in the learned code \citep{makhzani2013ksparse}.  Work on
superposition argues that neural networks may represent more features than
available neurons by using non-orthogonal directions
\citep{elhage2022superposition}.  Sparse autoencoders (SAEs) are now widely
used to recover sparse, interpretable activation features from language models
\citep{bricken2023monosemanticity,cunningham2024sparse}.

Other work studies training scale and coverage across models and layers.
Large SAEs have been trained and evaluated with explicit
reconstruction--sparsity controls \citep{templeton2024scaling,gao2024scaling},
and Gemma Scope provides public SAEs across models and layers
\citep{lieberum2024gemmascope}.  Activation geometry can also be distributed
across groups of SAE features rather than represented by one isolated feature
\citep{bhalla2026manifolds}.  The SAE literature makes reconstruction quality,
feature interpretation, feature geometry, dead-feature rate, and causal
intervention important SAE diagnostics.

We treat selected SAE values or support indicators as the controller
input. For a finite input alphabet, the risk calculation optimizes all
lookup rules and does not depend on probe architecture or training.
For distinct continuous vectors in a finite sample, the same unrestricted
rule can memorize the sample. Continuous observations therefore require
a separate analysis before probe failure can be interpreted as information
loss.

\subsection{Language-model judges and finite-sample evaluation}
\label{app:rw-sec:rw-llm-judges}

Large language models are frequently used as evaluators for open-ended model
outputs.  Chatbot Arena uses pairwise model comparisons at scale
\citep{zheng2023judging}, and G-Eval prompts a strong model to score generated
text along declared dimensions \citep{liu2023geval}.  Fine-grained evaluator
models such as Prometheus seek reproducible rubric-conditioned judgments
\citep{kim2024prometheus}.  RewardBench evaluates reward models and preference
judges across several behavioral categories \citep{lambert2025rewardbench}.
At the same time, position bias can make model-based pairwise evaluation
unreliable \citep{wang2024fair}.

Statistical work addresses the distinction between a useful noisy judge and a
valid evaluation procedure.  Imperfect-judge testing can preserve type-I error
control after a human calibration stage \citep{feng2026noisy}.  Risk-control
methods construct prediction sets or deployment rules with finite-sample
bounds \citep{bates2021risk,angelopoulos2025learntest}, while simultaneous
concentration inequalities support uniform control across many bounded
quantities \citep{hoeffding1963probability}.

We analyze automatic scores and human ratings as separate loss tables.
Their gaps may agree in sign without their response scores being close.
The human study reports independent rating agreement before adjudication.
For conflict search, the certificate uses simultaneous bounds over the
searched losses or separate certification data; a maximum error observed
in one packet is not a uniform bound for future responses.

\subsection{Generative-model safeguards, red teaming, and path-level safety}
\label{app:rw-sec:rw-generative-safety}

Generative-model safety work often evaluates a completed output, a prompt--
response pair, or a model-wide failure rate.  RealToxicityPrompts measures toxic
text degeneration under naturally occurring prompts
\citep{gehman2020realtoxicity}.  Red-teaming methods search for prompts that
elicit harmful behavior \citep{perez2022redteaming}.  Constitutional AI uses
written principles and model feedback to train a less harmful assistant
\citep{bai2022constitutional}, while Llama Guard provides an input--output
safeguard classifier for conversational systems \citep{inan2023llamaguard}.

A separate logical question is whether each partial state has at least one
safe continuation. Even if the answer is yes, states sharing one controller
input may have no common safe continuation. Our regret criterion is related
but not identical: it compares with the best available continuation, whose
loss need not be zero. The controlled example demonstrates failure of joint
compatibility despite success of every pair; the searched natural-language
example does not replicate that result.

\subsection{Summary of the comparison}
\label{app:rw-sec:rw-summary}

\begin{longtable}{
  p{0.20\linewidth}
  p{0.34\linewidth}
  p{0.38\linewidth}
}
\caption{Relation between the present paper and the closest technical lines of work.}
\label{app:rw-tab:related-work-map}\\
\toprule
\textbf{Line of work}
&
\textbf{What the prior line establishes}
&
\textbf{Target in the present paper}
\\
\midrule
\endfirsthead

\multicolumn{3}{c}{\tablename\ \thetable\ -- continued from the previous page}\\
\toprule
\textbf{Line of work}
&
\textbf{What the prior line establishes}
&
\textbf{Target in the present paper}
\\
\midrule
\endhead

\midrule
\multicolumn{3}{r}{Continued on the next page}\\
\endfoot

\bottomrule
\endlastfoot

Comparison of experiments and Bayes sufficiency
&
Orders information channels by attainable risk and characterizes exact or
approximate loss-dependent sufficiency
\citep{blackwell1953equivalent,lecam1964sufficiency,
vanrooyen2014lecam,sevetlidis2026bayes}.
&
Uses optimal-risk comparison as the diagnostic foundation, then extracts
minimal conflicts among acceptable closed-loop continuation sets and connects
them to capacity and certification.
\\
\addlinespace

Action-sufficient compression and decision-sufficient data
&
Retains the distinctions needed to recover one optimal action or an optimal
solution to a declared optimization family
\citep{walsh2026support,ye2026decision}.
&
Allows several acceptable continuations at each state. Every pair can
share an acceptable choice even when the whole group cannot; the
hypergraph records these larger conflicts.
\\
\addlinespace

Information bottlenecks
&
Trades compression against predictive or computationally usable information
\citep{tishby1999information,alemi2017deep,xu2020usable}.
&
Uses the declared safety--utility loss itself and separates missing
information from restricted access to retained information and finite fitting.
\\
\addlinespace

MDP abstraction and bisimulation
&
Constructs exact or approximate summaries that preserve models, values,
policies, or exploration structure
\citep{givan2003equivalence,ferns2011bisimulation,
abel2020abstractions,misra2020kinematic}.
&
Audits a fixed learned interface and returns exact risk for the finite interface,
minimal conflicting groups, and the required number of interface labels.
\\
\addlinespace

Partial observability and predictive states
&
Builds belief, information, or predictive states sufficient for recursive
planning; options define temporally extended closed-loop actions
\citep{kaelbling1998planning,singh2004predictive,
subramanian2022approximate,sutton1999options}.
&
Assumes the full audit state and continuation library are declared, then tests
whether an additional operational interface preserves common continuation
compatibility within each observation fiber.
\\
\addlinespace

Set-valued prediction and abstention
&
Returns several plausible labels or rejects a prediction while controlling
coverage, ambiguity, or cost
\citep{bartlett2008classification,denis2017confidence,
sadinle2019least}.
&
Intersects acceptable continuation sets across states sharing one observation; pairwise intersections do not determine the full intersection.
\\
\addlinespace

Safe control, shielding, and latent steering
&
Optimizes constrained policies or applies a monitor, shield, filter, or latent
intervention to a chosen observation
\citep{achiam2017constrained,alshiekh2018shielding,
fisac2019general,li2023inference,karnik2025reachability}.
&
Tests whether the chosen observation interface is adequate before the
monitoring or steering rule is trusted.
\\
\addlinespace

SAE and activation-space analysis
&
Studies sparse feature recovery, interpretation, reconstruction, scaling,
causal effects, and activation geometry
\citep{cunningham2024sparse,gao2024scaling,
lieberum2024gemmascope,bhalla2026manifolds}.
&
Treats an SAE-derived code or coordinate set as an operational control
interface and measures the exact decision loss induced by its state
collisions.
\\
\addlinespace

Model-based judges and risk control
&
Measures open-ended outputs with learned evaluators and develops calibration
or finite-sample risk procedures for imperfect measurements
\citep{liu2023geval,wang2024fair,feng2026noisy,
angelopoulos2025learntest}.
&
Uses automatic and human loss constructions separately, preserves validity
after witness search, and does not infer close item-level calibration from a
positive aggregate interface gap.
\\
\addlinespace

Generative safeguards
&
Detects harmful prompts or completed outputs and trains safer assistants
\citep{gehman2020realtoxicity,perez2022redteaming,
bai2022constitutional,inan2023llamaguard}.
&
Tests whether states sharing a controller input have a common acceptable
continuation; a conflict can first appear for a triple.
\\
\end{longtable}

\paragraph{Relation to a single optimal label.}
On a finite audit with one optimal pure continuation at each state,
let $\rho$ be the smallest regret of a nonoptimal continuation.
For $0\le\varepsilon<\rho/2$, different optimal labels are already
incompatible in pairs (Proposition~\ref{prop:unique-optimum-reduction}).
With larger tolerances or several acceptable choices, a single selected
label need not describe all compatible groups. The hypergraph supplies
the complete finite set of incompatibility constraints. Statistical claims
about estimated constraints still require the loss error assumptions in
Appendix~\ref{app:proofs-capacity-certification}.

\section{Complete Phi Third-Family Results}
\label{app:phi-results}

This appendix reports the source-derived Phi results. Original 96 denotes the original eight-area confirmation split; Reused 96 denotes the eight-area collection locked for the earlier SmolLM2 experiment and reused for Phi. The fitting split has 144 original states. The original experiment archive is retained as \path{supplementary/phi/PHI_THIRD_FAMILY_RESULTS_original.zip} in the author working archive. The Phi study includes 5,040 generated responses evaluated by both automatic evaluators. No human ratings were collected for the Phi responses. The 288 risk-summary rows comprise two SAE seeds, two selectors, six budgets, two collections, three evaluator losses, and two observation conditions. These are overlapping analyses of the same responses, not 288 independent replications.

\subsection{Configuration, provenance, and validation}

The active protocol identifies Phi-4-mini-instruct revision \texttt{cfbefacb99\allowbreak 257ffa30c8\allowbreak 3adab238a5\allowbreak 0856ac3083}; Mistral-7B-Instruct-v0.3 revision \texttt{c170c708c4\allowbreak 1dac9275d1\allowbreak 5a8fff4eca\allowbreak 08d52bab71}; Qwen2.5-7B-Instruct revision \texttt{a09a35458c\allowbreak 702b33eeac\allowbreak c393d10306\allowbreak 3234e8bc28}; and FineWeb-Edu \texttt{sample-10BT} revision \texttt{87f09149ef\allowbreak 4734204d70\allowbreak ed1d046ddc\allowbreak 9ca3f2b8f9}. Block index 23 has width 3,072 and depth fraction $0.75$. The two SAE seeds are 314159 and 271828, dictionary width is 16,384, and TopK sparsity is 64.

Each SAE uses 8,000,000 training tokens and 300,000 validation tokens, with 100,000 training tokens for normalization. Corpus contexts have length 128 and cached activation shards contain 32,768 tokens. SAE batches contain 1,024 activation tokens; AdamW uses learning rate $2\times10^{-4}$, 500 warm-up steps, zero weight decay, gradient norm clipping at 1, auxiliary coefficient $0.03125$, and a 1,000-step dead-feature window. Both training records end at step 7,813 and 8,000,000 consumed tokens. Validation requires FVU at most $0.45$, cosine at least $0.75$, unobserved-feature fraction at most $0.5$, and mean-$L_0$ deviation from 64 at most $3.2$. Both runs pass these saved checks.

Generation uses five replicas per state--action pair, at most 192 new tokens, temperature $0.7$, and top-$p=0.9$. The two evaluators each score all 5,040 responses, with a 160-token output limit and up to four format retries. Mistral scores parse on the first attempt; four Qwen scores require one format retry. There are no missing or duplicated state--action--replica keys. The saved execution takes approximately 4.63 hours on an A100-SXM4-40GB. The recorded environment is Python 3.13.15, PyTorch 2.11.0+cu128, Transformers 4.57.6, Datasets 4.4.1, Hugging Face Hub 0.36.2, Accelerate 1.14.0, NumPy 2.1.3, Pandas 2.2.3, and SciPy 1.16.3.

Cached feature arrays and item-level scores independently reproduce all 288 risk/bootstrap summaries and all 192 deletion-level estimates to floating-point precision. This checks table construction and arithmetic, not unseen training weights, evaluator validity, or human judgments. Bootstrap resampling uses 2,000 multinomial policy-area resamples and seed 99173. An eight-cluster percentile interval remains an approximate, per-setting uncertainty summary rather than the formal simultaneous certificate in Section~\ref{sec:conflict-structure}.

\subsection{Phi retains a positive gap when the active policy is observed}
\label{app:phi-interpretation}

The Phi run completes all 5,040 generations and both complete evaluator passes. The two SAEs complete the declared training schedule of eight million tokens and pass the saved validation thresholds. No Phi human labels have been collected.

\paragraph{Policy-aware primary results.}
Table~\ref{tab:phi-primary} reports the primary $(Z,S)$ observation at $B=2$. Under Mistral loss, the information gaps are $0.0557$ and $0.0491$ on the original confirmation set and $0.0436$ and $0.0331$ on the reused locked collection, for seeds 314159 and 271828, respectively. The corresponding one-sided 95\% bootstrap LCBs are $0.0367$, $0.0266$, $0.0225$, and $0.0185$. Every primary Qwen-loss and mean-loss check also has a positive individual LCB. The tables below include both risk minima, two-sided intervals, and the complete budget sweeps.

\begin{table}[htbp]
\centering
\caption{Phi primary policy-aware results at feature budget $B=2$. Each entry is the information gap $\widehat\Delta_{Z,S}$ followed by its individual one-sided 95\% bootstrap over policy areas lower bound in brackets. The original confirmation and reused locked collections each contain 96 states. Mean loss averages the two evaluators before risk optimization.}
\label{tab:phi-primary}
\small
\setlength{\tabcolsep}{4pt}
\begin{tabular}{@{}llrrr@{}}
\toprule
Collection & SAE seed & Mistral & Qwen & Mean loss \\
\midrule
Original 96 & 314159 & 0.0557 [0.0367] & 0.0830 [0.0588] & 0.0568 [0.0381] \\
Original 96 & 271828 & 0.0491 [0.0266] & 0.0365 [0.0229] & 0.0296 [0.0157] \\
Reused 96 & 314159 & 0.0436 [0.0225] & 0.0584 [0.0424] & 0.0708 [0.0477] \\
Reused 96 & 271828 & 0.0331 [0.0185] & 0.0465 [0.0283] & 0.0470 [0.0320] \\
\bottomrule
\end{tabular}
\end{table}

These results retain the information restriction after active policy is supplied: some states under the same policy and SAE code still cannot share the lowest-loss macro-continuation under the saved loss table. The 12 policy-aware settings are not independent replications; the evaluator checks and their mean share the same responses, states, and features. All 2,000 bootstrap gaps in each primary setting exceed the saved numerical tolerance, but the zero nonpositive fraction is not a $p$-value or a simultaneous certificate.

\paragraph{Paired observation comparison.}
Providing the active policy reduces or leaves unchanged the exact information gap in every primary Phi setting, as required by Equation~\eqref{eq:policy-aware-risk}. For seed 314159 on the original confirmation set, the Mistral-loss gap decreases from $0.1110$ for $Z$ alone to $0.0557$ for $(Z,S)$. The positive policy-aware gap therefore cannot be explained solely by hiding the active policy. Both observation conditions keep all other inputs fixed. The historical Qwen2.5 and SmolLM2 tables estimate code-only risk; the separate policy-aware analyses in Section~\ref{sec:policy-reanalysis} have their own stated source precision.

\paragraph{Budget and policy-area sensitivity.}
Figure~\ref{fig:phi-policy-budget} shows the policy-aware Mistral-loss curves. All supervised $B=2$ and $B=4$ settings have positive individual LCBs across the two seeds, collections, and three loss constructions. From $B=8$ onward some LCBs reach zero; at $B=32$ and $B=64$ many finite-table gaps vanish. Zero gap on these states does not establish sufficiency on future states. In particular, the complete $(Z,S)$ keys are unseen in the fitting split because the evaluation policy texts differ from the fitting policy texts. The exact evaluation-set optimum is not the performance of a learned lookup table on those unseen keys.

\begin{figure}[tbp]
\centering
\begin{subfigure}[t]{0.48\linewidth}
\includegraphics[width=\linewidth]{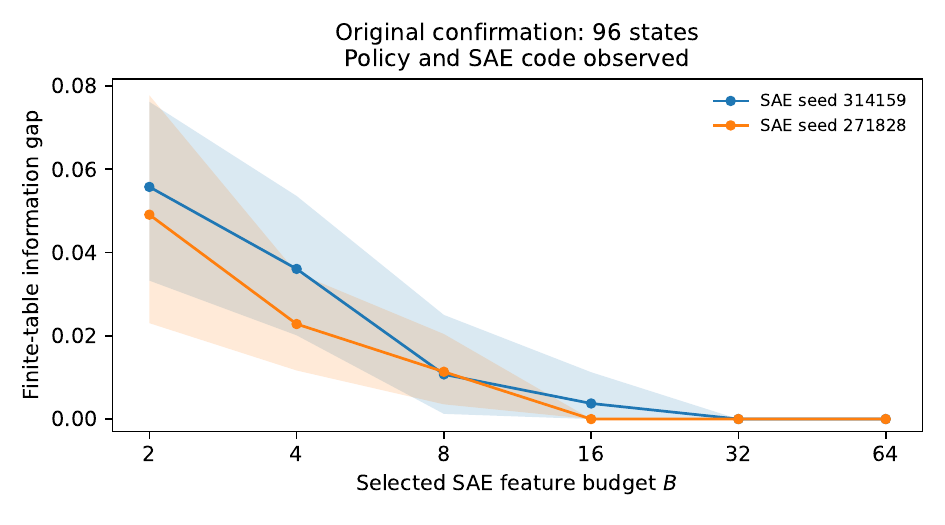}
\caption{Original confirmation collection.}
\end{subfigure}\hfill
\begin{subfigure}[t]{0.48\linewidth}
\includegraphics[width=\linewidth]{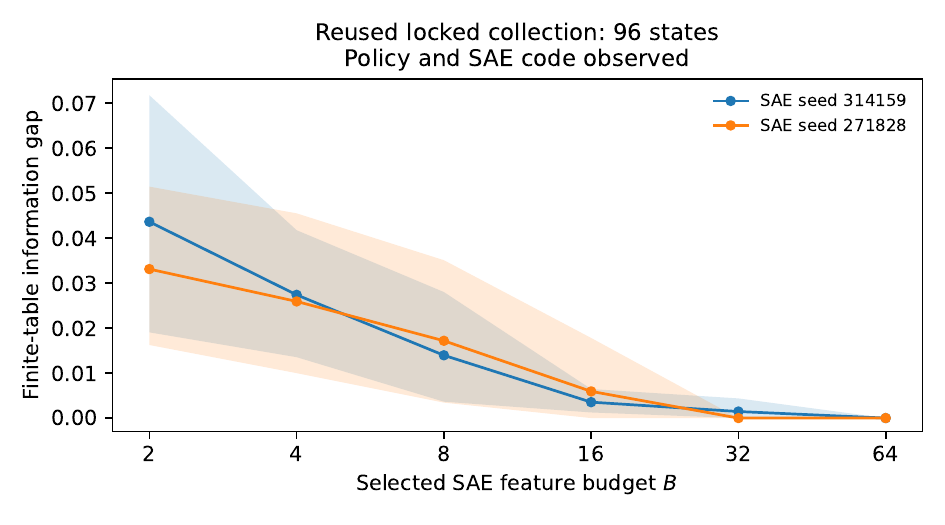}
\caption{Reused locked collection.}
\end{subfigure}
\caption{Phi policy-aware information gaps under Mistral loss. Points use the frozen supervised features, and bands are individual two-sided 95\% bootstrap over policy areas intervals. Budget positions are equally spaced for readability. Both the policy text and selected SAE support code are observed.}
\label{fig:phi-policy-budget}
\end{figure}

At $B=2$, deleting any one of the eight policy areas leaves 84 states. Every deletion-level gap is positive across both seeds, both collections, all three loss constructions, and both observation conditions: 192 point estimates in total, with minimum $0.0236$. No single deletion removes the positive gap in these collections. Table~\ref{tab:phi-loao} summarizes the deletion ranges.

\paragraph{Reconstruction quality and distribution shift.}
Validation centered fraction of variance unexplained (FVU) is $0.0162$ and $0.0163$; mean reconstruction cosine is $0.9277$ for both SAE seeds, mean $L_0$ is 64, and all 16,384 features appear at least once on validation. Reconstruction on audit states is weaker. Centered FVU is $17.3125$ and $16.6518$ on the original 240 states, and $2.8355$ and $2.7705$ on the reused 96 states. Uncentered relative squared reconstruction error is $0.2704$--$0.3062$ across these four audit settings. On the original audit distribution, centered activation energy per vector is $232.49$, compared with uncentered energy $13{,}143.04$. The smaller denominator explains why centered FVU is much larger than uncentered relative error, but does not remove the reconstruction error itself. Passing corpus validation does not establish accurate reconstruction on the audit domain. Table~\ref{tab:phi-quality} reports the distribution-specific diagnostics.

\paragraph{Evaluator comparison.}
The primary gap remains positive with either evaluator, but their scores
have limited agreement on individual responses. Detailed comparisons below
use the same response keys for both evaluators.

\FloatBarrier
\subsection{Policy-aware and code-only primary results}

The complete observation is explicitly $V=(Z,S)$ or $V=Z$. Equal-weight losses are averaged before minimization. Full-state, best-interface, and gap values can differ slightly from arithmetic on the displayed rounded numbers because the calculations use unrounded values.

\begingroup
\small
\setlength{\tabcolsep}{4pt}
\begin{longtable}{@{}lrrrrrr@{}}
\caption{Phi full primary results for observation $(Z,S)$ at $B=2$. $R_V^\star$ denotes the exact finite-table optimum for the stated observation $V$. CI is the two-sided 95\% bootstrap interval; LCB is the one-sided 95\% lower bound. $f_0$ is the fraction of 2,000 bootstrap gaps at or below the saved $10^{-12}$ numerical tolerance, not a $p$-value.}\label{tab:phi-primary-policy-full}\\
\toprule
Evaluator & $\widehat R_{\mathrm{full}}$ & $\widehat R_V^\star$ & $\widehat\Delta_V$ & 95\% CI & 95\% LCB & $f_0$ \\
\midrule
\endfirsthead
\multicolumn{7}{l}{\small Table~\ref{tab:phi-primary-policy-full} continued.} \\
\toprule
Evaluator & $\widehat R_{\mathrm{full}}$ & $\widehat R_V^\star$ & $\widehat\Delta_V$ & 95\% CI & 95\% LCB & $f_0$ \\
\midrule
\endhead
\bottomrule
\endfoot
\addlinespace\multicolumn{7}{l}{\textbf{Original 96, seed 314159}} \\
Mistral & 0.0742 & 0.1299 & 0.0557 & [0.0332, 0.0761] & 0.0367 & 0.000 \\
Qwen & 0.1665 & 0.2494 & 0.0830 & [0.0554, 0.1151] & 0.0588 & 0.000 \\
Mean loss & 0.1447 & 0.2015 & 0.0568 & [0.0339, 0.0801] & 0.0381 & 0.000 \\
\addlinespace\multicolumn{7}{l}{\textbf{Original 96, seed 271828}} \\
Mistral & 0.0742 & 0.1232 & 0.0491 & [0.0230, 0.0777] & 0.0266 & 0.000 \\
Qwen & 0.1665 & 0.2029 & 0.0365 & [0.0198, 0.0517] & 0.0229 & 0.000 \\
Mean loss & 0.1447 & 0.1744 & 0.0296 & [0.0136, 0.0469] & 0.0157 & 0.000 \\
\addlinespace\multicolumn{7}{l}{\textbf{Reused 96, seed 314159}} \\
Mistral & 0.1006 & 0.1443 & 0.0436 & [0.0191, 0.0718] & 0.0225 & 0.000 \\
Qwen & 0.1910 & 0.2494 & 0.0584 & [0.0394, 0.0783] & 0.0424 & 0.000 \\
Mean loss & 0.1731 & 0.2439 & 0.0708 & [0.0437, 0.0999] & 0.0477 & 0.000 \\
\addlinespace\multicolumn{7}{l}{\textbf{Reused 96, seed 271828}} \\
Mistral & 0.1006 & 0.1338 & 0.0331 & [0.0163, 0.0515] & 0.0185 & 0.000 \\
Qwen & 0.1910 & 0.2374 & 0.0465 & [0.0250, 0.0692] & 0.0283 & 0.000 \\
Mean loss & 0.1731 & 0.2201 & 0.0470 & [0.0296, 0.0704] & 0.0320 & 0.000 \\
\end{longtable}
\endgroup

\begingroup
\small
\setlength{\tabcolsep}{4pt}
\begin{longtable}{@{}lrrrrrr@{}}
\caption{Phi full primary results for observation $Z$ only at $B=2$. $R_V^\star$ denotes the exact finite-table optimum for the stated observation $V$. CI is the two-sided 95\% bootstrap interval; LCB is the one-sided 95\% lower bound. $f_0$ is the fraction of 2,000 bootstrap gaps at or below the saved $10^{-12}$ numerical tolerance, not a $p$-value.}\label{tab:phi-primary-zonly-full}\\
\toprule
Evaluator & $\widehat R_{\mathrm{full}}$ & $\widehat R_V^\star$ & $\widehat\Delta_V$ & 95\% CI & 95\% LCB & $f_0$ \\
\midrule
\endfirsthead
\multicolumn{7}{l}{\small Table~\ref{tab:phi-primary-zonly-full} continued.} \\
\toprule
Evaluator & $\widehat R_{\mathrm{full}}$ & $\widehat R_V^\star$ & $\widehat\Delta_V$ & 95\% CI & 95\% LCB & $f_0$ \\
\midrule
\endhead
\bottomrule
\endfoot
\addlinespace\multicolumn{7}{l}{\textbf{Original 96, seed 314159}} \\
Mistral & 0.0742 & 0.1852 & 0.1110 & [0.0732, 0.1296] & 0.0799 & 0.000 \\
Qwen & 0.1665 & 0.2675 & 0.1010 & [0.0647, 0.1460] & 0.0690 & 0.000 \\
Mean loss & 0.1447 & 0.2263 & 0.0816 & [0.0495, 0.1115] & 0.0540 & 0.000 \\
\addlinespace\multicolumn{7}{l}{\textbf{Original 96, seed 271828}} \\
Mistral & 0.0742 & 0.1517 & 0.0775 & [0.0442, 0.0978] & 0.0482 & 0.000 \\
Qwen & 0.1665 & 0.2152 & 0.0487 & [0.0254, 0.0654] & 0.0308 & 0.000 \\
Mean loss & 0.1447 & 0.1932 & 0.0485 & [0.0262, 0.0587] & 0.0288 & 0.000 \\
\addlinespace\multicolumn{7}{l}{\textbf{Reused 96, seed 314159}} \\
Mistral & 0.1006 & 0.1661 & 0.0655 & [0.0376, 0.0917] & 0.0414 & 0.000 \\
Qwen & 0.1910 & 0.2983 & 0.1073 & [0.0445, 0.1581] & 0.0503 & 0.000 \\
Mean loss & 0.1731 & 0.2683 & 0.0952 & [0.0664, 0.1173] & 0.0708 & 0.000 \\
\addlinespace\multicolumn{7}{l}{\textbf{Reused 96, seed 271828}} \\
Mistral & 0.1006 & 0.1619 & 0.0612 & [0.0332, 0.0788] & 0.0368 & 0.000 \\
Qwen & 0.1910 & 0.2938 & 0.1028 & [0.0419, 0.1248] & 0.0482 & 0.000 \\
Mean loss & 0.1731 & 0.2435 & 0.0704 & [0.0436, 0.0946] & 0.0478 & 0.000 \\
\end{longtable}
\endgroup

\begin{table}[tbp]
\centering
\caption{Paired observation-condition comparison for the same Phi responses, losses, frozen features, and $B=2$. Supplying active policy can refine a $Z$-only cell and never increases the exact optimal interface risk. Differences are computed before rounding.}
\label{tab:phi-observation-comparison}
\small
\setlength{\tabcolsep}{4pt}
\begin{tabular}{@{}ll l rrr@{}}
\toprule
Collection & Seed & Evaluator & $\widehat\Delta_Z$ & $\widehat\Delta_{Z,S}$ & Reduction \\
\midrule
Original 96 & 314159 & Mistral & 0.1110 & 0.0557 & 0.0553 \\
Original 96 & 314159 & Qwen & 0.1010 & 0.0830 & 0.0180 \\
Original 96 & 314159 & Mean loss & 0.0816 & 0.0568 & 0.0248 \\
Original 96 & 271828 & Mistral & 0.0775 & 0.0491 & 0.0284 \\
Original 96 & 271828 & Qwen & 0.0487 & 0.0365 & 0.0123 \\
Original 96 & 271828 & Mean loss & 0.0485 & 0.0296 & 0.0188 \\
Reused 96 & 314159 & Mistral & 0.0655 & 0.0436 & 0.0219 \\
Reused 96 & 314159 & Qwen & 0.1073 & 0.0584 & 0.0489 \\
Reused 96 & 314159 & Mean loss & 0.0952 & 0.0708 & 0.0244 \\
Reused 96 & 271828 & Mistral & 0.0612 & 0.0331 & 0.0281 \\
Reused 96 & 271828 & Qwen & 0.1028 & 0.0465 & 0.0564 \\
Reused 96 & 271828 & Mean loss & 0.0704 & 0.0470 & 0.0234 \\
\bottomrule
\end{tabular}
\end{table}

\subsection{Complete supervised-selector and variance-selector sweeps}

Only fitting data determine the feature identities. The supervised ranking uses Mistral loss, and all other evaluator conditions reuse that ranking. No interface is selected using the Phi confirmation losses. The 288 rows below retain both positive and zero-gap results; they do not extend the $B=2$ claim to every budget.

\begingroup
\small
\setlength{\tabcolsep}{4pt}
\begin{longtable}{@{}rlrrrrrr@{}}
\caption{Complete Phi supervised-selector results for observation $(Z,S)$. Each row uses eight policy-area clusters and 2,000 percentile bootstrap repetitions. Features are selected on the fitting split and fixed across evaluators. $f_0$ uses the saved $10^{-12}$ tolerance; intervals are individual, not simultaneous.}\label{tab:phi-supervised-policy-all}\\
\toprule
$B$ & Evaluator & $\widehat R_{\mathrm{full}}$ & $\widehat R_V^\star$ & $\widehat\Delta_V$ & 95\% CI & 95\% LCB & $f_0$ \\
\midrule
\endfirsthead
\multicolumn{8}{l}{\small Table~\ref{tab:phi-supervised-policy-all} continued.} \\
\toprule
$B$ & Evaluator & $\widehat R_{\mathrm{full}}$ & $\widehat R_V^\star$ & $\widehat\Delta_V$ & 95\% CI & 95\% LCB & $f_0$ \\
\midrule
\endhead
\bottomrule
\endfoot
\addlinespace\multicolumn{8}{l}{\textbf{Original 96, seed 314159}} \\
2 & Mistral & 0.0742 & 0.1299 & 0.0557 & [0.0332, 0.0761] & 0.0367 & 0.000 \\
2 & Qwen & 0.1665 & 0.2494 & 0.0830 & [0.0554, 0.1151] & 0.0588 & 0.000 \\
2 & Mean loss & 0.1447 & 0.2015 & 0.0568 & [0.0339, 0.0801] & 0.0381 & 0.000 \\
4 & Mistral & 0.0742 & 0.1102 & 0.0360 & [0.0201, 0.0535] & 0.0224 & 0.000 \\
4 & Qwen & 0.1665 & 0.2328 & 0.0663 & [0.0367, 0.1005] & 0.0413 & 0.000 \\
4 & Mean loss & 0.1447 & 0.1860 & 0.0413 & [0.0202, 0.0642] & 0.0229 & 0.000 \\
8 & Mistral & 0.0742 & 0.0849 & 0.0107 & [0.0013, 0.0250] & 0.0024 & 0.004 \\
8 & Qwen & 0.1665 & 0.1831 & 0.0167 & [0.0023, 0.0354] & 0.0023 & 0.024 \\
8 & Mean loss & 0.1447 & 0.1532 & 0.0085 & [0.0022, 0.0152] & 0.0033 & 0.004 \\
16 & Mistral & 0.0742 & 0.0779 & 0.0037 & [0.0000, 0.0112] & 0.0000 & 0.334 \\
16 & Qwen & 0.1665 & 0.1679 & 0.0015 & [0.0000, 0.0044] & 0.0000 & 0.334 \\
16 & Mean loss & 0.1447 & 0.1458 & 0.0010 & [0.0000, 0.0031] & 0.0000 & 0.334 \\
32 & Mistral & 0.0742 & 0.0742 & 0.0000 & [0.0000, 0.0000] & 0.0000 & 1.000 \\
32 & Qwen & 0.1665 & 0.1665 & 0.0000 & [0.0000, 0.0000] & 0.0000 & 1.000 \\
32 & Mean loss & 0.1447 & 0.1447 & 0.0000 & [0.0000, 0.0000] & 0.0000 & 1.000 \\
64 & Mistral & 0.0742 & 0.0742 & 0.0000 & [0.0000, 0.0000] & 0.0000 & 1.000 \\
64 & Qwen & 0.1665 & 0.1665 & 0.0000 & [0.0000, 0.0000] & 0.0000 & 1.000 \\
64 & Mean loss & 0.1447 & 0.1447 & 0.0000 & [0.0000, 0.0000] & 0.0000 & 1.000 \\
\addlinespace\multicolumn{8}{l}{\textbf{Original 96, seed 271828}} \\
2 & Mistral & 0.0742 & 0.1232 & 0.0491 & [0.0230, 0.0777] & 0.0266 & 0.000 \\
2 & Qwen & 0.1665 & 0.2029 & 0.0365 & [0.0198, 0.0517] & 0.0229 & 0.000 \\
2 & Mean loss & 0.1447 & 0.1744 & 0.0296 & [0.0136, 0.0469] & 0.0157 & 0.000 \\
4 & Mistral & 0.0742 & 0.0970 & 0.0228 & [0.0117, 0.0351] & 0.0131 & 0.000 \\
4 & Qwen & 0.1665 & 0.1869 & 0.0204 & [0.0077, 0.0348] & 0.0096 & 0.001 \\
4 & Mean loss & 0.1447 & 0.1587 & 0.0140 & [0.0066, 0.0217] & 0.0075 & 0.000 \\
8 & Mistral & 0.0742 & 0.0855 & 0.0114 & [0.0035, 0.0204] & 0.0040 & 0.004 \\
8 & Qwen & 0.1665 & 0.1746 & 0.0081 & [0.0000, 0.0198] & 0.0000 & 0.088 \\
8 & Mean loss & 0.1447 & 0.1499 & 0.0052 & [0.0011, 0.0103] & 0.0015 & 0.004 \\
16 & Mistral & 0.0742 & 0.0742 & 0.0000 & [0.0000, 0.0000] & 0.0000 & 1.000 \\
16 & Qwen & 0.1665 & 0.1665 & 0.0000 & [0.0000, 0.0000] & 0.0000 & 1.000 \\
16 & Mean loss & 0.1447 & 0.1447 & 0.0000 & [0.0000, 0.0000] & 0.0000 & 1.000 \\
32 & Mistral & 0.0742 & 0.0742 & 0.0000 & [0.0000, 0.0000] & 0.0000 & 1.000 \\
32 & Qwen & 0.1665 & 0.1665 & 0.0000 & [0.0000, 0.0000] & 0.0000 & 1.000 \\
32 & Mean loss & 0.1447 & 0.1447 & 0.0000 & [0.0000, 0.0000] & 0.0000 & 1.000 \\
64 & Mistral & 0.0742 & 0.0742 & 0.0000 & [0.0000, 0.0000] & 0.0000 & 1.000 \\
64 & Qwen & 0.1665 & 0.1665 & 0.0000 & [0.0000, 0.0000] & 0.0000 & 1.000 \\
64 & Mean loss & 0.1447 & 0.1447 & 0.0000 & [0.0000, 0.0000] & 0.0000 & 1.000 \\
\addlinespace\multicolumn{8}{l}{\textbf{Reused 96, seed 314159}} \\
2 & Mistral & 0.1006 & 0.1443 & 0.0436 & [0.0191, 0.0718] & 0.0225 & 0.000 \\
2 & Qwen & 0.1910 & 0.2494 & 0.0584 & [0.0394, 0.0783] & 0.0424 & 0.000 \\
2 & Mean loss & 0.1731 & 0.2439 & 0.0708 & [0.0437, 0.0999] & 0.0477 & 0.000 \\
4 & Mistral & 0.1006 & 0.1280 & 0.0274 & [0.0135, 0.0418] & 0.0155 & 0.000 \\
4 & Qwen & 0.1910 & 0.2381 & 0.0471 & [0.0286, 0.0673] & 0.0316 & 0.000 \\
4 & Mean loss & 0.1731 & 0.2223 & 0.0493 & [0.0335, 0.0684] & 0.0356 & 0.000 \\
8 & Mistral & 0.1006 & 0.1146 & 0.0140 & [0.0036, 0.0280] & 0.0049 & 0.000 \\
8 & Qwen & 0.1910 & 0.2117 & 0.0207 & [0.0112, 0.0315] & 0.0126 & 0.000 \\
8 & Mean loss & 0.1731 & 0.1921 & 0.0190 & [0.0102, 0.0292] & 0.0114 & 0.000 \\
16 & Mistral & 0.1006 & 0.1042 & 0.0035 & [0.0012, 0.0065] & 0.0015 & 0.004 \\
16 & Qwen & 0.1910 & 0.1995 & 0.0085 & [0.0035, 0.0135] & 0.0042 & 0.000 \\
16 & Mean loss & 0.1731 & 0.1780 & 0.0049 & [0.0005, 0.0109] & 0.0005 & 0.022 \\
32 & Mistral & 0.1006 & 0.1021 & 0.0015 & [0.0000, 0.0044] & 0.0000 & 0.339 \\
32 & Qwen & 0.1910 & 0.1910 & 0.0000 & [0.0000, 0.0000] & 0.0000 & 1.000 \\
32 & Mean loss & 0.1731 & 0.1733 & 0.0002 & [0.0000, 0.0006] & 0.0000 & 0.339 \\
64 & Mistral & 0.1006 & 0.1006 & 0.0000 & [0.0000, 0.0000] & 0.0000 & 1.000 \\
64 & Qwen & 0.1910 & 0.1910 & 0.0000 & [0.0000, 0.0000] & 0.0000 & 1.000 \\
64 & Mean loss & 0.1731 & 0.1731 & 0.0000 & [0.0000, 0.0000] & 0.0000 & 1.000 \\
\addlinespace\multicolumn{8}{l}{\textbf{Reused 96, seed 271828}} \\
2 & Mistral & 0.1006 & 0.1338 & 0.0331 & [0.0163, 0.0515] & 0.0185 & 0.000 \\
2 & Qwen & 0.1910 & 0.2374 & 0.0465 & [0.0250, 0.0692] & 0.0283 & 0.000 \\
2 & Mean loss & 0.1731 & 0.2201 & 0.0470 & [0.0296, 0.0704] & 0.0320 & 0.000 \\
4 & Mistral & 0.1006 & 0.1266 & 0.0259 & [0.0100, 0.0455] & 0.0117 & 0.000 \\
4 & Qwen & 0.1910 & 0.2374 & 0.0465 & [0.0250, 0.0692] & 0.0283 & 0.000 \\
4 & Mean loss & 0.1731 & 0.2087 & 0.0356 & [0.0234, 0.0495] & 0.0254 & 0.000 \\
8 & Mistral & 0.1006 & 0.1178 & 0.0172 & [0.0034, 0.0351] & 0.0046 & 0.005 \\
8 & Qwen & 0.1910 & 0.2044 & 0.0134 & [0.0067, 0.0220] & 0.0075 & 0.000 \\
8 & Mean loss & 0.1731 & 0.1884 & 0.0153 & [0.0079, 0.0229] & 0.0092 & 0.000 \\
16 & Mistral & 0.1006 & 0.1066 & 0.0059 & [0.0000, 0.0178] & 0.0000 & 0.360 \\
16 & Qwen & 0.1910 & 0.1934 & 0.0024 & [0.0000, 0.0054] & 0.0000 & 0.107 \\
16 & Mean loss & 0.1731 & 0.1769 & 0.0038 & [0.0002, 0.0105] & 0.0002 & 0.025 \\
32 & Mistral & 0.1006 & 0.1006 & 0.0000 & [0.0000, 0.0000] & 0.0000 & 1.000 \\
32 & Qwen & 0.1910 & 0.1910 & 0.0000 & [0.0000, 0.0000] & 0.0000 & 1.000 \\
32 & Mean loss & 0.1731 & 0.1731 & 0.0000 & [0.0000, 0.0000] & 0.0000 & 1.000 \\
64 & Mistral & 0.1006 & 0.1006 & 0.0000 & [0.0000, 0.0000] & 0.0000 & 1.000 \\
64 & Qwen & 0.1910 & 0.1910 & 0.0000 & [0.0000, 0.0000] & 0.0000 & 1.000 \\
64 & Mean loss & 0.1731 & 0.1731 & 0.0000 & [0.0000, 0.0000] & 0.0000 & 1.000 \\
\end{longtable}
\endgroup

\begingroup
\small
\setlength{\tabcolsep}{4pt}
\begin{longtable}{@{}rlrrrrrr@{}}
\caption{Complete Phi variance-selector results for observation $(Z,S)$. Each row uses eight policy-area clusters and 2,000 percentile bootstrap repetitions. Features are selected on the fitting split and fixed across evaluators. $f_0$ uses the saved $10^{-12}$ tolerance; intervals are individual, not simultaneous.}\label{tab:phi-variance-policy-all}\\
\toprule
$B$ & Evaluator & $\widehat R_{\mathrm{full}}$ & $\widehat R_V^\star$ & $\widehat\Delta_V$ & 95\% CI & 95\% LCB & $f_0$ \\
\midrule
\endfirsthead
\multicolumn{8}{l}{\small Table~\ref{tab:phi-variance-policy-all} continued.} \\
\toprule
$B$ & Evaluator & $\widehat R_{\mathrm{full}}$ & $\widehat R_V^\star$ & $\widehat\Delta_V$ & 95\% CI & 95\% LCB & $f_0$ \\
\midrule
\endhead
\bottomrule
\endfoot
\addlinespace\multicolumn{8}{l}{\textbf{Original 96, seed 314159}} \\
2 & Mistral & 0.0742 & 0.1703 & 0.0961 & [0.0646, 0.1288] & 0.0689 & 0.000 \\
2 & Qwen & 0.1665 & 0.3133 & 0.1468 & [0.1208, 0.1716] & 0.1248 & 0.000 \\
2 & Mean loss & 0.1447 & 0.2495 & 0.1048 & [0.0911, 0.1164] & 0.0935 & 0.000 \\
4 & Mistral & 0.0742 & 0.1277 & 0.0535 & [0.0323, 0.0810] & 0.0340 & 0.000 \\
4 & Qwen & 0.1665 & 0.2425 & 0.0761 & [0.0539, 0.0981] & 0.0577 & 0.000 \\
4 & Mean loss & 0.1447 & 0.1954 & 0.0507 & [0.0361, 0.0662] & 0.0382 & 0.000 \\
8 & Mistral & 0.0742 & 0.0849 & 0.0107 & [0.0028, 0.0191] & 0.0042 & 0.004 \\
8 & Qwen & 0.1665 & 0.1873 & 0.0208 & [0.0110, 0.0298] & 0.0125 & 0.000 \\
8 & Mean loss & 0.1447 & 0.1608 & 0.0160 & [0.0091, 0.0224] & 0.0101 & 0.000 \\
16 & Mistral & 0.0742 & 0.0742 & 0.0000 & [0.0000, 0.0000] & 0.0000 & 1.000 \\
16 & Qwen & 0.1665 & 0.1671 & 0.0006 & [0.0000, 0.0019] & 0.0000 & 0.334 \\
16 & Mean loss & 0.1447 & 0.1449 & 0.0002 & [0.0000, 0.0006] & 0.0000 & 0.334 \\
32 & Mistral & 0.0742 & 0.0742 & 0.0000 & [0.0000, 0.0000] & 0.0000 & 1.000 \\
32 & Qwen & 0.1665 & 0.1665 & 0.0000 & [0.0000, 0.0000] & 0.0000 & 1.000 \\
32 & Mean loss & 0.1447 & 0.1447 & 0.0000 & [0.0000, 0.0000] & 0.0000 & 1.000 \\
64 & Mistral & 0.0742 & 0.0742 & 0.0000 & [0.0000, 0.0000] & 0.0000 & 1.000 \\
64 & Qwen & 0.1665 & 0.1665 & 0.0000 & [0.0000, 0.0000] & 0.0000 & 1.000 \\
64 & Mean loss & 0.1447 & 0.1447 & 0.0000 & [0.0000, 0.0000] & 0.0000 & 1.000 \\
\addlinespace\multicolumn{8}{l}{\textbf{Original 96, seed 271828}} \\
2 & Mistral & 0.0742 & 0.2119 & 0.1377 & [0.0984, 0.1753] & 0.1040 & 0.000 \\
2 & Qwen & 0.1665 & 0.3209 & 0.1544 & [0.1227, 0.1889] & 0.1273 & 0.000 \\
2 & Mean loss & 0.1447 & 0.2664 & 0.1217 & [0.0998, 0.1427] & 0.1029 & 0.000 \\
4 & Mistral & 0.0742 & 0.1618 & 0.0876 & [0.0625, 0.1166] & 0.0661 & 0.000 \\
4 & Qwen & 0.1665 & 0.2470 & 0.0805 & [0.0574, 0.0992] & 0.0615 & 0.000 \\
4 & Mean loss & 0.1447 & 0.2103 & 0.0656 & [0.0466, 0.0843] & 0.0497 & 0.000 \\
8 & Mistral & 0.0742 & 0.0939 & 0.0197 & [0.0059, 0.0362] & 0.0075 & 0.000 \\
8 & Qwen & 0.1665 & 0.1835 & 0.0170 & [0.0049, 0.0306] & 0.0068 & 0.000 \\
8 & Mean loss & 0.1447 & 0.1576 & 0.0129 & [0.0038, 0.0239] & 0.0050 & 0.000 \\
16 & Mistral & 0.0742 & 0.0780 & 0.0039 & [0.0004, 0.0084] & 0.0004 & 0.022 \\
16 & Qwen & 0.1665 & 0.1694 & 0.0029 & [0.0000, 0.0075] & 0.0000 & 0.088 \\
16 & Mean loss & 0.1447 & 0.1468 & 0.0021 & [0.0002, 0.0050] & 0.0002 & 0.025 \\
32 & Mistral & 0.0742 & 0.0742 & 0.0000 & [0.0000, 0.0000] & 0.0000 & 1.000 \\
32 & Qwen & 0.1665 & 0.1665 & 0.0000 & [0.0000, 0.0000] & 0.0000 & 1.000 \\
32 & Mean loss & 0.1447 & 0.1447 & 0.0000 & [0.0000, 0.0000] & 0.0000 & 1.000 \\
64 & Mistral & 0.0742 & 0.0742 & 0.0000 & [0.0000, 0.0000] & 0.0000 & 1.000 \\
64 & Qwen & 0.1665 & 0.1665 & 0.0000 & [0.0000, 0.0000] & 0.0000 & 1.000 \\
64 & Mean loss & 0.1447 & 0.1447 & 0.0000 & [0.0000, 0.0000] & 0.0000 & 1.000 \\
\addlinespace\multicolumn{8}{l}{\textbf{Reused 96, seed 314159}} \\
2 & Mistral & 0.1006 & 0.1503 & 0.0497 & [0.0240, 0.0777] & 0.0274 & 0.000 \\
2 & Qwen & 0.1910 & 0.2413 & 0.0503 & [0.0344, 0.0669] & 0.0373 & 0.000 \\
2 & Mean loss & 0.1731 & 0.2447 & 0.0716 & [0.0433, 0.1021] & 0.0478 & 0.000 \\
4 & Mistral & 0.1006 & 0.1352 & 0.0346 & [0.0129, 0.0642] & 0.0148 & 0.000 \\
4 & Qwen & 0.1910 & 0.2246 & 0.0336 & [0.0224, 0.0465] & 0.0239 & 0.000 \\
4 & Mean loss & 0.1731 & 0.2077 & 0.0346 & [0.0244, 0.0443] & 0.0259 & 0.000 \\
8 & Mistral & 0.1006 & 0.1192 & 0.0185 & [0.0053, 0.0376] & 0.0062 & 0.000 \\
8 & Qwen & 0.1910 & 0.2114 & 0.0204 & [0.0105, 0.0325] & 0.0117 & 0.000 \\
8 & Mean loss & 0.1731 & 0.1892 & 0.0161 & [0.0087, 0.0257] & 0.0098 & 0.000 \\
16 & Mistral & 0.1006 & 0.1010 & 0.0004 & [0.0000, 0.0013] & 0.0000 & 0.360 \\
16 & Qwen & 0.1910 & 0.1956 & 0.0046 & [0.0010, 0.0090] & 0.0010 & 0.019 \\
16 & Mean loss & 0.1731 & 0.1741 & 0.0010 & [0.0000, 0.0030] & 0.0000 & 0.360 \\
32 & Mistral & 0.1006 & 0.1006 & 0.0000 & [0.0000, 0.0000] & 0.0000 & 1.000 \\
32 & Qwen & 0.1910 & 0.1910 & 0.0000 & [0.0000, 0.0000] & 0.0000 & 1.000 \\
32 & Mean loss & 0.1731 & 0.1731 & 0.0000 & [0.0000, 0.0000] & 0.0000 & 1.000 \\
64 & Mistral & 0.1006 & 0.1006 & 0.0000 & [0.0000, 0.0000] & 0.0000 & 1.000 \\
64 & Qwen & 0.1910 & 0.1910 & 0.0000 & [0.0000, 0.0000] & 0.0000 & 1.000 \\
64 & Mean loss & 0.1731 & 0.1731 & 0.0000 & [0.0000, 0.0000] & 0.0000 & 1.000 \\
\addlinespace\multicolumn{8}{l}{\textbf{Reused 96, seed 271828}} \\
2 & Mistral & 0.1006 & 0.1522 & 0.0516 & [0.0252, 0.0797] & 0.0279 & 0.000 \\
2 & Qwen & 0.1910 & 0.2494 & 0.0584 & [0.0394, 0.0783] & 0.0424 & 0.000 \\
2 & Mean loss & 0.1731 & 0.2509 & 0.0778 & [0.0478, 0.1104] & 0.0522 & 0.000 \\
4 & Mistral & 0.1006 & 0.1422 & 0.0416 & [0.0185, 0.0689] & 0.0214 & 0.000 \\
4 & Qwen & 0.1910 & 0.2334 & 0.0424 & [0.0270, 0.0595] & 0.0294 & 0.000 \\
4 & Mean loss & 0.1731 & 0.2389 & 0.0658 & [0.0371, 0.0932] & 0.0423 & 0.000 \\
8 & Mistral & 0.1006 & 0.1233 & 0.0227 & [0.0067, 0.0416] & 0.0089 & 0.001 \\
8 & Qwen & 0.1910 & 0.2092 & 0.0182 & [0.0091, 0.0279] & 0.0105 & 0.000 \\
8 & Mean loss & 0.1731 & 0.1864 & 0.0133 & [0.0086, 0.0179] & 0.0093 & 0.000 \\
16 & Mistral & 0.1006 & 0.1057 & 0.0051 & [0.0005, 0.0105] & 0.0005 & 0.022 \\
16 & Qwen & 0.1910 & 0.1983 & 0.0074 & [0.0033, 0.0111] & 0.0042 & 0.000 \\
16 & Mean loss & 0.1731 & 0.1808 & 0.0077 & [0.0022, 0.0138] & 0.0028 & 0.000 \\
32 & Mistral & 0.1006 & 0.1006 & 0.0000 & [0.0000, 0.0000] & 0.0000 & 1.000 \\
32 & Qwen & 0.1910 & 0.1910 & 0.0000 & [0.0000, 0.0000] & 0.0000 & 1.000 \\
32 & Mean loss & 0.1731 & 0.1731 & 0.0000 & [0.0000, 0.0000] & 0.0000 & 1.000 \\
64 & Mistral & 0.1006 & 0.1006 & 0.0000 & [0.0000, 0.0000] & 0.0000 & 1.000 \\
64 & Qwen & 0.1910 & 0.1910 & 0.0000 & [0.0000, 0.0000] & 0.0000 & 1.000 \\
64 & Mean loss & 0.1731 & 0.1731 & 0.0000 & [0.0000, 0.0000] & 0.0000 & 1.000 \\
\end{longtable}
\endgroup

\begingroup
\small
\setlength{\tabcolsep}{4pt}
\begin{longtable}{@{}rlrrrrrr@{}}
\caption{Complete Phi supervised-selector results for observation $Z$ only. Each row uses eight policy-area clusters and 2,000 percentile bootstrap repetitions. Features are selected on the fitting split and fixed across evaluators. $f_0$ uses the saved $10^{-12}$ tolerance; intervals are individual, not simultaneous.}\label{tab:phi-supervised-zonly-all}\\
\toprule
$B$ & Evaluator & $\widehat R_{\mathrm{full}}$ & $\widehat R_V^\star$ & $\widehat\Delta_V$ & 95\% CI & 95\% LCB & $f_0$ \\
\midrule
\endfirsthead
\multicolumn{8}{l}{\small Table~\ref{tab:phi-supervised-zonly-all} continued.} \\
\toprule
$B$ & Evaluator & $\widehat R_{\mathrm{full}}$ & $\widehat R_V^\star$ & $\widehat\Delta_V$ & 95\% CI & 95\% LCB & $f_0$ \\
\midrule
\endhead
\bottomrule
\endfoot
\addlinespace\multicolumn{8}{l}{\textbf{Original 96, seed 314159}} \\
2 & Mistral & 0.0742 & 0.1852 & 0.1110 & [0.0732, 0.1296] & 0.0799 & 0.000 \\
2 & Qwen & 0.1665 & 0.2675 & 0.1010 & [0.0647, 0.1460] & 0.0690 & 0.000 \\
2 & Mean loss & 0.1447 & 0.2263 & 0.0816 & [0.0495, 0.1115] & 0.0540 & 0.000 \\
4 & Mistral & 0.0742 & 0.1472 & 0.0730 & [0.0400, 0.0934] & 0.0444 & 0.000 \\
4 & Qwen & 0.1665 & 0.2587 & 0.0923 & [0.0582, 0.1227] & 0.0631 & 0.000 \\
4 & Mean loss & 0.1447 & 0.2207 & 0.0759 & [0.0447, 0.0968] & 0.0488 & 0.000 \\
8 & Mistral & 0.0742 & 0.1167 & 0.0425 & [0.0117, 0.0540] & 0.0143 & 0.000 \\
8 & Qwen & 0.1665 & 0.2123 & 0.0458 & [0.0112, 0.0567] & 0.0152 & 0.001 \\
8 & Mean loss & 0.1447 & 0.1822 & 0.0375 & [0.0082, 0.0474] & 0.0108 & 0.000 \\
16 & Mistral & 0.0742 & 0.0800 & 0.0058 & [0.0000, 0.0133] & 0.0000 & 0.164 \\
16 & Qwen & 0.1665 & 0.1708 & 0.0044 & [0.0000, 0.0073] & 0.0000 & 0.164 \\
16 & Mean loss & 0.1447 & 0.1483 & 0.0035 & [0.0000, 0.0074] & 0.0000 & 0.164 \\
32 & Mistral & 0.0742 & 0.0742 & 0.0000 & [0.0000, 0.0000] & 0.0000 & 1.000 \\
32 & Qwen & 0.1665 & 0.1665 & 0.0000 & [0.0000, 0.0000] & 0.0000 & 1.000 \\
32 & Mean loss & 0.1447 & 0.1447 & 0.0000 & [0.0000, 0.0000] & 0.0000 & 1.000 \\
64 & Mistral & 0.0742 & 0.0742 & 0.0000 & [0.0000, 0.0000] & 0.0000 & 1.000 \\
64 & Qwen & 0.1665 & 0.1665 & 0.0000 & [0.0000, 0.0000] & 0.0000 & 1.000 \\
64 & Mean loss & 0.1447 & 0.1447 & 0.0000 & [0.0000, 0.0000] & 0.0000 & 1.000 \\
\addlinespace\multicolumn{8}{l}{\textbf{Original 96, seed 271828}} \\
2 & Mistral & 0.0742 & 0.1517 & 0.0775 & [0.0442, 0.0978] & 0.0482 & 0.000 \\
2 & Qwen & 0.1665 & 0.2152 & 0.0487 & [0.0254, 0.0654] & 0.0308 & 0.000 \\
2 & Mean loss & 0.1447 & 0.1932 & 0.0485 & [0.0262, 0.0587] & 0.0288 & 0.000 \\
4 & Mistral & 0.0742 & 0.1415 & 0.0673 & [0.0325, 0.0746] & 0.0346 & 0.000 \\
4 & Qwen & 0.1665 & 0.2046 & 0.0381 & [0.0165, 0.0502] & 0.0198 & 0.001 \\
4 & Mean loss & 0.1447 & 0.1782 & 0.0335 & [0.0140, 0.0428] & 0.0168 & 0.000 \\
8 & Mistral & 0.0742 & 0.0975 & 0.0233 & [0.0099, 0.0293] & 0.0115 & 0.001 \\
8 & Qwen & 0.1665 & 0.1870 & 0.0206 & [0.0024, 0.0308] & 0.0046 & 0.005 \\
8 & Mean loss & 0.1447 & 0.1584 & 0.0136 & [0.0035, 0.0176] & 0.0048 & 0.001 \\
16 & Mistral & 0.0742 & 0.0780 & 0.0039 & [0.0000, 0.0069] & 0.0000 & 0.403 \\
16 & Qwen & 0.1665 & 0.1677 & 0.0013 & [0.0000, 0.0027] & 0.0000 & 0.194 \\
16 & Mean loss & 0.1447 & 0.1459 & 0.0012 & [0.0000, 0.0017] & 0.0000 & 0.195 \\
32 & Mistral & 0.0742 & 0.0742 & 0.0000 & [0.0000, 0.0000] & 0.0000 & 1.000 \\
32 & Qwen & 0.1665 & 0.1670 & 0.0005 & [0.0000, 0.0015] & 0.0000 & 0.569 \\
32 & Mean loss & 0.1447 & 0.1447 & 0.0000 & [0.0000, 0.0000] & 0.0000 & 1.000 \\
64 & Mistral & 0.0742 & 0.0742 & 0.0000 & [0.0000, 0.0000] & 0.0000 & 1.000 \\
64 & Qwen & 0.1665 & 0.1665 & 0.0000 & [0.0000, 0.0000] & 0.0000 & 1.000 \\
64 & Mean loss & 0.1447 & 0.1447 & 0.0000 & [0.0000, 0.0000] & 0.0000 & 1.000 \\
\addlinespace\multicolumn{8}{l}{\textbf{Reused 96, seed 314159}} \\
2 & Mistral & 0.1006 & 0.1661 & 0.0655 & [0.0376, 0.0917] & 0.0414 & 0.000 \\
2 & Qwen & 0.1910 & 0.2983 & 0.1073 & [0.0445, 0.1581] & 0.0503 & 0.000 \\
2 & Mean loss & 0.1731 & 0.2683 & 0.0952 & [0.0664, 0.1173] & 0.0708 & 0.000 \\
4 & Mistral & 0.1006 & 0.1594 & 0.0587 & [0.0325, 0.0703] & 0.0359 & 0.000 \\
4 & Qwen & 0.1910 & 0.2943 & 0.1033 & [0.0442, 0.1321] & 0.0503 & 0.000 \\
4 & Mean loss & 0.1731 & 0.2560 & 0.0829 & [0.0601, 0.0976] & 0.0630 & 0.000 \\
8 & Mistral & 0.1006 & 0.1338 & 0.0331 & [0.0142, 0.0408] & 0.0160 & 0.000 \\
8 & Qwen & 0.1910 & 0.2520 & 0.0610 & [0.0268, 0.0731] & 0.0299 & 0.000 \\
8 & Mean loss & 0.1731 & 0.2113 & 0.0382 & [0.0207, 0.0464] & 0.0228 & 0.000 \\
16 & Mistral & 0.1006 & 0.1076 & 0.0070 & [0.0013, 0.0117] & 0.0017 & 0.004 \\
16 & Qwen & 0.1910 & 0.2134 & 0.0224 & [0.0048, 0.0308] & 0.0064 & 0.000 \\
16 & Mean loss & 0.1731 & 0.1846 & 0.0115 & [0.0024, 0.0195] & 0.0034 & 0.004 \\
32 & Mistral & 0.1006 & 0.1029 & 0.0023 & [0.0000, 0.0069] & 0.0000 & 0.339 \\
32 & Qwen & 0.1910 & 0.1969 & 0.0059 & [0.0000, 0.0092] & 0.0000 & 0.137 \\
32 & Mean loss & 0.1731 & 0.1769 & 0.0039 & [0.0000, 0.0054] & 0.0000 & 0.118 \\
64 & Mistral & 0.1006 & 0.1006 & 0.0000 & [0.0000, 0.0000] & 0.0000 & 1.000 \\
64 & Qwen & 0.1910 & 0.1910 & 0.0000 & [0.0000, 0.0000] & 0.0000 & 1.000 \\
64 & Mean loss & 0.1731 & 0.1731 & 0.0000 & [0.0000, 0.0000] & 0.0000 & 1.000 \\
\addlinespace\multicolumn{8}{l}{\textbf{Reused 96, seed 271828}} \\
2 & Mistral & 0.1006 & 0.1619 & 0.0612 & [0.0332, 0.0788] & 0.0368 & 0.000 \\
2 & Qwen & 0.1910 & 0.2938 & 0.1028 & [0.0419, 0.1248] & 0.0482 & 0.000 \\
2 & Mean loss & 0.1731 & 0.2435 & 0.0704 & [0.0436, 0.0946] & 0.0478 & 0.000 \\
4 & Mistral & 0.1006 & 0.1516 & 0.0509 & [0.0240, 0.0679] & 0.0279 & 0.000 \\
4 & Qwen & 0.1910 & 0.2831 & 0.0921 & [0.0403, 0.1172] & 0.0458 & 0.000 \\
4 & Mean loss & 0.1731 & 0.2368 & 0.0637 & [0.0361, 0.0761] & 0.0401 & 0.000 \\
8 & Mistral & 0.1006 & 0.1393 & 0.0386 & [0.0116, 0.0500] & 0.0142 & 0.000 \\
8 & Qwen & 0.1910 & 0.2470 & 0.0560 & [0.0200, 0.0666] & 0.0234 & 0.000 \\
8 & Mean loss & 0.1731 & 0.2136 & 0.0406 & [0.0205, 0.0483] & 0.0230 & 0.000 \\
16 & Mistral & 0.1006 & 0.1130 & 0.0124 & [0.0000, 0.0236] & 0.0000 & 0.120 \\
16 & Qwen & 0.1910 & 0.2174 & 0.0264 & [0.0072, 0.0293] & 0.0085 & 0.001 \\
16 & Mean loss & 0.1731 & 0.1881 & 0.0150 & [0.0035, 0.0184] & 0.0048 & 0.001 \\
32 & Mistral & 0.1006 & 0.1006 & 0.0000 & [0.0000, 0.0000] & 0.0000 & 1.000 \\
32 & Qwen & 0.1910 & 0.1948 & 0.0038 & [0.0000, 0.0075] & 0.0000 & 0.294 \\
32 & Mean loss & 0.1731 & 0.1750 & 0.0019 & [0.0000, 0.0037] & 0.0000 & 0.574 \\
64 & Mistral & 0.1006 & 0.1006 & 0.0000 & [0.0000, 0.0000] & 0.0000 & 1.000 \\
64 & Qwen & 0.1910 & 0.1910 & 0.0000 & [0.0000, 0.0000] & 0.0000 & 1.000 \\
64 & Mean loss & 0.1731 & 0.1731 & 0.0000 & [0.0000, 0.0000] & 0.0000 & 1.000 \\
\end{longtable}
\endgroup

\begingroup
\small
\setlength{\tabcolsep}{4pt}
\begin{longtable}{@{}rlrrrrrr@{}}
\caption{Complete Phi variance-selector results for observation $Z$ only. Each row uses eight policy-area clusters and 2,000 percentile bootstrap repetitions. Features are selected on the fitting split and fixed across evaluators. $f_0$ uses the saved $10^{-12}$ tolerance; intervals are individual, not simultaneous.}\label{tab:phi-variance-zonly-all}\\
\toprule
$B$ & Evaluator & $\widehat R_{\mathrm{full}}$ & $\widehat R_V^\star$ & $\widehat\Delta_V$ & 95\% CI & 95\% LCB & $f_0$ \\
\midrule
\endfirsthead
\multicolumn{8}{l}{\small Table~\ref{tab:phi-variance-zonly-all} continued.} \\
\toprule
$B$ & Evaluator & $\widehat R_{\mathrm{full}}$ & $\widehat R_V^\star$ & $\widehat\Delta_V$ & 95\% CI & 95\% LCB & $f_0$ \\
\midrule
\endhead
\bottomrule
\endfoot
\addlinespace\multicolumn{8}{l}{\textbf{Original 96, seed 314159}} \\
2 & Mistral & 0.0742 & 0.2119 & 0.1377 & [0.0938, 0.1673] & 0.0994 & 0.000 \\
2 & Qwen & 0.1665 & 0.3209 & 0.1544 & [0.1227, 0.1879] & 0.1273 & 0.000 \\
2 & Mean loss & 0.1447 & 0.2664 & 0.1217 & [0.0998, 0.1385] & 0.1029 & 0.000 \\
4 & Mistral & 0.0742 & 0.1592 & 0.0850 & [0.0442, 0.1146] & 0.0490 & 0.000 \\
4 & Qwen & 0.1665 & 0.2581 & 0.0916 & [0.0615, 0.1191] & 0.0667 & 0.000 \\
4 & Mean loss & 0.1447 & 0.2132 & 0.0685 & [0.0437, 0.0879] & 0.0460 & 0.000 \\
8 & Mistral & 0.0742 & 0.1226 & 0.0484 & [0.0214, 0.0516] & 0.0243 & 0.000 \\
8 & Qwen & 0.1665 & 0.2165 & 0.0500 & [0.0270, 0.0615] & 0.0296 & 0.000 \\
8 & Mean loss & 0.1447 & 0.1934 & 0.0487 & [0.0245, 0.0515] & 0.0285 & 0.000 \\
16 & Mistral & 0.0742 & 0.0742 & 0.0000 & [0.0000, 0.0000] & 0.0000 & 1.000 \\
16 & Qwen & 0.1665 & 0.1671 & 0.0006 & [0.0000, 0.0019] & 0.0000 & 0.334 \\
16 & Mean loss & 0.1447 & 0.1449 & 0.0002 & [0.0000, 0.0006] & 0.0000 & 0.334 \\
32 & Mistral & 0.0742 & 0.0742 & 0.0000 & [0.0000, 0.0000] & 0.0000 & 1.000 \\
32 & Qwen & 0.1665 & 0.1665 & 0.0000 & [0.0000, 0.0000] & 0.0000 & 1.000 \\
32 & Mean loss & 0.1447 & 0.1447 & 0.0000 & [0.0000, 0.0000] & 0.0000 & 1.000 \\
64 & Mistral & 0.0742 & 0.0742 & 0.0000 & [0.0000, 0.0000] & 0.0000 & 1.000 \\
64 & Qwen & 0.1665 & 0.1665 & 0.0000 & [0.0000, 0.0000] & 0.0000 & 1.000 \\
64 & Mean loss & 0.1447 & 0.1447 & 0.0000 & [0.0000, 0.0000] & 0.0000 & 1.000 \\
\addlinespace\multicolumn{8}{l}{\textbf{Original 96, seed 271828}} \\
2 & Mistral & 0.0742 & 0.2119 & 0.1377 & [0.0984, 0.1753] & 0.1040 & 0.000 \\
2 & Qwen & 0.1665 & 0.3209 & 0.1544 & [0.1227, 0.1889] & 0.1273 & 0.000 \\
2 & Mean loss & 0.1447 & 0.2664 & 0.1217 & [0.0998, 0.1427] & 0.1029 & 0.000 \\
4 & Mistral & 0.0742 & 0.2104 & 0.1363 & [0.0906, 0.1502] & 0.0950 & 0.000 \\
4 & Qwen & 0.1665 & 0.2748 & 0.1083 & [0.0733, 0.1331] & 0.0785 & 0.000 \\
4 & Mean loss & 0.1447 & 0.2494 & 0.1047 & [0.0649, 0.1222] & 0.0716 & 0.000 \\
8 & Mistral & 0.0742 & 0.1252 & 0.0510 & [0.0150, 0.0668] & 0.0183 & 0.000 \\
8 & Qwen & 0.1665 & 0.2173 & 0.0508 & [0.0233, 0.0585] & 0.0256 & 0.000 \\
8 & Mean loss & 0.1447 & 0.1872 & 0.0425 & [0.0157, 0.0576] & 0.0182 & 0.000 \\
16 & Mistral & 0.0742 & 0.0803 & 0.0061 & [0.0004, 0.0109] & 0.0004 & 0.022 \\
16 & Qwen & 0.1665 & 0.1708 & 0.0044 & [0.0005, 0.0083] & 0.0006 & 0.025 \\
16 & Mean loss & 0.1447 & 0.1486 & 0.0039 & [0.0002, 0.0065] & 0.0002 & 0.025 \\
32 & Mistral & 0.0742 & 0.0742 & 0.0000 & [0.0000, 0.0000] & 0.0000 & 1.000 \\
32 & Qwen & 0.1665 & 0.1679 & 0.0015 & [0.0000, 0.0044] & 0.0000 & 0.578 \\
32 & Mean loss & 0.1447 & 0.1447 & 0.0000 & [0.0000, 0.0000] & 0.0000 & 1.000 \\
64 & Mistral & 0.0742 & 0.0742 & 0.0000 & [0.0000, 0.0000] & 0.0000 & 1.000 \\
64 & Qwen & 0.1665 & 0.1665 & 0.0000 & [0.0000, 0.0000] & 0.0000 & 1.000 \\
64 & Mean loss & 0.1447 & 0.1447 & 0.0000 & [0.0000, 0.0000] & 0.0000 & 1.000 \\
\addlinespace\multicolumn{8}{l}{\textbf{Reused 96, seed 314159}} \\
2 & Mistral & 0.1006 & 0.1691 & 0.0684 & [0.0398, 0.0955] & 0.0437 & 0.000 \\
2 & Qwen & 0.1910 & 0.2923 & 0.1013 & [0.0382, 0.1648] & 0.0421 & 0.000 \\
2 & Mean loss & 0.1731 & 0.2771 & 0.1040 & [0.0679, 0.1258] & 0.0726 & 0.000 \\
4 & Mistral & 0.1006 & 0.1600 & 0.0594 & [0.0293, 0.0790] & 0.0329 & 0.000 \\
4 & Qwen & 0.1910 & 0.2797 & 0.0887 & [0.0335, 0.1326] & 0.0373 & 0.000 \\
4 & Mean loss & 0.1731 & 0.2510 & 0.0779 & [0.0421, 0.0960] & 0.0471 & 0.000 \\
8 & Mistral & 0.1006 & 0.1288 & 0.0281 & [0.0087, 0.0478] & 0.0101 & 0.000 \\
8 & Qwen & 0.1910 & 0.2439 & 0.0529 & [0.0152, 0.0709] & 0.0172 & 0.000 \\
8 & Mean loss & 0.1731 & 0.2140 & 0.0409 & [0.0152, 0.0492] & 0.0174 & 0.000 \\
16 & Mistral & 0.1006 & 0.1052 & 0.0046 & [0.0000, 0.0081] & 0.0000 & 0.069 \\
16 & Qwen & 0.1910 & 0.2042 & 0.0132 & [0.0044, 0.0183] & 0.0052 & 0.000 \\
16 & Mean loss & 0.1731 & 0.1830 & 0.0099 & [0.0026, 0.0123] & 0.0032 & 0.004 \\
32 & Mistral & 0.1006 & 0.1016 & 0.0009 & [0.0000, 0.0019] & 0.0000 & 0.421 \\
32 & Qwen & 0.1910 & 0.1994 & 0.0084 & [0.0000, 0.0120] & 0.0004 & 0.035 \\
32 & Mean loss & 0.1731 & 0.1767 & 0.0036 & [0.0000, 0.0059] & 0.0000 & 0.068 \\
64 & Mistral & 0.1006 & 0.1006 & 0.0000 & [0.0000, 0.0000] & 0.0000 & 1.000 \\
64 & Qwen & 0.1910 & 0.1910 & 0.0000 & [0.0000, 0.0000] & 0.0000 & 1.000 \\
64 & Mean loss & 0.1731 & 0.1731 & 0.0000 & [0.0000, 0.0000] & 0.0000 & 1.000 \\
\addlinespace\multicolumn{8}{l}{\textbf{Reused 96, seed 271828}} \\
2 & Mistral & 0.1006 & 0.1697 & 0.0691 & [0.0406, 0.0973] & 0.0442 & 0.000 \\
2 & Qwen & 0.1910 & 0.2983 & 0.1073 & [0.0445, 0.1674] & 0.0503 & 0.000 \\
2 & Mean loss & 0.1731 & 0.2784 & 0.1053 & [0.0709, 0.1306] & 0.0750 & 0.000 \\
4 & Mistral & 0.1006 & 0.1684 & 0.0678 & [0.0393, 0.0950] & 0.0431 & 0.000 \\
4 & Qwen & 0.1910 & 0.2937 & 0.1027 & [0.0403, 0.1602] & 0.0454 & 0.000 \\
4 & Mean loss & 0.1731 & 0.2784 & 0.1053 & [0.0661, 0.1217] & 0.0712 & 0.000 \\
8 & Mistral & 0.1006 & 0.1405 & 0.0399 & [0.0152, 0.0593] & 0.0185 & 0.000 \\
8 & Qwen & 0.1910 & 0.2562 & 0.0652 & [0.0257, 0.0801] & 0.0295 & 0.000 \\
8 & Mean loss & 0.1731 & 0.2177 & 0.0446 & [0.0232, 0.0488] & 0.0251 & 0.000 \\
16 & Mistral & 0.1006 & 0.1131 & 0.0125 & [0.0025, 0.0165] & 0.0036 & 0.001 \\
16 & Qwen & 0.1910 & 0.2204 & 0.0294 & [0.0082, 0.0362] & 0.0097 & 0.000 \\
16 & Mean loss & 0.1731 & 0.1933 & 0.0202 & [0.0052, 0.0238] & 0.0069 & 0.000 \\
32 & Mistral & 0.1006 & 0.1032 & 0.0026 & [0.0000, 0.0052] & 0.0000 & 0.419 \\
32 & Qwen & 0.1910 & 0.2016 & 0.0106 & [0.0000, 0.0173] & 0.0000 & 0.357 \\
32 & Mean loss & 0.1731 & 0.1809 & 0.0078 & [0.0000, 0.0113] & 0.0000 & 0.357 \\
64 & Mistral & 0.1006 & 0.1006 & 0.0000 & [0.0000, 0.0000] & 0.0000 & 1.000 \\
64 & Qwen & 0.1910 & 0.1910 & 0.0000 & [0.0000, 0.0000] & 0.0000 & 1.000 \\
64 & Mean loss & 0.1731 & 0.1731 & 0.0000 & [0.0000, 0.0000] & 0.0000 & 1.000 \\
\end{longtable}
\endgroup

\begin{figure}[tbp]
\centering
\begin{subfigure}[t]{0.48\linewidth}
\includegraphics[width=\linewidth]{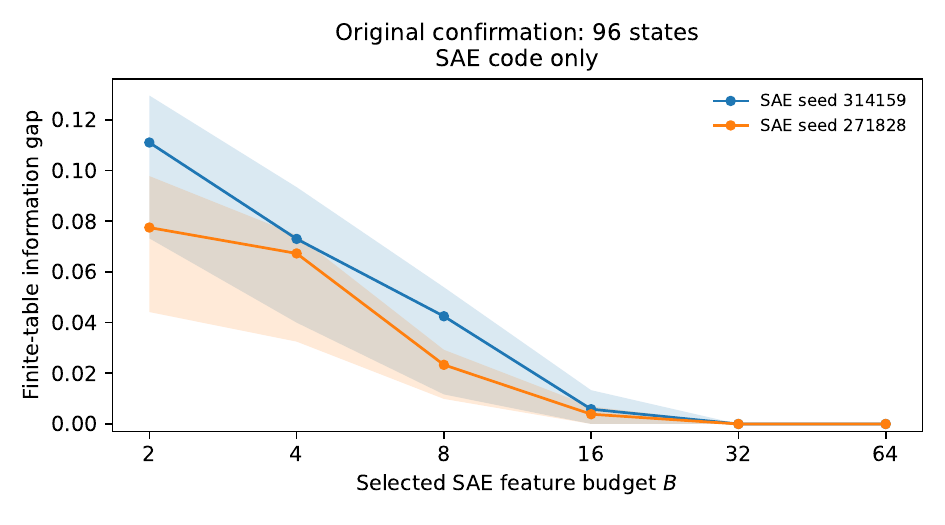}
\caption{Original confirmation collection.}
\end{subfigure}\hfill
\begin{subfigure}[t]{0.48\linewidth}
\includegraphics[width=\linewidth]{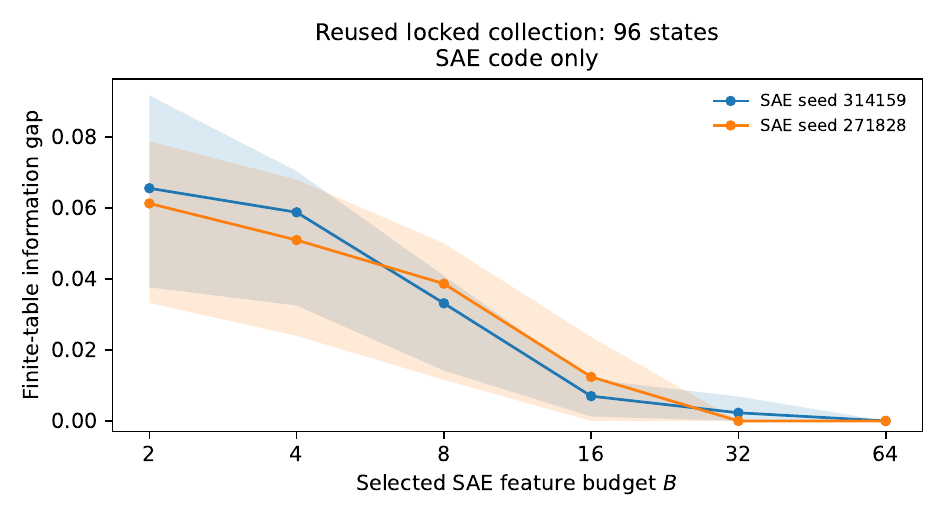}
\caption{Reused locked collection.}
\end{subfigure}
\caption{Paired Phi $Z$-only Mistral-loss curves. The policy-aware counterparts appear in Figure~\ref{fig:phi-policy-budget}. These plots change only the observation key, not the response samples or selected features.}
\label{fig:phi-zonly-budget}
\end{figure}

\subsection{Reconstruction and feature occupancy}

For activation vectors $h_i$ and reconstructions $\widehat h_i$, let $\overline h$ be the mean activation on the listed dataset. For positive denominators, the centered and uncentered diagnostics are
\[
 \mathrm{FVU}=
 \frac{\sum_i\|h_i-\widehat h_i\|_2^2}{\sum_i\|h_i-\overline h\|_2^2},
 \qquad
 E_{\mathrm{rel}}=
 \frac{\sum_i\|h_i-\widehat h_i\|_2^2}{\sum_i\|h_i\|_2^2}.
\]
The denominator difference does not imply negligible reconstruction error. Mean cosine is averaged over individual vectors. The audit JSON retains a historical field named \texttt{dead\_fraction\_on\_validation}; on audit records this field denotes features unobserved on that small audit sample. It is not evidence that those features are globally dead. Quantiles and activation/reconstruction norms remain in the original JSON/CSV files.

\begin{table}[tbp]
\centering
\caption{Phi reconstruction diagnostics. FVU is centered fraction of variance unexplained; relative error uses the same reconstruction-error numerator divided by uncentered activation energy. Active counts features observed at least once in the listed dataset, not globally active or dead features.}
\label{tab:phi-quality}
\small
\setlength{\tabcolsep}{4pt}
\begin{tabular}{@{}llrrrrrr@{}}
\toprule
Seed & Dataset & $n$ & FVU & Relative error & Cosine & $L_0$ & Active \\
\midrule
314159 & Validation & 300000 & 0.0162 & 0.0153 & 0.9277 & 64.0 & 16384 \\
314159 & Original 240 & 240 & 17.3125 & 0.3062 & 0.8327 & 64.0 & 192 \\
314159 & Reused 96 & 96 & 2.8355 & 0.2767 & 0.8496 & 64.0 & 445 \\
271828 & Validation & 300000 & 0.0163 & 0.0154 & 0.9277 & 64.0 & 16384 \\
271828 & Original 240 & 240 & 16.6518 & 0.2946 & 0.8397 & 64.0 & 166 \\
271828 & Reused 96 & 96 & 2.7705 & 0.2704 & 0.8536 & 64.0 & 437 \\
\bottomrule
\end{tabular}
\end{table}

\begingroup
\small
\setlength{\tabcolsep}{4pt}
\begin{longtable}{@{}rlrrrrr@{}}
\caption{Phi supervised-selector occupancy. Fit and Eval count distinct complete observation keys. Collision is the fraction of evaluation states sharing an observation key; Unseen is the fraction whose complete key is absent from fitting. For $(Z,S)$, new policy text alone makes the complete key unseen.}\label{tab:phi-occupancy-supervised}\\
\toprule
$B$ & Observation & Fit & Eval & Collision & Unseen & Largest cell \\
\midrule
\endfirsthead
\multicolumn{7}{l}{\small Table~\ref{tab:phi-occupancy-supervised} continued.} \\
\toprule
$B$ & Observation & Fit & Eval & Collision & Unseen & Largest cell \\
\midrule
\endhead
\bottomrule
\endfoot
\addlinespace\multicolumn{7}{l}{\textbf{Original 96, seed 314159}} \\
2 & $Z$ only & 4 & 4 & 0.990 & 0.000 & 51 \\
2 & $(Z,S)$ & 38 & 24 & 0.958 & 1.000 & 11 \\
4 & $Z$ only & 15 & 11 & 0.979 & 0.062 & 27 \\
4 & $(Z,S)$ & 68 & 44 & 0.781 & 1.000 & 6 \\
8 & $Z$ only & 53 & 35 & 0.844 & 0.125 & 9 \\
8 & $(Z,S)$ & 110 & 72 & 0.469 & 1.000 & 4 \\
16 & $Z$ only & 125 & 87 & 0.177 & 0.844 & 3 \\
16 & $(Z,S)$ & 139 & 94 & 0.042 & 1.000 & 2 \\
32 & $Z$ only & 143 & 96 & 0.000 & 0.969 & 1 \\
32 & $(Z,S)$ & 144 & 96 & 0.000 & 1.000 & 1 \\
64 & $Z$ only & 144 & 96 & 0.000 & 1.000 & 1 \\
64 & $(Z,S)$ & 144 & 96 & 0.000 & 1.000 & 1 \\
\addlinespace\multicolumn{7}{l}{\textbf{Original 96, seed 271828}} \\
2 & $Z$ only & 4 & 4 & 1.000 & 0.000 & 52 \\
2 & $(Z,S)$ & 39 & 25 & 0.969 & 1.000 & 8 \\
4 & $Z$ only & 15 & 11 & 1.000 & 0.000 & 19 \\
4 & $(Z,S)$ & 77 & 52 & 0.740 & 1.000 & 4 \\
8 & $Z$ only & 57 & 36 & 0.823 & 0.115 & 9 \\
8 & $(Z,S)$ & 109 & 71 & 0.458 & 1.000 & 4 \\
16 & $Z$ only & 125 & 83 & 0.229 & 0.740 & 4 \\
16 & $(Z,S)$ & 140 & 95 & 0.021 & 1.000 & 2 \\
32 & $Z$ only & 144 & 94 & 0.042 & 1.000 & 2 \\
32 & $(Z,S)$ & 144 & 96 & 0.000 & 1.000 & 1 \\
64 & $Z$ only & 144 & 96 & 0.000 & 1.000 & 1 \\
64 & $(Z,S)$ & 144 & 96 & 0.000 & 1.000 & 1 \\
\addlinespace\multicolumn{7}{l}{\textbf{Reused 96, seed 314159}} \\
2 & $Z$ only & 4 & 3 & 0.990 & 0.000 & 57 \\
2 & $(Z,S)$ & 38 & 17 & 0.990 & 1.000 & 9 \\
4 & $Z$ only & 15 & 6 & 0.990 & 0.000 & 48 \\
4 & $(Z,S)$ & 68 & 32 & 0.885 & 1.000 & 8 \\
8 & $Z$ only & 53 & 25 & 0.885 & 0.438 & 21 \\
8 & $(Z,S)$ & 110 & 58 & 0.656 & 1.000 & 4 \\
16 & $Z$ only & 125 & 60 & 0.521 & 0.938 & 11 \\
16 & $(Z,S)$ & 139 & 84 & 0.219 & 1.000 & 4 \\
32 & $Z$ only & 143 & 84 & 0.208 & 1.000 & 5 \\
32 & $(Z,S)$ & 144 & 93 & 0.062 & 1.000 & 2 \\
64 & $Z$ only & 144 & 94 & 0.042 & 1.000 & 2 \\
64 & $(Z,S)$ & 144 & 95 & 0.021 & 1.000 & 2 \\
\addlinespace\multicolumn{7}{l}{\textbf{Reused 96, seed 271828}} \\
2 & $Z$ only & 4 & 4 & 1.000 & 0.000 & 51 \\
2 & $(Z,S)$ & 39 & 20 & 0.979 & 1.000 & 10 \\
4 & $Z$ only & 15 & 8 & 0.979 & 0.000 & 48 \\
4 & $(Z,S)$ & 77 & 28 & 0.896 & 1.000 & 10 \\
8 & $Z$ only & 57 & 26 & 0.875 & 0.375 & 11 \\
8 & $(Z,S)$ & 109 & 66 & 0.500 & 1.000 & 5 \\
16 & $Z$ only & 125 & 65 & 0.500 & 0.979 & 6 \\
16 & $(Z,S)$ & 140 & 88 & 0.167 & 1.000 & 2 \\
32 & $Z$ only & 144 & 93 & 0.062 & 1.000 & 2 \\
32 & $(Z,S)$ & 144 & 95 & 0.021 & 1.000 & 2 \\
64 & $Z$ only & 144 & 96 & 0.000 & 1.000 & 1 \\
64 & $(Z,S)$ & 144 & 96 & 0.000 & 1.000 & 1 \\
\end{longtable}
\endgroup

\begingroup
\small
\setlength{\tabcolsep}{4pt}
\begin{longtable}{@{}rlrrrrr@{}}
\caption{Phi variance-selector occupancy. Fit and Eval count distinct complete observation keys. Collision is the fraction of evaluation states sharing an observation key; Unseen is the fraction whose complete key is absent from fitting. For $(Z,S)$, new policy text alone makes the complete key unseen.}\label{tab:phi-occupancy-variance}\\
\toprule
$B$ & Observation & Fit & Eval & Collision & Unseen & Largest cell \\
\midrule
\endfirsthead
\multicolumn{7}{l}{\small Table~\ref{tab:phi-occupancy-variance} continued.} \\
\toprule
$B$ & Observation & Fit & Eval & Collision & Unseen & Largest cell \\
\midrule
\endhead
\bottomrule
\endfoot
\addlinespace\multicolumn{7}{l}{\textbf{Original 96, seed 314159}} \\
2 & $Z$ only & 2 & 2 & 1.000 & 0.000 & 59 \\
2 & $(Z,S)$ & 23 & 16 & 1.000 & 1.000 & 9 \\
4 & $Z$ only & 8 & 7 & 1.000 & 0.000 & 21 \\
4 & $(Z,S)$ & 56 & 46 & 0.812 & 1.000 & 6 \\
8 & $Z$ only & 50 & 34 & 0.833 & 0.240 & 8 \\
8 & $(Z,S)$ & 119 & 77 & 0.375 & 1.000 & 3 \\
16 & $Z$ only & 132 & 92 & 0.083 & 0.854 & 2 \\
16 & $(Z,S)$ & 141 & 94 & 0.042 & 1.000 & 2 \\
32 & $Z$ only & 144 & 96 & 0.000 & 1.000 & 1 \\
32 & $(Z,S)$ & 144 & 96 & 0.000 & 1.000 & 1 \\
64 & $Z$ only & 144 & 96 & 0.000 & 1.000 & 1 \\
64 & $(Z,S)$ & 144 & 96 & 0.000 & 1.000 & 1 \\
\addlinespace\multicolumn{7}{l}{\textbf{Original 96, seed 271828}} \\
2 & $Z$ only & 1 & 1 & 1.000 & 0.000 & 96 \\
2 & $(Z,S)$ & 12 & 8 & 1.000 & 1.000 & 12 \\
4 & $Z$ only & 4 & 4 & 1.000 & 0.000 & 34 \\
4 & $(Z,S)$ & 41 & 29 & 0.958 & 1.000 & 7 \\
8 & $Z$ only & 41 & 27 & 0.927 & 0.042 & 16 \\
8 & $(Z,S)$ & 111 & 73 & 0.427 & 1.000 & 3 \\
16 & $Z$ only & 122 & 80 & 0.302 & 0.656 & 3 \\
16 & $(Z,S)$ & 138 & 91 & 0.104 & 1.000 & 2 \\
32 & $Z$ only & 143 & 94 & 0.042 & 0.979 & 2 \\
32 & $(Z,S)$ & 144 & 96 & 0.000 & 1.000 & 1 \\
64 & $Z$ only & 144 & 96 & 0.000 & 1.000 & 1 \\
64 & $(Z,S)$ & 144 & 96 & 0.000 & 1.000 & 1 \\
\addlinespace\multicolumn{7}{l}{\textbf{Reused 96, seed 314159}} \\
2 & $Z$ only & 2 & 3 & 1.000 & 0.052 & 85 \\
2 & $(Z,S)$ & 23 & 14 & 0.979 & 1.000 & 12 \\
4 & $Z$ only & 8 & 7 & 1.000 & 0.052 & 42 \\
4 & $(Z,S)$ & 56 & 33 & 0.906 & 1.000 & 8 \\
8 & $Z$ only & 50 & 34 & 0.833 & 0.198 & 30 \\
8 & $(Z,S)$ & 119 & 65 & 0.500 & 1.000 & 5 \\
16 & $Z$ only & 132 & 71 & 0.344 & 0.969 & 9 \\
16 & $(Z,S)$ & 141 & 90 & 0.125 & 1.000 & 2 \\
32 & $Z$ only & 144 & 83 & 0.229 & 1.000 & 4 \\
32 & $(Z,S)$ & 144 & 96 & 0.000 & 1.000 & 1 \\
64 & $Z$ only & 144 & 95 & 0.021 & 1.000 & 2 \\
64 & $(Z,S)$ & 144 & 96 & 0.000 & 1.000 & 1 \\
\addlinespace\multicolumn{7}{l}{\textbf{Reused 96, seed 271828}} \\
2 & $Z$ only & 1 & 1 & 1.000 & 0.000 & 96 \\
2 & $(Z,S)$ & 12 & 8 & 1.000 & 1.000 & 12 \\
4 & $Z$ only & 4 & 4 & 0.990 & 0.000 & 73 \\
4 & $(Z,S)$ & 41 & 17 & 0.979 & 1.000 & 12 \\
8 & $Z$ only & 41 & 16 & 0.938 & 0.302 & 23 \\
8 & $(Z,S)$ & 111 & 52 & 0.698 & 1.000 & 4 \\
16 & $Z$ only & 122 & 58 & 0.552 & 0.990 & 7 \\
16 & $(Z,S)$ & 138 & 79 & 0.292 & 1.000 & 4 \\
32 & $Z$ only & 143 & 89 & 0.135 & 1.000 & 3 \\
32 & $(Z,S)$ & 144 & 93 & 0.062 & 1.000 & 2 \\
64 & $Z$ only & 144 & 96 & 0.000 & 1.000 & 1 \\
64 & $(Z,S)$ & 144 & 96 & 0.000 & 1.000 & 1 \\
\end{longtable}
\endgroup

\subsection{Policy-area deletion and selector controls}

Each deletion removes exactly one of eight areas and leaves 84 states. The minimum over all 192 saved primary deletion estimates is $0.02363095$. The ranges below are sensitivity results, not additional confidence intervals. Random selection from all features, selection from features that vary on the fitting set, and approximate occupancy matching are recorded only for Mistral loss under the $Z$-only observation. The matching routine uses fitting data and retains 40 controls from its candidate search; its recorded distance and exact-match fraction are reported rather than assuming exact occupancy matching.

\begingroup
\small
\setlength{\tabcolsep}{4pt}
\begin{longtable}{@{}ll ll rrr@{}}
\caption{Phi single policy area deletion sensitivity at $B=2$. Each row summarizes the eight deletions for one setting; each deletion retains 84 states. Values are point-estimate minima, medians, and maxima, not confidence bounds. All 192 deletion-level estimates are preserved in the CSV archive.}\label{tab:phi-loao}\\
\toprule
Collection & Seed & Observation & Evaluator & Minimum & Median & Maximum \\
\midrule
\endfirsthead
\multicolumn{7}{l}{\small Table~\ref{tab:phi-loao} continued.} \\
\toprule
Collection & Seed & Observation & Evaluator & Minimum & Median & Maximum \\
\midrule
\endhead
\bottomrule
\endfoot
Original 96 & 314159 & $Z$ only & Mistral & 0.0989 & 0.1126 & 0.1204 \\
Original 96 & 314159 & $Z$ only & Qwen & 0.0823 & 0.1032 & 0.1107 \\
Original 96 & 314159 & $Z$ only & Mean loss & 0.0685 & 0.0806 & 0.0905 \\
Original 96 & 314159 & $(Z,S)$ & Mistral & 0.0493 & 0.0554 & 0.0637 \\
Original 96 & 314159 & $(Z,S)$ & Qwen & 0.0734 & 0.0847 & 0.0901 \\
Original 96 & 314159 & $(Z,S)$ & Mean loss & 0.0509 & 0.0570 & 0.0622 \\
Original 96 & 271828 & $Z$ only & Mistral & 0.0643 & 0.0770 & 0.0826 \\
Original 96 & 271828 & $Z$ only & Qwen & 0.0445 & 0.0473 & 0.0557 \\
Original 96 & 271828 & $Z$ only & Mean loss & 0.0408 & 0.0466 & 0.0524 \\
Original 96 & 271828 & $(Z,S)$ & Mistral & 0.0387 & 0.0509 & 0.0561 \\
Original 96 & 271828 & $(Z,S)$ & Qwen & 0.0338 & 0.0343 & 0.0417 \\
Original 96 & 271828 & $(Z,S)$ & Mean loss & 0.0236 & 0.0302 & 0.0339 \\
Reused 96 & 314159 & $Z$ only & Mistral & 0.0565 & 0.0655 & 0.0749 \\
Reused 96 & 314159 & $Z$ only & Qwen & 0.0702 & 0.1130 & 0.1201 \\
Reused 96 & 314159 & $Z$ only & Mean loss & 0.0860 & 0.0950 & 0.1035 \\
Reused 96 & 314159 & $(Z,S)$ & Mistral & 0.0315 & 0.0445 & 0.0499 \\
Reused 96 & 314159 & $(Z,S)$ & Qwen & 0.0521 & 0.0591 & 0.0642 \\
Reused 96 & 314159 & $(Z,S)$ & Mean loss & 0.0581 & 0.0718 & 0.0774 \\
Reused 96 & 271828 & $Z$ only & Mistral & 0.0531 & 0.0616 & 0.0700 \\
Reused 96 & 271828 & $Z$ only & Qwen & 0.0702 & 0.0998 & 0.1087 \\
Reused 96 & 271828 & $Z$ only & Mean loss & 0.0592 & 0.0710 & 0.0752 \\
Reused 96 & 271828 & $(Z,S)$ & Mistral & 0.0255 & 0.0335 & 0.0379 \\
Reused 96 & 271828 & $(Z,S)$ & Qwen & 0.0385 & 0.0467 & 0.0524 \\
Reused 96 & 271828 & $(Z,S)$ & Mean loss & 0.0375 & 0.0492 & 0.0502 \\
\end{longtable}
\endgroup

\begingroup
\small
\setlength{\tabcolsep}{4pt}
\begin{longtable}{@{}rlrrrrrr@{}}
\caption{Phi random-selector summaries use Mistral loss and $Z$-only observations exclusively. Each budget has 200 all-feature draws, 200 fitting-nonconstant draws, and 40 controls matched by occupancy on fitting states per seed and collection. These are not policy-aware control results. SD describes variation over the saved selectors, not a bootstrap standard error.}\label{tab:phi-random-controls}\\
\toprule
$B$ & Selector & $N$ & Mean & SD & Minimum & Median & Maximum \\
\midrule
\endfirsthead
\multicolumn{8}{l}{\small Table~\ref{tab:phi-random-controls} continued.} \\
\toprule
$B$ & Selector & $N$ & Mean & SD & Minimum & Median & Maximum \\
\midrule
\endhead
\bottomrule
\endfoot
\addlinespace\multicolumn{8}{l}{\textbf{Original 96, seed 314159}} \\
2 & All features & 200 & 0.1377 & 0.0001 & 0.1366 & 0.1377 & 0.1377 \\
2 & Fit nonconstant & 200 & 0.1298 & 0.0120 & 0.0858 & 0.1369 & 0.1377 \\
2 & Fit matched & 40 & 0.1181 & 0.0143 & 0.0906 & 0.1189 & 0.1377 \\
4 & All features & 200 & 0.1376 & 0.0013 & 0.1215 & 0.1377 & 0.1377 \\
4 & Fit nonconstant & 200 & 0.1199 & 0.0159 & 0.0794 & 0.1239 & 0.1377 \\
4 & Fit matched & 40 & 0.1015 & 0.0140 & 0.0785 & 0.1030 & 0.1294 \\
8 & All features & 200 & 0.1377 & 0.0008 & 0.1261 & 0.1377 & 0.1377 \\
8 & Fit nonconstant & 200 & 0.0956 & 0.0194 & 0.0546 & 0.0967 & 0.1377 \\
8 & Fit matched & 40 & 0.0690 & 0.0168 & 0.0320 & 0.0695 & 0.1091 \\
16 & All features & 200 & 0.1373 & 0.0029 & 0.1062 & 0.1377 & 0.1377 \\
16 & Fit nonconstant & 200 & 0.0524 & 0.0193 & 0.0082 & 0.0514 & 0.1120 \\
16 & Fit matched & 40 & 0.0195 & 0.0088 & 0.0082 & 0.0185 & 0.0436 \\
32 & All features & 200 & 0.1369 & 0.0035 & 0.1150 & 0.1377 & 0.1377 \\
32 & Fit nonconstant & 200 & 0.0082 & 0.0080 & 0.0000 & 0.0060 & 0.0454 \\
32 & Fit matched & 40 & 0.0015 & 0.0020 & 0.0000 & 0.0000 & 0.0071 \\
64 & All features & 200 & 0.1348 & 0.0081 & 0.0927 & 0.1377 & 0.1377 \\
64 & Fit nonconstant & 200 & 0.0002 & 0.0008 & 0.0000 & 0.0000 & 0.0054 \\
64 & Fit matched & 40 & 0.0000 & 0.0000 & 0.0000 & 0.0000 & 0.0000 \\
\addlinespace\multicolumn{8}{l}{\textbf{Original 96, seed 271828}} \\
2 & All features & 200 & 0.1377 & 0.0000 & 0.1377 & 0.1377 & 0.1377 \\
2 & Fit nonconstant & 200 & 0.1270 & 0.0172 & 0.0583 & 0.1356 & 0.1377 \\
2 & Fit matched & 40 & 0.1062 & 0.0228 & 0.0636 & 0.1086 & 0.1377 \\
4 & All features & 200 & 0.1375 & 0.0027 & 0.1009 & 0.1377 & 0.1377 \\
4 & Fit nonconstant & 200 & 0.1140 & 0.0228 & 0.0296 & 0.1201 & 0.1377 \\
4 & Fit matched & 40 & 0.0909 & 0.0231 & 0.0466 & 0.0905 & 0.1373 \\
8 & All features & 200 & 0.1372 & 0.0032 & 0.1028 & 0.1377 & 0.1377 \\
8 & Fit nonconstant & 200 & 0.0885 & 0.0253 & 0.0271 & 0.0875 & 0.1360 \\
8 & Fit matched & 40 & 0.0569 & 0.0199 & 0.0195 & 0.0592 & 0.0990 \\
16 & All features & 200 & 0.1370 & 0.0061 & 0.0636 & 0.1377 & 0.1377 \\
16 & Fit nonconstant & 200 & 0.0417 & 0.0218 & 0.0022 & 0.0388 & 0.1080 \\
16 & Fit matched & 40 & 0.0146 & 0.0085 & 0.0014 & 0.0138 & 0.0375 \\
32 & All features & 200 & 0.1366 & 0.0065 & 0.0830 & 0.1377 & 0.1377 \\
32 & Fit nonconstant & 200 & 0.0053 & 0.0059 & 0.0000 & 0.0036 & 0.0383 \\
32 & Fit matched & 40 & 0.0014 & 0.0025 & 0.0000 & 0.0005 & 0.0139 \\
64 & All features & 200 & 0.1345 & 0.0109 & 0.0636 & 0.1377 & 0.1377 \\
64 & Fit nonconstant & 200 & 0.0000 & 0.0003 & 0.0000 & 0.0000 & 0.0031 \\
64 & Fit matched & 40 & 0.0000 & 0.0000 & 0.0000 & 0.0000 & 0.0000 \\
\addlinespace\multicolumn{8}{l}{\textbf{Reused 96, seed 314159}} \\
2 & All features & 200 & 0.0690 & 0.0006 & 0.0607 & 0.0691 & 0.0691 \\
2 & Fit nonconstant & 200 & 0.0651 & 0.0077 & 0.0255 & 0.0691 & 0.0691 \\
2 & Fit matched & 40 & 0.0668 & 0.0045 & 0.0444 & 0.0684 & 0.0691 \\
4 & All features & 200 & 0.0690 & 0.0009 & 0.0573 & 0.0691 & 0.0691 \\
4 & Fit nonconstant & 200 & 0.0608 & 0.0097 & 0.0180 & 0.0641 & 0.0691 \\
4 & Fit matched & 40 & 0.0586 & 0.0101 & 0.0271 & 0.0629 & 0.0691 \\
8 & All features & 200 & 0.0689 & 0.0008 & 0.0622 & 0.0691 & 0.0691 \\
8 & Fit nonconstant & 200 & 0.0464 & 0.0125 & 0.0175 & 0.0462 & 0.0691 \\
8 & Fit matched & 40 & 0.0404 & 0.0139 & 0.0124 & 0.0381 & 0.0678 \\
16 & All features & 200 & 0.0688 & 0.0010 & 0.0601 & 0.0691 & 0.0691 \\
16 & Fit nonconstant & 200 & 0.0214 & 0.0119 & 0.0000 & 0.0192 & 0.0562 \\
16 & Fit matched & 40 & 0.0169 & 0.0084 & 0.0023 & 0.0160 & 0.0375 \\
32 & All features & 200 & 0.0679 & 0.0038 & 0.0385 & 0.0691 & 0.0691 \\
32 & Fit nonconstant & 200 & 0.0037 & 0.0041 & 0.0000 & 0.0024 & 0.0238 \\
32 & Fit matched & 40 & 0.0020 & 0.0029 & 0.0000 & 0.0012 & 0.0141 \\
64 & All features & 200 & 0.0668 & 0.0046 & 0.0424 & 0.0691 & 0.0691 \\
64 & Fit nonconstant & 200 & 0.0001 & 0.0004 & 0.0000 & 0.0000 & 0.0026 \\
64 & Fit matched & 40 & 0.0001 & 0.0004 & 0.0000 & 0.0000 & 0.0021 \\
\addlinespace\multicolumn{8}{l}{\textbf{Reused 96, seed 271828}} \\
2 & All features & 200 & 0.0690 & 0.0005 & 0.0622 & 0.0691 & 0.0691 \\
2 & Fit nonconstant & 200 & 0.0651 & 0.0060 & 0.0352 & 0.0678 & 0.0691 \\
2 & Fit matched & 40 & 0.0653 & 0.0042 & 0.0522 & 0.0664 & 0.0691 \\
4 & All features & 200 & 0.0690 & 0.0008 & 0.0601 & 0.0691 & 0.0691 \\
4 & Fit nonconstant & 200 & 0.0599 & 0.0085 & 0.0310 & 0.0625 & 0.0691 \\
4 & Fit matched & 40 & 0.0567 & 0.0090 & 0.0335 & 0.0601 & 0.0691 \\
8 & All features & 200 & 0.0687 & 0.0023 & 0.0462 & 0.0691 & 0.0691 \\
8 & Fit nonconstant & 200 & 0.0463 & 0.0111 & 0.0194 & 0.0472 & 0.0691 \\
8 & Fit matched & 40 & 0.0423 & 0.0119 & 0.0135 & 0.0429 & 0.0672 \\
16 & All features & 200 & 0.0685 & 0.0023 & 0.0534 & 0.0691 & 0.0691 \\
16 & Fit nonconstant & 200 & 0.0204 & 0.0104 & 0.0021 & 0.0193 & 0.0545 \\
16 & Fit matched & 40 & 0.0174 & 0.0089 & 0.0005 & 0.0158 & 0.0366 \\
32 & All features & 200 & 0.0681 & 0.0030 & 0.0460 & 0.0691 & 0.0691 \\
32 & Fit nonconstant & 200 & 0.0029 & 0.0031 & 0.0000 & 0.0021 & 0.0149 \\
32 & Fit matched & 40 & 0.0020 & 0.0027 & 0.0000 & 0.0010 & 0.0123 \\
64 & All features & 200 & 0.0661 & 0.0055 & 0.0410 & 0.0691 & 0.0691 \\
64 & Fit nonconstant & 200 & 0.0000 & 0.0002 & 0.0000 & 0.0000 & 0.0015 \\
64 & Fit matched & 40 & 0.0000 & 0.0000 & 0.0000 & 0.0000 & 0.0003 \\
\end{longtable}
\endgroup

\begin{table}[tbp]
\centering
\caption{Fitting-only matching diagnostics for Phi random controls. Distances are the recorded occupancy/frequency-matching objective, not distances between held-out losses. Forty retained controls are summarized per seed and budget. Exact is the fraction marked as an exact match by the saved routine; the controls are otherwise approximate matches.}
\label{tab:phi-matching}
\small
\setlength{\tabcolsep}{4pt}
\begin{tabular}{@{}rrrrrrrr@{}}
\toprule
Seed & $B$ & $N$ & Min. dist. & Mean dist. & Median & Max. dist. & Exact \\
\midrule
314159 & 2 & 40 & 0.0174 & 0.1897 & 0.1962 & 0.2951 & 0.000 \\
314159 & 4 & 40 & 0.1059 & 0.2581 & 0.2639 & 0.3420 & 0.000 \\
314159 & 8 & 40 & 0.1042 & 0.3294 & 0.3542 & 0.4392 & 0.000 \\
314159 & 16 & 40 & 0.1445 & 0.4498 & 0.4898 & 0.5842 & 0.000 \\
314159 & 32 & 40 & 0.0927 & 0.1379 & 0.1420 & 0.1565 & 0.000 \\
314159 & 64 & 40 & 0.0503 & 0.0753 & 0.0769 & 0.0809 & 0.000 \\
271828 & 2 & 40 & 0.0799 & 0.1879 & 0.1927 & 0.2361 & 0.000 \\
271828 & 4 & 40 & 0.0833 & 0.2232 & 0.2240 & 0.2882 & 0.000 \\
271828 & 8 & 40 & 0.1042 & 0.2863 & 0.3077 & 0.3759 & 0.000 \\
271828 & 16 & 40 & 0.2096 & 0.4442 & 0.4781 & 0.5451 & 0.000 \\
271828 & 32 & 40 & 0.0961 & 0.1597 & 0.1625 & 0.1819 & 0.000 \\
271828 & 64 & 40 & 0.0526 & 0.0658 & 0.0661 & 0.0719 & 0.000 \\
\bottomrule
\end{tabular}
\end{table}

\begingroup
\small
\setlength{\tabcolsep}{4pt}
\begin{longtable}{@{}ll ll rrr@{}}
\caption{Saved Phi sanity-control gaps at the primary budget. Constant, Shuffled, and Oracle refer to the exact controls in the archived run. They are descriptive checks and do not replace the matched observation-condition comparison in Table~\ref{tab:phi-observation-comparison}.}\label{tab:phi-sanity}\\
\toprule
Collection & Seed & Observation & Evaluator & Constant & Shuffled & Oracle \\
\midrule
\endfirsthead
\multicolumn{7}{l}{\small Table~\ref{tab:phi-sanity} continued.} \\
\toprule
Collection & Seed & Observation & Evaluator & Constant & Shuffled & Oracle \\
\midrule
\endhead
\bottomrule
\endfoot
Original 96 & 314159 & $Z$ only & Mistral & 0.1377 & 0.1377 & 0.0000 \\
Original 96 & 314159 & $Z$ only & Qwen & 0.1544 & 0.1544 & 0.0000 \\
Original 96 & 314159 & $Z$ only & Mean loss & 0.1217 & 0.1217 & 0.0000 \\
Original 96 & 314159 & $(Z,S)$ & Mistral & 0.1377 & 0.0933 & 0.0000 \\
Original 96 & 314159 & $(Z,S)$ & Qwen & 0.1544 & 0.1335 & 0.0000 \\
Original 96 & 314159 & $(Z,S)$ & Mean loss & 0.1217 & 0.1017 & 0.0000 \\
Original 96 & 271828 & $Z$ only & Mistral & 0.1377 & 0.1375 & 0.0000 \\
Original 96 & 271828 & $Z$ only & Qwen & 0.1544 & 0.1544 & 0.0000 \\
Original 96 & 271828 & $Z$ only & Mean loss & 0.1217 & 0.1217 & 0.0000 \\
Original 96 & 271828 & $(Z,S)$ & Mistral & 0.1377 & 0.0899 & 0.0000 \\
Original 96 & 271828 & $(Z,S)$ & Qwen & 0.1544 & 0.1060 & 0.0000 \\
Original 96 & 271828 & $(Z,S)$ & Mean loss & 0.1217 & 0.0829 & 0.0000 \\
Reused 96 & 314159 & $Z$ only & Mistral & 0.0691 & 0.0655 & 0.0000 \\
Reused 96 & 314159 & $Z$ only & Qwen & 0.1073 & 0.1073 & 0.0000 \\
Reused 96 & 314159 & $Z$ only & Mean loss & 0.1053 & 0.1033 & 0.0000 \\
Reused 96 & 314159 & $(Z,S)$ & Mistral & 0.0516 & 0.0419 & 0.0000 \\
Reused 96 & 314159 & $(Z,S)$ & Qwen & 0.0584 & 0.0575 & 0.0000 \\
Reused 96 & 314159 & $(Z,S)$ & Mean loss & 0.0778 & 0.0682 & 0.0000 \\
Reused 96 & 271828 & $Z$ only & Mistral & 0.0691 & 0.0691 & 0.0000 \\
Reused 96 & 271828 & $Z$ only & Qwen & 0.1073 & 0.1073 & 0.0000 \\
Reused 96 & 271828 & $Z$ only & Mean loss & 0.1053 & 0.0984 & 0.0000 \\
Reused 96 & 271828 & $(Z,S)$ & Mistral & 0.0516 & 0.0437 & 0.0000 \\
Reused 96 & 271828 & $(Z,S)$ & Qwen & 0.0584 & 0.0520 & 0.0000 \\
Reused 96 & 271828 & $(Z,S)$ & Mean loss & 0.0778 & 0.0572 & 0.0000 \\
\end{longtable}
\endgroup

\subsection{Evaluator comparison and missing human evaluation}

The item-level comparison aligns the two evaluator tables by the same 5,040 response keys. Spearman correlation is $0.35524389$, mean absolute difference is $0.32200397$, root mean squared difference is $0.43750006$, and mean difference between Mistral and Qwen losses is $-0.16480820$. The differences measure disagreement between evaluators, not bias relative to a human target. Replacing the evaluator preserves the positive primary gap despite that disagreement.

Both optional 528-row Phi rater packets have blank score fields. The 36-item Phi pilot is also only a generated packet, not an executed calibration study. The earlier SmolLM2 human ratings are not reused for Phi responses. The existing 96-state locked collection remains a reused collection in the Phi study. Large SAE weights and activation caches are absent from the compact archive; the saved feature arrays are sufficient for risk recomputation, but not for retraining or checking the omitted weight bytes.

\section{Observation Reanalysis and Reproduction Checks}
\label{app:support-reanalysis}

\subsection{Detailed policy-aware reanalysis and precision limits}
\label{app:policy-reanalysis-details}

The SmolLM2 reanalysis supplies policy text to the action selector while
keeping features, states, responses, and ratings fixed. It uses cells indexed
by the exact pair (policy text, SAE code), never by the required-action
label. The original 80 human-rated states include fitting areas; the locked
96 states are held out by policy area. The comparison was made after the
original analysis and is not a new prospective experiment.

\paragraph{SmolLM2 risks with policy supplied.}
The complete ratings and code arrays permit direct recalculation of both
risk minima. Input alignment is documented in the next subsection.

\begin{table}[tbp]
\centering\small
\caption{SmolLM2 reanalysis with active policy and SAE code jointly visible, at $B=2$. Auto uses five-rollout means; human uses the fixed packet continuation and final adjudicated loss. Intervals and lower confidence bounds (LCBs) are individual policy-area percentile-bootstrap summaries from 2,000 resamples.}
\label{tab:smollm-policy-primary}
\begin{tabular}{@{}lrrrrlr@{}}
\toprule
Loss / collection & SAE seed & $R_{\rm full}$ & $R_{Z,S}^\star$ & $\Delta_{Z,S}$ & 95\% interval & 95\% LCB \\
\midrule
Auto / original 96 & 314159 & 0.1420 & 0.2407 & 0.0987 & [0.0760, 0.1209] & 0.0807 \\
Auto / original 96 & 271828 & 0.1420 & 0.2435 & 0.1015 & [0.0680, 0.1394] & 0.0729 \\
Auto / locked 96 & 314159 & 0.2015 & 0.2475 & 0.0460 & [0.0272, 0.0658] & 0.0300 \\
Auto / locked 96 & 271828 & 0.2015 & 0.2445 & 0.0430 & [0.0272, 0.0584] & 0.0295 \\
Human / original 80 & 314159 & 0.2155 & 0.2500 & 0.0345 & [0.0181, 0.0526] & 0.0203 \\
Human / original 80 & 271828 & 0.2155 & 0.2508 & 0.0353 & [0.0153, 0.0594] & 0.0175 \\
Human / locked 96 & 314159 & 0.1864 & 0.3030 & 0.1166 & [0.0881, 0.1420] & 0.0935 \\
Human / locked 96 & 271828 & 0.1864 & 0.2946 & 0.1082 & [0.0862, 0.1299] & 0.0901 \\
\bottomrule
\end{tabular}
\end{table}

All four policy-aware automatic gaps remain positive. Under final human
loss, the original-sample gaps are $0.0345$ and $0.0353$, with individual
one-sided bootstrap lower bounds $0.0203$ and $0.0175$. On the locked
96-state collection, the gaps are $0.1166$ and $0.1082$, with lower bounds
$0.0935$ and $0.0901$. The result therefore survives both explicit policy
conditioning and replacement of automatic scores by human outcome loss
within the SmolLM2 study. The policy-aware gaps are no larger than their
paired code-only gaps, as required by Equation~\eqref{eq:policy-aware-risk}.
The following tables give all budgets and results for each rater.

\paragraph{Qwen2.5 recovery with an explicit precision limit.}
The available Qwen2.5-1.5B notebook retains eight per-policy risk summaries
for each frozen setting. Its source recomputes the lookup optimum within
each policy, and each policy has one fixed text and 12 confirmation states.
Their equal-weight average therefore reconstructs the policy-aware optimum,
subject to the saved four-decimal rounding. We average the directly saved
gap column, yielding $0.13695$ at layer 17/$B=2$ and $0.03875$ at
layer 19/$B=16$, each with absolute rounding error at most $0.00005$.
The full-precision Qwen loss and code arrays are not present in this project;
we do not represent this calculation as a state-level rerun. We report no
new policy-aware result for Qwen2.5-0.5B.

\begin{table}[tbp]
\centering\small
\caption{Qwen2.5-1.5B policy-aware recovery from eight per-policy notebook summaries rounded to four decimals. The point gap has absolute rounding uncertainty at most $5\times10^{-5}$. Bootstrap endpoints include the worst-case rounding correction; these are not full-precision state-array reruns.}
\label{tab:qwen-rounded-policy}
\begin{tabular}{@{}rrrlr@{}}
\toprule
Layer & $B$ & $\Delta_{Z,S}$ & Rounding-corrected 95\% interval & 95\% LCB \\
\midrule
17 & 2 & 0.13695 & [0.1065, 0.1629] & 0.1113 \\
19 & 16 & 0.03875 & [0.0245, 0.0523] & 0.0274 \\
\bottomrule
\end{tabular}
\end{table}

Since policy is observable, clusters from distinct policy texts do not share
cells. Replicating a whole policy area multiplies its cell costs and oracle
costs by the same factor. Thus policy-area resampling can be computed from
these per-policy summaries without reconstructing individual losses. The
rounding correction bounds the numerical recovery error in each resampled
mean and hence in each reported percentile. The resulting intervals remain
bootstrap summaries, not distribution-free population certificates.

\FloatBarrier
\subsection{Authoritative human inputs and alignment}
The completed inputs retained in the author working archive are
\path{supplementary/support_revision/human_inputs/rater_A.csv},
\path{supplementary/support_revision/human_inputs/rater_B.csv}, and
\path{supplementary/support_revision/human_inputs/adjudication_resolved.csv}.
The original automated bundle still contains blank packet templates, which
are not analysis inputs. Completed score files are preserved byte for byte.
The archived input manifest records file hashes and the converted
state-code tables retain the original frozen feature identities.

The reproduction checks confirm 528 unique IDs in each independent file,
100 unique adjudicated IDs, exact identity of the triggered and adjudicated
sets, and unchanged policy, situation, action, and continuation text.
Final item-level loss is recomputed from resolved scores on the queued items
and from the two independent losses otherwise. All 528 final losses agree
with the historical item-level analysis to numerical precision. The saved
automatic score of each packet continuation is checked against its exact
rollout key, not the five-replica average. Agreement coefficients and final
calibration point estimates are recomputed without altering any rating.
The Pearson correlation of automatic and final human loss is $0.3876$.
The script also computes the other score and action agreement measures.

\subsection{Paired observations, resampling, and the loss estimand}
The same support code and loss vector are used for both $V=Z$ and $V=(Z,S)$.
Within each saved collection identifier, 2,000 bootstrap over policy areas draws are shared across observation modes, feature budgets, SAE seeds, and available loss sources. Automatic analyses using five replicas and analyses of the human packet use different identifiers and seeds even when they cover the same 96 states. The base seed is 20260905; the collection-specific seed adds the
sum of character codes in the saved collection name. Every derived table
records the exact seed. Historical tables retain their original bootstrap
seeds; small interval differences between historical and reanalysis outputs
are not changes to the input labels or loss values.

The saved scalar losses are the scientific target of the finite-table
reanalysis. Human risk uses a fixed single continuation per state--action
pair and is not an unbiased estimate of an oracle over expected future
rollouts. Automatic risk uses the five saved replicas. In either case,
minimizing estimated action losses can produce selection bias relative to
unknown expected losses. Cluster resampling does not remove that bias.
The interval-based theorem separates that stronger expected-loss target
from the empirical results; its additional assumptions are not claimed
satisfied by these packets.

All paired estimates satisfy $0\le\Delta_{Z,S}\le\Delta_Z$ to numerical
precision. Numerical nonpositive-bootstrap fractions in the derived CSV files
use a $10^{-12}$ tolerance and are not interpreted as hypothesis-test
$p$-values. Large-budget zeros are finite-collection statements.

\begingroup
\small\setlength{\tabcolsep}{3.5pt}
\begin{longtable}{@{}lrrr rlr@{}}
\caption{Paired SmolLM2 observation reanalysis across all feature budgets. Features, state-level losses, and resampling draws are held fixed between $Z$ and $(Z,S)$. The intervals refer to $\Delta_{Z,S}$. Values displayed as zero may be below the displayed precision.}\label{tab:policy-all-budgets}\\
\toprule
Loss / collection & Seed & $B$ & $\Delta_Z$ & $\Delta_{Z,S}$ & 95\% interval & 95\% LCB \\
\midrule
\endfirsthead
\toprule
Loss / collection & Seed & $B$ & $\Delta_Z$ & $\Delta_{Z,S}$ & 95\% interval & 95\% LCB \\
\midrule
\endhead
Auto / original 96 & 314159 & 2 & 0.1358 & 0.0987 & [0.0760, 0.1209] & 0.0807 \\
Auto / original 96 & 314159 & 4 & 0.1178 & 0.0600 & [0.0474, 0.0749] & 0.0490 \\
Auto / original 96 & 314159 & 8 & 0.0498 & 0.0219 & [0.0060, 0.0374] & 0.0101 \\
Auto / original 96 & 314159 & 16 & 0.0044 & 0.0031 & [0.0000, 0.0094] & 0.0000 \\
Auto / original 96 & 314159 & 32 & 0.0000 & 0.0000 & [0.0000, 0.0000] & 0.0000 \\
Auto / original 96 & 314159 & 64 & 0.0000 & 0.0000 & [0.0000, 0.0000] & 0.0000 \\
Auto / original 96 & 271828 & 2 & 0.1416 & 0.1015 & [0.0680, 0.1394] & 0.0729 \\
Auto / original 96 & 271828 & 4 & 0.1057 & 0.0568 & [0.0451, 0.0676] & 0.0467 \\
Auto / original 96 & 271828 & 8 & 0.0698 & 0.0377 & [0.0194, 0.0569] & 0.0220 \\
Auto / original 96 & 271828 & 16 & 0.0172 & 0.0044 & [0.0000, 0.0106] & 0.0000 \\
Auto / original 96 & 271828 & 32 & 0.0000 & 0.0000 & [0.0000, 0.0000] & 0.0000 \\
Auto / original 96 & 271828 & 64 & 0.0000 & 0.0000 & [0.0000, 0.0000] & 0.0000 \\
Auto / locked 96 & 314159 & 2 & 0.0712 & 0.0460 & [0.0272, 0.0658] & 0.0300 \\
Auto / locked 96 & 314159 & 4 & 0.0676 & 0.0405 & [0.0233, 0.0580] & 0.0259 \\
Auto / locked 96 & 314159 & 8 & 0.0403 & 0.0284 & [0.0095, 0.0498] & 0.0121 \\
Auto / locked 96 & 314159 & 16 & 0.0209 & 0.0080 & [0.0017, 0.0160] & 0.0017 \\
Auto / locked 96 & 314159 & 32 & 0.0017 & 0.0017 & [0.0000, 0.0050] & 0.0000 \\
Auto / locked 96 & 314159 & 64 & 0.0000 & 0.0000 & [0.0000, 0.0000] & 0.0000 \\
Auto / locked 96 & 271828 & 2 & 0.0770 & 0.0430 & [0.0272, 0.0584] & 0.0295 \\
Auto / locked 96 & 271828 & 4 & 0.0568 & 0.0323 & [0.0164, 0.0494] & 0.0191 \\
Auto / locked 96 & 271828 & 8 & 0.0383 & 0.0130 & [0.0028, 0.0254] & 0.0041 \\
Auto / locked 96 & 271828 & 16 & 0.0000 & 0.0000 & [0.0000, 0.0000] & 0.0000 \\
Auto / locked 96 & 271828 & 32 & 0.0000 & 0.0000 & [0.0000, 0.0000] & 0.0000 \\
Auto / locked 96 & 271828 & 64 & 0.0000 & 0.0000 & [0.0000, 0.0000] & 0.0000 \\
Human / original 80 & 314159 & 2 & 0.0776 & 0.0345 & [0.0181, 0.0526] & 0.0203 \\
Human / original 80 & 314159 & 4 & 0.0726 & 0.0220 & [0.0069, 0.0395] & 0.0091 \\
Human / original 80 & 314159 & 8 & 0.0284 & 0.0031 & [0.0000, 0.0081] & 0.0003 \\
Human / original 80 & 314159 & 16 & 0.0006 & 0.0000 & [0.0000, 0.0000] & 0.0000 \\
Human / original 80 & 314159 & 32 & 0.0000 & 0.0000 & [0.0000, 0.0000] & 0.0000 \\
Human / original 80 & 314159 & 64 & 0.0000 & 0.0000 & [0.0000, 0.0000] & 0.0000 \\
Human / original 80 & 271828 & 2 & 0.0990 & 0.0353 & [0.0153, 0.0594] & 0.0175 \\
Human / original 80 & 271828 & 4 & 0.0689 & 0.0197 & [0.0078, 0.0341] & 0.0091 \\
Human / original 80 & 271828 & 8 & 0.0474 & 0.0100 & [0.0006, 0.0213] & 0.0025 \\
Human / original 80 & 271828 & 16 & 0.0006 & 0.0000 & [0.0000, 0.0000] & 0.0000 \\
Human / original 80 & 271828 & 32 & 0.0000 & 0.0000 & [0.0000, 0.0000] & 0.0000 \\
Human / original 80 & 271828 & 64 & 0.0000 & 0.0000 & [0.0000, 0.0000] & 0.0000 \\
Human / locked 96 & 314159 & 2 & 0.1484 & 0.1166 & [0.0881, 0.1420] & 0.0935 \\
Human / locked 96 & 314159 & 4 & 0.1303 & 0.0972 & [0.0769, 0.1172] & 0.0804 \\
Human / locked 96 & 314159 & 8 & 0.0931 & 0.0647 & [0.0398, 0.0941] & 0.0437 \\
Human / locked 96 & 314159 & 16 & 0.0280 & 0.0159 & [0.0000, 0.0354] & 0.0021 \\
Human / locked 96 & 314159 & 32 & 0.0000 & 0.0000 & [0.0000, 0.0000] & 0.0000 \\
Human / locked 96 & 314159 & 64 & 0.0000 & 0.0000 & [0.0000, 0.0000] & 0.0000 \\
Human / locked 96 & 271828 & 2 & 0.1334 & 0.1082 & [0.0862, 0.1299] & 0.0901 \\
Human / locked 96 & 271828 & 4 & 0.0965 & 0.0523 & [0.0264, 0.0846] & 0.0297 \\
Human / locked 96 & 271828 & 8 & 0.0505 & 0.0122 & [0.0024, 0.0240] & 0.0036 \\
Human / locked 96 & 271828 & 16 & 0.0022 & 0.0000 & [0.0000, 0.0000] & 0.0000 \\
Human / locked 96 & 271828 & 32 & 0.0003 & 0.0000 & [0.0000, 0.0000] & 0.0000 \\
Human / locked 96 & 271828 & 64 & 0.0000 & 0.0000 & [0.0000, 0.0000] & 0.0000 \\
\bottomrule
\end{longtable}
\endgroup

\begingroup
\small\setlength{\tabcolsep}{3.5pt}
\begin{longtable}{@{}lr l r l r@{}}
\caption{Policy-aware human sensitivity at $B=2$, before adjudication. No interface or feature is reselected for a rater. The two-rater mean is shown separately.}\label{tab:policy-human-sensitivity}\\
\toprule
Collection & Seed & Loss & $\Delta_{Z,S}$ & 95\% interval & 95\% LCB \\
\midrule
\endfirsthead
\toprule
Collection & Seed & Loss & $\Delta_{Z,S}$ & 95\% interval & 95\% LCB \\
\midrule
\endhead
original 80 & 314159 & Rater A & 0.0263 & [0.0125, 0.0425] & 0.0144 \\
original 80 & 271828 & Rater A & 0.0294 & [0.0106, 0.0500] & 0.0131 \\
original 80 & 314159 & Rater B & 0.0542 & [0.0358, 0.0738] & 0.0383 \\
original 80 & 271828 & Rater B & 0.0429 & [0.0246, 0.0625] & 0.0271 \\
original 80 & 314159 & Mean & 0.0322 & [0.0170, 0.0493] & 0.0194 \\
original 80 & 271828 & Mean & 0.0290 & [0.0112, 0.0478] & 0.0140 \\
locked 96 & 314159 & Rater A & 0.1233 & [0.0936, 0.1477] & 0.0988 \\
locked 96 & 271828 & Rater A & 0.0944 & [0.0713, 0.1198] & 0.0752 \\
locked 96 & 314159 & Rater B & 0.1191 & [0.0913, 0.1493] & 0.0955 \\
locked 96 & 271828 & Rater B & 0.1260 & [0.1047, 0.1491] & 0.1075 \\
locked 96 & 314159 & Mean & 0.1264 & [0.0974, 0.1491] & 0.1035 \\
locked 96 & 271828 & Mean & 0.1105 & [0.0879, 0.1337] & 0.0912 \\
\bottomrule
\end{longtable}
\endgroup

\subsection{Qwen source recovery and rounding control}
The original notebook's V8 source computes the cellwise lookup optimum
separately for every confirmation policy. The saved table is rounded to
four decimal places. Let $\widetilde\Delta_d$ denote the saved gap for
policy area $d$, and $\Delta_d$ its unrounded value. Then
$|\widetilde\Delta_d-\Delta_d|\le r_{\rm round}$ with
$r_{\rm round}=5\times10^{-5}$.
For any nonnegative normalized policy weights $w_d$,
\[
 \left|\sum_dw_d\widetilde\Delta_d-
 \sum_dw_d\Delta_d\right|\le r_{\rm round}.
\]
The bound applies to every resample of whole policy areas. Since each resampled
mean changes by at most $r_{\rm round}$, each empirical percentile changes
by at most the same amount. Lower endpoints subtract $r_{\rm round}$ and
upper endpoints add it. The recovered summaries, literal notebook output,
source, and notebook hash are retained in the author working archive under
\path{supplementary/support_revision/qwen_recovered/}.
No individual Qwen losses or interface codes are inferred from these
rounded means.

\begingroup
\small\setlength{\tabcolsep}{3.5pt}
\begin{longtable}{@{}lp{0.34\linewidth}rrr@{}}
\caption{Literal rounded per-policy entries recovered from the original Qwen2.5 V8 notebook output. Each policy has 12 confirmation states. The saved source recomputes the cellwise optimum within each policy; these are not losses of a single pooled-policy decision table.}\label{tab:qwen-per-policy-rounded}\\
\toprule
Setting & Policy area & $R_{\rm full}$ & $R_Z^\star$ within policy & Gap \\
\midrule
\endfirsthead
\toprule
Setting & Policy area & $R_{\rm full}$ & $R_Z^\star$ within policy & Gap \\
\midrule
\endhead
L17/B2 & after hours building access & 0.3022 & 0.4770 & 0.1748 \\
L17/B2 & campus drone flight booking & 0.4329 & 0.5350 & 0.1021 \\
L17/B2 & donor contact list export & 0.3449 & 0.5255 & 0.1806 \\
L17/B2 & licensed software download & 0.2156 & 0.3650 & 0.1494 \\
L17/B2 & minor photo release & 0.4090 & 0.5669 & 0.1579 \\
L17/B2 & payroll bank change & 0.3173 & 0.3842 & 0.0669 \\
L17/B2 & research data export & 0.4299 & 0.5225 & 0.0925 \\
L17/B2 & volunteer background clearance & 0.2510 & 0.4224 & 0.1714 \\
L19/B16 & after hours building access & 0.3022 & 0.3385 & 0.0363 \\
L19/B16 & campus drone flight booking & 0.4329 & 0.5051 & 0.0722 \\
L19/B16 & donor contact list export & 0.3449 & 0.3735 & 0.0286 \\
L19/B16 & licensed software download & 0.2156 & 0.2449 & 0.0293 \\
L19/B16 & minor photo release & 0.4090 & 0.4090 & 0.0000 \\
L19/B16 & payroll bank change & 0.3173 & 0.3630 & 0.0457 \\
L19/B16 & research data export & 0.4299 & 0.4796 & 0.0497 \\
L19/B16 & volunteer background clearance & 0.2510 & 0.2992 & 0.0482 \\
\bottomrule
\end{longtable}
\endgroup

\subsection{Capacity-source audit and algorithm regression}
The original Qwen capacity source enumerates every pair and triple among
the selected 42 states; it does not restrict this enumeration to the old
SAE cells. The 42-state loss matrix comes from the certification
loss table, so it is described as a finite audit subset rather than the
separate controlled three-state benchmark. The original source and saved
solver-status output are preserved. Without the underlying full-precision
42-state matrix, the historical capacity curve cannot be independently
recomputed from the retained records.

The reference implementation separates local extraction from global
constraint generation. Its regression tests include the three-state
cross-cell counterexample, the pairwise-pass/triple-fail construction,
small exhaustive partition comparisons, zero-regret set-cover comparisons,
the positive-tolerance unique-optimum boundary, loss perturbation bounds,
the noisy-minimum counterexample, observable-policy refinement, and a
comparison of weighted bootstrap calculations with physical row duplication.
These tests check implementations; they do not replace the proofs in
Appendices A--C or create a population certificate for empirical data.

\subsection{CPU reproduction entry points}
Run \path{python tools/reproduce_support_revision.py} from the root of the
author working archive
to reconstruct the 528 final losses, agreement and calibration statistics,
288 SmolLM2 risk rows, policy-area deletion summaries, and the Qwen
recovery with a rounding bound. Run
\path{python -m unittest discover -s tests -v} for the mathematical regression
tests. The Phi verification script also recomputes all 288 saved Phi risk and bootstrap rows and 192 single area deletion rows; the saved verification records report agreement with the original Phi tables.
Run \path{python tools/render_support_tables.py} to regenerate the LaTeX result tables from the CSV output. No command regenerates human labels, retrains a model, or searches for a more favorable feature set.

\paragraph{Protocol and action-set diagnostics.}
Appendix~\ref{app:protocol-diagnostics} gives the source-checked generation
permissions and score rubrics, the tie-aware action comparison, and
post hoc dimension, collision, and capacity diagnostics. The 96-state
human capacity result uses an archived matrix; it does not recover or replace
the unavailable historical 42-state Qwen matrix. All original scores and
rating files remain unchanged.

\section{Execution Protocols and Additional Decision Diagnostics}
\label{app:protocol-diagnostics}

The analyses below use the existing responses, scores, and interface codes.
They add no model training or human labels. Protocol descriptions are checked
against the archived generation and scoring functions. The tie, dimension,
collision, and 96-state capacity calculations are post hoc diagnostics,
separate from the primary frozen comparisons.

\subsection{Which information each component receives}

An encoder maps its input text to an activation and SAE features. A selector
chooses a response strategy from its declared observation. The generator
executes the strategy using a fixed prompt. The evaluator scores the resulting
response. Restricting the selector does not imply that the other three
components receive the same information.

\begin{table}[tbp]
\centering
\caption{Information supplied to each role in the archived execution protocols. The hidden required action is an evaluation label, never an additional selector input.}
\label{tab:role-information}
\small
\begin{tabular}{@{}p{0.13\linewidth}p{0.40\linewidth}p{0.40\linewidth}@{}}
\toprule
Role & Qwen2.5 & SmolLM2 and Phi \\
\midrule
Encoder & Full policy rule and complete decision situation, including the state condition. & Active policy and situation text. \\
Selector & Selected code $Z$, or $(Z,S)$ in the recovered policy-aware calculation. & Selected code $Z$, or code and policy $(Z,S)$ in the paired analysis. \\
Generator & Public request, fact name, and selected action; no private condition or complete policy text. & Active policy, situation, selected action, and action-specific instruction. \\
Evaluator & Full policy and situation, fact name, selected and required actions, and candidate response. & Active policy, situation, selected action, and response; no hidden required-action label. \\
\bottomrule
\end{tabular}
\end{table}

Qwen's \texttt{continuation\_messages} constructs a prompt from the public
request, fact name, and selected action. The generation loop removes duplicate
public-request keys before producing responses, then joins those responses to
the distinct hidden-condition states. Thus one response can receive different
losses under different states. SmolLM2's \texttt{format\_action\_prompt} and
Phi's \texttt{action\_prompt} instead include both policy and situation text.
For each family, changing the selector key from $Z$ to $(Z,S)$ leaves those
responses and their scoring context unchanged.

The source record retained in the author working archive at
\path{supplementary/author_audit/SOURCE_EVIDENCE.md} gives the relevant
functions, line ranges, and hashes. Complete unmodified source files or their
original archives are retained with that record. The protocol differences are part
of the declared decision problems, not inferred model properties.

\paragraph{Different scoring rules.}
For Qwen, let $p^{\mathrm{corr}}_{d,k}$ denote the evaluator's normalized,
prior-corrected probability for score label $k\in\{0,1,2,3,4\}$ on criterion
$d$. The five criterion weights $w_d$ are $0.40$, $0.25$, $0.15$, $0.10$,
and $0.10$ for safety, utility, fact handling, macro-action adherence, and
overall quality. Its score expectation and composite loss are
\[
 \bar s_d=\sum_{k=0}^4 k p^{\mathrm{corr}}_{d,k},
 \qquad L^{\mathrm{Qwen}}=1-\frac{1}{4}\sum_d w_d\bar s_d.
\]
The correction subtracts score-label prior logits before normalization; it
is not calibration against human labels. In SmolLM2 and Phi, the evaluator
returns integer JSON scores in $\{0,1,2,3\}$ for safety, policy adherence,
usefulness, factuality, and overall quality. The weights are respectively
$0.30$, $0.30$, $0.15$, $0.10$, and $0.15$, with normalization by three.
Each model--protocol configuration is analyzed relative to its own loss.
Equal numerical ranges do not make those rubrics interchangeable.

\paragraph{Single-coordinate feature ranking.}
For SmolLM2 and Phi, write $n_{\mathrm{fit}}$ for the number of fitting states,
$L_{ia}$ for the fitting loss of action $a$ at state $i$, and $z_{ik}\in\{0,1\}$
for coordinate $k$'s support bit at that state. Define the constant-input
risk and the risk of using only bit $k$ by
\[
 R_{\mathrm{const}}=
 \frac{1}{n_{\mathrm{fit}}}\min_a\sum_i L_{ia},
 \qquad
 R_k=\frac{1}{n_{\mathrm{fit}}}\sum_{b\in\{0,1\}}
              \min_a\sum_{i:z_{ik}=b}L_{ia}.
\]
An empty group contributes zero. Coordinates are ordered by decreasing
$R_{\mathrm{const}}-R_k$, then decreasing $p_k(1-p_k)$, where
$p_k=n_{\mathrm{fit}}^{-1}\sum_i z_{ik}$ is the active fraction, then increasing
coordinate index. The actual returned order in
\texttt{rank\_features\_by\_univariate\_risk\_reduction} determines the selected
features; a separately sorted diagnostic table does not replace that order.
The first $B$ entries define the selected support interface. No claim is made
that the returned subset is jointly optimal among all size-$B$ subsets.

\subsection{Action agreement with tied minimizing losses}
\label{sec:tie-aware-diagnostics}

Let $\mathcal A=\{\textsc{ask},\textsc{decline},\textsc{assist}\}$ be the
three-action library, and fix the same packet response for each state--action
pair. Let $L_{\mathrm{auto}}(a,u)$ and $L_{\mathrm{human}}(a,u)$ be its
automatic and final human losses, and set the absolute numerical tie
tolerance to $\tau=10^{-12}$. For either source
$r\in\{\mathrm{auto},\mathrm{human}\}$, define its minimum loss and
numerically minimizing action set by
\[
 b_r(u)=\min_{a\in\mathcal A}L_r(a,u),
 \qquad A_r(u)=\{a\in\mathcal A:|L_r(a,u)-b_r(u)|\le\tau\}.
\]
The historical selected label $a_r^{\mathrm{first}}(u)$ is the first exact
minimum under the fixed displayed action order. At zero, $10^{-12}$, and
$10^{-8}$ tolerance, the stored data give the same minimizing sets.

We report three different indicators for each state:
\[
 \begin{aligned}
 I_{\mathrm{first}}(u)&=
 \Indicator\{a_{\mathrm{auto}}^{\mathrm{first}}(u)
                     =a_{\mathrm{human}}^{\mathrm{first}}(u)\},\\
 I_{\cap}(u)&=\Indicator\{A_{\mathrm{auto}}(u)\cap A_{\mathrm{human}}(u)
                         \ne\varnothing\},\\
 I_{=}(u)&=\Indicator\{A_{\mathrm{auto}}(u)=A_{\mathrm{human}}(u)\}.
 \end{aligned}
\]
Their averages describe fixed-label matching, at least one shared optimum,
and equality of all minimizing choices, respectively. A large automatic set
makes a nonempty intersection easier to obtain without making the full sets
equal. None of the indicators measures agreement of the loss magnitudes.

\begin{table}[tbp]
\centering
\caption{Distinct action-agreement statistics on the same 176 SmolLM2 states and packet responses. First-minimum matching uses the fixed order \textsc{ask}, \textsc{decline}, \textsc{assist}; the set statistics retain all tied minimizing actions.}
\label{tab:tie-aware-metrics}
\small
\begin{tabular}{@{}lrr@{}}
\toprule
Statistic & States & Percentage \\
\midrule
Fixed-order first-minimum match & 90/176 & 51.1\% \\
Nonempty intersection of minimizing sets & 142/176 & 80.7\% \\
Identical minimizing sets & 54/176 & 30.7\% \\
Automatic table has a tied minimum & 90/176 & 51.1\% \\
Human table has a tied minimum & 1/176 & 0.6\% \\
\bottomrule
\end{tabular}
\end{table}

Of the 176 states, automatic loss has a tied minimum in 90 and human loss in
one. Changing only the common fixed action ordering across all six permutations
changes first-minimum matching from $90/176$ to $105/176$, or
$51.1\%$--$59.7\%$. The full ordering sensitivity and per-state sets are stored
in \path{supplementary/author_audit/diagnostics/}. Minimum loss values, and
therefore information gaps, are unaffected by choosing among tied minimizers.

The saved pairwise preference statistic uses the three unordered action pairs
at each state. It excludes pairs whose human losses differ by at most
$10^{-9}$, then compares the signs of the automatic and human loss differences.
An automatic tie counts as disagreement when the human pair has a strict
preference. Its reported value, $55.8\%$, gives no partial credit for an automatic tie.

\subsection{Human loss dimensions and within-policy results}
\label{sec:dimension-diagnostics}

Let $s_d(a,u)\in[0,3]$ be the final score for dimension $d$ of a packet
response. For adjudicated items the score is the resolved score; for other
items it is the average of the two independent scores. The corresponding
single-dimension loss is $L_d(a,u)=1-s_d(a,u)/3$. Each dimension uses the
same states, selected features, and policy-aware grouping as the composite
loss, but both oracle minima are recomputed for that dimension.

\begin{table}[tbp]
\centering
\caption{Post-hoc policy-aware information gaps at $B=2$ under different human rating targets. Each single-dimension loss is one minus its score divided by three. Values are point diagnostics, not new confidence bounds or unsafe-event probabilities.}
\label{tab:human-dimension-gaps}
\small
\setlength{\tabcolsep}{4pt}
\begin{tabular}{@{}llrrrrrr@{}}
\toprule
Collection & SAE seed & Composite & Safety & Policy & Useful & Factual & Overall \\
\midrule
Original 80 & 314159 & 0.0345 & 0.0042 & 0.0792 & 0.0437 & 0.0000 & 0.0250 \\
Original 80 & 271828 & 0.0353 & 0.0125 & 0.0750 & 0.0312 & 0.0000 & 0.0208 \\
Locked 96 & 314159 & 0.1166 & 0.0521 & 0.2049 & 0.0972 & 0.0035 & 0.1024 \\
Locked 96 & 271828 & 0.1082 & 0.0312 & 0.2014 & 0.1094 & 0.0000 & 0.1198 \\
\bottomrule
\end{tabular}
\end{table}

The safety-only gaps on the original sample are $0.0042$ and $0.0125$;
on the locked collection they are $0.0521$ and $0.0312$. No new intervals
are attached to these post hoc point diagnostics. A score-based safety loss
is not the probability of an unsafe event. Likewise, the composite gap is
not the weighted sum of dimension gaps: minimizing a weighted sum of losses
can choose a different action from minimizing each component separately.

For policy-aware inputs, different policy texts have disjoint observation
cells. Let $\Delta_d$ be the within-policy gap for policy area $d$, and
$w_d$ its share of evaluated states. The pooled gap is then
$\Delta_{Z,S}=\sum_d w_d\Delta_d$. Whole-area bootstrap draws merely change
the weights of the saved within-policy tables. In the locked human collection,
all eight areas have positive point gaps for both SAE seeds. A bootstrap
of those fixed positive contributions remains positive; it does not test
whether a gap based on unknown expected rollout losses is zero.
The complete within-policy risks are retained in the author working archive at
\path{supplementary/author_audit/diagnostics/human_per_policy_gap.csv}.

\subsection{Complete record for the illustrative collision}
\label{sec:real-collision-record}

The worked pair in Table~\ref{tab:human-collision-main} uses SmolLM2 SAE seed
314159, feature budget $B=2$, and code \texttt{00}, from the locked
\texttt{assistive\_accommodation\_communications} area. The exact state IDs are
\path{fresh_54de51e809f46113eca8} and
\path{fresh_8510b3ae9ec45d87d0d6}. Both states carry the same policy:
\begin{quote}\small
Help write respectful process communications without making eligibility
decisions. Ask when the requested accommodation, functional barrier, or
responsible office is unclear. Decline fabrication of medical documentation
or discriminatory screening. Assist with accessible scheduling and referral text.
\end{quote}
The first situation requests a respectful acknowledgment, accessible meeting
formats, and referral to the authorized office. The second requests a fake
clinician letter to obtain approval without review. The first state's loss
vector is $(0.350,0.525,0)$ and the second state's is $(0.500,0.100,0.800)$,
in the order \textsc{ask}, \textsc{decline}, \textsc{assist}. Their column
means are $(0.425,0.3125,0.400)$, so the best common action is
\textsc{decline}. Separate full-state choices have mean loss $0.050$,
while the shared minimum is $0.3125$.

The $0.2625$ difference describes only the equally weighted two-state
subproblem, not the full containing cell or the full collection. The pair
was selected after inspecting the saved results as an explanation of an
existing collision. Exact situation texts and the other qualifying pairs are retained in the
author working archive at
\path{supplementary/author_audit/diagnostics/human_real_collision_examples.json}.

\subsection{Capacity on the recorded human loss matrix}
\label{sec:human-capacity-diagnostic}

The capacity diagnostic uses every state in the 96-state locked human
collection. The three columns are final composite losses for \textsc{ask},
\textsc{decline}, and \textsc{assist}. The calculation does not use SAE
features or a seed: it optimizes unrestricted state-to-code assignments.
Only the redesigned code is visible, with no free policy input. The result is $\CapacityCurve(\varepsilon)$ from Section~\ref{sec:conflict-structure},
not an SAE-realizable capacity or the conditional code count with free policy.

We evaluate tolerances $0$, $0.02$, $0.10$, $0.20$, $0.30$, $0.35$, and
$0.40$, using the same saved human labels for all tolerances. Global
constraint generation starts with no conflicts, solves the coloring problem,
and validates every new cell. Every retained conflict, returned label, cell
margin, and iteration record is saved. The final capacity is the matching
lower and upper output count under numerical tolerance $10^{-9}$.

\begin{table}[tbp]
\centering
\caption{Unrestricted code-only capacity on all 96 locked SmolLM2 states under final human composite loss. Constraint generation and the independent two-label feasibility formulation agree. Margins are reoptimized within every returned cell; numerical tolerance is $10^{-9}$.}
\label{tab:human96-capacity}
\small
\begin{tabular}{@{}rrrrrl@{}}
\toprule
Tolerance & Outputs & Max cell margin & Iterations & Edges & Two-label check \\
\midrule
0.00 & 3 & 0.00000 & 121 & 188 & Infeasible \\
0.02 & 3 & 0.00000 & 115 & 185 & Infeasible \\
0.10 & 3 & 0.08000 & 116 & 181 & Infeasible \\
0.20 & 3 & 0.20000 & 82 & 138 & Infeasible \\
0.30 & 2 & 0.29836 & 50 & 65 & Feasible \\
0.35 & 2 & 0.34746 & 25 & 31 & Feasible \\
0.40 & 1 & 0.36950 & 1 & 0 & Feasible \\
\bottomrule
\end{tabular}
\end{table}

\paragraph{Independent two-label feasibility check.}
With three actions, labeling each state by an optimal action supplies at
most three compatible cells. A direct linear program gives the one-output margin
$0.36950037$. To check whether two outputs suffice without enumerating
conflicts, let $e_i\in[0,1]^3$ be the regret vector at state $i$,
$q^{(0)},q^{(1)}\in\Delta_3$ be the two shared continuation distributions,
and $x_i\in\{0,1\}$ select a label. Feasibility is equivalent to
\[
 e_i^\top q^{(0)}\le\varepsilon+x_i,
 \qquad
 e_i^\top q^{(1)}\le\varepsilon+1-x_i,
 \qquad i\in[96].
\]
For $x_i=0$, the first constraint is active; for $x_i=1$, the second is
active. The other constraint is automatically satisfied because regrets
are bounded by one and $\varepsilon\ge0$. Thus a feasible solution supplies
two compatible cells, and any two-cell compatible labeling supplies a
feasible solution. The independent feasibility MILP agrees with the
constraint-generation capacity at all seven tolerances. Feasible assignments
are rechecked against the original regret constraints; solver-reported
infeasibility and numerical tolerance are recorded rather than described as
exact rational proofs.

The matrix, including integer score units with denominator 120, is retained
in the author working archive as
\path{supplementary/author_audit/capacity/human96_loss_matrix.csv}.
The capacity script reconstructs the matrix from the existing human-loss
file and checks its precision, complete state coverage, and input hash.
The historical 42-state Qwen curve in Appendix~\ref{app:higher-order-details}
is a different experiment. Its matrix is still unavailable, and no
96-state result is described as a recovery of that historical curve.

\section{Continuous Interfaces and Population Certificates}
\label{app:continuous-certificates}

The main audit computes exact optima for a supplied finite loss table.
This extension asks what can be inferred about an unquantized continuous
observation and about unknown expected losses. The complete assumptions,
finite-sample bounds, and proofs are collected here so that the finite
interface analysis does not require this additional estimation setup.

\subsection{Continuous interfaces: identification and transfer}
\label{sec:continuous-interfaces}

\paragraph{What distinct finite observations identify.}
Let the observation space satisfy $\RepSpace\subseteq\mathbb R^d$, where
$d$ is the observation dimension. For audited pairs $(u_i,z_i)_{i=1}^n$,
define the empirical risks using the same statewise losses:
\[
 \widehat R_Z^\star:=\inf_{\delta\in\mathfrak D_Z}
    \frac1n\sum_{i=1}^n\ell(u_i,\delta(z_i)),
 \qquad
 \widehat R_{\mathrm{full}}:=\frac1n\sum_{i=1}^n\ell^\star(u_i).
\]

\begin{theorem}[Finite-sample interpolation limit]
\label{thm:continuous-interpolation-limit}
If the audited observations $z_1,\ldots,z_n$ are distinct and all measurable
rules are allowed, then $\widehat R_Z^\star=\widehat R_{\mathrm{full}}$.
Consequently, the audited points alone do not certify a positive population
information gap for the exact continuous observation.
\end{theorem}

An unrestricted lookup rule can choose an optimal continuation separately
at every distinct observed vector. Failure of a trained probe therefore
does not identify missing information or the optimum of its rule class.
An exact best-in-class comparison can identify a rule-class limitation,
but a population information claim needs additional distributional evidence
or assumptions, such as a specified overlapping observation channel.

\paragraph{Quantization with a controlled cost error.}
Equip $\RepSpace$ with a declared metric $d_Z$, such as Euclidean
distance. Let $Q_h:\RepSpace\to[L]$ be a measurable partition map, with
finite $L\ge1$, cells $A_r:=Q_h^{-1}(r)$, and diameter bound $h\ge0$
measured using $d_Z$.
For each continuation $\pi_j$, define its conditional expected loss
\[
 c_j(z):=\mathbb E[\ell_j(U)\mid Z=z].
\]
The quantity $c_j(z)$ is an expected \emph{loss}, whereas the earlier
$\bar e_j(z)$ is an expected \emph{regret}. Let $R_{Q_h(Z)}^\star$ denote
the optimum when only the partition label $Q_h(Z)$ is observed.

\begin{assumption}[Modulus for conditional costs]
\label{ass:conditional-cost-modulus}
There are versions of the conditional expected losses $c_j$ and a known
nondecreasing function $\Modulus:[0,\infty)\to[0,\infty)$ with
$\Modulus(0)=0$ such that
\[
 |c_j(z)-c_j(z')|\le\Modulus(d_Z(z,z'))
\]
for every $j\in[m]$ and every $z,z'$ in the same positive-probability
partition cell. Each such cell has diameter at most $h$ under $d_Z$.
\end{assumption}

The modulus $\Modulus$ controls how expected continuation costs vary within
an entire population cell, not just between the sampled vectors. Small
empirical distances or low reconstruction error do not establish the assumption.

\begin{theorem}[Continuous-interface transfer bound]
\label{thm:continuous-transfer}
Under Assumptions~\ref{ass:basic-audit-setting}
and~\ref{ass:conditional-cost-modulus},
\[
 0\le R_{Q_h(Z)}^\star-R_Z^\star\le2\Modulus(h),
\]
and consequently
\[
 R_Z^\star-\Rfull
 \ge R_{Q_h(Z)}^\star-\Rfull-2\Modulus(h).
\]
If every conditional expected loss $c_j$ is Lipschitz with common constant
$L_c\ge0$, one may take $\Modulus(h)=L_ch$.
\end{theorem}

The partition can increase risk by hiding distinctions present in $Z$.
Subtracting $2\Modulus(h)$ accounts for that increase before a partitioned
risk gap is used to infer a gap for the original continuous observation.
Without a valid transfer bound, the finite partition is the audited interface,
not a proxy that automatically certifies its continuous input.

\subsection{Finite-sample certificates with exact or estimated losses}

The transfer theorem compares population risks. A finite-sample certificate
must additionally control state sampling and, when losses are estimated,
within-state loss uncertainty. The two error sources are kept separate below.

\paragraph{Exact conditional expected loss vectors.}
Let $(U_t,Z_t)_{t=1}^n$ be independent and identically distributed audit pairs
from the declared joint law, with the continuation library, loss functions,
observation channel, partition $Q_h$, and valid modulus fixed independently
of the certification sample. The partition has $L$ cells and the library
has $m$ continuations. Suppose each observed entry
$L_{tj}=\ell_j(U_t)$ is the \emph{exact} conditional expected loss.
Define the empirical joint costs and full-state risk by
\[
 \widehat C_{rj}:=\frac1n\sum_{t=1}^n
     \Indicator\{Q_h(Z_t)=r\}L_{tj},
 \qquad
 \widehat R_{\mathrm{full}}:=\frac1n\sum_{t=1}^n\min_{j\in[m]}L_{tj}.
\]
The corresponding population joint cost is
$C_{rj}:=\mathbb E[\Indicator\{Q_h(Z)=r\}\ell_j(U)]$.

\begin{corollary}[Finite-sample certificate with exact loss vectors]
\label{cor:continuous-finite-sample-certificate}
Under the independent-sampling and exact-loss conditions just stated and
Assumption~\ref{ass:conditional-cost-modulus}, fix $0<\alpha<1$ and set
\[
 r_n:=\sqrt{\frac{\log(4Lm/\alpha)}{2n}},
 \qquad
 s_n:=\sqrt{\frac{\log(4/\alpha)}{2n}}.
\]
With probability at least $1-\alpha$,
\[
 R_Z^\star-\Rfull\ge
 \sum_{r=1}^L\max\{0,\min_j\widehat C_{rj}-r_n\}
 -\min\{1,\widehat R_{\mathrm{full}}+s_n\}-2\Modulus(h).
\]
\end{corollary}

The radius $r_n$ simultaneously controls the $Lm$ joint costs, while $s_n$
controls the mean full-state oracle loss. The certificate subtracts an
upper bound on full-state risk from a lower bound on partitioned-interface
risk, followed by the transfer penalty. Positivity certifies a population
gap only under the stated sampling and modulus assumptions.

\paragraph{Why unbiased noisy scores are insufficient.}
Let $Y_{tj}\in[0,1]$ be a noisy score with
$\mathbb E[Y_{tj}\mid U_t]=\ell_j(U_t)$. Unbiasedness for each continuation
does not imply unbiasedness after choosing the smallest score:
\[
 \mathbb E[\min_jY_{tj}\mid U_t]
 \le\min_j\mathbb E[Y_{tj}\mid U_t]
 =\ell^\star(U_t).
\]
For a single state and two independent Bernoulli$(1/2)$ scores,
the expected score minimum is $1/4$, whereas the true oracle loss is $1/2$.
More sampled states do not remove that minimum-selection bias.
Corollary~\ref{cor:continuous-finite-sample-certificate} therefore cannot
use an uncorrected finite-rollout estimate in place of an exact loss vector.

\paragraph{Simultaneous intervals for estimated losses.}
Retain independent and identically distributed audit pairs and a library,
loss, channel, partition, and valid modulus fixed independently of the
certification sample. Choose a state-sampling error budget
$\alpha_s\in(0,1)$ and a loss-estimation error budget
$\alpha_\ell\in[0,1-\alpha_s)$. Suppose random interval endpoints satisfy
$0\le\underline L_{tj}\le\overline L_{tj}\le1$ and
\[
 \mathbb P\!\left(
   \underline L_{tj}\le\ell_j(U_t)\le\overline L_{tj}
   \text{ for all }t\in[n],\ j\in[m]
 \right)\ge1-\alpha_\ell.
\]
The intervals cover expected losses, not merely individual rollout scores.
Define
\[
 \widehat C^-_{rj}:=\frac1n\sum_{t=1}^n
      \Indicator\{Q_h(Z_t)=r\}\underline L_{tj},
 \qquad
 \widehat R^+_{\mathrm{full}}:=
      \frac1n\sum_{t=1}^n\min_j\overline L_{tj}.
\]
The first quantity uses loss lower bounds for interface costs; the second
uses loss upper bounds for the oracle comparison.

\begin{theorem}[Continuous certificate with simultaneous loss intervals]
\label{thm:noisy-continuous-certificate}
Under the stated simultaneous interval coverage and independent audit
sampling, with Assumption~\ref{ass:conditional-cost-modulus}, use the
state-sampling radii
\[
 r_n:=\sqrt{\frac{\log(4Lm/\alpha_s)}{2n}},
 \qquad
 s_n:=\sqrt{\frac{\log(4/\alpha_s)}{2n}}.
\]
Then, with probability at least $1-\alpha_s-\alpha_\ell$,
\begin{equation}
 \begin{aligned}
 R_Z^\star-\Rfull\ge{}&
 \sum_{r=1}^L\max\{0,\min_j\widehat C^-_{rj}-r_n\}\\
 &-\min\{1,\widehat R^+_{\mathrm{full}}+s_n\}
 -2\Modulus(h).
 \end{aligned}
 \label{eq:noisy-continuous-bound}
\end{equation}
The interval-coverage event and state-sampling concentration event need not
be independent.
\end{theorem}

On simultaneous interval coverage, $\widehat C^-_{rj}$ is no larger than
the empirical joint cost computed from exact losses, and
$\widehat R^+_{\mathrm{full}}$ is no smaller than the empirical exact oracle
risk. State-sampling concentration is applied to those exact-loss quantities;
the interval endpoints then bound them. A union bound accounts for both
failure events without an independence assumption.

\paragraph{One valid rollout-based interval construction.}
Let each sampled state--continuation pair receive $N\ge1$ conditionally
independent scores $Y_{tjk}\in[0,1]$, $k\in[N]$, drawn under the fixed
scoring protocol. Conditionally on the full state $U_t$, assume the scores for each pair
are independent and their means differ from $\ell_j(U_t)$ by at most a
known deterministic bound $b_{\mathrm{judge}}\ge0$. For $0<\alpha_\ell<1$, define
\[
 \overline Y_{tj}:=\frac1N\sum_{k=1}^N Y_{tjk},
 \qquad
 r_N:=b_{\mathrm{judge}}
       +\sqrt{\frac{\log(2nm/\alpha_\ell)}{2N}}.
\]
Conditional Hoeffding bounds and a union bound over the $nm$ pairs give
valid simultaneous intervals
\[
 \underline L_{tj}:=\max\{0,\overline Y_{tj}-r_N\},
 \qquad
 \overline L_{tj}:=\min\{1,\overline Y_{tj}+r_N\}.
\]
Here $N$ controls within-state score uncertainty, whereas $n$ controls
state-sampling uncertainty. A judge-bias bound obtained with its own failure
probability requires an additional error-budget term. Exact loss vectors
are the zero-width interval case with $\alpha_\ell=0$.

The saved finite-table bootstrap analyses do not establish a uniform
$b_{\mathrm{judge}}$ or a validated population modulus $\Modulus$.
They are therefore not empirical instances of
Equation~\eqref{eq:noisy-continuous-bound}.
Algorithm~\ref{alg:continuous-interface-certificate} gives the interval-based procedure. The proofs follow below.

\subsection{Proofs of the continuous-interface results}

\begin{proof}[Proof of Theorem~\ref{thm:continuous-interpolation-limit}]
At state $u_i$, choose $j_i\in\argmin_j\ell_j(u_i)$. Distinct inputs
$z_i$ allow a single rule to assign continuation $j_i$ to each $z_i$;
use any fixed continuation elsewhere. The rule is measurable because
singletons in Euclidean space are Borel. It attains
$n^{-1}\sum_i\ell^\star(u_i)$. No rule can beat those row minima,
proving the equality.
\end{proof}

\begin{proof}[Proof of Theorem~\ref{thm:continuous-transfer}]
Data processing gives $R_{Q_h(Z)}^\star\ge R_Z^\star$. Fix a partition
cell $A_r$ of positive probability, choose $z_r\in A_r$, and set
$j_r\in\argmin_jc_j(z_r)$. Let $c_\star(z):=\min_jc_j(z)$.
The minimum of finitely many functions with common modulus $\Modulus$
has the same modulus: compare at an index minimizing either argument.
For every $z\in A_r$,
\[
 c_{j_r}(z)\le c_{j_r}(z_r)+\Modulus(h)
 =c_\star(z_r)+\Modulus(h)
 \le c_\star(z)+2\Modulus(h).
\]
Using $j_r$ throughout each cell gives a feasible rule on the partition
labels. Averaging the inequality bounds its risk by
$\mathbb E[c_\star(Z)]+2\Modulus(h)=R_Z^\star+2\Modulus(h)$.
The partition optimum is no larger than this feasible risk. Subtracting
$\Rfull$ proves the gap bound. Lipschitz costs allow
$\Modulus(h)=L_ch$.
\end{proof}

\begin{proof}[Proof of Corollary~\ref{cor:continuous-finite-sample-certificate}]
Each exact observation
$\Indicator\{Q_h(Z_t)=r\}L_{tj}\in[0,1]$ has mean $C_{rj}$ and is
independent across $t$. Hoeffding's inequality and a union bound give
simultaneous cost error at most $r_n$ with failure probability
$2Lm e^{-2nr_n^2}=\alpha/2$. The independent observations
$\min_jL_{tj}=\ell^\star(U_t)$ have mean $\Rfull$; the radius $s_n$
bounds their error with failure probability at most
$2e^{-2ns_n^2}=\alpha/2$.
On the intersection of these events,
\[
 R_{Q_h(Z)}^\star\ge
 \sum_r\max\{0,\min_j\widehat C_{rj}-r_n\},\qquad
 \Rfull\le\min\{1,\widehat R_{\rm full}+s_n\}.
\]
Clipping uses the known loss range. Subtract the upper bound from the
lower bound and apply Theorem~\ref{thm:continuous-transfer}. A union
bound limits total failure probability to $\alpha$.
\end{proof}

\begin{proof}[Proof of Theorem~\ref{thm:noisy-continuous-certificate}]
Define the unobserved exact sample quantities
\[
 \widetilde C_{rj}:=\frac1n\sum_t
 \Indicator\{Q_h(Z_t)=r\}\ell_j(U_t),\qquad
 \widetilde R_{\rm full}:=\frac1n\sum_t\ell^\star(U_t).
\]
The preceding concentration proof applies with failure probability
$\alpha_s$. On simultaneous interval coverage,
$\widehat C^-_{rj}\le\widetilde C_{rj}$ and
$\widehat R^+_{\rm full}\ge\widetilde R_{\rm full}$, since entrywise
inequalities are preserved by sums and minima. On both events,
\[
 R_{Q_h(Z)}^\star\ge\sum_r\max\{0,\min_j\widehat C^-_{rj}-r_n\},
 \qquad
 \Rfull\le\min\{1,\widehat R^+_{\rm full}+s_n\}.
\]
The transfer bound proves the result. A union bound gives failure
probability at most $\alpha_s+\alpha_\ell$; no independence of the
events or concentration of the interval endpoints is needed.

For the stated rollout construction, condition on each sampled state.
Hoeffding bounds its $N$ scores around their average conditional mean
with radius $\sqrt{\log(2nm/\alpha_\ell)/(2N)}$ and failure probability
at most $\alpha_\ell/(nm)$. The deterministic bias bound increases the
radius by $b_{\rm judge}$. Integrating over states and taking a union
bound over all pairs gives coverage. Clipping to $[0,1]$ preserves it.
Any separately uncertain bias bound requires its own failure budget.
\end{proof}

\section{Detailed Experimental and Reproduction Protocols}
\label{app:full-protocol}

The following sections specify the model settings, feature selection,
human ratings, and reproduction inputs for the comparisons
in Section~\ref{sec:empirical-audit}.

\subsection{Qwen2.5: separating information, rule, and fitting effects}

The primary model is \texttt{Qwen/Qwen2.5-1.5B-Instruct}, with public SAEs
applied to the activations after layers 17, 18, 19, and 21. Layer indices are
zero-based. The original collection has 240 states from 20 policy areas, each
containing three required actions and four paraphrases. Twelve areas
(144 states) are used for setting selection and fitting; eight areas
(96 states) are held out until confirmation. Qwen2.5-7B-Instruct is the
primary outcome evaluator, and Mistral-7B-Instruct-v0.3 provides a separate
evaluator check.

The model activation width is 1,536 and each SAE dictionary has 24,576
coordinates. Extraction uses \texttt{hidden\_states[layer+1]}, the residual
activation after the specified layer. Before computing risks, we check
extraction alignment, agreement with the direct SAE forward pass, state and
paraphrase counts, and completeness of the three-action loss tables.
Table~\ref{tab:empirical-setup} gives the full configuration.

\paragraph{Qwen outcome scoring.}
The five criteria are safety, utility, fact handling, macro-action adherence,
and overall quality, with weights $(0.40,0.25,0.15,0.10,0.10)$ on a
scale from zero to four. The evaluator produces an expected score from
prior-corrected next-token label probabilities; the weighted score is divided
by four and subtracted from one to obtain loss. The evaluator sees the full
situation, policy, selected action, response, and required-action label.
The required-action label is an evaluation input, not a selector input.
SmolLM2 and Phi instead use integer JSON scores from zero to three with the
five criteria and weights given below. Normalization to $[0,1]$ does not
make the different rubrics the same measurement scale across models.

\subsection{SmolLM2: independent SAE training and additional policies}

For \texttt{HuggingFaceTB/SmolLM2-1.7B-Instruct}, we use activations after
layer 17 and Mistral-7B-Instruct-v0.3 as the automatic evaluator. Alongside
the original 240 states, a locked collection contains 96 states from eight
additional policy areas. The locked collection is fixed before response
generation and scoring, with its SHA-256 digest recorded in the archive.
Five responses for each of the three choices at each of the 336 states give
$336\times3\times5=5{,}040$ responses.

Two TopK SAEs are trained from scratch with random seeds 314159 and 271828.
Each dictionary has 16,384 coordinates, with sparsity $k=64$: the TopK
selection keeps 64 coordinates per activation. Each run uses eight million
FineWeb-Edu activation tokens for training, 300,000 for validation, and a
100,000-token normalization sample. The runs share the model, layer,
architecture, and cached activation source, but differ in initialization and
training order. Feature identities are selected on the original 144 fitting
states and frozen before evaluation on the original and locked 96-state
collections. Table~\ref{tab:cross-family-design} gives the full configuration.

The automatic evaluator assigns integer scores in $\{0,1,2,3\}$ for safety,
policy adherence, usefulness, factuality, and overall quality, with weights
$(0.30,0.30,0.15,0.10,0.15)$. Loss is one minus the weighted score divided
by three. Its prompt includes the policy, situation, selected action, and
response, but not the hidden required action.

\paragraph{How the support features are selected.}
For both SmolLM2 and Phi, each coordinate is first assessed alone on the
144 fitting states. Its ranking score is the reduction in minimum loss
relative to a constant input when the selector observes that coordinate's
on/off bit. Coordinates are ordered by decreasing reduction, then by
activation-frequency balance, then by index. The first $B$ coordinates form
the interface with $B$ features. The selector uses a ranking of individual coordinates, not joint
optimization over all subsets of $B$ features. The same ordered features
are retained for all evaluation conditions; Appendix~\ref{app:protocol-diagnostics}
gives the ranking formula and tie rules.

\paragraph{Matching automatic scores to human-rated responses.}
The automatic risk table averages five responses per state--action pair.
For human evaluation, a score-independent SHA-256 ordering fixes one response
from each pair. Human risks use that single response, and automatic--human
calibration compares scores of the \emph{same} response, rather than comparing
a single human-rated response with a five-response automatic average.
The human study evaluates SmolLM2 responses only.

\subsection{Phi: matched feature-only and policy-aware comparisons}

The Phi model is \texttt{microsoft/Phi-4-mini-instruct}. We extract the
3,072-dimensional residual activation after block 23, the 24th of 32 blocks.
The relative depth $24/32=0.75$ matches SmolLM2's $18/24$ extraction fraction;
relative depth fixes the layer choice without assuming identical semantics.
The extraction hook agrees exactly with \texttt{hidden\_states[24]}.

Phi evaluates the original 240-state collection and the 96-state locked
collection used for SmolLM2. We call the latter the \emph{reused locked
collection}: its scenarios are fixed before Phi generation but are not new
to the researchers. Feature selection uses the original 144 fitting states;
evaluation uses the original 96 confirmation states and the reused 96 states.
Five responses for every state--action pair give 5,040 Phi responses, each
scored by Mistral-7B-Instruct-v0.3 and Qwen2.5-7B-Instruct. Each evaluator
assigns five scores in $\{0,1,2,3\}$ for safety, policy adherence, usefulness,
factuality, and overall quality. With weights
$(0.30,0.30,0.15,0.10,0.15)$, the automatic loss is one minus the weighted
score divided by three, matching the normalization of the human loss.

Two TopK SAEs, with seeds 314159 and 271828, use 16,384 coordinates and
sparsity 64. Each uses eight million FineWeb-Edu activation tokens for
training and 300,000 for validation. The 100,000-token normalization sample
comes from the training stream. Initialization and training order differ;
the model, block, architecture, and cached activation corpus are shared.
Each run completes 7,813 optimization steps.
Appendix~\ref{app:phi-results} gives model revisions, settings, and diagnostics.

Supervised feature ranking uses only Mistral losses on the 144 fitting states.
The selected features are fixed across evaluation collections, controller
inputs, and evaluators. We evaluate feature budgets
$B\in\{2,4,8,16,32,64\}$, with $B=2$ primary, and compare policy plus code
$(Z,S)$ with code alone $Z$. A third loss averages the two evaluator losses
\emph{before} optimizing state--action risk; its information gap need not
equal the mean of their separate gaps. Neither evaluator replacement nor
the mean-loss analysis selects new features.

Additional diagnostics measure reconstruction, the number and size of code
cells, and sensitivity to deleting a policy area. Variance selection and
three fitting-based random-selector controls provide further comparisons;
the random-selector controls use Mistral loss and code-only inputs, not
policy-aware inputs. Phi is evaluated automatically and has no completed
human ratings.

\subsection{Complete human-rating procedure}\label{app:human-protocol-details}

The human evaluation covers 176 states: four states from each of the 20
original policy areas and all 96 locked states. Three response choices per
state give 528 items per rater. The original 80-state subset includes
feature-selection areas, whereas the locked collection is held out at the
policy-area level. Raters see the active policy, user situation, selected
action, and candidate response. The following procedure keeps independent
rating separate from joint adjudication.

\paragraph{Rater recruitment and background.}
Two volunteer raters affiliated with the authors' research institution were recruited through internal convenience sampling. The raters received no financial or non-financial compensation. Both provided informed consent before participating. No demographic variables or direct personal identifiers were recorded in the rating dataset. We therefore report only the raters' institutional affiliation and do not infer unrecorded educational or professional qualifications. Domain-specific professional credentials were not required because each item presented the active policy in full, and the task was to evaluate the candidate continuation against that policy and the declared action rather than to make independent legal, medical, or regulatory decisions. Accordingly, the factuality dimension is interpreted as consistency with the information presented in the item and the absence of evident unsupported claims, not as expert verification of domain-specific facts.

\paragraph{Rater training, independence, and adjudication.}
Before the main rating round, both raters completed the same 36-item pilot and discussed the interpretation of the written rubric. The pilot labels were excluded from all reported analyses. Each rater then independently evaluated all 528 main-study items. During this independent stage, neither rater had access to the other rater's labels, the hidden required action, model identity, SAE configuration, automatic scores, dataset split, or rollout-replica information. After both independent rating files had been completed and locked, the prespecified disagreement rules were applied to construct the adjudication queue. Only then did the raters jointly resolve the 100 queued items. The original independent labels remained unchanged. Direct rater identifiers were not included in the analysis files or in materials prepared for release.

\paragraph{From ratings to loss.}
For each item, let $s_{\rm safety},s_{\rm policy},s_{\rm use},s_{\rm fact},s_{\rm overall}$
be its five integer scores in $\{0,1,2,3\}$. The dimension names denote
safety, policy adherence, usefulness, factuality, and overall quality.
The composite human loss is
\[
 L^{\rm human}=1-\frac{0.30s_{\rm safety}+0.30s_{\rm policy}
       +0.15s_{\rm use}+0.10s_{\rm fact}+0.15s_{\rm overall}}{3}.
\]
The division by three places the weighted score in $[0,1]$; subtracting from
one makes lower loss correspond to better rated responses. The composite
loss is not an unsafe-event probability. Agreement is
computed from the independent ratings. An item enters adjudication if safety
or policy-adherence scores differ by at least two points, acceptable labels
or recommended actions differ, or composite losses differ by more than
$0.25$. The 100 queued items use the jointly resolved scores; the other
428 items use the mean of the two independent losses.

For automatic--human action comparisons, we distinguish a selected lowest-loss
action from the set of all minimizing actions. The selected-action statistic
uses the first minimum in the fixed order \textsc{ask}, \textsc{decline},
\textsc{assist}. We additionally compare nonempty intersection and equality
of the two minimizing sets. Numerical ties use absolute tolerance $10^{-12}$;
zero and $10^{-8}$ tolerances give the same set counts. Definitions and
counts are in Appendix~\ref{sec:tie-aware-diagnostics}.

\subsection{Fixed comparisons and uncertainty}

Feature selection uses the declared fitting split and does not change across
the corresponding evaluation comparisons. The Phi observation conditions are
fixed before response generation, although the policy collections were
already used for SmolLM2. Selector sweeps, matched codebooks, noisy observation
channels, and fitted continuous rules are diagnostic analyses unless a
frozen confirmation setting is explicitly identified.

To measure sensitivity to the sampled policy areas, we use a cluster
bootstrap: each resampling draw selects whole policy areas, retaining their
states, and recomputes both risk minima in
Equation~\eqref{eq:finite-audit-bayes-rule}. Reported intervals are two-sided
95\% bootstrap intervals. An individual one-sided 95\% lower confidence bound
(LCB) above zero reports stability to the declared policy-area resampling. The Phi analysis
uses 2,000 resamples, eight policy-area clusters per evaluation collection,
and numerical zero tolerance $10^{-12}$. Its 12 primary policy-aware settings
share states, responses, and features; their intervals are not simultaneous
familywise guarantees. A zero fraction of nonpositive bootstrap draws is not
a zero $p$-value. In the policy-aware condition, states from distinct policy
texts cannot share a cell. Reweighting whole areas therefore reweights their
within-policy gaps. Positive lower bounds describe sensitivity of these fixed tables, not a test that unknown expected rollout losses have a positive
gap.

The empirical target is the finite outcome-loss table. Resampling policy
areas does not by itself provide simultaneous bounds on unknown expected
rollout losses. The bootstrap results are therefore distinct from the
population and search-valid certificates in
Sections~\ref{sec:continuous-interfaces} and~\ref{sec:capacity-certification}.
The corresponding configurations, cached losses, interface codes, ratings,
and diagnostic outputs are retained in the author working archive described
in Appendix~\ref{app:reproduction-inventory}.

\subsection{Reproduction inventory}
\label{app:reproduction-inventory}

The arXiv source package contains the manuscript, bibliography, figures,
and LaTeX tables. Experiment code and cached records are retained in a
separate author working archive and are not included in the arXiv source
package. All \path{supplementary/} and \path{tools/} paths in this paper
refer to that working archive. The Phi archive preserves the original 5,040 responses, both item-level evaluator files, frozen state tables, SAE feature arrays, selector registries, the active protocol lock, training summaries, and complete risk and diagnostic outputs. Machine-readable verification records are retained with the tables computed from archived records. Risk and bootstrap statistics can be recomputed from these cached records without rerunning the language model. Displayed values are rounded; aggregation and evaluator averaging occur before rounding.

Phi's active configuration is \texttt{protocol\_lock.json}. The nested \texttt{data/experiment\_config.json} is an unchanged file recording the SmolLM2 input configuration, not the Phi runtime configuration. Large SAE checkpoints, training activation shards, and raw hidden activation arrays for the audit states are not included in the compact Phi archive; recorded checkpoint hashes do not provide the missing weights. Optional Phi human packets are blank and are not completed human data. The completed SmolLM2 rating files and resolved queue are included under
\path{supplementary/support_revision/human_inputs/}; blank templates in the original automated archive are not treated as ratings. The CPU reproduction script reconstructs the final item-level losses, agreement statistics, and paired observation-risk tables. Bootstrap seeds are recorded separately for the code-only and matched observation analyses. Qwen source and notebook summaries are archived with a rounding bound; the full-precision historical Qwen loss/code arrays and the underlying 42-state capacity matrix are not present. The global-capacity regression suite validates the algorithm on declared test matrices, not on the unavailable 42-state matrix. The original experiment archives are preserved unchanged in the author working archive.

The author working archive retains the source excerpts, input hashes,
statewise minimizing sets, dimension-specific losses, and worked-example
records underlying the diagnostics described in
Appendix~\ref{app:protocol-diagnostics}. The author working archive also retains the loss matrix, every returned
labeling, validated cell margins, and independent two-label feasibility
results for the 96-state human capacity calculation. Within that archive,
the scripts \path{tools/reproduce_author_audit.py},
\path{tools/reproduce_human_capacity.py}, and
\path{tools/render_author_audit_tables.py} regenerate the records and
tables from archived inputs. These scripts do not change human labels or
perform model training. First-minimum agreement in older result files retains its original
number and is accompanied by the explicit tie convention and set-based
metrics.

\subsection{Ethical use}

The intended use of the framework is diagnostic: a representation can be tested before a safety monitor, steering rule, or downstream controller is trusted. Separating information loss from rule and training loss can reduce false confidence and can direct engineering effort toward the module that fails. Small conflict witnesses can also make failures easier to inspect, document, and reproduce.

The same analysis can be misrepresented as a deployment guarantee. Every result is conditional on the audited states, policy specification, continuation library, loss, interface, model, and sampling protocol. A representation that passes one audit can fail under another distribution, while failure of a finite support interface does not imply failure of the full continuous hidden state. We reduce this risk by reporting exact claim boundaries, retaining negative replication results, separating automatic and human losses, and avoiding a population-level uniform evaluator-error claim.

The human-rating procedure, including recruitment, consent, compensation,
rater independence, adjudication, and data handling, is described in
Appendix~\ref{app:human-protocol-details}.

\end{document}